\documentclass[twoside]{article}

\usepackage[preprint]{aistats2027}

\usepackage[utf8]{inputenc}
\usepackage[T1]{fontenc}
\usepackage[round]{natbib}

\usepackage{hyperref}
\usepackage{url}
\usepackage{booktabs}
\usepackage{array}
\usepackage{amsfonts}
\usepackage{nicefrac}
\usepackage{microtype}
\usepackage{amsmath,amssymb,amsthm}
\usepackage{mathtools}
\usepackage{bbm}
\usepackage{graphicx}
\graphicspath{{./}}
\usepackage{subcaption}
\usepackage{xcolor}

\definecolor{linkblue}{RGB}{0,68,136}
\hypersetup{
    colorlinks=true,
    linkcolor=linkblue,
    citecolor=linkblue,
    urlcolor=linkblue,
}

\usepackage{tikz}
\usetikzlibrary{patterns, decorations.pathreplacing, calc, matrix, positioning, fit, backgrounds}

\colorlet{Acol}{green!45!black}
\colorlet{Ncol}{violet!70!black}
\definecolor{qDcol}{RGB}{200,80,0}   %
\definecolor{qUcol}{RGB}{160,45,130} %
\definecolor{qPcol}{RGB}{0,125,90}   %
\definecolor{qWcol}{RGB}{140,95,0}   %
\newcommand{\dfn}[2]{\textbf{\color{#1}#2}}
\newcommand{\qD}[1]{{\color{qDcol}#1}}
\newcommand{\qU}[1]{{\color{qUcol}#1}}
\newcommand{\qP}[1]{{\color{qPcol}#1}}
\newcommand{\qW}[1]{{\color{qWcol}#1}}
 \let\plainvarphi\varphi \let\plainPhi\Phi
\newcommand{\stf}{\qD{\omega^2}}
\newcommand{\stft}{\qD{\tilde\omega^2}}
\newcommand{\cu}{\qU{H}}
\newcommand{\cun}{\qU{H_n}}
\newcommand{\cuc}{\qU{H_c}}
\newcommand{\cw}{\qW{\lambda}}
\definecolor{mcA}{HTML}{377EB8}%
\definecolor{mcB}{HTML}{984EA3}%
\definecolor{mcC}{HTML}{FF7F00}%
\DeclareRobustCommand{\lgline}[1]{\tikz[baseline=-0.55ex]\draw[#1,line width=0.9pt](0,0)--(1.1em,0);}
\DeclareRobustCommand{\lgdot}[1]{\tikz[baseline=-0.55ex]\draw[fill=#1,draw=black,line width=0.35pt](0,0)circle(0.23em);}
\DeclareRobustCommand{\lgdash}{\tikz[baseline=-0.55ex]\draw[black,line width=0.9pt,dash pattern=on 2.6pt off 1.6pt](0,0)--(1.5em,0);}
\definecolor{qDscol}{RGB}{0,0,0}     %
\definecolor{qUscol}{RGB}{0,0,0}     %
\definecolor{qPscol}{RGB}{0,0,0}     %
\newcommand{\Ds}{{\color{qDscol}\omega^{*2}}}
\newcommand{\us}{{\color{qUscol}H^*}}
\newcommand{\uts}{{\color{qUscol}\tilde H^*}}
\newcommand{\ps}{{\color{qPscol}\plainvarphi^*}}
\newcommand{\pts}{{\color{qPscol}\tilde\plainvarphi^*}}
\newcommand{\Phis}{{\color{qPscol}\plainPhi^*}}
\newcommand{\colorkeysymbols}{\DeclareRobustCommand{\varphi}{\qP{\plainvarphi}}\DeclareRobustCommand{\Phi}{\qP{\plainPhi}}%
  \let\plainut\ut \let\plainpt\pt \DeclareRobustCommand{\ut}{\qU{\tilde H}}\DeclareRobustCommand{\pt}{\qP{\tilde\plainvarphi}}}
\newcommand{\plainkeysymbols}{\let\varphi\plainvarphi \let\Phi\plainPhi \let\ut\plainut \let\pt\plainpt}
\usepackage[capitalize,noabbrev]{cleveref}
\usepackage{adjustbox}
\usepackage{placeins}

\newtheorem{theorem}{Theorem}
\newtheorem{proposition}[theorem]{Proposition}
\newtheorem{lemma}[theorem]{Lemma}
\newtheorem{corollary}[theorem]{Corollary}
\newtheorem{assumption}{Assumption}

\theoremstyle{definition}
\newtheorem{definition}[theorem]{Definition}
\theoremstyle{remark}
\newtheorem{remark}[theorem]{Remark}

\newcommand{\E}{\mathbb{E}}
\newcommand{\Prob}{\mathbb{P}}
\newcommand{\R}{\mathbb{R}}
\DeclareMathOperator{\Ber}{Bernoulli}

\newcommand{\N}{\mathcal{N}}

\newcommand{\Pbatch}{P_{\mathrm{batch}}}

\newcommand{\eps}{\varepsilon}

\newif\ifshowcomments
\showcommentsfalse

\newcommand{\C}{\mathbb{C}}
\newcommand{\cA}{\mathcal{A}}

\newcommand{\ebar}{\bar\varepsilon}
\newcommand{\ut}{\tilde H}
\newcommand{\pt}{\tilde\varphi}
\newcommand{\vinc}{v_{\mathrm{inc}}}
\newcommand{\vadd}{v_{\mathrm{add}}}
\newcommand{\Sh}{\hat\Sigma}
\newcommand{\Rh}{\hat R}
\newcommand{\TV}{\operatorname{TV}}
\makeatletter
\newcommand{\flowfighook}[1]{\global\@namedef{flowfighook@#1}{}}
\newcommand{\FlowFigure}[1]{\immediate\write\@auxout{\string\flowfighook{#1}}\input{appendix/unified/flow/fig_#1}}
\newcommand{\FlowFigureDefault}[1]{\@ifundefined{flowfighook@#1}{\input{appendix/unified/flow/fig_#1}}{}}
\makeatother
\newcommand{\appphasefig}{figures/fig_phase_diagram_app.pdf}
\newcommand{\appphasesigma}{1.2}

\providecommand{\cA}{\mathcal{A}}

\providecommand{\ebar}{\bar\varepsilon}
\providecommand{\ut}{\tilde H}
\providecommand{\pt}{\tilde\varphi}
\newcommand{\gres}{\gamma_{\mathrm{res}}}

\begin{document}
\colorkeysymbols

\twocolumn[
\aistatstitle{Dynamics of Stochastic Momentum with Sparse Updates in High Dimensions}

\aistatsauthor{Katie Everett \And Elliot Paquette}

\aistatsaddress{Google DeepMind \& MIT \\ \texttt{everettk@google.com}
  \And McGill University \& Mila \\ \texttt{elliot.paquette@mcgill.ca}}
]

\begin{abstract}
Existing theory of momentum assumes that every parameter receives a gradient at every step. In language models this fails for the embedding and readout layers, indexed by a heavy-tailed vocabulary. We analyze momentum under sparse updates in two models of these layers: least squares with sparse inputs and logistic regression with a rare class. We characterize their high-dimensional scaling limits across regimes of batch size, momentum, sparsity and learning rate. We consider a vocabulary of tokens with heterogeneous sparsity and show that, even with isotropic data, the varying sparsity plays the role of ill-conditioning. At large batch sizes, momentum gains the familiar square-root speedup in the condition number and increases the critical batch size. We show similar batch scaling behavior on a bigram language model.
\end{abstract}

\addtocontents{toc}{\protect\setcounter{tocdepth}{-2}}

\section{Introduction}
\label{sec:intro}
It is well known that heavy-ball or Nesterov momentum can converge faster than gradient descent on ill-conditioned problems \citep{polyak1964some,nesterov1983method}. However, most theoretical analyses of momentum assume that every parameter receives a gradient at every step, while in practice many parameters receive sparse updates due to the model architecture or data distribution. In particular, under the power-law distribution of token frequencies in natural language, the embeddings of common tokens are updated at every step, while those of rare tokens go through long intervals of inactivity. Similarly, the readout layer gradients are mostly concentrated in those tokens which are targets, making it a softer version of sparsity (see Figure~\ref{fig:tokens}). In addition, mixture-of-experts models can have similar sparsity induced by the expert routing \citep{shazeer2017outrageously,fedus2022switch}.

\begin{figure}[t]
\centering
\includegraphics[width=0.78\linewidth]{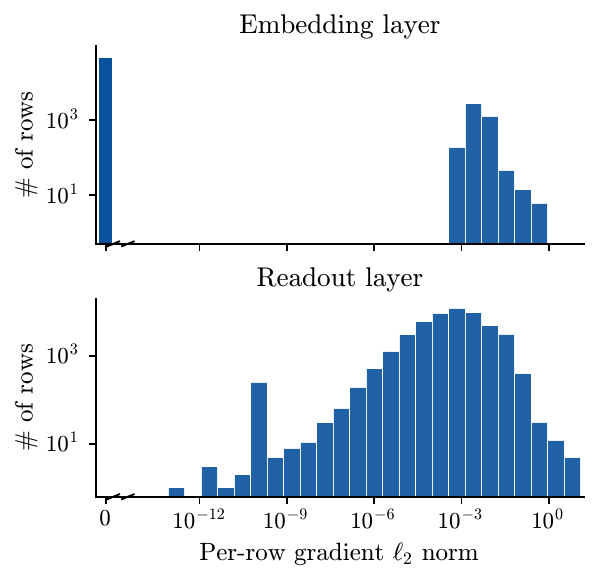}
\caption{\textbf{Embedding and readout gradients are sparse in different ways.} Per-row gradient $\ell_2$ norm from one forward and backward pass on Pythia-70M \citep{biderman2023pythia} using 16 packed 1024-token sequences of WikiText-103 \citep{merity2017pointer}. Embedding rows have a nonzero gradient only for tokens in the batch, as in the least squares model. The softmax gives every readout row a gradient, mostly small, as in the logistic model.}
\label{fig:tokens}
\end{figure}

Sparse updates appear to change the momentum dynamics, since the buffer decays and the parameters update on every step whether or not a gradient arrives. For rare features, it can decay to zero between \emph{active} batches containing the feature, but increasing momentum to bridge the gap makes the parameters \emph{drift} along stale gradients on the steps between active batches.

We analyze two tractable models with isotropic data: least squares (LS) with Bernoulli-gated inputs and label noise, a model of the embedding layer, and logistic regression (LR) on a Gaussian mixture with a rare class, a model of the readout layer. We fully characterize the high-dimensional limits of both models across regimes of sparsity, batch size and momentum. Our contributions are:
\begin{itemize}
\setlength{\itemsep}{1pt}\setlength{\parskip}{0pt}
\item For LS, the second moments of the error and the momentum obey an exact per-step recursion, in which sparsity mainly scales the curvature by the probability $p$ that a feature is active. Our main reduction theorem (Theorem~\ref{thm:main-comparison}) shows that the LS risk follows the second-moment system of a stochastic heavy ball SDE: a damped spring driven by noise. The limits are fully characterized by the \emph{stiffness}, the ratio of the \emph{learning} timescale on which gradient descent decays the risk, and the \emph{retention} timescale of the momentum buffer.

\item For LR we give a parallel analysis (Section~\ref{sec:logistic}), which follows the same high-level picture as LS, although with a few key differences.
\item With label noise (and no learning rate schedule), momentum achieves the same decay rate as SGD, although with smaller learning rates. Thus, when the learning rates are chosen to reach the same floor, momentum converges up to twice as fast as SGD near critical damping (Section~\ref{sec:helps}).

\item When one learning rate is shared across a vocabulary, the curvature of the most frequent token limits the learning rate at large batch sizes. The frequency imbalance acts as a condition number, and momentum increases the critical batch size by a factor of approximately its square root (Proposition~\ref{prop:vocab}). We show similar batch scaling behavior on a bigram language model (Section~\ref{sec:helps}).
\end{itemize}

\section{Related Work}
\label{sec:relwork}
Appendix~\ref{app:relwork} contains additional related work, including continuous-time analyses of acceleration and the sources of sparsity in transformers.

\textbf{Momentum in high-dimensional limits.} In many high-dimensional or noise-dominated stochastic limits, standard heavy-ball momentum disappears from the limiting dynamics entirely, which reduce to SGD with a rescaled learning rate \citep{paquette2021dynamics}. Acceleration then requires dimension- or noise-dependent scaling, as in accelerated SGD \citep{jain2018accelerating, kidambi2018insufficiency, varre2022accelerated} or the dimension-adjusted heavy ball \citep{paquette2021dynamics}. \citet{wang2024marginal} show that at small learning rates, where gradient noise dominates, SGD with and without momentum behave similarly, which is evidence against noise reduction as the source of momentum's benefit. In our models, momentum disappears from the limit only when learning is slower than retention (Section~\ref{sec:main_results}).

\textbf{Critical batch size.} Training speeds up in proportion to the batch size only up to a critical batch size \citep{mccandlish2018empirical, shallue2019measuring}, which grows with the amount of training data \citep{zhang2024criticalbatch} and over the course of training \citep{merrill2026critical}. In a noisy quadratic model, momentum extends perfect batch scaling \citep{zhang2019which}, and heavy ball accelerates only above a batch-size threshold \citep{lee2022trajectory, bollapragada2025fast}. Recent theory shows that heavy-ball momentum's critical batch size can be $\sqrt{\kappa}$ times SGD's in linear regression with condition number $\kappa$ \citep{morwani2026momentumtradeoffs}, and that Polyak momentum increases the critical batch size in power-law kernel regression \citep{wang2026momentumlargebatch}.

\textbf{Frequency imbalance and adaptive methods.} \citet{kunstner2024heavy} argue that heavy-tailed class frequencies make gradient descent slow on rare classes, and that this is why Adam \citep{kingma2015adam} outperforms gradient descent on language models. \citet{kunstner2026scaling} derive scaling laws for full-batch gradient descent and sign descent on linear bigram models under Zipf's law. Adaptive methods such as AdaGrad \citep{duchi2011adaptive} avoid the imbalance by giving each coordinate its own step size. Our work adds a theoretical analysis of stochastic momentum under frequency imbalance and shows that at large batch sizes the imbalance acts as a condition number, so momentum increases the critical batch size even without adaptive methods (Section~\ref{sec:helps}).

\section{SGD with Momentum Under Sparse Updates}
\label{sec:setup}

We study minibatch SGD with heavy-ball momentum in a high-dimensional scaling regime on two sparse-update problems: a least squares model and a logistic regression model.

\paragraph{SGD with momentum.} Given a learning rate $\eta > 0$, momentum parameter $\beta\in[0,1)$, and minibatch size $B$, the iterates evolve as
\begin{equation}
\begin{gathered}
    g_k = \frac{1}{B}\sum_{i=1}^B g_{k,i}, \qquad
    m_{k+1} = \beta\, m_k + (1-\beta)\, g_k, \\
    \theta_{k+1} = \theta_k - \eta\, m_{k+1},
\end{gathered}
\end{equation}
where $g_{k,i}$ is the per-sample gradient supplied by the data model below. We write $\eps:=1-\beta\in(0,1]$ for the momentum-decay rate, so that $1/\eps$ is the number of steps over which the momentum buffer $m_k$ retains a gradient; $\eps \to 0$ corresponds to strong momentum. We measure the error $e_k:=\theta_k-\theta_\star$ to a target $\theta_\star$ and the risk $R_k:=\E|e_k|^2$. Time $k$ counts minibatch steps.

\paragraph{Sparse least squares model.}\label{sec:ls-model} Each sample arrives through a sparse gating mechanism. We draw an activation indicator $s_i \sim \Ber(p)$ and an independent latent feature $\tilde{x}_i \sim \N(0, I_d)$, and form the gated input $x_i := s_i\, \tilde{x}_i$. The label is $y_i := \langle x_i, \theta_\star \rangle+\zeta_i$, where the label noise $\zeta_i$ is independent of $x_i$, with $\E\zeta_i=0$ and $\E\zeta_i^2=\varsigma^2$; the noiseless model is $\varsigma=0$. The per-sample squared-error gradient is
\begin{equation}
    g_i \;=\; x_i\, \bigl(\langle x_i,\, \theta\rangle - y_i\bigr).
\end{equation}
Let $N_k := \sum_{i=1}^B s_{k,i}$ count the active samples in minibatch $k$, and $\Pbatch := \Prob(N_k \ge 1) = 1 - (1-p)^B$ the probability of an \emph{active batch}. When $N_k = 0$, every per-sample gradient vanishes: the buffer decays as $m_{k+1} = \beta\, m_k$, while the parameters drift under the residual momentum. Empty batches count as steps.

\paragraph{Sparse logistic regression model.}\label{sec:lr-model} The logistic regression problem is a two-class Gaussian mixture in which one class is rare. A sample has a label $y\sim\Ber(p)$, where $y=1$ is the rare class, and features $x=y\mu+z$ with $z\sim\N(0,I_d)$, where $\mu \in \R^d$ is the signal direction with fixed norm $|\mu| = r > 0$. We fit a logistic classifier on this data with the bias held at its Bayes-optimal value $b^*=\log\frac{p}{1-p}-\frac{r^2}{2}$, training only the weights $\theta$; the target is then $\theta_\star=\mu$ (Appendix~\ref{sec:lr}). Unlike the LS model, every minibatch contributes a nonzero gradient, but only the rare-class fraction of samples carries information about the signal direction.

\paragraph{Co-scaling.} We analyze the high-dimensional limit $d \to \infty$ under a power-law co-scaling of the four parameters that govern the dynamics,
\begin{equation}
\begin{aligned}
    p &= p_* d^{-\kappa}, & B &= B_* d^{\sigma},\\
    \eps &= \eps_* d^{-\gamma}, & \eta &= \eta_* d^{-\alpha},
\end{aligned}
    \label{eq:coscaling-main}
\end{equation}
with fixed positive constants $(p_*, B_*, \eps_*, \eta_*)$ and a fixed label-noise level $\varsigma$. The exponents $\kappa, \sigma, \gamma \ge 0$ are free inputs. The learning-rate exponent $\alpha$ is also free; which values are admissible is decided by stability (Section~\ref{sec:stability_main}). We write $f \asymp g$ for equality up to sub-polynomial factors, $\log(f/g)/\log d \to 0$, and $\gtrsim$, $\lesssim$ for the corresponding one-sided relations.

\section{Main Results}
\label{sec:main_results}

\begin{figure*}[t]
\centering
\includegraphics[width=\textwidth]{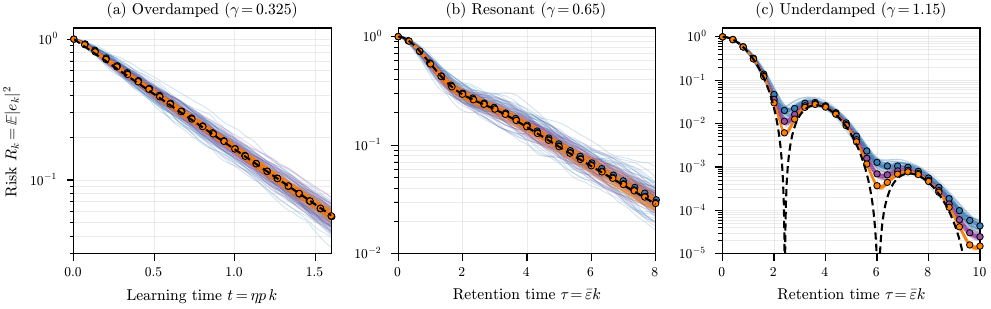}
\caption{\textbf{LS risk in the overdamped, resonant and underdamped regimes: Monte Carlo, exact recursion and limit.} The panels vary the momentum exponent $\gamma$ at fixed $\sigma=1.2$ and $\kappa=0.85$, a vertical slice of Figure~\ref{fig:phase-diagram}(a) through the resonance line $\gamma=0.65$, with zero initial momentum and noiseless labels. Thin lines are individual Monte Carlo runs (64 seeds per $d$) and circles the exact recursion~\eqref{eq:main-L1}, for $d=100$~(\lgline{mcA}\,\lgdot{mcA}), $d=300$~(\lgline{mcB}\,\lgdot{mcB}) and $d=1000$~(\lgline{mcC}\,\lgdot{mcC}); the dashed line~(\lgdash) is the limit of Section~\ref{sec:limits_main}. \textbf{(a)} Overdamped, $\gamma=0.325$, $\us=0.1$, on the learning clock. \textbf{(b)} Resonant, $\gamma=0.65$, $(\Ds,\us)=(1,\frac12)$. \textbf{(c)} Underdamped, $\gamma=1.15$, $\Ds=1$; the dashed spikes are the zeros of $x$. The recursion matches the Monte Carlo mean and approaches the limit as $d$ grows.}
\label{fig:convergence}
\end{figure*}

\begin{figure*}[t]
\centering
\includegraphics[width=\textwidth]{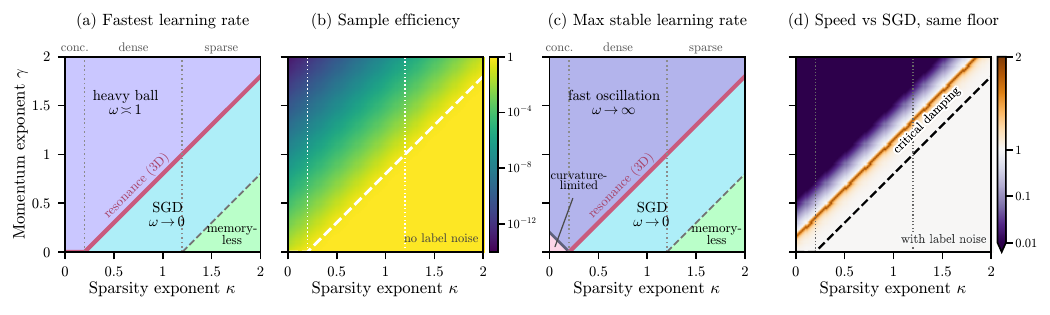}
\caption{\textbf{The dynamics of least squares in the $(\kappa,\gamma)$ plane.} \textbf{(a)} Limits at the fastest learning rate, $\eta\asymp d^{-\max(1-\sigma,\gamma-\kappa)}$. \textbf{(b)} Sample efficiency $d/N_1$, where $N_1$ is the number of nonzero gradients a feature receives per decay time at the fastest learning rate (Corollary~\ref{cor:sample_cost}), without label noise. \textbf{(c)} Limits at a fixed fraction of the maximum stable learning rate, $\eta\asymp d^{-\alpha_c}$, outside a narrow band around $\cw=1+\beta$. \textbf{(d)} Momentum's decay rate per step over SGD's, each at the learning rate whose floor is $\delta=10^{-3}$ times the label noise. Momentum is faster only near critical damping (orange ridge), by up to about a factor of two, and much slower far above the resonance line (dashed). Panels (b) and (d) use the exact per-step recursion at $\sigma=1.2$ and $d=10^6$.}
\label{fig:phase-diagram}\label{fig:phase-main}
\end{figure*}

We discuss the results for both models together and keep the mathematics light. Appendix~\ref{sec:app-overview} states the results formally and maps each claim to its proof.

\subsection{Two timescales and an exact recursion}
\label{sec:reduction_main}

Conditioned on the past, the LS gradient estimate has mean $p\,e_k$, which implies the mean curvature is scaled by $p$ (Appendix~\ref{sec:L1}).  This leads to two a competition between two timescales: while the \emph{retention} timescale $1/\eps$ is the number of steps over which the buffer remembers a gradient is unaffected by sparsity, the \emph{learning} timescale $1/(\eta p)$ is the number of steps that gradient descent needs to reduce the error by a constant factor. Their ratio is the \dfn{qDcol}{stiffness}
\begin{equation}
\label{eq:delta-def}
\stf:=\frac{\eta p}{\eps}=\frac{\text{retention timescale}}{\text{learning timescale}},
\end{equation}
which will determine the effective momentum of the algorithm (the nomenclature is adopted from the spring picture of Section~\ref{sec:comparison_main}).

The learning dynamics are encoded in the $2\times 2$ matrix $\Sigma_k$ of second moments of $e_k$ and $\eta m_k$, with $R_k=(\Sigma_k)_{11}$. For every $k$ and every $d$ (Lemma~\ref{lem:L1}),
\begin{equation}
\label{eq:main-L1}
\begin{gathered}
\Sigma_{k+1}=F\Sigma_kF^\top+2\cw\eps\,(\cun R_k+\varphi)\,bb^\top,\\
F=\begin{pmatrix}1-\cw&-\beta\\ \cw&\beta\end{pmatrix},\qquad b=\begin{pmatrix}-1\\1\end{pmatrix},
\end{gathered}
\end{equation}
with the \dfn{qWcol}{per-step curvature} $\cw:=\eta\eps p$, the \dfn{qUcol}{noise feedback} $\cun:=\eta(d+2-p)/(2B)$, the gradient noise per unit of risk, and the \dfn{qPcol}{ambient temperature} $\varphi:=\eta\varsigma^2d/(2B)$, the gradient noise due to label noise.

\subsection{Stability: two ceilings}
\label{sec:stability_main}

Write $\cuc:=\cw/(2(1+\beta))$ for the \dfn{qUcol}{curvature feedback} and $\cu:=\cun+\cuc$ for the \dfn{qUcol}{total feedback}. The recursion~\eqref{eq:main-L1} is stable exactly when the total feedback is below one, $\cu<1$ (Corollary~\ref{cor:stab} and Lemma~\ref{lem:stab-all}), so $\cun$ and $\cuc$ are the shares of the stability limit used by the gradient noise and by the curvature. For LS this is $\eta<\eta_+$, with
\begin{equation}
\label{eq:eta-plus}
\frac1{\eta_+}=\frac{d+2-p}{2B}+\frac p2\cdot\frac{1-\beta}{1+\beta}.
\end{equation}
Inverting each term gives a ceiling on $\eta$: a noise ceiling, which does not depend on $\beta$, and a curvature ceiling. Momentum raises only the curvature ceiling, from $2/p$ for SGD to $2(1+\beta)/((1-\beta)p)$, which is how it can raise the critical batch size of a vocabulary (Section~\ref{sec:helps}). It never lowers $\eta_+$, and when the noise term dominates, the threshold is that of SGD for all $\beta$.

Below the threshold the risk converges to the floor $\varphi/(1-\cu)$. Momentum enters it only through $\cuc$, which decreases in $\beta$: the learning rate sets the floor, and at a fixed learning rate momentum never raises it.

Under the co-scaling~\eqref{eq:coscaling-main}, stability requires
\begin{equation}
\label{eq:alpha-c}
\alpha\ \ge\ \alpha_c:=\max(1-\sigma,\,-\kappa-\gamma).
\end{equation}
The learning rate is \emph{noise-limited}, $\alpha_c=1-\sigma$, when $\gamma>\sigma-1-\kappa$, and \emph{curvature-limited} when $\gamma<\sigma-1-\kappa$. At the \emph{critical learning rate} $\alpha=\alpha_c$, the largest stable order of $\eta$, the noise-limited region has $\stf\asymp d^{\gamma-\gres}$ with $\gres:=1-\sigma+\kappa$, so the two timescales match on the \emph{resonance line} $\gamma=\gres$. Above it, learning is faster than retention and still stable.

\subsection{The stochastic heavy ball SDE and its cubic}
\label{sec:comparison_main}

The recursion~\eqref{eq:main-L1} has a continuous-time counterpart, the stochastic heavy ball SDE $(X,Y)$ driven by a continuous martingale $M$,
\begin{equation}
\label{eq:main-SHB}
\begin{gathered}
dX=-\stft Y\,d\tau,\qquad dY=(X-Y)\,d\tau+dM,\\
d\langle M\rangle=\tfrac2{\stft}\bigl(\ut|X|^2+\pt\bigr)\,d\tau,
\end{gathered}
\end{equation}
on the retention clock $\tau$, whose unit is the memory $1/\eps$ of the buffer, with $X(\ebar k)\approx e_k$ and $Y(\ebar k)\approx m_k/p$, where $\ebar:=-\log\beta\approx\eps$. Its feedback and temperature are
\[
\ut:=\frac{\cun}{1-\cuc},\qquad \pt:=\frac{\varphi}{1-\cuc},
\]
and $\stft$ is defined by a matching condition (Definition~\ref{def:matched}). As $\cw\to0$ they tend to the nominal values: $\stft=\stf\,(1+O(\eps+\cw))$, $\ut=\cu+O(\cw)$ and $\pt=\varphi\,(1+O(\cw))$.

Eliminating $Y$ gives $X''+X'+\stft X=-\stft\,dM/d\tau$: a heavy ball on a spring, with unit friction and stiffness $\stft$, driven by noise at temperature $\stft(\ut|X|^2+\pt)$. The feedback part $\ut|X|^2$ grows with the displacement and drives the spring unstable at $\ut=1$, which is exactly the threshold $\cu=1$ of the recursion; the ambient part $\pt$ only sets the floor $\pt/(1-\ut)=\varphi/(1-\cu)$.
The unforced spring $X''+X'+\stft X=0$ has fundamental solutions $e^{z_\pm\tau}$, $z_\pm=\frac12(-1\pm\sqrt{1-4\stft})$: real exponentials when $\stft<\frac14$ (\emph{overdamped}), oscillations when $\stft>\frac14$ (\emph{underdamped}), and a double root at $\stft=\frac14$ (\emph{critically damped}). It is \emph{resonant} when $\stft$ is of order one, near critical damping, and where the noise acts on the timescale of both fundamental solutions.

The second moments $R=\E|X|^2$, $V=\E|Y|^2$ and $C=\E\langle X,Y\rangle$ solve (Remark~\ref{rem:SHB})
\begin{equation}
\label{eq:main-L3}
\begin{gathered}
R'=-2\stft C,\qquad C'=R-C-\stft V,\\
V'=2C-2V+\tfrac2{\stft}(\ut R+\pt),
\end{gathered}
\end{equation}
a linear system whose characteristic polynomial is the cubic (Lemma~\ref{lem:L3})
\begin{equation}
\label{eq:main-cubic}
\chi(z)=z^3+3z^2+(2+4\stft)z+4\stft(1-\ut).
\end{equation}
Its roots have negative real parts exactly when $\ut<1$.  Moreover, a full analysis of the algorithmic limits can be read off these roots and their coefficients (Remark~\ref{rem:roots}).

\begin{theorem}[informal; Theorem~\ref{thm:M}]
\label{thm:main-comparison}
Fix $\delta\in(0,\frac14]$. If $\beta\ge\frac12$, $\ut\le1-\delta$ and $|\cw-(1+\beta)|\ge2\sqrt\beta\sin\delta$, then the LS risk satisfies
\[
\sup_{k\ge0}\bigl|R_k-\E|X(k\ebar)|^2\bigr|=O(\ebar).
\]
\end{theorem}

Both constraints on $\ut$ and $\cw$ are needed (see Remark~\ref{rem:nyquist}), but they are also satisfied with $\delta=\frac14$ whenever $\eta\le\eta_+/4$ (Remark~\ref{rem:margins}); they only become delicate as one drives the learning rate $\eta$ to its stability threshold. One must take $\beta\to1$ to make the errors vanish, and the error is uniform in all other parameters; every scaling limit of the risk is a limit of the single system~\eqref{eq:main-L3} as $(\stft,\ut,\pt)$ converge, even with $\stft\to\infty$.

\subsection{The limits}
\label{sec:limits_main}

As $d\to\infty$ with $\gamma>0$ at an admissible learning rate, $\stf$ tends to $0$, to $\Ds=\eta_*p_*/\eps_*$ or to infinity, and the feedback and the temperature tend to $0$ or to $\us=\eta_*/(2B_*)$ and $\ps=\varsigma^2\us$. The parameters $(\stft,\ut,\pt)$ of~\eqref{eq:main-SHB} have the same limits, except in the curvature-limited corner, where $\cw$ stays of order one (last paragraph; Remark~\ref{rem:limits}). When the learning rate $\eta$ scales like $\eta^+$, the position relative to the resonance line decides the limit: $\stf\to0$ below it, $\stf\to\Ds$ on it and $\stf\to\infty$ above it. Figure~\ref{fig:phase-diagram}(a,c) places the limits in the $(\kappa,\gamma)$ plane, and Figure~\ref{fig:convergence} compares them with Monte Carlo runs and the exact recursion.

In every limit below, the excess risk falls by a constant factor every $d^{\max(\gamma,\gres)}$ steps, the fewest at any stable learning rate (Remark~\ref{rem:cost}): one unit of the learning clock below the resonance line, and one unit of the retention clock, $1/\ebar\asymp d^\gamma$ steps, on and above it.

\paragraph{Learning slower than retention, $\stf\to0$.} This is the regime below the resonance line, $\gamma<\gres$. The spring is overdamped: its friction dominates its stiffness, so the momentum $Y$ relaxes to $X$ long before $X$ moves. On the learning clock $t=\stf\tau=\eta pk$, whose unit is $1/(\eta p)\asymp d^{\gres}$ steps, the risk is the second moment of the SGD diffusion
\[
dX=-X\,dt+dM,\qquad d\langle M\rangle=2\bigl(\us|X|^2+\ps\bigr)\,dt,
\]
that is, $R'=-2(1-\us)R+2\ps$ (Corollary~\ref{cor:tikhonov}). The momentum is invisible: in this regime the algorithm is equivalent to SGD at the same learning rate, whose risk follows the same equation (Lemma~\ref{lem:small_delta}).

\paragraph{Comparable timescales, $\stf\to\Ds$.} On the retention clock $\tau=\ebar k$ the risk is the second moment of the stochastic heavy ball SDE~\eqref{eq:main-SHB} with stiffness $\Ds$ (Corollary~\ref{cor:regular}), can be achieved in two ways. Above the resonance line, at a learning rate $\eta\asymp d^{\kappa-\gamma}$ well below stability (which is faster), the feedback and the temperature vanish and the small $\eta$ suppresses the noise. The risk follows $R=x^2(\tau)$ for the deterministic heavy ball ODE
\[
x''+x'+\Ds x=0,\qquad x(0)=\sqrt{R_0},\quad x'(0)=0
\]
(Corollary~\ref{cor:lift}). This regime also arises on the resonance line with maximal stable learning rate.  There, the noise survives and the limit is the full SDE (Corollary~\ref{cor:resonance}) with the limiting feedback $\us$ and temperature $\ps$. %

\paragraph{Learning faster than retention, $\stf\to\infty$.} This is the regime above the resonance line with maximal stable learning rate. The spring is strongly underdamped: its stiffness dominates its friction, and $X$ oscillates many times per unit of retention time. The risk oscillates around its local average $\frac12\bar S$, with an amplitude never larger than the average, and strictly smaller after the start when $\uts>0$, where the mean energy $\bar S$ solves $\bar S'=-(1-\uts)\bar S+2\pts$ on the retention clock (Corollaries~\ref{cor:window} and~\ref{cor:curv}). Raising the learning rate into this regime only increases $\ut$, which slows the decay and raises the floor (Remark~\ref{rem:window}).

\paragraph{The curvature-limited corner.} For $\sigma>1$ the corner $\kappa+\gamma<\sigma-1$ is curvature-limited: $\cw$ stays of order one. It is the only regime in which the curvature corrections matter. The stiffness diverges and the noise vanishes, so the limit is strongly underdamped. On the edge of the curvature-corner $\gamma=\sigma-1-\kappa$, the noise feedback $\cun$ also stays of order one, and the curvature amplifies it to $\ut=\cun/(1-\cuc)$, which slows convergence (Corollary~\ref{cor:curv}). %

\subsection{Logistic regression}
\label{sec:logistic}

LR has the same phase diagram as LS, with a few key differences. Its iteration reduces exactly to a Markov chain in five variables: the overlap $\theta_\parallel$ with the signal direction, its momentum, and the second moments $R_\perp,V_\perp,C_\perp$ of their components orthogonal to $\mu$, which we call the \emph{bulk} (Lemma~\ref{lem:LR0}). The bulk part of its drift, the expected one-step change, is the recursion~\eqref{eq:main-L1} with parameters that move with the state. When $p$ is small enough (Lemma~\ref{lem:B}), the curvature is $p\varrho$ where $\varrho:=e^{(|\theta|^2-r^2)/2}$, and $\varphi$ is replaced by $\Phi/\varrho$, with $\Phi:=\eta(d-1)/(2B)$ (Lemma~\ref{lem:LRdrift}).

When $\stf=O(1)$, we can give a similar mathematical treatment to LS (Corollary~\ref{cor:fluid}), showing that these five variables follow a deterministic system of ODEs that describes a stochastic heavy ball in $\R^d$ (Proposition~\ref{prop:LR34} and Remark~\ref{rem:LR-SHB}),
\begin{equation}
\label{eq:main-LR5}
\begin{gathered}
dY=\bigl(\varrho(X)X-\mu-Y\bigr)\,d\tau+dM,\\
dX=-\Ds Y\,d\tau,\qquad d\langle M\rangle=\tfrac2{\Ds}\Phis\,d\tau,
\end{gathered}
\end{equation}
with $X=\theta$, $Y=m/p$ and $\Phis:=\lim\Phi$. The force $\varrho(X)X-\mu$ is $\nabla U(X)$ for the convex potential $U(x)=e^{(|x|^2-r^2)/2}-\langle\mu,x\rangle$, minimized at $\theta_\star=\mu$.

In high dimensions $|X|^2$ concentrates, so $\varrho(X)$ is a deterministic function of time. Eliminating $Y$, the component $X_\perp$ of $X$ orthogonal to $\mu$ solves $X_\perp''+X_\perp'+\Ds\varrho\,X_\perp=-\Ds\,dM/d\tau$: the spring of Section~\ref{sec:comparison_main} with stiffness $\Ds\varrho$, driven by noise at temperature $\Ds\Phis$, with no feedback. While $|X|$ stays bounded, $\varrho$ is of order one, so the stiffness and the temperature scale as in LS and select the same limits on the same clocks (Figure~\ref{fig:phase-diagram}a,c). Below the resonance line the momentum drops out, and on the learning clock $t=\eta pk$ the limit is $\theta_\parallel'=r-\varrho\theta_\parallel$, $R_\perp'=2\Phis-2\varrho R_\perp$ (Proposition~\ref{prop:LR2}). On the line $\Phis>0$; above it, at $\eta\asymp d^{\kappa-\gamma}$, $\Phis=0$ and the heavy ball is deterministic. At $\eta\asymp d^{\sigma-1}$ above the line, $\stf\to\infty$ and, as in LS, the risk envelope decays more slowly than retention (Proposition~\ref{prop:W1}).

A key difference from LS is that in LR there is no noise feedback. The bias $b^*\approx\log p$ keeps every predicted probability of order $p$, so the residual of a sample, its predicted probability minus its label $y$, is about $-y$ whatever $\theta$ is. The gradient noise comes from the labels, like label noise in LS, and has no feedback (Lemma~\ref{lem:B}). In LS the feedback is what drives the spring unstable at $\ut=1$ and sets the noise ceiling; the limit~\eqref{eq:main-LR5} is stable at every temperature (Proposition~\ref{prop:S}). The scale $\eta\asymp B/d$ of the noise ceiling still limits the learning rate, but through the floor: beyond it $\Phi\to\infty$ and the bulk grows without bound. Only the curvature ceiling is a stability limit $\cw<\lambda_c\le2$ (in place of the LS $\cw<2(1+\beta)$) where $\cw=\eta\eps p\varrho^*$ (Proposition~\ref{prop:W2}).

Unlike in LS, the noise also biases the estimator. The overlap and the bulk moments converge to a fixed point, as in LS (Propositions~\ref{prop:LR2}, \ref{prop:S} and~\ref{prop:D}), and there the bulk second moment $R_\perp=\Phi/\varrho^*$ plays the role of the LS floor. But the bulk also raises the curvature factor to $\varrho^*>1$, so the signal force $\varrho^*\theta_\parallel-r$ balances at $\theta_\parallel=r/\varrho^*<r$. In LS the force is linear and the iterate stays centered at $\theta_\star$; here the classifier underestimates the signal, by an amount of order $\Phi$. The excess population loss, about $\frac12p\Phi$, comes from the bulk, and both vanish as $\Phi\to0$.

\section{Benefits of Momentum}
\label{sec:helps}

\label{sec:momentum_buys}

For a single feature in our isotropic models, momentum trains faster than SGD only when the learning rate is constant and noise floors are matched. However, for a vocabulary of tokens with differing frequencies trained with a single learning rate, we show that momentum increases the critical batch size and that the frequency imbalance acts like a condition number.

We measure speed by the \emph{decay time} $1/\Lambda$ of a loss that decays at rate $\Lambda$ per step, and sample efficiency by the \emph{sample cost}, the number of nonzero gradients a feature receives in one decay time at the fastest learning rate, the one that maximizes $\Lambda$.

\begin{figure*}[t]
\centering
\includegraphics[width=\textwidth]{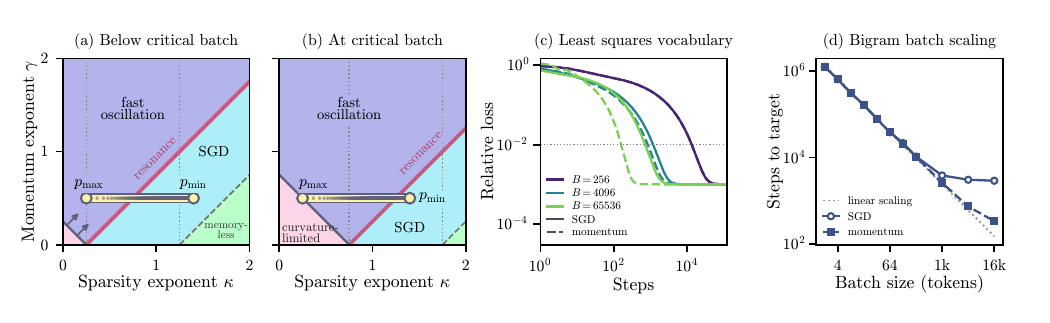}
\caption{\textbf{With one learning rate for a vocabulary, momentum increases the critical batch size.} \textbf{(a, b)} A vocabulary at one $\beta$ is a horizontal segment (yellow) from $p_{\max}$ to $p_{\min}$ on the phase diagram at the maximum stable learning rate (Figure~\ref{fig:phase-diagram}c), with $\sigma=1.25$ in (a) and $\sigma=1.75$ in (b); dots show individual tokens of a Zipf vocabulary. The edge of the curvature-limited corner is the line $Bp\eps=d$, which moves outward as the batch grows (arrows). Up to a constant factor, the vocabulary reaches its critical batch size when this edge intersects the most frequent token (b). \textbf{(c)} Exact relative loss (expected excess loss over its initial value) of the least-squares vocabulary with the bigram token frequencies ($d=64$, $\varsigma^2=10^{-3}$). At each batch size, both optimizers use the learning rate whose floor is $10^{-3}$ of the initial loss, as in Figure~\ref{fig:phase-diagram}d, and $\beta$ is tuned for momentum to minimize the steps to $1\%$ of the initial loss (dotted); momentum speeds up training at large batch sizes. SGD's $\eta$ and momentum's $(\eta,\beta)$ are $3.8$ and $(3.9,0.11)$ at $B=256$, $28$ and $(58,0.70)$ at $B=4096$, and $38$ and $(279,0.79)$ at $B=65536$. \textbf{(d)} Bigram experiment: steps to close $92\%$ of the gap to the best loss, with $\eta$ tuned for SGD and $(\eta,\beta)$ for momentum; momentum follows linear batch scaling up to a larger batch size than SGD.}
\label{fig:vocab}
\end{figure*}

\paragraph{Single feature.}\label{sec:single_feature}\label{sec:sample_cost}
When the learning rate can be tuned freely, momentum does not train a single feature faster than SGD at its best learning rate (Corollary~\ref{cor:sample_cost}, Lemma~\ref{lem:sgd-cost}, Figure~\ref{fig:phase-diagram}b). With label noise and a constant learning rate schedule, the noise floor $\delta:=R^\star/\varsigma^2=(\eta d/2B)/(1-\cu)$ depends only on the learning rate and not momentum. For a target floor, the momentum decay rate can be up to twice the SGD decay rate near critical damping, $\stf=\frac14$ (Proposition~\ref{prop:fixed_floor}, Figure~\ref{fig:phase-diagram}d). In other words, for a target rate, momentum reaches it with a lower $\eta$ (and hence lower floor) than SGD.

\paragraph{Spectral conflict.}\label{sec:spectral_conflict}
Parameters that receive gradients at very different rates often share one learning rate and momentum. In an embedding layer, for example, the row of a frequent token receives a gradient far more often than the row of a rare one. However, the single learning rate has to keep the most frequent tokens stable while the loss can decrease only as fast as the rare tokens are learned, creating a tension we refer to as the \emph{spectral conflict}. Since all tokens share one $\beta$ but differ in frequency, they span a horizontal segment of the phase diagram, from the most frequent token on the left to the rarest on the right (Figure~\ref{fig:vocab}a).

\paragraph{Vocabulary model.}\label{sec:vocab_main}
We stack $|\mathcal V|$ copies of the least squares model as the rows of a matrix $\Theta\in\mathbb{R}^{|\mathcal V|\times d}$, similar to an embedding layer. We sample minibatches of $B$ tokens from the vocabulary, drawn independently with probabilities $p_{\max}=p_1\ge\dots\ge p_{|\mathcal V|}=p_{\min}$. Only the row corresponding to each token receives a gradient, so each row is exactly the single-feature least squares model with $p=p_j$ (Lemma~\ref{lem:vocab_rows}).

Many optimizers have a linear scaling regime at small batch sizes, where increasing the batch size reduces the number of training steps required to reach a target loss by the same factor \citep{mccandlish2018empirical,shallue2019measuring}. Above some batch size, the number of steps is reduced by a smaller factor than the batch size grows, so reaching the target loss takes more samples. Following \citet{mccandlish2018empirical}, we define this \emph{critical batch size}, where the linear scaling regime ends, as $B_{\mathrm{crit}}:=E_{\min}/S_{\min}$, with $E_{\min}$ the smallest number of samples and $S_{\min}$ the smallest number of steps needed per decay time, over all choices of batch sizes and hyperparameters.

\begin{proposition}[Momentum increases the critical batch size of a vocabulary; informal]
\label{prop:vocab}
For a large frequency imbalance $\kappa_{\mathcal V}:=p_{\max}/p_{\min}$, when all tokens are trained with one learning rate:
\textup{(i)} below the batch size $B_\times:=(d+2)/(p_{\max}-p_{\min})$, momentum is no faster than SGD to leading order in high dimension;
\textup{(ii)} the critical batch size is approximately $4(d+2)/p_{\max}$ for SGD and approximately $4(d+2)/\sqrt{p_{\max}p_{\min}}$ for momentum with tuned $\beta$ (Propositions~\ref{prop:helps-critical} and~\ref{prop:vocab_full}).
\end{proposition}

For SGD below $B_\times$, the gradient noise imposes about the same limit on the learning rate of every token, so one learning rate works well for the whole vocabulary. In this regime, doubling the batch halves the gradient noise, so the learning rate can double and the number of steps roughly halves. Above $B_\times$, the curvature of the most frequent token bounds the learning rate below $2/p_{\max}$, independent of the batch size (Figure~\ref{fig:vocab}a,b). Since a larger batch no longer allows a larger learning rate, SGD's decay time approaches $\kappa_{\mathcal V}/4$ steps as the batch grows, and its critical batch size (Proposition~\ref{prop:vocab}) depends on the vocabulary only through $p_{\max}$.

Above $B_\times$, momentum raises the curvature limit of the most frequent token by a factor of $(1+\beta)/(1-\beta)$ (Section~\ref{sec:stability_main}). With tuned $\beta$, the large-batch decay time is reduced to about $\sqrt{\kappa_{\mathcal V}}/4$ steps, so that momentum's critical batch size is about $\sqrt{\kappa_{\mathcal V}}$ times SGD's (Proposition~\ref{prop:vocab}, Figure~\ref{fig:vocab}c).

Even though the data in this vocabulary model is isotropic, we can interpret the frequency imbalance as a form of ill-conditioning, where momentum has the classical conditioning benefit. In expectation the vocabulary loss is a quadratic with curvature $p_j$ for row $j$, so its mean Hessian is $\operatorname{diag}(p_j)\otimes I_d$ and its condition number is the frequency imbalance $\kappa_{\mathcal V}$. Once the batch is large enough that the gradient noise no longer limits the learning rate, SGD is essentially gradient descent on this quadratic, and momentum with Polyak's choice of $\beta$ attains the accelerated rate \citep{polyak1964some}. The decay times $\kappa_{\mathcal V}/4$ and $\sqrt{\kappa_{\mathcal V}}/4$ above are the classical rates of these two methods. Momentum's critical batch size in Proposition~\ref{prop:vocab} is likewise the implicit conditioning ratio of \citet{lee2022trajectory}, with token frequencies in place of the eigenvalues of the Hessian.

\paragraph{Bigram experiment.}\label{sec:bigram_main}
We train a bigram softmax model on text and show that it follows the same batch-scaling pattern as our vocabulary model. The model has an embedding vector and a readout vector for each of the 4096 most frequent Pythia tokens \citep{biderman2023pythia}, and we train it on next-token pairs from WikiText-103 \citep{merity2017pointer} with one learning rate and one $\beta$ for all parameters (Appendix~\ref{app:bigram}). An embedding receives a gradient only in batches that contain its token, so its frequency plays the role of $p$, as in the least-squares vocabulary. Through the softmax, however, every readout vector receives a gradient from every sample, which couples the rows in a way the least-squares vocabulary does not. At small batch sizes, tuned momentum is only slightly faster than SGD. As the batch grows, the tuned learning rate for SGD stops growing while for momentum it keeps growing with $\beta$, so momentum has a larger critical batch size (Figure~\ref{fig:vocab}d).

\section{Discussion and Limitations}
\label{sec:discussion}

We analyzed momentum under sparse updates in the high-dimensional limits of least squares and logistic regression. For least squares, one comparison theorem identifies every limit with a stochastic heavy ball, organized by the ratio $\stf$ of the retention timescale to the learning timescale: SGD when $\stf\to0$, the heavy ball when $\stf$ is of order one, and a strongly underdamped oscillation that is stable but decays no faster than retention when $\stf\to\infty$. The data in our models is isotropic, so the analysis ignores better-understood gains that occur in usual dense gradient models (Section~\ref{sec:helps}); an anisotropic feature would show how momentum interacts with sparsity and how the effective condition number responds. The vocabulary model ignores the softmax coupling of the rows, which the bigram model includes.

Our vocabulary model suggests when Adam helps over SGD, if we idealize Adam's per-coordinate normalization as a per-token learning rate. Below $B_\times$, one shared SGD learning rate is already close to every token's own optimum, so neither per-token learning rates nor momentum speed up training by more than a constant factor (Proposition~\ref{prop:vocab_full}). This agrees with \citet{marek2026small}, who find that SGD without momentum is competitive for training language models at very small batch sizes, and that the gap between optimizers grows with the batch size. Above $B_\times$, momentum relieves the spectral conflict only partly, while per-token learning rates remove it. The frequency imbalance in language model embedding and readout layers is aligned with the parameter basis, which may explain why these layers often use Adam instead of Muon or SOAP \citep{jordan2024muon,vyas2024soap}.

Our results also suggest that scaling studies should consider momentum and batch size interactions. With a fixed tokenizer, token frequencies do not change with model size; in our theoretical model, both critical batch sizes then grow linearly with the width (Proposition~\ref{prop:vocab}). Below $B_\times$, the choice of $\beta$ barely matters, but above it the best $\beta$ is large and grows with the batch size. A $\beta$ tuned at or extrapolated from small batch sizes may therefore be a poor choice at large ones.

\section{Conclusions and Future Work}
\label{sec:conclusions}
Our sparse models provide a theoretically tractable setting where we can fully characterize how momentum interacts with sparsity, batch size and the learning rate, which is valuable in particular for understanding momentum in language model embedding and readout layers.  The same sparsity also models mixture-of-experts layers, where an expert that receives a fraction of the tokens resembles a feature with that fraction as the activation probability, so that a momentum shared with the dense layers may place the experts in a different regime. Future work could also analyze iterate averaging, which in our models can reduce the label-noise floor without reducing the learning rate, removing the tradeoff behind the momentum advantage on a single token. An iterate average over a fixed window may adapt to token frequency in a way that one learning rate cannot. Optimizers could be aware of token frequency, with learning rates or momentum that depend on frequency, momentum schedules that bring each frequency near critical damping in turn, or a separate clock for each parameter, as in SparseAdam \citep{paszke2019pytorch} and ADANA-STAR \citep{ferbach2026logarithmic}.

\subsection*{AI use statement}
In this work, we used generative AI tools throughout the research process: for theoretical and conceptual development, for deriving and proving the mathematical claims, for designing experiments, for implementing code, and for interpreting results. Synthetic-data generation, translation and qualitative analysis are not applicable to this work. Additionally, we used generative AI tools to create all figures, for literature discovery and summarization, and to format references. The main text was written by the authors, with generative AI as a writing assistant; parts of the appendix were drafted with generative AI and reviewed by the authors.
We checked some mathematical results by hand. We formalized all the theorems, propositions, lemmas and corollaries of the appendix in Lean, in the companion repository \citep{sparsesgdrepo}; the Lean code was written by language-model agents and is checked by Lean's kernel. The formalization and the written proofs were developed together over several rounds: each round found gaps in the written statements and proofs, such as missing hypotheses, vacuous assumptions and lost error terms, which we corrected in the text before the next. We validated the theoretical results with numerical simulations.
We reviewed experiment code. The authors reviewed all AI-assisted work and take full responsibility for the content of this paper.

\subsubsection*{Acknowledgments}
The authors thank Courtney Paquette, Jeffrey Pennington and Atish Agarwala for technical discussions and feedback.
E.\ P.\ was supported by an NSERC Discovery Grant RGPIN-2025-04643, an FRQNT–NSERC NOVA Grant, a CIFAR Catalyst Grant and a gift from Google Canada.

\bibliographystyle{plainnat}
\bibliography{references}

\plainkeysymbols

\clearpage
\appendix
\counterwithin{theorem}{section}
\thispagestyle{empty}

\onecolumn
\aistatstitle{Dynamics of Stochastic Momentum with Sparse Updates\\
in High Dimensions: Appendix}

\addtocontents{toc}{\protect\setcounter{tocdepth}{2}}
\renewcommand{\contentsname}{Appendix Contents}
\tableofcontents
\clearpage

\section{The results at a glance}\label{sec:app-overview}

The main text describes the results without their hypotheses. This section
states them precisely, maps every claim of the main text to the result that
proves it, and explains how the proofs fit together. The proofs are in
Sections~\ref{sec:L1} to~\ref{sec:helps-app}.

\subsection{Six results}\label{sec:app-six}

\paragraph{Result 1: an exact reduction (Lemma~\ref{lem:L1}).} Suppose only
that the gradient estimate has conditional mean $\cA e_k$ and a conditional
second moment affine in $|e_k|^2$ (Assumption~\ref{ass:G}); LS satisfies this
exactly, with $\cA=p$. Then the second moments of the error and the momentum
obey the affine recursion~\eqref{eq:L1} for every $k$ and every $d$. The risk
depends on the model only through four numbers $(\beta,\lambda,H_n,\varphi)$ and the
initial second moments. The probability of an active batch and the gaps
between active batches do not appear.

\paragraph{Result 2: stability and the floor (Corollary~\ref{cor:stab}).} The
recursion is stable exactly when the total feedback $H=H_n+H_c$ is below one; for LS
this is $\eta<\eta_+$, with $\eta_+$ given by~\eqref{eq:eta-plus}. The risk
then converges geometrically to $\varphi/(1-H)$. The statement holds for every
per-step curvature $\lambda$ (Lemma~\ref{lem:stab-all}).

\paragraph{Result 3: the moment comparison theorem (Theorem~\ref{thm:M}).} Fix
a margin $\delta>0$. For every chain whose renormalized feedback satisfies
$\ut\le1-\delta$, and which keeps a Nyquist margin $\theta^*\le\frac\pi2-\delta$
when the timescale ratio is large,
\[
\sup_{k\ge0}\bigl|R_k-\Rh(k\ebar)\bigr|\le C_\delta\,\ebar\,(N_0+\pt),
\]
where $\Rh$ solves the continuum system~\eqref{eq:L3} at the matched parameters
and $N_0$ measures the initial condition. The constant does not depend on the
timescale ratio $\bar\omega^2$, so one statement covers learning slower than,
comparable to, and faster than retention. Equivalently, the risk is the second
moment of the stochastic heavy ball~\eqref{eq:SHB}, sampled every $\ebar$ units
of retention time (Remark~\ref{rem:SHB}).

\paragraph{Cells.} Under the co-scaling~\eqref{eq:coscaling-main}, $\omega^2$,
$H_n$ and $\lambda$ are powers of $d$ whose exponents are affine in
$(\kappa,\sigma,\gamma,\alpha)$, and $\varphi$ scales as $H_n$. For $\gamma>0$,
a \emph{cell} is a sign pattern of these three exponents. Table~\ref{tab:cells}
in Section~\ref{sec:dict} numbers the stable cells $1$ to $8$, and cell~$9$ is
the edge $\gamma=0$. The main text does not number the cells. It describes
those it discusses by their position relative to the resonance line and their
learning rate: cell~2 lies below the line, cell~4 on it, cell~3 above it at the
subcritical learning rate $\eta\asymp d^{\kappa-\gamma}$, cell~6 above it at
the critical learning rate, and cells~7 and~8 are the curvature-limited corner
and its edge.

\paragraph{Result 4: the limits (Section~\ref{sec:cor}).} Under the
co-scaling~\eqref{eq:coscaling-main} with an admissible learning rate, the risk
converges, on the clock of its cell, to the limit listed in
Table~\ref{tab:limits}: one-dimensional SGD with a floor when the timescale
ratio vanishes (Corollary~\ref{cor:tikhonov}), the system~\eqref{eq:L3} when it
converges (Corollaries~\ref{cor:regular} and~\ref{cor:resonance}), its
noiseless square-root lift when the feedback and the temperature vanish (Corollary~\ref{cor:lift}), and
a window law when it diverges (Corollaries~\ref{cor:window}
and~\ref{cor:curv}). Every limiting rate is read off the roots of the cubic
of Lemma~\ref{lem:L3}(c), and no chain decays faster than $\beta^k$
(Lemma~\ref{lem:retention-cap}).

\paragraph{Result 5: logistic regression (Section~\ref{sec:lr}).} The
$d$-dimensional iteration reduces exactly to a five-variable Markov chain whose
bulk drift is the recursion~\eqref{eq:L1} with state-dependent parameters
(Lemmas~\ref{lem:LR0} and~\ref{lem:LRdrift}). When learning is no faster than
retention, the chain stays within $O(\sqrt{\log d/d})$ of its drift recursion
over a fixed horizon on the slow clock, with probability at least $1-d^{-10}$
(Theorem~\ref{thm:A} and Corollary~\ref{cor:fluid}). The drift recursion
converges to the limits of Table~\ref{tab:limits}
(Propositions~\ref{prop:LR34} and~\ref{prop:LR2}), and these converge to one
equilibrium, biased against alignment, whose floor is set by the learning rate
(Propositions~\ref{prop:D} and~\ref{prop:S}). Above the resonance line the
rates are explicit, and the curvature ceiling is $\lambda_c\le2$
(Propositions~\ref{prop:W1} and~\ref{prop:W2}).

\paragraph{Result 6: when momentum helps (Section~\ref{sec:helps-app}).} As the
retention rate vanishes, the per-step rate of the risk is $\eps\,r_c(\omega^2,H)$
to leading order. At the fastest learning rate one decay time costs
$\frac12\bigl(N_{\mathrm{stab}}+N_{\mathrm{mem}}+\sqrt{N_{\mathrm{stab}}^2+N_{\mathrm{mem}}^2}\bigr)$
active samples, with $N_{\mathrm{stab}}=2B/\eta_+$ and
$N_{\mathrm{mem}}=Bp/\eps$, a hyperbola in steps and samples
(Corollary~\ref{cor:samplecost}). At a fixed learning rate momentum relaxes
at most twice as fast as SGD, and $2/(1+H^{1/3}+H^{2/3})$ times as fast at
critical damping (Proposition~\ref{lem:speedup}). For a vocabulary trained with one
learning rate, with $\kappa_{\mathcal V}=p_{\max}/p_{\min}$, the large-batch decay time is
$S_{\mathrm{SGD}}=1/(2\ln\frac{\kappa_{\mathcal V}+1}{\kappa_{\mathcal V}-1})\approx\kappa_{\mathcal V}/4$
with SGD and
$S_{\mathrm{mom}}=1/(2\ln\frac{\sqrt{\kappa_{\mathcal V}}+1}{\sqrt{\kappa_{\mathcal V}}-1})\approx\sqrt{\kappa_{\mathcal V}}/4$
with tuned momentum (Propositions~\ref{prop:vocab_full}
and~\ref{prop:helps-critical}). Without gradient noise these decay times are
attained, and no single step size does better on the two extreme frequencies
(Lemmas~\ref{lem:helps-vocab}(iv) and~\ref{lem:helps-twocurv}). With absolute
constants, momentum multiplies the critical batch size of a vocabulary by a
factor between $\sqrt{\kappa_{\mathcal V}}/19$ and about
$\sqrt{\kappa_{\mathcal V}}$ (Proposition~\ref{prop:helps-critical}).

\subsection{From the main text to the proofs}\label{sec:app-ledger}

Table~\ref{tab:ledger} lists the claims of the main text with the results
that prove them. \emph{Exact} means that no limit is taken; \emph{proved}
means proved under the hypotheses of the cited result.
Remarks are discussion: no theorem, proposition, lemma or corollary relies
on a remark, and only the numbered results are formalized in Lean.
The Lean formalization is in the directory \texttt{lean} of the companion repository \citep{sparsesgdrepo}.
Its file \texttt{RESULTS.md} lists every result of
Appendices~\ref{sec:L1}--\ref{sec:helps-app} by its number here, with the Lean declarations
that prove it, and \texttt{NOTATION.md} translates our notation into the Lean names. Every
theorem, proposition, lemma and corollary is formalized in full.

\begin{table}[ht]
\centering
\small
\renewcommand{\arraystretch}{1.15}
\begin{tabular}{@{}l>{\raggedright\arraybackslash}p{6.2cm}>{\raggedright\arraybackslash}p{4.3cm}>{\raggedright\arraybackslash}p{3.6cm}@{}}
\toprule
where & claim & result & status\\
\midrule
\S\ref{sec:reduction_main} & exact per-step recursion; $\Pbatch$ does not appear & Lemma~\ref{lem:L1} & exact\\
\S\ref{sec:reduction_main} & LR reduces to five variables, with bulk drift~\eqref{eq:L1} & Lemmas~\ref{lem:LR0}, \ref{lem:LRdrift}, \ref{lem:B} & exact; tame regime for the parameters\\
\S\ref{sec:stability_main} & stable iff $H<1$; floor $\varphi/(1-H)$ & Corollary~\ref{cor:stab}, Lemma~\ref{lem:stab-all} & exact\\
\S\ref{sec:stability_main} & momentum raises only the curvature ceiling and never the floor & \eqref{eq:eta-plus}, Corollary~\ref{cor:stab} & exact\\
\S\ref{sec:comparison_main} & stochastic heavy ball, its moment system and cubic & Remark~\ref{rem:SHB}, Lemma~\ref{lem:L3} & exact\\
\S\ref{sec:comparison_main} & risk equals the heavy-ball moment up to $O(\ebar)$ & Theorem~\ref{thm:M} & proved; constants depend on the margins (Remark~\ref{rem:uniform})\\
\S\ref{sec:limits_main} & SGD with floor & Corollary~\ref{cor:tikhonov} & proved\\
\S\ref{sec:limits_main} & 3D stochastic heavy ball; heavy ball & Corollaries~\ref{cor:regular}, \ref{cor:resonance}, \ref{cor:lift} & proved\\
\S\ref{sec:limits_main} & window & Corollaries~\ref{cor:window}, \ref{cor:curv} & proved outside the Nyquist band (Remark~\ref{rem:nyquist})\\
\S\ref{sec:limits_main} & a constant-factor decay takes $d^{\max(\gamma,\gres)}$ steps, the fewest at any stable learning rate & Remark~\ref{rem:cost} & proved\\
\S\ref{sec:limits_main} & below the line, equivalent to SGD at the same learning rate & Corollary~\ref{cor:tikhonov}, Lemma~\ref{lem:small_delta} & proved\\
\S\ref{sec:limits_main} & on the line, no faster than SGD at its best learning rate; at most twice as fast at a fixed one & Lemma~\ref{lem:sgd-cost}, Proposition~\ref{prop:fixed_floor} & proved as $\eps\to0$\\
\S\ref{sec:limits_main} & raising the learning rate into the strongly underdamped regime slows the decay and raises the floor & Remark~\ref{rem:window} & proved\\
\S\ref{sec:logistic} & LR follows its drift recursion & Theorem~\ref{thm:A}, Corollary~\ref{cor:fluid} & proved over finite horizons\\
\S\ref{sec:logistic} & LR limits, equilibrium and convergence & Propositions~\ref{prop:LR34}, \ref{prop:LR2}, \ref{prop:D}, \ref{prop:S}, Corollary~\ref{cor:recursion} & proved\\
\S\ref{sec:logistic} & LR window rates and curvature ceiling & Propositions~\ref{prop:W1}, \ref{prop:W2} & proved; convergence in the curvature window is Assumption~\ref{ass:W2}\\
\S\ref{sec:helps} & sample-cost hyperbola & Corollary~\ref{cor:samplecost} & exact bound (i); hyperbola proved as $\eps\to0$ along co-scaling rays\\
\S\ref{sec:helps} & speedup at most $2$; value at critical damping & Proposition~\ref{lem:speedup} & limit $\eps\to0$ at fixed $(\omega^2,H)$, ratio at most $2$, exact value at critical damping\\
\S\ref{sec:helps} & vocabulary: critical batch sizes & Proposition~\ref{rem:vocab} & exact two-sided bounds with absolute constants (Proposition~\ref{prop:helps-critical}); fixed-$\beta$ critical batch size open\\
\S\ref{sec:helps} & vocabulary: critical batch size multiplied by $\sqrt{\kappa_{\mathcal V}}$ up to an absolute constant & Proposition~\ref{prop:helps-critical}, Lemma~\ref{lem:helps-twocurv} & proved, explicit constants\\
\S\ref{sec:helps} & bigram language model & Figure~\ref{fig:vocab} & experiment\\
Figure~\ref{fig:convergence} & Monte Carlo, recursion and limit agree & Section~\ref{app:numerics} & numerical\\
\bottomrule
\end{tabular}
\caption{Claims of the main text and the results that prove them.}
\label{tab:ledger}
\end{table}

\subsection{How the proofs fit together}\label{sec:app-howproofs}

We study mini-batch SGD with heavy-ball momentum, written in its
exponential-moving-average form, with momentum parameter
$\beta=1-\varepsilon$ and learning rate $\eta$; $1/\varepsilon$ is the number
of steps over which a gradient is retained. In least squares (LS) the features
are sparse Gaussian vectors and the labels carry additive noise of variance
$\varsigma^2$; in logistic regression (LR) a rare class is separated from the
bulk by a mean shift of norm $r$. The object of study is the risk
$R_k=\E|e_k|^2$, where $e_k=\theta_k-\theta_\star$ is the error of the
iterate at step $k$, in the co-scaling~\eqref{eq:coscaling-main}, with
$\varsigma^2$ held fixed. Outside a narrow band of learning rates (the
Nyquist band of Remark~\ref{rem:nyquist}, reached only when the per-step
curvature is of order one), every limit of $R_k$ with $\varepsilon\to0$ is a
limit of one three-dimensional linear system, the second-moment system of a
stochastic heavy ball. The chain is compared with
this system through one statement, the moment comparison theorem; the only
limits obtained without it are the one-dimensional limit $\bar\omega^2\to0$, which
follows from the renewal equation directly, and the square-root lift without
noise, which is exact. The argument has four steps. The first two are exact
identities, the third is the only approximation, and the fourth reads off
the scaling limits.

\paragraph{Step 1: an exact reduction to four numbers.} The analysis uses
only the first two conditional moments of the gradient estimates
(Assumption~\ref{ass:G}), which LS satisfies exactly. Under this assumption
the $2\times2$ matrix $\Sigma_k$ of second moments of the error and of the
momentum obeys a closed affine recursion,
$\Sigma_{k+1}=F\Sigma_kF^\top+(\text{rank-one noise})$
(Lemma~\ref{lem:L1}), which depends on the hyperparameters only through four
numbers: the momentum $\beta$, the per-step curvature $\lambda$, the noise feedback
$H_n$ and the ambient temperature $\varphi$. Neither the probability that a batch
contains an active sample nor the gaps between active batches enter.

Unrolling the recursion gives a renewal equation with a nonnegative kernel
(Lemma~\ref{lem:L2}). The mass $\ut=H_n/(1-H_c)$ of this kernel is the gain
of the noise feedback loop, and positivity alone gives the stability
threshold $\ut<1$ (for LS, $\eta<\eta_+$), the noise floor $\varphi/(1-H)$
and an a priori bound (Corollary~\ref{cor:stab}).

\paragraph{Step 2: an exact embedding into a continuum system.} The candidate
limit is the continuum moment system~\eqref{eq:L3}, the second-moment system
of the stochastic heavy ball~\eqref{eq:SHB}, an SDE with heavy-ball drift
whose noise enters the momentum with variance affine in $|X|^2$
(Remark~\ref{rem:SHB}). Its parameters are the stiffness $\omega^2$, the ratio of the learning rate to the retention rate, the feedback and the ambient temperature. Lemma~\ref{lem:L3} records its
structure: a renewal form with a kernel of mass one, an energy bound, the
stability threshold and the Perron rate $r_c$, two exact eigenfunctionals $S$
(the slow energy) and $Z$ (the oscillation), and a rank functional $Q$.

The chain is not merely close to this system. After the change of
coordinates of Lemma~\ref{lem:match}, and at the matched parameters
$(\bar\omega^2,\ut,\pt,\ebar)$ of Definition~\ref{def:matched}, the chain \emph{is}
the free flow of~\eqref{eq:L3} sampled on the grid $\tau_k=k\ebar$, with the
noise injected as impulses at the grid times (Lemma~\ref{lem:embed}). Without
noise the chain and the sampled continuum coincide.

\paragraph{Step 3: quadrature, and the moment comparison theorem.} What
remains is that the noise enters at grid times rather than continuously, so
that a Riemann sum replaces an integral. For a product of two exponential
modes the ratio of the sum to the integral is the explicit function
$m(y)=y/(e^y-1)$ of the frequency difference $y$ (Lemma~\ref{lem:quad}), and
$m(y)=1+O(|y|)$ as long as $y$ stays away from the poles $\pm2\pi i$.

When $\bar\omega^2$ is bounded a total-variation estimate suffices. When $\bar\omega^2$ is
large the continuum solution consists of three modes with controlled
coefficients (Lemma~\ref{lem:modes}), and the Nyquist margin
$\theta^*\le\frac\pi2-\delta$ keeps the relevant mode pairs away from the
poles. Here $\theta^*=\min(\theta,\pi-\theta)$ is the matched angle of
Definition~\ref{def:matched}: the angle $\theta$ by which the free map $F$
turns in one step, folded into $(0,\frac\pi2]$.
The outcome is the moment comparison theorem (Theorem~\ref{thm:M}): the chain
equals the sampled continuum up to $O(\ebar)$, uniformly in time and in
$\bar\omega^2\in(0,\infty)$, for every chain with a stability margin and a Nyquist
margin. Equivalently, the risk is the second moment of the stochastic heavy
ball~\eqref{eq:SHB} at the matched parameters, sampled on the grid
(Remark~\ref{rem:SHB}). The Nyquist band is a genuine exception
(Remark~\ref{rem:nyquist}).

\paragraph{Step 4: every limit is a corollary.} Because the error in
Theorem~\ref{thm:M} is uniform in the matched parameters, a scaling limit of
the chain with $\ebar\to0$ that keeps the margins of the theorem is a limit of~\eqref{eq:L3} in which each of $\bar\omega^2$, $\ut$ and
$\pt$ tends to $0$, to a constant or to $\infty$, and in which the per-step
curvature $\lambda$ either vanishes or stays of order one.

Sending $\bar\omega^2\to0$ gives one-dimensional SGD with a noise floor on the
learning clock (Corollary~\ref{cor:tikhonov}); this case needs only
$\omega^2\to0$, not $\varepsilon\to0$. Keeping $\bar\omega^2$ fixed gives the
three-dimensional limit (Corollary~\ref{cor:regular}), which on the resonance
line is irreducibly three-dimensional (Corollary~\ref{cor:resonance}). Sending
$\bar\omega^2\to\infty$ gives the long-oscillation window, in which a slow energy
decays at rate $1-\ut$ while an oscillation turns by $2\theta^*$ per step and
decays at rate $1+\ut/2$ (Corollary~\ref{cor:window}).

Sending $\ut,\pt\to0$ gives, from a rank-one start such as the cold start, the
square of a deterministic damped oscillator together with an explicit bound
on the defect (Corollary~\ref{cor:lift}). Keeping $\lambda$ of order one gives the
curvature ceiling and the noise/curvature switch (Corollary~\ref{cor:curv}).

\paragraph{The phase dictionary.} Under the
co-scaling~\eqref{eq:coscaling-main}, $\omega^2$, the feedback, the temperature and $\lambda$ are powers
of $d$ whose exponents are affine in $(\kappa,\sigma,\gamma,\alpha)$. The
cells of Section~\ref{sec:app-six} are the sign patterns of these exponents,
and the limit in each cell is given by
one of the corollaries of Step~4 (Table~\ref{tab:cells}). Stability forces
$\alpha\ge\alpha_c:=\max(1-\sigma,-\kappa-\gamma)$. Table~\ref{tab:limits}
and Remark~\ref{rem:limits} write out, for the cells discussed in
Section~\ref{sec:limits_main} of the main text, the location, the learning
rate, the clock, the limiting equation and the result that proves it, for LS
and for LR.

At the critical learning rate $\alpha=\alpha_c$ the $(\kappa,\gamma)$ plane
contains exactly three open phases, one-dimensional SGD with a floor (cell~2),
the noisy window (cell~6) and the curvature window (cell~7), and two phase
boundaries, the resonance line $\gamma=1-\sigma+\kappa$ (cell~4) and the
noise/curvature switch $\gamma=\sigma-1-\kappa$ (cell~8); see
Figure~\ref{fig:phase}. The noisy window lies above the resonance line and is
reached only at the critical learning rate. The lines $\kappa=\sigma-1$,
$\kappa=\sigma$ and $\gamma=\kappa-\sigma$, across which the bookkeeping of
active updates changes, are not phase boundaries (Lemma~\ref{lem:nonboundary}).

\paragraph{How convergence is measured.} In the window the risk oscillates,
and early in the dynamics its troughs are deep. We measure convergence by the
envelope rate, the asymptotic decay rate of the upper envelope of $R_k$;
neither the troughs nor time averages count as a convergence mechanism. In
the window the envelope and the local average decay at the same rate,
$\ebar(1-\ut)$ per step, which is the largest continuum rate available at a
given $\ut$ (Remark~\ref{rem:window}; when $\ut=0$ the oscillation decays at
this rate too, Remark~\ref{rem:window-equal}). Without any approximation, no chain
decays faster than $\beta^k$: the per-step rate is at most $\ebar$ in every
regime (Lemma~\ref{lem:retention-cap}). Raising the learning rate into the
window only increases $\ut$, which lowers this rate and raises the floor. The
floor $\varphi/(1-H)$ is set by the learning rate; a longer momentum memory
lowers it only to the extent that it comes with a smaller learning rate
(Remark~\ref{rem:cost}).

\paragraph{Logistic regression.} For LR the second moments do not close,
because the curvature and the gradient noise depend on the state through the
curvature factor $\varrho=e^{(q-r^2)/2}$, where $q=|\theta|^2$. The analysis of
LS carries over through three facts.

First, the $d$-dimensional iteration reduces exactly to a five-dimensional
Markov chain (Lemma~\ref{lem:LR0}), and the bulk part of its conditional
drift is exactly the recursion~\eqref{eq:L1} with state-dependent parameters
(Lemma~\ref{lem:LRdrift}); in the tame regime of Lemma~\ref{lem:B} the noise
feedback vanishes, so that the curvature ceiling is the only stability
constraint. Second, a general fluid limit for Markov chains
(Theorem~\ref{thm:A}), with the increment bounds of Lemma~\ref{lem:LRinc},
shows that the random chain follows its drift recursion over finite horizons
in cells~2 to~4 (Corollary~\ref{cor:fluid}).

Third, the new ingredient is the two-dimensional signal subsystem, coupled to
the bulk only through $\varrho$. In cells~3 and~4 the drift recursion
converges to a five-dimensional system with dynamic $\varrho$
(Proposition~\ref{prop:LR34}); and in cell~2 it converges to a planar slow
system, which is globally stable (Proposition~\ref{prop:LR2}). The fixed point
of the drift recursion is identified in Proposition~\ref{prop:D}: it is biased
against alignment with the signal, and its floor is set by the learning
rate. Its explicit form there assumes $\lambda,H_n\to0$; when $\lambda$ is of order one,
in cells~7 and~8, only the bulk floor $R_\perp^*=\varphi/(1-H)$ at $\varrho^*$
(Proposition~\ref{prop:D}(iii)) is used.

\subsection{What is proved and what is assumed}\label{sec:app-proved}

The exact identities hold
with no limit taken: Lemmas~\ref{lem:L1}, \ref{lem:L2}, \ref{lem:L3},
\ref{lem:match}, \ref{lem:embed} and~\ref{lem:quad},
Corollaries~\ref{cor:stab} and~\ref{cor:lift},
Remark~\ref{rem:SHB} and Lemma~\ref{lem:retention-cap}, and, for LR,
Lemmas~\ref{lem:LR0} and~\ref{lem:LRdrift}. Lemma~\ref{lem:nonboundary}
follows directly from Lemma~\ref{lem:L1}, and Lemma~\ref{lem:modes} holds with
absolute constants.

Theorem~\ref{thm:M} and Corollaries~\ref{cor:tikhonov}, \ref{cor:regular},
\ref{cor:resonance}, \ref{cor:window} and~\ref{cor:curv} are proved with
constants that depend on the stability margin $1-\ut\ge\delta$ and on the
Nyquist margin $\frac\pi2-\theta^*\ge\delta$. We do not prove that the
constants stay bounded as $\ut\to1$ (Remark~\ref{rem:uniform}).

Both margins hold with $\delta=\frac14$ whenever $H_n,H_c\le\frac14$, in
particular for LS whenever $\eta\le\eta_+/4$ (Remark~\ref{rem:margins}). The
stability margin can fail only when the learning rate is pushed toward
$\eta_+$, and the Nyquist margin only in a band around $\lambda=1+\beta$, which is
reached only when $\lambda$ is of order one (cells~7 and~8).

For LR, Lemmas~\ref{lem:B} and~\ref{lem:LRinc}, Theorem~\ref{thm:A},
Corollaries~\ref{cor:fluid} and~\ref{cor:recursion}, and
Propositions~\ref{prop:LR34}, \ref{prop:LR2}, \ref{prop:D}, \ref{prop:S},
\ref{prop:W1} and~\ref{prop:W2} are proved. The propositions are statements
about the drift recursion or its continuum limit, and
Corollary~\ref{cor:fluid} transfers Propositions~\ref{prop:LR34}
and~\ref{prop:LR2} to the random chain over finite horizons.
Proposition~\ref{prop:S} gives global exponential stability of the
dynamic-$\varrho$ system in cells~3 and~4, and Corollary~\ref{cor:recursion}
carries it to the drift recursion at each fixed $\omega^{*2}$.

Besides Assumption~\ref{ass:G}, which is a hypothesis on the model, one
statement is assumed and labelled as an assumption: convergence of the drift
recursion in the curvature window, cells~7 and~8, from a warm start below
the LR curvature ceiling $\lambda_c$ of Proposition~\ref{prop:W2}
(Assumption~\ref{ass:W2}). It is not used in the proof of any numbered
result. In the noisy window, cells~5 and~6, one gap remains and is not
stated as an assumption: the constants of Corollary~\ref{cor:recursion} are
not shown to stay bounded as $\omega^{*2}\to\infty$, so the large-$\omega^{*2}$
rates of Proposition~\ref{prop:W1} are not transferred to the drift recursion
in the limit.

In Section~\ref{sec:helps-app}, Lemmas~\ref{lem:helps-sgd},
\ref{lem:step_speed}, \ref{lem:step_spectrum}, \ref{lem:ray}, \ref{lem:vocab_rows}
and~\ref{lem:helps-twocurv} are exact. So is Lemma~\ref{lem:chi_roots},
except part~(d), which is proved as $\omega^2\to0$.
Lemma~\ref{lem:small_delta} holds with explicit constants.

Corollary~\ref{cor:sample_cost} is exact in part~(i) and proved as $\eps\to0$
along a co-scaling ray in part~(ii). Proposition~\ref{prop:fixed_floor} is an
explicit limit as $\eps\to0$ at fixed $(\omega^2,H)$, and
Remark~\ref{rem:schedule} is proved as $k\to\infty$.

Lemma~\ref{lem:helps-vocab}(i)--(v) is exact, except the last claim of
part~(iii), which holds along a co-scaling ray, and the last claim of
part~(iv), a limit as $B\to\infty$. Proposition~\ref{prop:helps-critical}(i)--(iv)
is proved from these results with explicit constants, and
Proposition~\ref{prop:vocab_full} is its consequence.

Three statements about a vocabulary are not proved. The critical batch size
at a fixed $\beta$ is open, since the fewest samples $E_\beta$ are not
determined. With momentum at a finite batch, that the slowest row is one of
the two extremes, and how close $S_{\mathrm{mom}}(B)$ comes to its large-batch
limit, are observed numerically. The curve of steps against batch size
between its two asymptotes is not described for momentum.

Numerical observations appear only inside remarks and figure captions, and
in Section~\ref{app:numerics}. They are evidence, not part of any proof.

\subsection{Organization and notation}\label{sec:app-notation}

\paragraph{Organization.} Section~\ref{sec:L1} proves the exact reduction of
Step~1 and introduces the four parameters. Section~\ref{sec:renewal} derives
the renewal equations and the structure of the continuum moment system.
Section~\ref{sec:M} carries out Steps~2 and~3 and proves the moment
comparison theorem. Section~\ref{sec:cor} derives the corollaries of Step~4,
and Section~\ref{sec:dict} assembles them into the phase dictionary for LS.
Section~\ref{sec:lr} treats logistic regression,
Section~\ref{sec:helps-app} proves the comparisons of when momentum helps, and
Section~\ref{app:numerics} compares the exact chains with their limits
numerically, and Section~\ref{app:relwork} is the extended related work. Each of Sections~\ref{sec:L1} to~\ref{sec:lr} opens with a short
paragraph stating what it takes from the earlier sections and what it
delivers to the later ones.

\paragraph{Asymptotic notation.} In the main text we work at polynomial
resolution in the dimension: for positive $f(d)$ and $g(d)$, $f\asymp g$
means $\log(f/g)/\log d\to0$, that is, $f/g=d^{o(1)}$; and $f\asymp d^a$
means $f=d^{a+o(1)}$. In Sections~\ref{sec:L1} to~\ref{sec:lr}, which state
non-asymptotic bounds, $f\asymp g$ has the stronger meaning that $f/g$ is
bounded above and below by positive constants, independent of $d$ or of the
parameter that varies; this implies equality at polynomial resolution. The
symbols $O(\cdot)$ and $o(\cdot)$ have their usual meaning, with constants
that are absolute unless their dependence is stated.

\paragraph{Notation.} Table~\ref{tab:notation} collects the main symbols. In
Section~\ref{sec:lr}, $\varrho$ denotes the LR curvature factor, and the
learning-rate exponent is written $\alpha_\eta$; there $\theta^*$ denotes the
equilibrium signal overlap rather than the matched angle, and
$\sigma(t)=e^t/(1+e^t)$ denotes the logistic function, while the batch
exponent $\sigma$ appears only in exponents of $d$. Three further overloads
are local: in Section~\ref{sec:renewal}, $\theta$ also denotes the
eigenvalue angle of $F$ (Lemma~\ref{lem:L2}); the free response $\Phi_{\omega^2}$
of Section~\ref{sec:renewal} is unrelated to the LR noise level $\Phi$; and
in the discussion after Corollary~\ref{cor:tikhonov}, $C_*$ denotes the quasi-steady value
of $C$, not one of the constants $C_*$ of Lemma~\ref{lem:modes} and
Theorem~\ref{thm:M}.

\begin{table}[ht]
\centering
\small
\renewcommand{\arraystretch}{1.2}
\begin{tabular}{@{}ll>{\raggedright\arraybackslash}p{6.4cm}@{}}
\toprule
symbol & definition & meaning\\
\midrule
$|\cdot|$, $\|\cdot\|$ & & vector (Euclidean) norm; matrix norm, or as subscripted\\
$\theta_\star$, $e_k$ & $e_k=\theta_k-\theta_\star$ & target, error at step $k$\\
$\beta$, $\varepsilon$, $\ebar$ & $\varepsilon=1-\beta$, $\ebar=-\log\beta$ & momentum, retention rate, matched time step\\
$\eta$, $p$, $B$, $\varsigma^2$ & & learning rate, activation probability, batch size, label-noise variance\\
$(\kappa,\sigma,\gamma,\alpha)$ & \eqref{eq:coscaling-main} & exponents of $p$, $B$, $\varepsilon$, $\eta$\\
$\cA$, $\vinc$, $\vadd$ & Assumption~\ref{ass:G} & curvature, multiplicative and additive noise\\
$\lambda$, $H_n$, $H_c$, $\varphi$ & \eqref{eq:params} & per-step curvature, noise and curvature feedback, ambient temperature\\
$H$, $\ut$, $\pt$ & $H_n+H_c$, $\frac{H_n}{1-H_c}$, $\frac{\varphi}{1-H_c}$ & total feedback, renormalized feedback and temperature\\
$\omega^2$, $\bar\omega^2$, $\theta^*$ & $\eta\cA/\varepsilon$; Definition~\ref{def:matched} & learning-to-retention ratio, matched ratio, matched angle\\
$\eta_+$ & Section~\ref{sec:L1} & stability threshold of LS ($H<1\iff\eta<\eta_+$)\\
$\Sigma_k$, $R_k$ & Lemma~\ref{lem:L1} & second moments of (error, momentum), risk\\
$(R,V,C)$ & \eqref{eq:L3} & continuum moment system\\
$k$, $t$, $\tau$ & $t=\eta\cA k$; $\tau=\varepsilon k$ or $k\ebar$ & step, learning and retention clocks\\
$e_{\omega^2}$, $e_H$, $e_\lambda$ & Section~\ref{sec:dict} & exponents of $\omega^2$, $H_n$, $\lambda$\\
$\alpha_c$ & $\max(1-\sigma,-\kappa-\gamma)$ & critical learning-rate exponent\\
$\gamma_{\mathrm{res}}$ & $1-\sigma+\kappa$ & height of the resonance line\\
rate-optimal & $\alpha_c\le\alpha\le\max(1-\sigma,\gamma-\kappa)$ & shortest decay time (Remark~\ref{rem:cost})\\
$\omega^{*2}$, $H^*$, $\varphi^*$, $\lambda^*$ & Section~\ref{sec:dict} & limiting constants\\
$r_c(\omega^2,H)$ & Lemma~\ref{lem:L3}(c) & Perron rate of the cubic $\chi$\\
$\Lambda$ & Lemma~\ref{lem:step_speed} & decay rate per step $-\ln\rho(\mathcal T)$; the decay time is $1/\Lambda$\\
$N_1$, $N_{\mathrm{stab}}$, $N_{\mathrm{mem}}$ & Corollary~\ref{cor:sample_cost} & samples per decay time, stability and memory scales\\
$|\mathcal V|$, $\kappa_{\mathcal V}$ & Section~\ref{app:vocab} & vocabulary size; frequency imbalance $p_{\max}/p_{\min}$ (not a scaling exponent)\\
\midrule
$\varrho$, $q$ & $e^{(q-r^2)/2}$, $|\theta|^2$ & LR curvature factor, squared norm\\
$\theta_\parallel$, $m_\parallel$ & & LR signal overlap and its momentum\\
$R_\perp$, $V_\perp$, $C_\perp$ & Section~\ref{sec:lr} & LR bulk second moments\\
$\Phi$, $\epsilon_B$ & $\eta(d-1)/(2B)$; Lemma~\ref{lem:B} & LR additive noise level, tameness error\\
\bottomrule
\end{tabular}
\caption{Notation used in the appendix and the main text.}
\label{tab:notation}
\end{table}

\clearpage
\section{The exact per-step reduction}\label{sec:L1}

This section reduces the $d$-dimensional stochastic iteration to a
deterministic recursion for a $2\times2$ matrix. It takes nothing from later
sections: its only input is an assumption on the first two conditional
moments of the gradient estimates (Assumption~\ref{ass:G}), which least squares
with label noise satisfies exactly. Its output is the recursion~\eqref{eq:L1}
and the four parameters~\eqref{eq:params} on which that recursion depends.
Every later section works with these objects and never returns to the
algorithm itself, and Section~\ref{sec:lr} applies the same reduction to
logistic regression, conditionally on the current state. We first state the
algorithm and the assumption, then verify the assumption for least squares,
then introduce the four parameters, and finally prove the reduction
(Lemma~\ref{lem:L1}).

\paragraph{The algorithm.} We study mini-batch stochastic gradient descent
with momentum, in the form of an exponential moving average.
The momentum parameter is
$\beta\in[0,1)$, and we write $\varepsilon:=1-\beta\in(0,1]$, so
that $1/\varepsilon$ is the number of steps over which the momentum retains a
gradient; $\beta=0$ is plain SGD.
Given gradient estimates $g_k$ and a learning rate $\eta>0$, the
iteration is
\[
m_{k+1}=\beta m_k+\varepsilon g_k,\qquad \theta_{k+1}=\theta_k-\eta\,m_{k+1},
\]
and we measure its error $e_k:=\theta_k-\theta_\star$ to the target $\theta_\star$.
Every result of this section and of Section~\ref{sec:renewal} holds for
$\beta\in[0,1)$. The matched parameters of Definition~\ref{def:matched} need
$\beta>0$, and the moment comparison theorem (Theorem~\ref{thm:M}) and the
corollaries derived from it assume $\beta\ge\frac12$.

All that the analysis uses about the gradient estimates is their first two
conditional moments. We write $\mathcal F_k$ for the $\sigma$-algebra
generated by $(\theta_0,m_0)$ and the gradient estimates $g_0,\dots,g_{k-1}$,
and we assume that $\E|e_0|^2+\E|m_0|^2<\infty$.

\begin{assumption}[Second-order oracle]\label{ass:G}
There are constants $\cA>0$ and $\vinc,\vadd\ge0$ such that, conditionally on
the past $\mathcal F_k$,
\[
\E[g_k\mid\mathcal F_k]=\cA e_k
\qquad\text{and}\qquad
\E\bigl[|g_k-\cA e_k|^2\bigm|\mathcal F_k\bigr]=\vinc|e_k|^2+\vadd .
\]
\end{assumption}

\noindent In words, $g_k$ is an unbiased estimate of the gradient of the
quadratic $\frac{\cA}{2}|\theta-\theta_\star|^2$, and its conditional variance is
affine in the current squared error. The term $\vinc|e_k|^2$ is the
multiplicative, state-dependent part of the gradient noise, and $\vadd$ is the
additive part.

\paragraph{Least squares with label noise.} Our main example is least squares
(LS) on sparse Gaussian features with label noise. A sample is a pair $(x,y)$
with $x=s\tilde x$, where $s\sim\mathrm{Ber}(p)$ and
$\tilde x\sim\mathcal N(0,I_d)$ are independent, and
$y=\langle x,\theta_\star\rangle+\zeta$, where the label noise $\zeta$ is
independent of $x$ with $\E\zeta=0$ and $\E\zeta^2=\varsigma^2$. At each step
the gradient estimate $g_k$ is the average, over a fresh batch of $B$
independent samples, of the per-sample gradient
$x(\langle x,\theta_k\rangle-y)=x(\langle x,e_k\rangle-\zeta)$. A batch may
contain no active sample at all, in which case $g_k=0$; such empty batches are
counted as steps. The noiseless model is the case $\varsigma=0$.

For this model Assumption~\ref{ass:G} holds exactly, with
\begin{equation}\label{eq:LSoracle}
\cA=p,\qquad \vinc=\frac{p(d+2-p)}{B},\qquad \vadd=\varsigma^2\frac{pd}{B}.
\end{equation}
Indeed, a single per-sample gradient $h:=x(\langle x,e\rangle-\zeta)$ has mean
$\E[xx^\top]e=pe$. Since $s^2=s$, the Gaussian moment identity
$\E[|\tilde x|^2\langle\tilde x,e\rangle^2]=\sum_i\E[\tilde x_i^2\langle\tilde x,e\rangle^2]=\sum_i(|e|^2+2e_i^2)=(d+2)|e|^2$
(by $\E\tilde x_i^4=3$) and the independence of $\zeta$ give
\[
\E|h|^2=p(d+2)|e|^2+\varsigma^2pd .
\]
Averaging $B$ independent copies divides the variance by $B$, so that
\[
\E|g-pe|^2=\frac{\E|h|^2-p^2|e|^2}{B}=\frac{p(d+2-p)}{B}|e|^2+\varsigma^2\frac{pd}{B},
\]
which is \eqref{eq:LSoracle}. Equivalently,
$\E|g|^2=p^2|e|^2+\frac{p(d+2-p)}{B}|e|^2+\varsigma^2\frac{pd}{B}$.

\paragraph{The four parameters.} It is convenient to measure the momentum in
the units of $\theta$, through $\tilde m:=\eta m$, which is the displacement
of $\theta$ in one step. We set
\begin{equation}\label{eq:params}
\lambda:=\eta\varepsilon\cA,\qquad
H_n:=\frac{\eta\vinc}{2\cA},\qquad
H_c:=\frac{\lambda}{2(1+\beta)},\qquad
\varphi:=\frac{\eta\vadd}{2\cA},
\end{equation}
together with the derived quantities
\[
H:=H_n+H_c,\qquad \ut:=\frac{H_n}{1-H_c},\qquad \pt:=\frac{\varphi}{1-H_c},\qquad
\omega^2:=\frac{\eta\cA}{\varepsilon}=\frac{\lambda}{\varepsilon^2}.
\]
Here $\lambda$ is the curvature seen by the error in one step, $H_n$ and $H_c$ are
the noise feedback and the curvature feedback, and $\varphi$ is the ambient
temperature, from the label noise. The stiffness $\omega^2$, the ratio of the learning rate to the retention
rate, is the single parameter of the continuum mean flow.

The renormalized feedback $\ut$ and temperature $\pt$ will turn out to be
the masses of the renewal kernel of the chain (Lemma~\ref{lem:L2}), and they,
rather than $H_n$ and $\varphi$, are the parameters of the continuum limit.
For LS, \eqref{eq:LSoracle} gives
\[
H_n=\frac{\eta(d+2-p)}{2B},\qquad \varphi=\frac{\eta\varsigma^2d}{2B},\qquad
H=\frac{\eta}{\eta_+},\qquad\text{where}\quad
\frac1{\eta_+}:=\frac{d+2-p}{2B}+\frac p2\cdot\frac{1-\beta}{1+\beta}.
\]
We shall see in Corollary~\ref{cor:stab} that the linear part of the
second-moment recursion is exponentially stable if and only if $H<1$; for LS
this is $\eta<\eta_+$, so $\eta_+$ is the stability threshold of the
learning rate.

The first result says that the second moments of the error and of the
momentum obey a closed affine recursion.

\begin{lemma}[Exact reduction]\label{lem:L1}
Let Assumption~\ref{ass:G} hold, and let $\Sigma_k$ be the $2\times2$ matrix of
traced second moments of the pair $(e_k,\tilde m_k)$,
\[
\Sigma_k:=\begin{pmatrix}\E|e_k|^2&\E\langle e_k,\tilde m_k\rangle\\
\E\langle e_k,\tilde m_k\rangle&\E|\tilde m_k|^2\end{pmatrix},
\qquad R_k:=(\Sigma_k)_{11}=\E|e_k|^2 .
\]
Then, for every $k\ge0$,
\begin{equation}\label{eq:L1}
\Sigma_{k+1}=F\Sigma_kF^\top+2\lambda\varepsilon\,(H_nR_k+\varphi)\,bb^\top,\qquad
F=\begin{pmatrix}1-\lambda&-\beta\\ \lambda&\beta\end{pmatrix},\quad b=\begin{pmatrix}-1\\1\end{pmatrix}.
\end{equation}
Consequently the sequence $(R_k)_{k\ge0}$ depends on $(\cA,\vinc,\vadd,\eta,\beta)$
only through $(\beta,\lambda,H_n,\varphi)$ and $\Sigma_0$.
\end{lemma}

\noindent For LS and a fixed $\Sigma_0$ (for instance $m_0=0$ and a fixed
$\E|e_0|^2$), the lemma says that $R_k$ does not depend on the
batch-activity probability $P_{\mathrm{batch}}=1-(1-p)^B$, nor on $p$, $B$,
$d$ or $\eta$ separately, but only on the four combinations displayed after
\eqref{eq:params}.

\begin{proof}
The mean dynamics is linear and the noise is a martingale difference, so the
second moments close. Write $g_k=\cA e_k+\xi_k$, where $\xi_k:=g_k-\cA e_k$
satisfies $\E[\xi_k\mid\mathcal F_k]=0$ and
$\E[|\xi_k|^2\mid\mathcal F_k]=\vinc|e_k|^2+\vadd$. Multiplying the momentum
update by $\eta$ gives $\tilde m_{k+1}=\beta\tilde m_k+\lambda e_k+\eta\varepsilon\xi_k$,
and therefore
$e_{k+1}=e_k-\tilde m_{k+1}=(1-\lambda)e_k-\beta\tilde m_k-\eta\varepsilon\xi_k$. In
block form,
\[
\begin{pmatrix}e_{k+1}\\ \tilde m_{k+1}\end{pmatrix}
=F\begin{pmatrix}e_k\\ \tilde m_k\end{pmatrix}+\eta\varepsilon\,b\otimes\xi_k .
\]
When we take traced second moments, the cross terms between
$(e_k,\tilde m_k)$ and $\xi_k$ vanish, because $(e_k,\tilde m_k)$ is
$\mathcal F_k$-measurable and $\E[\xi_k\mid\mathcal F_k]=0$. The remaining
noise term is $\eta^2\varepsilon^2\,\E|\xi_k|^2\,bb^\top$, and by the
definitions \eqref{eq:params},
\[
\eta^2\varepsilon^2\,\E|\xi_k|^2=\eta^2\varepsilon^2(\vinc R_k+\vadd)=2\lambda\varepsilon(H_nR_k+\varphi).
\]
This proves \eqref{eq:L1}. The recursion involves only $\beta$, $\lambda$, $H_n$ and
$\varphi$, which gives the second statement.
\end{proof}

Two features of Lemma~\ref{lem:L1} drive everything that follows.

The first is that the reduction is exact. There is no small parameter and no
limit, and nothing is assumed about the dimension: the second moments of SGD
with momentum obey a closed, affine, three-dimensional recursion. The state
dependence of the gradient noise, which is what distinguishes SGD from
gradient descent with additive noise, enters only as a rank-one forcing in
the fixed direction $b$, with strength affine in $R_k$. The direction $b$ has
opposite signs in its two entries because the noise enters the momentum and,
through $\theta_{k+1}=\theta_k-\tilde m_{k+1}$, the error with the opposite
sign.

The second feature is the collapse of parameters. For LS, the sparsity
$p$, the batch size $B$, the dimension $d$, the learning rate $\eta$ and the
label-noise level $\varsigma$ enter only through $(\beta,\lambda,H_n,\varphi)$. In
particular the probability that a batch contains an active sample, which
decides how many steps actually move $\theta$, plays no role on the per-step
clock. Every scaling limit of $R_k$ is therefore a statement about a
$2\times2$ matrix recursion in four parameters, and a line in the space of
scaling exponents can be a phase boundary only if it changes the asymptotics
of these four numbers. This is the basis of Lemma~\ref{lem:nonboundary}.

Finally, the closure uses that the conditional variance in
Assumption~\ref{ass:G} is affine in $|e_k|^2$ with deterministic
coefficients. For logistic regression this holds only conditionally on the
current state, and this is why that case needs a separate concentration
argument (Section~\ref{sec:lr}).

\begin{remark}[What Assumption~\ref{ass:G} rules out]
Assumption~\ref{ass:G} makes three demands: the conditional mean is a
scalar multiple of the error, the conditional variance is affine in
$|e_k|^2$, and both have deterministic coefficients. The reduction of
Lemma~\ref{lem:L1} uses each of them once.

If $\E[g_k\mid\mathcal F_k]=\mathbf He_k$ for a symmetric matrix $\mathbf H$ that is not a
multiple of the identity, then $\E\langle e_k,He_k\rangle$ is not a function
of $\E|e_k|^2$, and the traced moments do not close: they no longer satisfy
a recursion among themselves. The moments along the eigenspaces of $\mathbf H$
still close, provided the conditional variance of the projection of $g_k$
on each eigenspace is affine, with deterministic coefficients, in the
squared norms of the projections of $e_k$ on the eigenspaces; isotropic
noise with variance affine in $|e_k|^2$ is an example. There is then one
coupled block per eigenvalue, and the collapse to four parameters is lost. If the mean is not linear in $e_k$, as for a
non-quadratic loss, no finite set of moments closes in general.

The variance can be weakened to an inequality. Since $\vinc\ge0$, the map
\[
\Sigma\mapsto F\Sigma F^\top+2\lambda\varepsilon H_n\Sigma_{11}bb^\top
\]
preserves the order $\preceq$ of positive semidefinite matrices (proof of
Lemma~\ref{lem:stab-all}), so an upper bound
$\vinc|e_k|^2+\vadd$ on the conditional variance gives
$\Sigma_k\preceq\bar\Sigma_k$, where $\bar\Sigma_k$ solves the reduction with
equality; a lower bound gives the reverse inequality.

The mean has no such relaxation: if the conditional mean is only known to be
$\cA_ke_k$ with $\cA_k$ in a band $[\cA_-,\cA_+]$, no order comparison of
$\Sigma_k$ with the recursions at $\cA_-$ and $\cA_+$ follows, because the
momentum iteration can be destabilized by switching between two stable
curvatures. For instance, with $\beta=\frac34$ the matrices $F$
of~\eqref{eq:L1} at $\lambda=\frac14$ and at $\lambda=2$ both have spectral radius
$\sqrt\beta<1$, since their eigenvalues are complex with product $\beta$, but
their product
\[
\begin{pmatrix}-9/4&-9/8\\5/4&3/8\end{pmatrix}
\]
has trace $-\frac{15}8$ and determinant $\frac9{16}$, hence the eigenvalue
$-\frac32$. Alternating the two curvatures makes the mean of $(e_k,\tilde m_k)$ grow
geometrically from a generic start.

Deterministic coefficients fail for logistic regression. There the mean and
variance identities of Assumption~\ref{ass:G} hold at each step, but with
coefficients that depend on the current state (Lemma~\ref{lem:LRdrift}).
Taking expectations then produces products of a random coefficient and a
random second moment, and the moments do not close.
\end{remark}

\section{Renewal structure and the continuum moment system}\label{sec:renewal}

The recursion~\eqref{eq:L1} is affine, and its noise enters as a positive
rank-one term. This section uses that structure twice. First, it turns the
recursion into a discrete renewal equation (Lemma~\ref{lem:L2}) whose kernel
has mass $\ut$; positivity of the kernel alone then yields the stability
threshold, the noise floor and an a priori bound (Corollary~\ref{cor:stab}).
Second, we introduce the continuum moment system~\eqref{eq:L3}, the
three-dimensional linear system that turns out to be the limit in every phase,
and collect in Lemma~\ref{lem:L3} the facts about it that the later sections
consume.

Its renewal form, with a kernel of mass one and bounded variation, and
its energy bound are the inputs to the error equation in the proof of the
moment comparison theorem (Theorem~\ref{thm:M}). Its stability threshold and Perron rate
$r_c$ give the rates of the phase dictionary (Section~\ref{sec:dict}) and the
comparison in Remark~\ref{rem:window}. Its exact eigenfunctionals $S$ and $Z$
underlie the mode analysis for large $\omega^2$ (Lemma~\ref{lem:modes},
Theorem~\ref{thm:M}(ii) and Corollary~\ref{cor:window}), and its rank
functional $Q$ shows that the limit on the resonance line is irreducibly
three-dimensional (Corollary~\ref{cor:resonance}). The mass identity of
Lemma~\ref{lem:L2} is used once more in Section~\ref{sec:M}, to identify the
size of the noise impulses in the embedding.

Unrolling the recursion \eqref{eq:L1} expresses $R_k$ as the free decay of
the initial condition plus the response to every past injection of noise.
Because the noise enters as a positive rank-one term, this response is a
convolution of the past second moments with a nonnegative kernel, and the
chain becomes a renewal equation \citep[Chapter~XIII]{Feller1968}.

For $\beta>0$ the kernel can be written through an angle $\theta$ defined by
\[
\cos\theta=\frac{1+\beta-\lambda}{2\sqrt\beta},
\]
which is the argument of the eigenvalues $\sqrt\beta e^{\pm i\theta}$ of $F$.
When $|\cos\theta|\le1$ we take $\theta\in[0,\pi]$. Otherwise $\theta$ is
complex, and the ratio $\sin((n+1)\theta)/\sin\theta$ below is to be read as the
Chebyshev polynomial $U_n(\cos\theta)$, which is a polynomial in $\cos\theta$
and is real for real $\cos\theta$. At $\beta=0$ the eigenvalues of $F$ are
$1-\lambda$ and $0$, and no angle is needed.

\begin{lemma}[Discrete renewal]\label{lem:L2}
Let $0<\lambda<2(1+\beta)$. Then $\rho(F)<1$, and
\begin{equation}\label{eq:L2}
R_k=a_k+\sum_{j<k}K_{k-1-j}\,(\vinc R_j+\vadd),\qquad
K_n:=\eta^2\varepsilon^2x_n^2\ \ge0,
\end{equation}
where $a_k:=(F^k\Sigma_0F^{k\top})_{11}$ and $x_n:=e_1^\top F^nb$. For
$\beta>0$, $x_n=-\beta^{n/2}\sin((n+1)\theta)/\sin\theta$, and for $\beta=0$,
$x_n=-(1-\lambda)^n$. The kernel has total masses
\[
\vinc\sum_{n\ge0}K_n=\ut\qquad\text{and}\qquad\vadd\sum_{n\ge0}K_n=\pt .
\]
\end{lemma}

\noindent Here $a_k$ is the free decay of the initial condition. The renewal
form reads $R_k$ as a feedback loop. Each unit of squared error at time $j$
produces, through the multiplicative noise, a total of $\ut$ units of squared
error at all later times, and the additive noise injected at each step
contributes a further $\pt$. Stability is therefore a statement about the
gain $\ut$ of this loop, and it can be read off from positivity alone
(Corollary~\ref{cor:stab}).

The gain is $\ut$, not $H_n$, and the difference is a property of the mean
dynamics. As $\lambda$ approaches its ceiling $2(1+\beta)$ an eigenvalue of $F$
approaches $-1$, the free response to a noise impulse lives longer, and the
mass of the kernel is amplified by the factor $1/(1-H_c)$.

\begin{proof}
The kernel is the squared response of the first coordinate of the free map
$F$ to a kick in the noise direction $b$. We compute it from a two-term
recursion and its mass from a discrete Lyapunov equation.

Unrolling \eqref{eq:L1} gives \eqref{eq:L2} with
$K_n=\eta^2\varepsilon^2(e_1^\top F^nb)^2=\eta^2\varepsilon^2x_n^2$. Since
$\det F=\beta$ and $\operatorname{tr}F=1+\beta-\lambda$, the Cayley--Hamilton
theorem $F^2=(1+\beta-\lambda)F-\beta I$ shows that $x_n$ obeys
\[
x_{n+2}=(1+\beta-\lambda)\,x_{n+1}-\beta x_n,\qquad x_0=-1,\quad x_1=-(1+\beta-\lambda),
\]
for every $\beta\in[0,1)$. If $\beta>0$, the roots of
$z^2-(1+\beta-\lambda)z+\beta$ are the eigenvalues
$\sqrt\beta e^{\pm i\theta}$ of $F$, and the solution is
$x_n=-\beta^{n/2}\sin((n+1)\theta)/\sin\theta$. If $\beta=0$, the recursion
reads $x_{n+1}=(1-\lambda)x_n$ from $x_0=-1$, so $x_n=-(1-\lambda)^n$, which is also the
limit of the trigonometric form as $\beta\to0$. Squaring gives $K_n$.

The Jury conditions \citep[Theorem~2.37]{Elaydi2005} for $\rho(F)<1$ are $|\det F|=\beta<1$,
$1-\operatorname{tr}F+\det F=\lambda>0$ and $1+\operatorname{tr}F+\det F=2(1+\beta)-\lambda>0$,
and all three hold for every $\beta\in[0,1)$.
Hence the series $X:=\sum_{n\ge0}F^nbb^\top F^{n\top}$
converges, and it solves the discrete Lyapunov equation
$X-FXF^\top=bb^\top$. Solving this $3\times3$ linear system for the
entries of the symmetric matrix $X$ gives
\[
X_{11}=\frac{1}{2\lambda\varepsilon(1-H_c)} .
\]
Since $\sum_nK_n=\eta^2\varepsilon^2X_{11}$, it follows that
\[
\vinc\sum_{n\ge0}K_n=\frac{\eta^2\varepsilon^2\vinc}{2\lambda\varepsilon(1-H_c)}=\frac{\eta\vinc}{2\cA(1-H_c)}=\ut,
\]
and the same computation with $\vadd$ in place of $\vinc$ gives $\pt$.
\end{proof}

The renewal form gives the stability threshold, the size of the noise floor,
and an a priori bound, all from the sign of the kernel.

\begin{corollary}[Stability, floor, a priori bound]\label{cor:stab}
Let $0<\lambda<2(1+\beta)$, that is, $H_c<1$.
\begin{enumerate}
\item[(i)] $\ut<1$ if and only if $H<1$. For LS this is $\eta<\eta_+$.
\item[(ii)] If $\ut<1$, then
\[
\sup_kR_k\le\frac{\sup_ka_k+\pt}{1-\ut},
\]
and $R_k$ converges geometrically to
$R_\infty=\pt/(1-\ut)=\varphi/(1-H)$.
\item[(iii)] If $\ut\ge1$, $\varphi=0$, and either $(\Sigma_0)_{11}>0$, or
$\beta>0$ and $\Sigma_0\ne0$, then $\sum_kR_k=\infty$.
Consequently, when $\ut\ge1$ the linear part
$\Sigma\mapsto F\Sigma F^\top+2\lambda\varepsilon H_n\Sigma_{11}\,bb^\top$ of the
recursion is not exponentially stable.
\end{enumerate}
\end{corollary}

\noindent The argument uses only the positivity of the kernel; no spectral
computation for the three-dimensional affine map is needed.

\begin{proof}
Part (i) is the equivalence $H_n<1-H_c\iff H_n+H_c<1$. For (ii), every term
in \eqref{eq:L2} is nonnegative, so taking the supremum over $j\le k$ gives
$\sup_{j\le k}R_j\le\sup_ja_j+\ut\sup_{j\le k}R_j+\pt$, which rearranges to
the bound.

For the limit, put $D_k:=R_k-R_\infty$. Since
$R_\infty=\ut R_\infty+\pt=\sum_{n\ge0}K_n(\vinc R_\infty+\vadd)$,
\eqref{eq:L2} gives
\[
D_k=f_k+\vinc\sum_{j<k}K_{k-1-j}D_j,\qquad
f_k:=a_k-\sum_{n\ge k}K_n\,(\vinc R_\infty+\vadd).
\]
Fix $r\in(\rho(F),1)$. Then $a_k\le\|F^k\|^2\|\Sigma_0\|\le Cr^{2k}$ and
$K_n\le Cr^{2n}$, so $|f_k|\le C'r^{2k}$. Choose $s\in(1,r^{-2})$ so close to
$1$ that $\gamma_s:=\vinc\sum_{n\ge0}K_ns^{n+1}<1$, which is possible because
$\gamma_1=\ut<1$. Multiplying the equation for $D_k$ by $s^k$ gives
$\max_{k\le N}s^k|D_k|\le\sup_js^j|f_j|+\gamma_s\max_{k\le N}s^k|D_k|$ for
every $N$, so that $|D_k|\le s^{-k}\sup_js^j|f_j|/(1-\gamma_s)$, and the
convergence is geometric.

For (iii), let $\varphi=0$. By Tonelli's theorem the double sum
$\sum_k\sum_{j<k}$ in \eqref{eq:L2} counts every term $K_{k-1-j}\vinc R_j$
exactly once, so if $\sum_kR_k$ were finite we would have
$\sum_kR_k=\sum_ka_k+\ut\sum_kR_k\ge\sum_ka_k+\sum_kR_k$, which forces
$\sum_ka_k\le0$.
But $\sum_ka_k>0$: if $(\Sigma_0)_{11}>0$ then $a_0>0$, and
otherwise $\beta>0$ by hypothesis and $(\Sigma_0)_{12}=0$ by positive
semidefiniteness, so $(\Sigma_0)_{22}>0$ and $a_1=\beta^2(\Sigma_0)_{22}>0$.
The hypothesis is needed at $\beta=0$: there the momentum is forgotten after
one step, and from $e_0=0$ with $\varphi=0$ the error stays zero.
With $\varphi=0$, $R_k$
is an entry of the iterates of the linear part, which would be summable if
the linear part were exponentially stable.
\end{proof}

Corollary~\ref{cor:stab} assumes $\lambda<2(1+\beta)$. Beyond the curvature
ceiling the mean dynamics itself is unstable, and the criterion $H<1$
remains exact for every $\lambda>0$.

\begin{lemma}[Stability for every $\lambda>0$]\label{lem:stab-all}\label{rem:stab-large-w}
Let $\mathcal L(\Sigma):=F\Sigma F^\top+2\lambda\varepsilon H_n\Sigma_{11}\,bb^\top$
be the linear part of \eqref{eq:L1}. If $\lambda\ge2(1+\beta)$, then $H\ge H_c\ge1$
and $\mathcal L$ is not exponentially stable. Consequently, for every $\lambda>0$,
$\mathcal L$ is exponentially stable if and only if $H<1$.
\end{lemma}

\begin{proof}
Let $\lambda\ge2(1+\beta)$, so that $H_c=\lambda/(2(1+\beta))\ge1$. The characteristic
polynomial $z^2-(1+\beta-\lambda)z+\beta$ of $F$ equals
$2(1+\beta)-\lambda\le0$ at $z=-1$ and tends to $+\infty$ as
$z\to-\infty$, so $F$ has a real eigenvalue $z\le-1$ with a real
eigenvector $v$. The second row of $(F-z)v=0$ reads
$\lambda v_1+(\beta-z)v_2=0$, and $\beta-z>0$, so $v_1\ne0$. The map
$\mathcal L$ preserves the positive semidefinite order, because $H_n\ge0$ and
$\Sigma\mapsto\Sigma_{11}$ is monotone, and $\mathcal L(\Sigma)\succeq F\Sigma F^\top$.
By induction, $\mathcal L^k(vv^\top)\succeq F^kvv^\top F^{k\top}=z^{2k}vv^\top$,
so the iterates from $vv^\top$ have $(1,1)$ entry at least $v_1^2>0$ for
every $k$, and they do not decay.

For the equivalence, let first $H<1$. Then $H_c<1$, that is,
$\lambda<2(1+\beta)$, and $\ut<1$ by Corollary~\ref{cor:stab}(i). With
$\varphi=0$, Corollary~\ref{cor:stab}(ii) shows that the $(1,1)$ entry of
$\mathcal L^k(\Sigma_0)$ decays geometrically for every positive
semidefinite $\Sigma_0$. Unrolling,
$\mathcal L^k(\Sigma_0)$ is $F^k\Sigma_0F^{k\top}$ plus a convolution of these
entries with $F^nbb^\top F^{n\top}$, and $\rho(F)<1$ by Lemma~\ref{lem:L2},
so $\mathcal L^k(\Sigma_0)$ decays geometrically. Positive semidefinite
matrices span the symmetric matrices, so $\mathcal L$ is exponentially
stable. If $H\ge1$ and $\lambda<2(1+\beta)$, then $\ut\ge1$ by
Corollary~\ref{cor:stab}(i), and Corollary~\ref{cor:stab}(iii) applies. If
$\lambda\ge2(1+\beta)$, the first part applies.
\end{proof}

The noise and curvature feedback add in the threshold $H=H_n+H_c<1$, so a noise constraint and
a curvature constraint compete for the same unit budget.

The floor $\varphi/(1-H)$ is, for LS, the label-noise floor
$\varsigma^2\frac{\eta d}{2B}\big/(1-\eta/\eta_+)$. It is
proportional to the learning rate through $\varphi=\eta\vadd/(2\cA)$, and
the momentum enters it only through the curvature feedback $H_c$ in the
amplification factor $1/(1-H)$. It is the learning rate, not the momentum,
that sets the noise floor.

\paragraph{The continuum moment system.} The object that the chain will be
compared with is a three-dimensional linear system of ordinary differential
equations. For $\omega^2>0$, $H\in[0,1)$ and $\varphi\ge0$, let
$(R,V,C)(\tau)$ solve
\begin{equation}\label{eq:L3}
R'=-2\omega^2 C,\qquad V'=2C-2V+\frac2{\omega^2}(HR+\varphi),\qquad C'=R-C-\omega^2 V .
\end{equation}
This is the second-moment system of the stochastic differential equation
\[
dX=-\omega^2 Y\,d\tau,\qquad dY=(X-Y)\,d\tau+dM,\qquad
d\langle M\rangle=\frac2{\omega^2}\bigl(H|X|^2+\varphi\bigr)\,d\tau,
\]
with $R=\E|X|^2$, $V=\E|Y|^2$ and $C=\E\langle X,Y\rangle$. Its mean flow is
the matrix
\[
M_2(\omega^2)=\begin{pmatrix}0&-\omega^2\\1&-1\end{pmatrix}.
\]
In the variables of LS this is, formally, the heavy-ball pair $X=e$,
$Y=m/\cA$, on the retention clock $\tau=\varepsilon k$: one unit of $\tau$ is the memory length
$1/\varepsilon$ of the momentum.

Written with a Brownian motion, the stochastic differential equation is the
following. Fix a dimension $n\ge1$, let $W$ be a standard Brownian motion in
$\R^n$, and let $(X_0,Y_0)$ be independent of $W$ with
$\E(|X_0|^4+|Y_0|^4)<\infty$. The \emph{stochastic heavy ball} is
\begin{equation}\label{eq:SHB}
dX=-\omega^2 Y\,d\tau,\qquad
dY=(X-Y)\,d\tau+\sqrt{\frac{2\,(H|X|^2+\varphi)}{n\omega^2}}\;dW .
\end{equation}
Its drift is the heavy-ball field with unit friction and stiffness $\omega^2$.
The noise enters the momentum only, and its variance is affine in $|X|^2$,
which is why the second moments close. The martingale $M$ above is the
stochastic integral in~\eqref{eq:SHB}.

\begin{remark}[The stochastic heavy ball and its moments]\label{rem:SHB}
Let $\omega^2>0$ and $H,\varphi\ge0$.
\begin{enumerate}
\item[(a)] \eqref{eq:SHB} has a unique strong solution, and
$(R,V,C):=(\E|X|^2,\E|Y|^2,\E\langle X,Y\rangle)$ solves~\eqref{eq:L3} from
$(\E|X_0|^2,\E|Y_0|^2,\E\langle X_0,Y_0\rangle)$. Every positive
semidefinite initial condition of~\eqref{eq:L3} arises in this way.
\item[(b)] If $H=\varphi=0$ and $(X_0,Y_0)=(x_0v,y_0v)$ for a deterministic
unit vector $v$, then $(X,Y)=(xv,yv)$, where
\[
\frac{dx}{d\tau}=-\omega^2 y,\qquad\frac{dy}{d\tau}=x-y,
\]
so that $R=x^2$ and $x''+x'+\omega^2 x=0$.
\end{enumerate}
\end{remark}

\begin{proof}
(a) The drift is linear. The diffusion coefficient
$c(x):=\sqrt{2(H|x|^2+\varphi)/(n\omega^2)}$ is globally Lipschitz, with
constant $\sqrt{2H/(n\omega^2)}$, because $x\mapsto\sqrt{H|x|^2+\varphi}$ is
the Euclidean norm of $(\sqrt H\,x,\sqrt\varphi)$. Hence there is a unique
strong solution with continuous paths \citep[Theorem~5.2.1]{Oksendal2003}.

Write $\xi=(X,Y)$ and let $\mathcal G$ be the generator of~\eqref{eq:SHB}.
Since $dX$ carries no noise and $n\,c(X)^2=\frac2{\omega^2}(H|X|^2+\varphi)$,
\[
\mathcal G|X|^2=-2\omega^2\langle X,Y\rangle,\qquad
\mathcal G|Y|^2=2\langle X,Y\rangle-2|Y|^2+\tfrac2{\omega^2}\bigl(H|X|^2+\varphi\bigr),\qquad
\mathcal G\langle X,Y\rangle=|X|^2-\langle X,Y\rangle-\omega^2|Y|^2 .
\]
In expectation these are the three equations of~\eqref{eq:L3}, once the
stochastic integrals in Itô's formula are shown to have mean zero. Put
$g:=1+|X|^2+|Y|^2$. There is $K$, depending on $\omega^2,H,\varphi,n$, with
$\mathcal Gg\le Kg$, and hence $\mathcal G(g^2)=2g\,\mathcal Gg+4|Y|^2c(X)^2\le Kg^2$
after enlarging $K$. Let $\tau_N$ be the exit time of $\{g<N\}$. The stopped
stochastic integrals are martingales, and the Gr\"onwall inequality gives
$\E g(\xi_{\tau\wedge\tau_N})^2\le e^{K\tau}\E g(\xi_0)^2$ for all $N$. So, for
each quadratic $f\in\{|X|^2,|Y|^2,\langle X,Y\rangle\}$, the family
$f(\xi_{\tau\wedge\tau_N})$ is uniformly integrable in $N$, and
$|\mathcal Gf|\le Kg$ with $\int_0^\tau\E g(\xi_s)\,ds<\infty$ by Fatou's lemma.
Letting $N\to\infty$ in
$\E f(\xi_{\tau\wedge\tau_N})=\E f(\xi_0)+\E\int_0^{\tau\wedge\tau_N}\mathcal Gf(\xi_s)\,ds$
gives $\E f(\xi_\tau)=\E f(\xi_0)+\int_0^\tau\E\,\mathcal Gf(\xi_s)\,ds$. This
is~\eqref{eq:L3} in integral form, and its locally bounded solution is
unique. Finally, a positive semidefinite
$\Sigma_0=\bigl(\begin{smallmatrix}R_0&C_0\\C_0&V_0\end{smallmatrix}\bigr)$ is
obtained from the Gaussian start $(X_0,Y_0)\sim\N(0,\Sigma_0\otimes I_n/n)$.

(b) With $H=\varphi=0$, \eqref{eq:SHB} is a linear ordinary differential
equation, which maps the plane $\{(xv,yv)\}$ into itself; by uniqueness the
solution stays in it. Then $R=|X|^2=x^2$, and
$x''=-\omega^2 y'=-\omega^2 x+\omega^2 y=-\omega^2 x-x'$.
\end{proof}

\noindent In this language, Theorem~\ref{thm:M}(i) below says that, under its
hypotheses, $\sup_k\bigl|R_k-\E|X(k\ebar)|^2\bigr|\le C_\delta\,\ebar\,(N_0+\pt)$,
where $(X,Y)$ solves~\eqref{eq:SHB} at the matched parameters
$(\bar\omega^2,\ut,\pt)$ from an initial condition with second moments
$\Sh_0=P^{-1}\Sigma_0P^{-\top}$: the risk of the chain is the second moment
of a stochastic heavy ball sampled on the grid $\tau_k=k\ebar$. In cell~3 of
Section~\ref{sec:dict}, from the cold start, part (b) is the limit
(Corollary~\ref{cor:lift}).

Figure~\ref{fig:shb-graphs} draws the moment system~\eqref{eq:L3} and the stochastic heavy ball~\eqref{eq:SHB} as flow graphs.

\begin{figure}[t]
\centering
\includegraphics[width=0.42\textwidth]{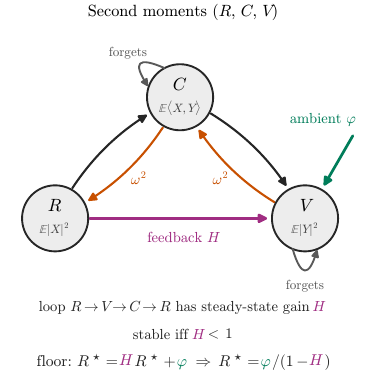}\hspace{0.06\textwidth}%
\includegraphics[width=0.42\textwidth]{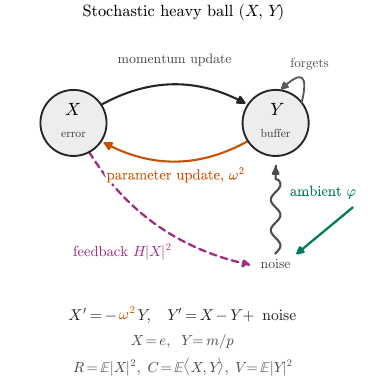}
\caption{\textbf{The moment system~\eqref{eq:L3} and the stochastic heavy ball~\eqref{eq:SHB} as flow graphs.}
\textbf{Left:} each arrow is one term of~\eqref{eq:L3}. The unlabelled arrows are $+R$ in $C'$ and $+2C$ in
$V'$, the loops are $-C$ and $-2V$, the orange arrows are $-2\omega^2C$ in $R'$ and $-\omega^2V$ in $C'$, and
$V'$ receives the feedback $\frac2{\omega^2}HR$ and the ambient source $\frac2{\omega^2}\varphi$. At equilibrium
$C=0$, $V=(HR+\varphi)/\omega^2$ and $R=\omega^2V$, so around the loop $R\to V\to C\to R$ a unit of risk
sustains $H$ units: the system is stable if and only if $H<1$ (Lemma~\ref{lem:L3}), and $R^\star=HR^\star+\varphi$
gives the floor $\varphi/(1-H)$. The kernel mass $\ut$ of Lemma~\ref{lem:L2} is the per-step counterpart of this
loop gain. \textbf{Right:} the error $X=e$ and the rescaled momentum $Y=m/\cA$ ($\cA=p$ for LS) on the retention
clock $\tau=\varepsilon k$. The momentum update averages $X$ into $Y$, which forgets at unit rate; the parameter
update moves $X$ against $Y$ with stiffness $\omega^2$; and the noise in $Y$ has temperature $H|X|^2+\varphi$.
$R$, $C$ and $V$ are the second moments of $(X,Y)$ (Remark~\ref{rem:SHB}). The colours are those of the main
text: stiffness, feedback and temperature.}
\label{fig:shb-graphs}
\end{figure}

We write
$\Phi_{\omega^2}(t):=e_1^\top e^{tM_2}e_2$ for the free response of $X$ to a unit
kick in $Y$, and $\Omega:=\sqrt{\omega^2-\frac14}$ for the free angular
frequency when $\omega^2>\frac14$. The next lemma collects the structure of
\eqref{eq:L3} that the moment comparison theorem uses.

\begin{lemma}[Structure of \eqref{eq:L3}]\label{lem:L3}
The system \eqref{eq:L3} has the following properties.
\begin{enumerate}
\item[(a)] \emph{Renewal.} Let
$\hat a(\tau):=(e^{\tau M_2}\Sigma_0e^{\tau M_2^\top})_{11}$ and
$k_{\omega^2}:=\frac2{\omega^2}\Phi_{\omega^2}^2\ge0$. Then
\[
R(\tau)=\hat a(\tau)+\int_0^\tau k_{\omega^2}(\tau-s)\,(HR(s)+\varphi)\,ds,
\qquad\text{and}\qquad
\tilde k_{\omega^2}(s)=\frac{4\omega^2}{(s+1)(s^2+2s+4\omega^2)}
\]
is the Laplace transform of $k_{\omega^2}$. The kernel has mass
$\int_0^\infty k_{\omega^2}=1$ and total variation $\TV(k_{\omega^2})\le2\omega$.
\item[(b)] \emph{Energy.} Under the free flow ($H=\varphi=0$) the energy
$X^2+\omega^2 Y^2$ is non-increasing pathwise. Hence
$\hat a(\tau)\le R_0+\omega^2 V_0$ for every positive semidefinite $\Sigma_0$.
\item[(c)] \emph{Stability and rate.} For every $H\ge0$, the system
\eqref{eq:L3} is exponentially stable if and only if $H<1$, with fixed point
$R_\infty=\varphi/(1-H)$. Its characteristic polynomial is
\[
p(z)=z^3+3z^2+(2+4\omega^2)z+4\omega^2(1-H),
\]
and its Perron rate $r_c(\omega^2,H):=-\max\{\operatorname{Re}z:p(z)=0\}$
satisfies $r_c<1-H$ for $H\in(0,1)$.
\item[(d)] \emph{Exact eigenfunctionals.} Let $z_\pm$ be the roots of
$z^2+2z+4\omega^2$. The functionals
\[
S:=R+\omega^2 V-C,\qquad Z_\pm:=R-\frac{z_\pm\omega^2}{z_\pm+2}V+z_\pm C
\]
diagonalize the free flow when $\omega^2\ne\frac14$, with eigenvalues $-1$ and
$z_\pm$. With noise,
\begin{equation}\label{eq:SZ}
S'=-S+2(HR+\varphi),\qquad Z_\pm'=z_\pm Z_\pm-\frac{2z_\pm}{z_\pm+2}(HR+\varphi).
\end{equation}
For $\omega^2>\frac14$, write $Z:=Z_+$, so that $z_+=-1+2i\Omega$. Then
$R=r_SS+2\operatorname{Re}(r_ZZ)$, with
\[
r_S=\frac{\omega^2}{2\Omega^2}\qquad\text{and}\qquad r_Z=-\frac{(1+2i\Omega)^2}{16\Omega^2}.
\]
\item[(e)] \emph{Rank.} The determinant $Q:=RV-C^2$ obeys
\[
Q'=-2Q+\frac2{\omega^2} R(HR+\varphi).
\]
Hence, for positive semidefinite $\Sigma_0$, $Q\ge0$ is preserved, and
$Q(\tau)>0$ for $\tau>0$ as soon as
$H+\varphi>0$ and $R\not\equiv0$.
\end{enumerate}
\end{lemma}

\noindent In (a), $\hat a$ is the free decay of the initial condition, so the
continuum system has the same renewal structure as the chain, with a kernel
of mass exactly one. The total feedback $H$ thus plays the role that $\ut$ plays for
the chain: it is the gain of the noise feedback, the system is stable
exactly when $H<1$, and the floor is $\varphi/(1-H)$.

Part (d) separates the dynamics into a slow energy $S$, which decays at the
retention rate $1$, and a conjugate pair of oscillating modes $Z_\pm$ with
decay rate $1$ and frequency $2\Omega$. The noise feeds all three through the
single scalar $HR+\varphi$.

Part (e) says that once any noise is present, the covariance of
$(X,Y)$ has full rank for all positive times, so $R$ is not the square of the
solution of a scalar equation. This is what we mean when we call the limit
irreducibly three-dimensional.

\begin{proof}
Most of the statements are direct computations. The one that needs an
observation is (a): the noise enters \eqref{eq:L3} only through the equation
for $V$, so $R$ responds to it through the free second-moment flow.

(a) Writing $\Sigma=\bigl(\begin{smallmatrix}R&C\\C&V\end{smallmatrix}\bigr)$,
the system \eqref{eq:L3} reads
$\Sigma'=M_2\Sigma+\Sigma M_2^\top+\frac2{\omega^2}(HR+\varphi)e_2e_2^\top$. By
variation of constants,
\[
\Sigma(\tau)=e^{\tau M_2}\Sigma_0e^{\tau M_2^\top}+\int_0^\tau e^{(\tau-s)M_2}e_2e_2^\top e^{(\tau-s)M_2^\top}\,\frac2{\omega^2}\bigl(HR(s)+\varphi\bigr)\,ds,
\]
and the $(1,1)$ entry of this identity is the renewal equation, because
$(e^{tM_2}e_2e_2^\top e^{tM_2^\top})_{11}=\Phi_{\omega^2}(t)^2$. The Laplace transform
is a direct computation, and $\tilde k_{\omega^2}(0)=1$ is the mass. For
the total variation, note that $\Phi_{\omega^2}=-\omega^2\psi$, where $\psi$ solves
\[
\psi''+\psi'+\omega^2\psi=0,\qquad \psi(0)=0,\quad\psi'(0)=1 .
\]
The energy identity $\frac{d}{dt}(\psi'^2+\omega^2\psi^2)=-2\psi'^2$ gives
$\int_0^\infty\psi'^2=\frac12$, and $\int k_{\omega^2}=1$ gives
$\int_0^\infty\psi^2=\frac1{2\omega^2}$. By the Cauchy--Schwarz inequality,
\[
\TV(k_{\omega^2})=4\omega^2\int_0^\infty|\psi\psi'|\le4\omega^2\,(2\omega^2)^{-1/2}\,2^{-1/2}=2\omega .
\]

(b) Along the free flow,
$\frac{d}{d\tau}(X^2+\omega^2 Y^2)=-2\omega^2 XY+2\omega^2 Y(X-Y)=-2\omega^2 Y^2\le0$.
Taking expectations, the free solution satisfies
$R(\tau)+\omega^2 V(\tau)\le R_0+\omega^2 V_0$, and in particular
$\hat a(\tau)\le R_0+\omega^2 V_0$.

(c) Expanding $\det(z I-A)$ for the generator $A$ of \eqref{eq:L3} on
$(R,V,C)$ gives $p$, and $p(z)=(z+1)(z^2+2z+4\omega^2)-4\omega^2 H$,
so that
\[
p(s)=(s+1)(s^2+2s+4\omega^2)\bigl(1-H\tilde k_{\omega^2}(s)\bigr).
\]
The roots of $(s+1)(s^2+2s+4\omega^2)$ have negative real parts, and for
$\operatorname{Re}s\ge0$ we have $|\tilde k_{\omega^2}(s)|\le\int_0^\infty k_{\omega^2}=1$
because $k_{\omega^2}\ge0$. Hence, if $H<1$, then $p$ has no root with
$\operatorname{Re}s\ge0$, and \eqref{eq:L3} is exponentially stable. If
$H\ge1$, then $p(0)=4\omega^2(1-H)\le0$ while $p(z)\to\infty$ as
$z\to\infty$, so $p$ has a root in $[0,\infty)$. The fixed point solves
$R'=V'=C'=0$: the first equation gives $C=0$, the third $R=\omega^2 V$, and the
second $\omega^2 V=HR+\varphi$, so that $R_\infty=\varphi/(1-H)$. Finally, at
$z=H-1$,
\[
p(H-1)=H\bigl((H-1)^2+2(H-1)+4\omega^2\bigr)-4\omega^2 H=-H(1-H^2)<0<p(0),
\]
so $p$ has a real root in $(H-1,0)$, and therefore $r_c<1-H$.

(d), (e) The identities follow by direct substitution into \eqref{eq:L3};
for (d) one uses only $z_\pm^2+2z_\pm+4\omega^2=0$, and
$z_\pm+2\ne0$. For the positivity in (e), let $\Sigma_0\succeq0$. Then
$\hat a\ge0$ and $k_{\omega^2}\ge0$, so the Neumann series of the Volterra
equation in (a) has nonnegative terms and $R\ge0$. Variation of constants
gives
\[
Q(\tau)=e^{-2\tau}Q_0+\int_0^\tau e^{-2(\tau-s)}\,\frac2{\omega^2} R(s)\bigl(HR(s)+\varphi\bigr)\,ds\ge0 .
\]
Since $R$ solves a linear system with constant coefficients, it is real
analytic, so if $R\not\equiv0$ its zeros are isolated. When $H+\varphi>0$ the
integrand is then positive almost everywhere, and $Q(\tau)>0$ for $\tau>0$.
\end{proof}

\section{The moment comparison theorem}\label{sec:M}

This section proves the central result of the analysis: the chain of
Section~\ref{sec:renewal} is the continuum system~\eqref{eq:L3} sampled at the
times $\tau_k=k\ebar$, up to an error of order $\ebar$ that is uniform in
$\bar\omega^2$. The argument has an exact part, which comes first, and an approximate
part.

In the exact part we define the matched parameters
(Definition~\ref{def:matched}), show that $F$ is similar, up to a sign, to the
time-$\ebar$ flow map of the continuum mean flow (Lemma~\ref{lem:match}), and
conclude that in matched coordinates the chain is the free continuum flow
sampled on a grid, with the noise injected as impulses at the grid times
(Lemma~\ref{lem:embed}). The size of the impulses is identified through the
kernel mass of Lemma~\ref{lem:L2}.

What remains is to compare a Riemann sum
with an integral. Lemma~\ref{lem:quad} does this for a single pair of
exponential modes, and Lemma~\ref{lem:modes} shows that for large $\bar\omega^2$ the
continuum solution consists of three such modes with bounded coefficients.
Theorem~\ref{thm:M} combines the two with the a priori bounds of
Corollary~\ref{cor:stab} and Lemma~\ref{lem:L3}; its part~(ii), on the
eigenfunctionals $S$ and $Z$, is what Section~\ref{sec:cor} needs for the
window. The Nyquist margin enters the proof at a single point, which
Figure~\ref{fig:nyquist} illustrates. Three remarks close the section: on the
size of the constants, on why the matched parameters rather than the nominal
ones are needed, and on why the Nyquist band is a genuine exception.

\begin{figure}[t]
\centering
\includegraphics[width=\textwidth]{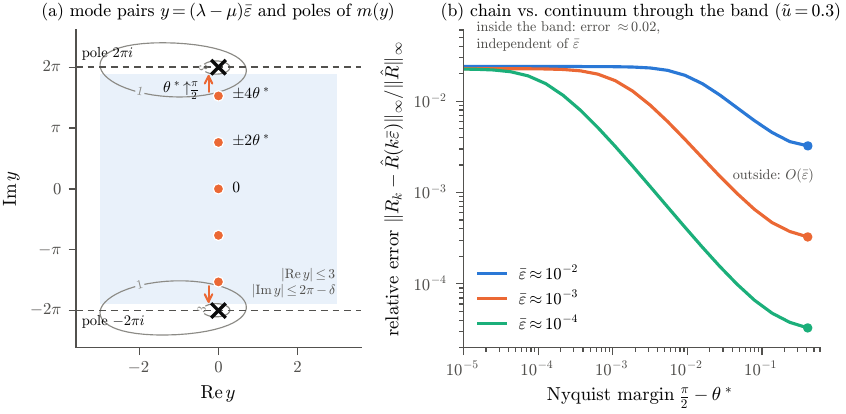}
\caption{The Nyquist geometry of the quadrature step (Lemma~\ref{lem:quad} and
Step~3 of the proof of Theorem~\ref{thm:M}). (a) In the plane of
$y=(z-\mu)\ebar$, the ratio $m(y)=y/(e^y-1)$ between a Riemann sum and
its integral is analytic except for poles at $\pm2\pi i$ (crosses); the grey
curves are the level sets $|m(y)-1|/|y|=1$ and $3$, so the constant
$c_\delta$ of Lemma~\ref{lem:quad} grows as $y$ approaches a pole. The mode
pairs of Step~3 sit at $\operatorname{Re}y=O(\ebar)$ and
$\operatorname{Im}y\in\{0,\pm2\theta^*,\pm4\theta^*\}$ (dots, drawn for
$\theta^*=1.2$). The Nyquist margin $\theta^*\le\frac\pi2-\delta$ keeps the
outer pair inside the shaded rectangle $|\operatorname{Re}y|\le3$,
$|\operatorname{Im}y|\le2\pi-\delta$; as $\theta^*\to\frac\pi2$ (arrows) the
outer pair reaches the poles, which is the aliasing of
Remark~\ref{rem:nyquist}. (b) The error
$\sup_k|R_k-\Rh(k\ebar)|/\sup\Rh$ between the exact chain and the continuum at
matched parameters, against the Nyquist margin $\frac\pi2-\theta^*$, for three
values of $\ebar$ a decade apart (with $\ut=0.3$, $\varphi=0$, the matched
start $\Sh_0=\operatorname{diag}(1,0)$ and horizon $\tau=10$). Away from the
band the error is a bounded multiple of $\ebar$. Inside a band whose width is
proportional to $\ebar$ it saturates at a value independent of $\ebar$.}
\label{fig:nyquist}
\end{figure}

The chain \eqref{eq:L1} and the continuum system \eqref{eq:L3} are both linear
in the second moments, and both are driven by noise of the form
$\text{(feedback)}\times R+\text{(temperature)}$. To compare them we choose the
parameters of the continuum system so that its free flow, run for one time
step, reproduces the free map $F$ of the chain exactly.
This requires $\beta>0$, since $\det e^{\ebar M_2(\omega^2)}=e^{-\ebar}>0$ while
$\det F=\beta$. For $\beta>0$ the eigenvalues of
$F$ are $\sqrt\beta e^{\pm i\theta}$, while those of $e^{\ebar M_2(\omega^2)}$
are $e^{-\ebar/2}e^{\pm i\ebar\sqrt{\omega^2-1/4}}$, so we must match the
modulus through the time step $\ebar$ and the angle through $\omega^2$.

\begin{definition}[Matched parameters]\label{def:matched}
Let $\beta\in(0,1)$, let $\ebar:=-\log\beta$, and let $\sigma_F:=1$ if $\cos\theta\ge0$ and
$\sigma_F:=-1$ if $\cos\theta<0$. In the
underdamped case $|\cos\theta|\le1$, let $\theta^*:=\min(\theta,\pi-\theta)$
and
\[
\bar\omega^2:=\frac14+\Bigl(\frac{\theta^*}{\ebar}\Bigr)^2 .
\]
In the overdamped case $|\cos\theta|>1$, let
\[
\bar\omega^2:=\frac14-\Bigl(\frac{\operatorname{arccosh}|\cos\theta|}{\ebar}\Bigr)^2 .
\]
Set $G:=e^{\ebar M_2(\bar\omega^2)}$. The matched feedback and temperature are $(\ut,\pt)$.
\end{definition}

\noindent In the overdamped case $\bar\omega^2\in(0,\frac14)$: the eigenvalues of
$F$ are $\pm\sqrt\beta e^{\pm a}$ with $\cosh a=|\cos\theta|$, and
$\rho(F)<1$ gives $a<-\frac12\log\beta=\ebar/2$. The matched feedback and temperature are
$(\ut,\pt)$ and not $(H_n,\varphi)$; Remark~\ref{rem:matched} explains why
this matters.

The matched time step $\ebar=-\log\beta=\varepsilon+O(\varepsilon^2)$
is the retention rate measured logarithmically. For small $\lambda$ one has
$\bar\omega^2=\omega^2(1+O(\varepsilon+\lambda))$, so the matched and nominal ratios agree.
Precisely, for $\beta\in[\frac12,1)$ and $0<\lambda\le\frac12$ the constant in
$O(\varepsilon+\lambda)$ is absolute. To see this, put
$c(\lambda):=\cos\theta=(1+\beta-\lambda)/(2\sqrt\beta)$, which lies in $[\frac12,2]$
for these parameters, and $h(c):=(\operatorname{arccosh}c)^2$, which is
analytic for $c>-1$ and equals $-(\arccos c)^2$ for $c\le1$. Both cases of
Definition~\ref{def:matched} then read $\bar\omega^2=\frac14-h(c(\lambda))/\ebar^2$. At
$\lambda=0$ we have $c=\cosh(\ebar/2)$, hence $\bar\omega^2=0$, and since
$c'(\lambda)=-1/(2\sqrt\beta)$ and $h'(c)=2\operatorname{arccosh}c/\sqrt{c^2-1}$,
\[
\frac{d\bar\omega^2}{d\lambda}\Big|_{\lambda=0}=\frac{1}{2\sqrt\beta\,\ebar\sinh(\ebar/2)}=\frac{1}{\varepsilon\ebar},
\]
because $2\sqrt\beta\sinh(\ebar/2)=1-\beta$. The second derivative is
$-h''(c)/(4\beta\ebar^2)=O(\ebar^{-2})$. Taylor's theorem and
$\omega^2=\lambda/\varepsilon^2$ give
$\bar\omega^2=\omega^2\,\varepsilon/\ebar+O(\lambda^2/\ebar^2)=\omega^2\bigl(\varepsilon/\ebar+O(\lambda)\bigr)$,
and $\varepsilon/\ebar=1+O(\varepsilon)$.
When
$\lambda\asymp1$, on the other hand, the angle $\theta^*$ is a fixed fraction of
$\pi$ and $\bar\omega^2\asymp\ebar^{-2}$, which can be very different from $\omega^2$.

The sign $\sigma_F$ records whether the chain overshoots: when $\lambda>1+\beta$ the
eigenvalues of $F$ have negative real part, and $F$ is matched to $-G$.

\begin{lemma}[Matching]\label{lem:match}
There is a unique real matrix
$P=\bigl(\begin{smallmatrix}1&0\\ \pi_{21}&\pi_{22}\end{smallmatrix}\bigr)$
such that $F=\sigma_FPGP^{-1}$. It satisfies
\[
P^{-1}b=\frac{\sigma_F}{\beta\pi_{22}}\,Ge_2 .
\]
\end{lemma}

\begin{proof}
Two real $2\times2$ matrices that are not multiples of the identity and have
the same trace and determinant are similar over $\R$, since both are similar
to the same companion matrix. We check these conditions and then normalize
the similarity.

First, $\det G=e^{\ebar\operatorname{tr}M_2}=e^{-\ebar}=\beta=\det F$
and $\operatorname{tr}G=2\sqrt\beta\cos\theta^*=\operatorname{tr}(\sigma_FF)$,
with $\cosh$ in place of $\cos$ in the overdamped case. Neither matrix is
scalar, since $G_{12}=\Phi_{\bar\omega^2}(\ebar)\ne0$ and $F_{12}=-\beta\ne0$. Hence
there is an invertible $P_0$ with $\sigma_FFP_0=P_0G$, and since the commutant
of the non-scalar matrix $G$ is $\{xI+yG\}$, the solutions of
$\sigma_FFP=PG$ are exactly the matrices $P_0(xI+yG)$. Requiring the first row
of $P$ to be $e_1^\top$ is a $2\times2$ linear system for $(x,y)$, and it is
nonsingular because $e_1^\top$ is not a left eigenvector of $F$ (as
$F_{12}\ne0$).

Comparing first rows in $FP=\sigma_FPG$ gives
$(1-\lambda-\beta\pi_{21},\,-\beta\pi_{22})=\sigma_F(G_{11},G_{12})$, that is,
\[
\pi_{22}=-\sigma_FG_{12}/\beta,\qquad \pi_{21}=(1-\lambda-\sigma_FG_{11})/\beta .
\]
Here $G_{12}=\Phi_{\bar\omega^2}(\ebar)=-\bar\omega^2\psi(\ebar)$, where
\[
\psi(\ebar)=
\begin{cases}
e^{-\ebar/2}\sin(\theta^*)/\Omega & \text{in the underdamped case},\\
e^{-\ebar/2}\sinh(a)/\sqrt{\tfrac14-\bar\omega^2} & \text{in the overdamped case},\\
\ebar e^{-\ebar/2} & \text{when }\bar\omega^2=\tfrac14,
\end{cases}
\]
with $\cosh a=|\cos\theta|$; it is nonzero because $0<\theta^*\le\frac\pi2$
(or $a>0$). Hence $\pi_{22}\ne0$ and $P$ is
invertible. This gives existence and uniqueness.

Finally, $Fe_2=\beta b$, so
\[
P^{-1}b=\beta^{-1}P^{-1}Fe_2=\beta^{-1}\sigma_FGP^{-1}e_2,
\]
and $P^{-1}e_2=e_2/\pi_{22}$ because $P$ is lower triangular.
\end{proof}

The matching has an immediate and exact consequence: in the coordinates
defined by $P$, the chain is the continuum system observed on a grid.

\begin{lemma}[Exact embedding: impulses on a grid]\label{lem:embed}
Let $\Sh_k:=P^{-1}\Sigma_kP^{-\top}$. Then $(\Sh_k)_{11}=R_k$, and
\begin{equation}\label{eq:embed}
\Sh_{k+1}=G\bigl(\Sh_k+z_d\,(\ut R_k+\pt)\,e_2e_2^\top\bigr)G^\top,
\qquad z_d:=\Bigl(\sum_{n\ge1}\Phi_{\bar\omega^2}(n\ebar)^2\Bigr)^{-1}.
\end{equation}
Equivalently, with $\hat a$ the free continuum solution from $\Sh_0$,
\begin{equation}\label{eq:discrete-renewal}
R_k=\hat a(k\ebar)+\sum_{n=1}^{k}h_n\,(\ut R_{k-n}+\pt),\qquad
h_n:=\frac{k_{\bar\omega^2}(n\ebar)}{\sum_{m\ge1}k_{\bar\omega^2}(m\ebar)},\qquad\sum_{n\ge1}h_n=1 .
\end{equation}
In particular, if $H_n=\varphi=0$, then $R_k=\Rh(k\ebar)$ exactly.
\end{lemma}

\noindent Equation~\eqref{eq:embed} says that the chain is the \emph{free}
continuum flow at $\bar\omega^2$, sampled at the times $\tau_k=k\ebar$, with the
noise injected into $\hat V$ as impulses at the grid times.

\begin{proof}
Conjugating \eqref{eq:L1} by $P^{-1}$ and using Lemma~\ref{lem:match} gives
\eqref{eq:embed}, with $z_d(\ut R_k+\pt)$ replaced by the impulse
\[
c_k:=\eta^2\varepsilon^2(\beta\pi_{22})^{-2}(\vinc R_k+\vadd);
\]
the sign $\sigma_F$ cancels in $(\sigma_FG)\Sh(\sigma_FG)^\top$.

It remains to identify $c_k$, and for this
we compare two expressions for the kernel of Lemma~\ref{lem:L2}. Since the
first row of $P$ is $e_1^\top$, Lemma~\ref{lem:match} gives
$e_1^\top F^nb=\pm(\beta\pi_{22})^{-1}\Phi_{\bar\omega^2}((n+1)\ebar)$, and hence
\[
K_n=\eta^2\varepsilon^2(\beta\pi_{22})^{-2}\,\Phi_{\bar\omega^2}((n+1)\ebar)^2 .
\]
Summing over $n$ gives
$\sum_nK_n=\eta^2\varepsilon^2(\beta\pi_{22})^{-2}/z_d$, so that
$c_k=z_d\bigl(\sum_nK_n\bigr)(\vinc R_k+\vadd)=z_d(\ut R_k+\pt)$ by
the mass identities of Lemma~\ref{lem:L2}. Taking the $(1,1)$ entry of the unrolled recursion \eqref{eq:embed}
gives \eqref{eq:discrete-renewal}, because
$z_d\Phi_{\bar\omega^2}(n\ebar)^2=h_n$.
\end{proof}

Lemma~\ref{lem:embed} is the structural heart of the analysis (Step~2 of the
outline in Section~\ref{sec:app-overview}). The chain is not
an approximation of the continuum system; it \emph{is} the continuum system,
observed at the grid times, except that the noise arrives in discrete
impulses rather than continuously. The renewal equations
\eqref{eq:discrete-renewal} and Lemma~\ref{lem:L3}(a) have the same forcing,
and their kernels differ only in that $h$ is the normalized sampling of
$k_{\bar\omega^2}$ on the grid.

Everything that remains is therefore quadrature: we
must bound the difference between an integral against $k_{\bar\omega^2}$ and its
Riemann sum, uniformly in $\bar\omega^2$. When $\bar\omega^2$ is large the kernel
oscillates with frequency $2\Omega$, and the generic Riemann-sum bound
$\ebar\TV(k_{\bar\omega^2})\approx2\ebar\Omega=2\theta^*$ is of order one. The way
around this is that both the kernel and the forcing are
sums of a few exponential modes, for which quadrature can be done exactly.

\begin{lemma}[Quadrature of exponential pairs]\label{lem:quad}
Let $\mu,z\in\C$, $\ebar>0$, $k\in\mathbb{N}$, $T=k\ebar$ and
$y:=(z-\mu)\ebar\notin2\pi i\mathbb Z\setminus\{0\}$, and let $m(y):=y/(e^y-1)$, with $m(0)=1$. Then
\[
\sum_{n=1}^{k}\ebar\,e^{\mu n\ebar}e^{z(k-n)\ebar}=m(y)\int_0^Te^{\mu t}e^{z(T-t)}\,dt,
\qquad\text{and}\qquad
\sum_{n\ge1}\ebar e^{\mu n\ebar}=\frac{m(-\mu\ebar)}{-\mu}\quad(\operatorname{Re}\mu<0).
\]
If $|\operatorname{Re}y|\le3$ and $|\operatorname{Im}y|\le2\pi-\delta$, then
$|m(y)-1|\le c_\delta|y|$. Hence, if moreover
$\operatorname{Re}\mu,\operatorname{Re}z\le0$, the sum and the integral
differ by at most
\[
c_\delta\,\ebar\,|e^{z T}-e^{\mu T}|\le2c_\delta\,\ebar,
\]
uniformly in $|\mu|$, $|z|$ and $T$.
\end{lemma}

\noindent The point of Lemma~\ref{lem:quad} is that the error does not
deteriorate with the frequencies of the modes. Only the relative phase per
step, $\operatorname{Im}y$, matters, and the condition
$|\operatorname{Im}y|\le2\pi-\delta$ excludes aliasing: a pair of modes
whose relative phase per step is close to $2\pi$ cannot be distinguished, on
the grid, from a pair that does not oscillate at all.

\begin{proof}
Both sides are geometric, so the ratio can be computed in closed form. With
$q=e^{-y}$, the sum equals $\ebar q(e^{z T}-e^{\mu T})/(1-q)$, and the
integral equals $(e^{z T}-e^{\mu T})/(z-\mu)$. Their ratio is
$yq/(1-q)=m(y)$. The function $(m(y)-1)/y$ is analytic on the given compact
set, since the poles of $m$ closest to the origin are at $\pm2\pi i$, so it is
bounded there by some $c_\delta$. Therefore
\[
|m(y)-1|\cdot\Bigl|\frac{e^{z T}-e^{\mu T}}{z-\mu}\Bigr|\le c_\delta\,\ebar\,|e^{z T}-e^{\mu T}| .
\]
\end{proof}

\begin{figure}[t]
\centering
\includegraphics[width=\textwidth]{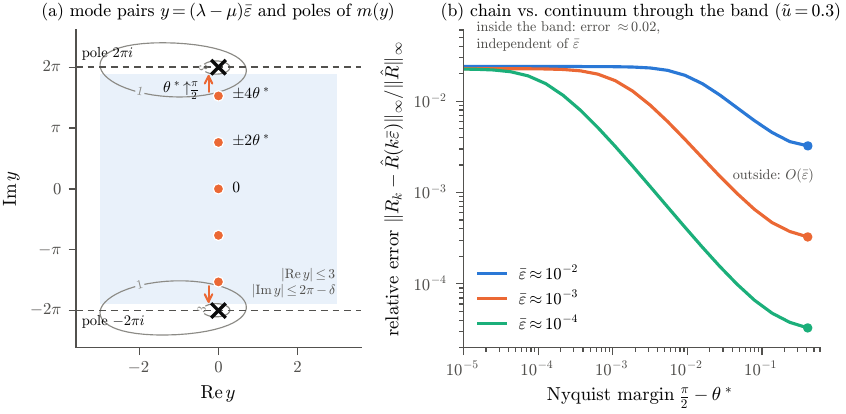}
\caption{The Nyquist geometry of the quadrature step (Lemma~\ref{lem:quad} and
Step~3 of the proof of Theorem~\ref{thm:M}). (a) In the plane of
$y=(z-\mu)\ebar$, the ratio $m(y)=y/(e^y-1)$ between a Riemann sum and
its integral is analytic except for poles at $\pm2\pi i$ (crosses); the grey
curves are the level sets $|m(y)-1|/|y|=1$ and $3$, so the constant
$c_\delta$ of Lemma~\ref{lem:quad} grows as $y$ approaches a pole. The mode
pairs of Step~3 sit at $\operatorname{Re}y=O(\ebar)$ and
$\operatorname{Im}y\in\{0,\pm2\theta^*,\pm4\theta^*\}$ (dots, drawn for
$\theta^*=1.2$). The Nyquist margin $\theta^*\le\frac\pi2-\delta$ keeps the
outer pair inside the shaded rectangle $|\operatorname{Re}y|\le3$,
$|\operatorname{Im}y|\le2\pi-\delta$; as $\theta^*\to\frac\pi2$ (arrows) the
outer pair reaches the poles, which is the aliasing of
Remark~\ref{rem:nyquist}. (b) The error
$\sup_k|R_k-\Rh(k\ebar)|/\sup\Rh$ between the exact chain and the continuum at
matched parameters, against the Nyquist margin $\frac\pi2-\theta^*$, for three
values of $\ebar$ a decade apart (with $\ut=0.3$, $\varphi=0$, the matched
start $\Sh_0=\operatorname{diag}(1,0)$ and horizon $\tau=10$). Away from the
band the error is a bounded multiple of $\ebar$. Inside a band whose width is
proportional to $\ebar$ it saturates at a value independent of $\ebar$.}
\label{fig:nyquist}
\end{figure}

To use Lemma~\ref{lem:quad} we need the continuum solution itself to be a sum
of a few exponential modes, with frequencies close to the free ones and with
controlled coefficients. For large $\bar\omega^2$ this is a perturbative statement.

\begin{lemma}[Modes for large $\omega^2$]\label{lem:modes}
There are absolute constants $\omega_*^2$ and $C_*$ (the proof uses
$\Omega=\sqrt{\omega^2-\frac14}\ge9$, so one may take $\omega_*^2=81.25$) such that the following
holds for $\omega^2\ge\omega_*^2$, $H\in[0,1)$, $\varphi\ge0$ and positive
semidefinite initial data. Every solution of \eqref{eq:L3} has the form
\[
R(\tau)=R_\infty+\sum_{i\in\{0,+,-\}}c_ie^{z_i\tau},
\]
where $z_0\in[-1,0)$ is real, the oscillatory exponents satisfy
\[
|z_\pm-(-1\pm2i\Omega)|\le3\qquad\text{and}\qquad
|\operatorname{Im}z_\pm\mp2\Omega|\le\frac{C_*}{\Omega},
\]
and the coefficients satisfy
\[
\sum_i|c_i|\le C_*\Bigl(R_0+\omega^2 V_0+\frac{\varphi}{1-H}\Bigr).
\]
\end{lemma}

\begin{proof}
In the eigenfunctionals of Lemma~\ref{lem:L3}(d), the noise feedback is a
perturbation of bounded size on a diagonal matrix whose entries are separated
by at least $2\Omega$, and Gershgorin's theorem does the rest. We use
the coordinates $x=(S,Z_+,Z_-)$. By \eqref{eq:SZ} the homogeneous part of the
system is $x'=(D+H\,\nu r^\top)x$, where
\[
D=\operatorname{diag}(-1,z_+,z_-),\qquad \nu=(2,b_+,b_-)^\top,\quad b_\pm=-\frac{2z_\pm}{z_\pm+2},\qquad r=(r_S,r_Z,\bar r_Z)^\top .
\]

\emph{Sizes.} We have $|b_\pm|=2$, $r_S=\omega^2/(2\Omega^2)\le0.6$ and
$|r_Z|=\omega^2/(4\Omega^2)\le0.3$. The diagonal of $E:=H\,\nu r^\top$ is real,
with $E_{SS}=H\omega^2/\Omega^2$ and $E_{\pm\pm}=ub_\pm r_Z^{(\pm)}=-H\omega^2/(2\Omega^2)$.
The off-diagonal row sums of $E$ are at most $2$, while
$|D_{ii}-D_{jj}|\ge2\Omega$ for $i\ne j$.

\emph{Eigenvalues.} The eigenvalues of $D+E$ are those of the real generator
of \eqref{eq:L3}, so they are closed under complex conjugation. For
$\Omega\ge3$ the Gershgorin discs \citep[Theorem~6.1.1]{HornJohnson2013} about
the points $D_{ii}+E_{ii}$ have radii at most $2$ and centres at mutual
distance at least $2\Omega$, so they are disjoint and each contains exactly one
eigenvalue \citep[Section~6.1]{HornJohnson2013}. The disc about the real point
$D_{SS}+E_{SS}$ is symmetric under conjugation, so its eigenvalue $z_0$
is real (the other two discs contain no real point). It is therefore the
real root of $p$ found in Lemma~\ref{lem:L3}(c) (and $-1$ when $H=0$), so
$z_0\in[-1,0)$.

To locate $z_\pm$ more precisely, replace $D+E$
by $J^{-1}(D+E)J$, where $J$ is diagonal with entry $1/t$, $t:=2/\Omega$, in
the $\pm$ position and $1$ elsewhere; this multiplies the off-diagonal
entries of the $\pm$ row by $t$ and those of the $\pm$ column by $1/t$. The
disc about $D_{\pm\pm}+E_{\pm\pm}$ then has radius at most $1.8\,t<4/\Omega$,
the other two discs have radii at most $0.6/t+1.2=0.3\Omega+1.2$, and for
$\Omega\ge3$ these are disjoint from the first. Hence
$|z_\pm-(D_{\pm\pm}+E_{\pm\pm})|<4/\Omega$, and since $E_{\pm\pm}$ is real
and $|E_{\pm\pm}|\le0.3$, both bounds on $z_\pm$ follow, with
$C_*\ge4$.

\emph{Eigenvectors.} The eigenvector $v_i$ normalized by $(v_i)_i=1$
satisfies $|(v_i)_j|\le\|E\|_\infty|v_i|_\infty/|D_{jj}-z_i|=O(1/\Omega)$
for $j\ne i$. (For $\Omega\ge3$, $|D_{jj}-z_i|\ge2\Omega-2$ and
$\|E\|_\infty\le2.4$, so the maximum of $|v_i|$ is attained at the $i$-th
coordinate and $|(v_i)_j|\le2.4/(2\Omega-2)$; for $\Omega\ge9$ this gives
$\|V-I\|_\infty\le\frac13$, hence $\|V\|_\infty\|V^{-1}\|_\infty\le\frac43\cdot\frac32=2$.)
Hence the eigenvector matrix is $V=I+O(1/\Omega)$, and
$\|V\|\|V^{-1}\|\le2$ for $\omega^2\ge\omega_*^2$.

\emph{Coefficients.} We have $R-R_\infty=r^\top(x-x_\infty)$ and
$x-x_\infty=Ve^{\Lambda\tau}V^{-1}(x_0-x_\infty)$, so
$\sum_i|c_i|\le C|x_0-x_\infty|$. For positive semidefinite data,
$|C_0|\le\sqrt{R_0V_0}\le(R_0+\omega^2 V_0)/(2\omega)$, and
$|z_\pm|=2\omega$, $|z_\pm\omega^2/(z_\pm+2)|=\omega^2$ give
$|S_0|,|Z_0|\le2(R_0+\omega^2 V_0)$. Finally,
$S_\infty=2\varphi/(1-H)$ and
$|Z_\infty|=|2/(z_++2)|\,\varphi/(1-H)\le\varphi/(1-H)$.
\end{proof}

We can now state the moment comparison theorem. It compares the chain with the
continuum system at the matched parameters, uniformly in time and in the
ratio $\bar\omega^2\in(0,\infty)$.

\begin{theorem}[Moment comparison theorem]\label{thm:M}
Fix $\delta\in(0,\frac14]$. There are $\omega_0^2(\delta)$ and $C_\delta<\infty$
such that the following holds for all $\beta\in[\frac12,1)$,
$\lambda\in(0,2(1+\beta))$ and $H_n,\varphi\ge0$ with $\ut\le1-\delta$, and for all
$\Sigma_0\succeq0$, subject to the \emph{Nyquist margin}: whenever
$\bar\omega^2\ge\omega_0^2$, also $\theta^*\le\frac\pi2-\delta$. Let $\Rh$ solve \eqref{eq:L3} at the
matched parameters $(\bar\omega^2,\ut,\pt)$ from $\Sh_0=P^{-1}\Sigma_0P^{-\top}$, and
set $N_0:=\Rh_0+\bar\omega^2\hat V_0+|\hat C_0|$. Then:
\begin{enumerate}
\item[(i)] for the error,
\[
\sup_{k\ge0}|R_k-\Rh(k\ebar)|\le C_\delta\,\ebar\,(N_0+\pt);
\]
\item[(ii)] if $\bar\omega^2\ge1$, the same bound holds for the eigenfunctionals
$S$ and $Z$ of Lemma~\ref{lem:L3}(d) at $\bar\omega^2$, evaluated on $\Sh_k$ and on
$\Sh(k\ebar)$.
\end{enumerate}
\end{theorem}

\noindent The Nyquist margin $\theta^*\le\frac\pi2-\delta$ is equivalent to
$|\lambda-(1+\beta)|\ge2\sqrt\beta\sin\delta$.

Theorem~\ref{thm:M} is one approximation statement for all regimes. The
error is $O(\ebar)$ uniformly in the ratio $\bar\omega^2$, uniformly in time, and
at every finite value of the parameters; nothing is sent to a limit. Each
phase of the co-scaling diagram will be obtained afterwards by letting
$\bar\omega^2$, $\ut$ or $\pt$ tend to $0$, to a constant or to $\infty$ inside the
continuum system, and no further comparison with the chain is needed.

The comparison is with the continuum system at the matched parameters,
which are exact functions of $(\beta,\lambda,H_n,\varphi)$. The nominal parameters
$(\omega^2,H_n,\varphi,\varepsilon)$ would not do (Remark~\ref{rem:matched}).

Both margins in the hypothesis are used. The stability margin
$\ut\le1-\delta$ is needed to invert the renewal equation for the error,
which costs a factor $1/(1-\ut)$. The Nyquist margin is needed only for
large $\bar\omega^2$, and it is not an artefact of the proof
(Remark~\ref{rem:nyquist}).

\begin{remark}[The margins below a quarter of the threshold]\label{rem:margins}
Let $\beta\in[\frac12,1)$, $H_n\le\frac14$ and $H_c\le\frac14$. Then the
hypotheses of Theorem~\ref{thm:M} hold with $\delta=\frac14$, whatever the
value of $\bar\omega^2$. For LS, $H_n$ and $H_c$ are $\eta$ divided by the noise
ceiling $2B/(d+2-p)$ and by the curvature ceiling $2(1+\beta)/((1-\beta)p)$,
and $H=H_n+H_c=\eta/\eta_+$ (Section~\ref{sec:L1}). It therefore suffices
that $\eta\le\eta_+/4$, or that $\eta$ is at most a quarter of each ceiling.

To check this, note first that $H_c<1$ gives $\lambda<2(1+\beta)$, and that
$\ut=H_n/(1-H_c)\le\frac13<1-\delta$. For the Nyquist margin,
$\lambda=2(1+\beta)H_c\le\frac12(1+\beta)$, so
\[
|\lambda-(1+\beta)|\ge\tfrac12(1+\beta)\ge\sqrt\beta>2\sqrt\beta\sin\tfrac14 ,
\]
which is the Nyquist margin at $\delta=\frac14$ in the form stated after
Theorem~\ref{thm:M}.
\end{remark}

Of the two margins, only the stability margin depends on how close $\eta$ is
to $\eta_+$. Since $1-\ut=(1-H)/(1-H_c)\ge1-H$, and $H=\eta/\eta_+$ for LS,
the feedback $\ut$ can approach $1$ only if $\eta/\eta_+\to1$. What
Theorem~\ref{thm:M} leaves out is therefore a learning rate pushed up toward
the stability threshold as $\ebar\to0$, for which its constants are not shown
to stay bounded (Remark~\ref{rem:uniform}).

The quarter is a convenient choice, not a sharp one. For $\eta\le c\,\eta_+$
with $c<\frac12$, the same computation gives $\ut\le c$ and
$|\lambda-(1+\beta)|\ge2\sqrt\beta(1-2c)$, so both margins hold with $\delta$
depending only on $c$. For $c\ge\frac12$ the stability margin still holds,
with $1-\ut\ge1-c$, but $H_c$ can reach $\frac12$. That is $\lambda=1+\beta$, the
centre of the Nyquist band (Remark~\ref{rem:nyquist}), and it requires $\lambda$ of
order one, which along the co-scaling happens only in cells~7 and~8 of
Section~\ref{sec:dict}.

The restriction $\beta\ge\frac12$ enters only through $\ebar\le\log2$, in
Steps~3 and~4 of the proof. In Step~3 it gives $|\operatorname{Re}y|\le3$
and the bound on $|\operatorname{Im}y|$ needed for Lemma~\ref{lem:quad}; in
Step~4 it bounds $E-1$ in the lower bound on $\mathsf N$, and the factor
$\ebar/(1-e^{-\ebar})$.

Part (ii) matters in the long-oscillation regime, where $R_k$ itself has no
limit but the slow energy $S$ and the amplitude $|Z|$ of the oscillation do
(Corollary~\ref{cor:window}).

\begin{proof}
Subtracting the two renewal equations shows that the error obeys the same
renewal equation, forced by a quadrature defect: the difference between an
integral against $k_{\bar\omega^2}$ and its normalized Riemann sum on the grid. The
proof bounds this defect in four steps.
\begin{enumerate}
\item We write the error equation and reduce the claim to a bound on the
quadrature defect.
\item For bounded $\bar\omega^2$, we bound the defect by total variation.
\item For large $\bar\omega^2$, we expand the kernel and the forcing in exponential
modes and apply Lemma~\ref{lem:quad} to each pair of modes.
\item We repeat the argument for the eigenfunctionals.
\end{enumerate}

\emph{Step 1: the error equation.} Let $D_k:=R_k-\Rh(k\ebar)$,
$f:=\ut\Rh+\pt$ and $T:=k\ebar$. Both renewal equations contain the same free
term $\hat a$. Subtracting Lemma~\ref{lem:L3}(a) at $\tau=T$ from
\eqref{eq:discrete-renewal} gives
\[
D_k=\ut\sum_{n=1}^kh_nD_{k-n}-\rho_k,\qquad
\rho_k:=\int_0^Tk_{\bar\omega^2}(t)f(T-t)\,dt-\sum_{n=1}^kh_nf(T-n\ebar).
\]
Since $\sum_nh_n=1$ and $D_0=0$, this gives
$\sup_k|D_k|\le(1-\ut)^{-1}\sup_k|\rho_k|\le\delta^{-1}\sup_k|\rho_k|$. Let
$\mathsf N:=\sum_{n\ge1}\ebar k_{\bar\omega^2}(n\ebar)$, so that
$h_n=\ebar k_{\bar\omega^2}(n\ebar)/\mathsf N$. Then
\[
|\rho_k|\le\Bigl|\int_0^Tk_{\bar\omega^2}(t)f(T-t)\,dt-\sum_{n\le k}\ebar k_{\bar\omega^2}(n\ebar)f(T-n\ebar)\Bigr|+|\mathsf N-1|\sup|f| .
\]
By Lemma~\ref{lem:L3}(a),(b), $\sup\Rh\le(N_0+\pt)/\delta$, and hence
$\sup|f|\le(N_0+\pt)/\delta$.

\emph{Step 2: $\bar\omega^2\le\omega_0^2$, by total variation.} A right-endpoint
Riemann sum of a function $g$ has error at most $\ebar\TV(g)$. We take
$g(t)=k_{\bar\omega^2}(t)f(T-t)$, for which
$\TV(g)\le\TV(k_{\bar\omega^2})\sup|f|+\sup|f'|$ because $\int k_{\bar\omega^2}=1$; also
$|\mathsf N-1|\le\ebar\TV(k_{\bar\omega^2})$. Lemma~\ref{lem:L3}(a) gives
$\TV(k_{\bar\omega^2})\le2\omega_0$. For the derivative we have
$f'=-2\ut\bar\omega^2\hat C$. The pair $(\hat C,W)$, with $W:=\bar\omega^2\hat V$, solves
\[
\hat C'=-\hat C+\Rh-W,\qquad W'=-2W+2\bar\omega^2\hat C+2(\ut\Rh+\pt).
\]
The matrix of this system has trace $-3$ and determinant $2+2\bar\omega^2$, so all
its eigenvalues have real part at most $-1$, uniformly for
$\bar\omega^2\in[0,\omega_0^2]$. Hence
$\sup|\hat C|\le C(\omega_0^2)(N_0+\sup\Rh+\pt)$, and therefore
$\sup_k|\rho_k|\le C(\omega_0^2)\,\ebar\,(N_0+\pt)/\delta$.

\emph{Step 3: $\bar\omega^2\ge\omega_0^2$, by modes.} Take
$\omega_0^2(\delta):=\max\bigl(\omega_*^2,\frac14+(C_*\log2/\delta)^2\bigr)$, and
write $\Omega=\sqrt{\bar\omega^2-1/4}$. Then
\[
k_{\bar\omega^2}(t)=\frac{\bar\omega^2}{\Omega^2}e^{-t}(1-\cos2\Omega t)=\sum_jb_je^{\mu_jt},\qquad
\mu_j\in\{-1,-1\pm2i\Omega\},\qquad \sum_j|b_j|=\frac{2\bar\omega^2}{\Omega^2}\le3,
\]
and by Lemma~\ref{lem:modes},
$f=(\ut R_\infty+\pt)+\ut\sum_ic_ie^{z_is}$. The Riemann-sum defect is
therefore a sum over pairs of modes $(\mu_j,z_i)$, the constant mode
$z=0$ included, and each pair is controlled by Lemma~\ref{lem:quad}
with $y=(z_i-\mu_j)\ebar$ (all exponents have nonpositive real part,
because $\ut<1$ and Lemma~\ref{lem:L3}(c)). For these $y$ we have
$|\operatorname{Re}y|\le3\ebar\le3$, since $\ebar\le\log2$. Since
$\Omega\ebar=\theta^*$ and $\theta^*\le\frac\pi2-\delta$ by the Nyquist margin,
\[
|\operatorname{Im}y|\le4\theta^*+\frac{C_*\ebar}{\Omega}\le4\theta^*+\frac{C_*\log2}{\sqrt{\omega_0^2-\tfrac14}}\le2\pi-3\delta .
\]
Figure~\ref{fig:nyquist}(a) shows these mode pairs in the $y$-plane, together
with the poles of $m$ that the Nyquist margin avoids. So each pair contributes at most $2c_{3\delta}\,\ebar\,|b_j||c_i|$. For the
normalization,
\[
|\mathsf N-1|=\Bigl|\sum_jb_j\frac{m(-\mu_j\ebar)-1}{-\mu_j}\Bigr|\le c\,\ebar\sum_j|b_j|,
\]
since $|\operatorname{Im}\mu_j\ebar|\le2\theta^*\le\pi$. Combining these
bounds with the coefficient bound of Lemma~\ref{lem:modes} gives
$\sup_k|\rho_k|\le C_\delta\,\ebar\,(N_0+\pt)$.

\emph{Step 4: part (ii).} Let $\ell\in\{S,Z\}$, with free eigenvalue
$z_\ell$. Then $|\ell(e_2e_2^\top)|=\bar\omega^2$, because
$|z/(z+2)|=1$. From \eqref{eq:embed} and
$z_d=2\ebar/(\bar\omega^2\mathsf N)$,
\[
\ell(\Sh_k)=e^{z_\ell T}\ell(\Sh_0)+\frac{2\ell(e_2e_2^\top)}{\bar\omega^2\mathsf N}\sum_{n=1}^k\ebar e^{z_\ell n\ebar}(\ut R_{k-n}+\pt),
\]
while the continuum has $\frac{2\ell(e_2e_2^\top)}{\bar\omega^2}\int_0^Te^{z_\ell t}f(T-t)\,dt$
in place of the sum.

We also need $\mathsf N$ to be bounded below. Summing the three geometric
series in $\mathsf N=\sum_{n\ge1}\ebar k_{\bar\omega^2}(n\ebar)$ gives, with
$E:=e^{\ebar}=1/\beta$,
\[
\mathsf N=\frac{\bar\omega^2}{\Omega^2}\cdot\frac{\ebar}{1-\beta}\cdot
\frac{(1+E)(1-\cos2\theta^*)}{(E-1)^2+2E(1-\cos2\theta^*)} .
\]
The first two factors are at least $1$. For $\bar\omega^2\ge1$ we have
$\theta^*=\Omega\ebar\ge\frac{\sqrt3}2\ebar$, so
$\sin\theta^*\ge\frac2\pi\theta^*\ge\frac{\sqrt3}{\pi}\ebar$, while
$E-1\le2\ebar$; since $1-\cos2\theta^*=2\sin^2\theta^*$, the last factor is
at least $2(1+E)/(\frac{4\pi^2}{3}+4E)\ge\frac15$. Hence $\mathsf N\ge\frac15$.

We split the difference into two parts. The part involving $D$ is at most
$\frac{2}{\mathsf N}\frac{\ebar}{1-e^{-\ebar}}\sup|D|\le C\sup|D|$. The rest
equals $\mathsf N^{-1}$ times the quadrature defect of Steps~2 and~3 with the
single mode $e^{z_\ell t}$ in place of $k_{\bar\omega^2}$, plus
$(\mathsf N^{-1}-1)\int_0^Te^{z_\ell t}f(T-t)\,dt$, which is at most
$5|\mathsf N-1|\sup|f|$; for the quadrature defect,
$\TV(e^{z_\ell\cdot})=|z_\ell|\le2\omega_0$ in Step~2, and
$|\operatorname{Im}z_\ell|=2\Omega$ in Step~3.
\end{proof}

\begin{remark}[What the constants look like]\label{rem:uniform}
Tracking the constants, the proof gives $C_\delta=O(\delta^{-4})$, with
$\omega_0^2(\delta)\asymp\delta^{-2}$ in the Nyquist step: one factor
$\delta^{-1}$ comes from inverting the renewal equation (Step~1), one from
$\sup|f|$, and the rest from $\sup|f'|\lesssim\omega_0^2\sup|\hat C|$ in Step~2
(in Step~3, from $c_{3\delta}\asymp\delta^{-1}$ in Lemma~\ref{lem:quad} and the
constant mode $\pt/(1-\ut)$ of $f$).

Numerically, the ratio $\sup_k|R_k-\Rh(k\ebar)|/(\ebar\sup\Rh)$ stays of
order one across the whole range of $\bar\omega^2$, even for $\ut$ close to $1$, so
the actual constant shows no sign of the factor $\delta^{-4}$. We expect that
at least the factor $\delta^{-1}$ from the inversion of the renewal equation
is an artefact, for the following heuristic reason. The error lives in the
Perron mode, and numerically the discrete Perron rate equals
$\ebar\,r_c(\bar\omega^2,\ut)(1+O(\ebar))$. The accumulated phase error is
therefore at most $C\ebar r\tau e^{-r\tau}\le C\ebar/e$. Making this rigorous
would require tilting both renewal equations by $e^{r\tau}$; we do not
pursue it.
\end{remark}

\begin{remark}[The matched parameters are necessary]\label{rem:matched}
If the matched values $(\bar\omega^2,\ut,\pt,\ebar)$ are replaced by the nominal
values $(\omega^2,H_n,\varphi,\varepsilon)$, the conclusion of
Theorem~\ref{thm:M} fails: numerically the error is then larger than $\ebar$
by several orders of magnitude when $\lambda\asymp1$. There are two reasons. The feedback $\ut$ differs from $H_n$ when $\lambda\asymp1$, because the curvature feedback
enters the mass of the kernel. And $\bar\omega^2$ differs from $\omega^2$ when
$\theta^*\asymp1$, because a per-step angle of order one is aliased on the
grid.
\end{remark}

\begin{remark}[The Nyquist band is genuine]\label{rem:nyquist}
At $\theta^*=\pi/2$, that is, at $\lambda=1+\beta$, both oscillatory free
multipliers equal $e^{z_\pm\ebar}=-\beta$. The rank-one noise then
couples $Z_+$ and $Z_-$ resonantly in the chain, although it does not in the
continuum, and the oscillatory part of $R$ visible on the grid decays at rate
$1+\ut$ instead of $1+\ut/2$.

This is visible numerically. Started from a matched initial condition with
real $Z_0$ at a fixed $\ut>0$, the error $\sup_k|R_k-\Rh(k\ebar)|/\sup\Rh$ at
$\lambda=1+\beta$ does not shrink as $\ebar\to0$, whereas at a fixed distance from
the band it is a bounded multiple of $\ebar$. Figure~\ref{fig:nyquist}(b)
follows the error continuously through the band: it is of order $\ebar$
until the Nyquist margin $\frac\pi2-\theta^*$ is itself of order $\ebar$, and
then it saturates at a value that does not depend on $\ebar$. The band meets
only the curvature-dominated cells~7 and~8 of Section~\ref{sec:dict}.
\end{remark}

\section{Corollaries: the limits in every cell}\label{sec:cor}

This section takes the limits of the continuum system~\eqref{eq:L3} to which
Theorem~\ref{thm:M} reduces every scaling regime, ordered by the value of
$\bar\omega^2$. Corollary~\ref{cor:tikhonov} treats
$\bar\omega^2\to0$, where~\eqref{eq:L3} collapses to a one-dimensional equation on the
learning clock $t=\eta\cA k$; it is the one corollary that is proved
from Lemma~\ref{lem:L2} directly rather than through Theorem~\ref{thm:M}.
Corollary~\ref{cor:regular} treats $\bar\omega^2$ fixed, combining
Theorem~\ref{thm:M} with the Lipschitz dependence of~\eqref{eq:L3} on its
parameters, and Corollary~\ref{cor:resonance} applies it to least squares on
the resonance line, in the co-scaling~\eqref{eq:coscaling-main}, using the exponent
bookkeeping of Section~\ref{sec:dict}. Corollary~\ref{cor:window} treats
$\bar\omega^2\to\infty$, where~\eqref{eq:L3} separates into a slow energy and a fast
oscillation, and Remark~\ref{rem:window} explains why the envelope, rather
than the troughs or the average, measures convergence.

The last two results
vary the other parameters. Corollary~\ref{cor:lift} sends $\ut$ and $\pt$ to
$0$, where, from a rank-one start, the chain is exactly the square of a linear
recursion, and
Corollary~\ref{cor:curv} keeps the per-step curvature $\lambda$ of order one, where
$\bar\omega^2\asymp\ebar^{-2}$ and the window applies with the renormalized feedback
$\ut=H_n/(1-H_c)$. Figure~\ref{fig:regimes} shows the three regimes of $\bar\omega^2$
for one choice of feedback and temperature. Section~\ref{sec:dict} assigns each corollary to its
cells.

\begin{remark}[The limits through the roots]\label{rem:roots}
The corollaries below are one route to the limits; the roots of the cubic $p$
of Lemma~\ref{lem:L3}(c) are another. Since~\eqref{eq:L3} is linear with
constant coefficients, each of its solutions is $R_\infty$ plus a combination
of exponentials $e^{z\tau}$ over the roots $z$ of $p$ (with polynomial factors
at repeated roots), and the coefficients are fixed by the initial condition.
Lemma~\ref{lem:L3}(c) reads stability off $p$ and gives the floor, and each limit of this section can also be derived by
following the roots and their coefficients as $(\bar\omega^2,\ut,\pt)$ converge. As
$\omega^2\to0$, the root nearest $0$ is $-2\omega^2(1-H)(1+O(\omega^2))$
(Lemma~\ref{lem:chi_roots}(d)), which is the rate of
Corollary~\ref{cor:tikhonov} on the retention clock. For large $\omega^2$,
Lemma~\ref{lem:modes} locates all three roots and bounds the coefficients, and
Remark~\ref{rem:window} identifies the real root as
$-(1-\ut)+O(\bar\omega^{-1})$.

The same route applies to the chain itself: the eigenvalues of its per-step
map are $1+\varepsilon z$, with $z$ a root of a cubic that differs from $p$
by terms of order $\varepsilon$ (Lemma~\ref{lem:step_spectrum}). Section~\ref{sec:helps-app} uses this
to obtain decay rates without Theorem~\ref{thm:M}. The roots give stability
and rates most directly. The proofs below go instead through the renewal
equations and the eigenfunctionals of~\eqref{eq:L3}, which also show the shape
of each limit: a one-dimensional equation, the square of a deterministic
oscillator, or a slow energy carrying a fast oscillation.
\end{remark}

\begin{figure}[t]
\centering
\includegraphics[width=\textwidth]{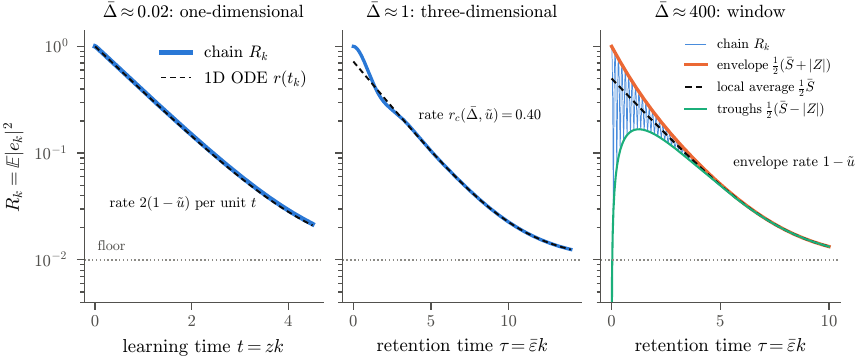}
\caption{The three regimes of $\bar\omega^2$ for the exact chain, with the same
matched feedback $\ut=0.5$ and temperature $\pt=0.005$ (so the floor $\pt/(1-\ut)=0.01$ is
the dotted line) and the cold start $\Sigma_0=\operatorname{diag}(1,0)$; the
vertical axes are logarithmic. Left: $\bar\omega^2\approx0.02$ ($\beta=0.99$), on the
learning clock. The chain is indistinguishable from the one-dimensional
equation of Corollary~\ref{cor:tikhonov}, which decays at rate $2(1-\ut)$ per
unit of learning time. Middle: $\bar\omega^2\approx1$ ($\beta=0.999$), on the retention
clock. After a transient carried by the complex modes of~\eqref{eq:L3}, $R_k$
decays at the Perron rate $r_c(1,0.5)\approx0.40$ of Lemma~\ref{lem:L3}(c)
(dashed). Right: $\bar\omega^2\approx400$ ($\beta=0.999$), the window of
Corollary~\ref{cor:window}. The chain turns by $2\theta^*$ per step and
oscillates between the troughs $\frac12(\bar S-|Z|)$ and the envelope
$\frac12(\bar S+|Z|)$. Since $|Z|$ decays at rate $1+\ut/2$ and
$\bar S-2\pt/(1-\ut)$ at rate $1-\ut$, the envelope merges with the local
average $\frac12\bar S$, and both decay at rate $1-\ut$ per unit of retention
time. The early troughs are deep, but they are transient and say nothing about
convergence: the envelope rate is the notion of convergence used throughout
(Remark~\ref{rem:window}).}
\label{fig:regimes}
\end{figure}

The first corollary concerns the regime in which the learning rate is small
compared with the retention rate, $\omega^2\to0$. In this regime the momentum
equilibrates to the current gradient long before the error moves
appreciably, and one expects plain SGD on the \emph{learning clock}
$t=k\eta\cA$. The statement below makes this precise.

\begin{corollary}[Small $\omega^2$: one-dimensional SGD with floor]\label{cor:tikhonov}
Fix $\delta>0$. There are an absolute constant $\omega_1^2>0$ and a constant
$C_\delta$ such that the following holds for every $\varepsilon\in(0,1]$,
every $\omega^2=\lambda/\varepsilon^2\le\omega_1^2$, every $\ut\le1-\delta$, and the
cold start $\Sigma_0=\operatorname{diag}(R_0,0)$. Let $t_k:=k\eta\cA$
be the learning clock, and let $r$ solve
\[
r'=-2(1-\ut)\,r+2\pt,\qquad r(0)=R_0 .
\]
Then, with $R_\infty=\pt/(1-\ut)$,
\[
\sup_k|R_k-r(t_k)|\le C_\delta\,\omega^2\,(R_0+R_\infty).
\]
\end{corollary}

\noindent The hypothesis is only that $\omega^2$ is small; $\varepsilon$ need
not tend to zero. In particular the corollary covers the case $\gamma=0$ of a
fixed momentum parameter $\beta\in[0,1)$, plain SGD included, where
Theorem~\ref{thm:M}, whose error is
$O(\ebar)$, gives nothing.

On the learning clock the momentum is invisible: heavy-ball SGD with EMA
momentum behaves as plain SGD with learning rate $\eta$, whose error relaxes
at rate $2\eta\cA(1-\ut)$ per step toward the floor $\pt/(1-\ut)$
(Figure~\ref{fig:regimes}, left panel). The memory of the momentum enters
only through the curvature feedback $H_c$ in the renormalized feedback $\ut$ and
temperature $\pt$, and $H_c=\eta\cA\varepsilon/(2(1+\beta))$ is itself small. In this regime,
therefore, the learning rate alone sets both the rate and the noise floor.

\begin{proof}
When $\omega^2$ is small, the eigenvalues of $F$ separate into a slow
eigenvalue close to $1-\eta\cA$ and a fast eigenvalue close to $\beta$. We compare
the discrete renewal equation of Lemma~\ref{lem:L2} with the renewal
equation whose kernel is the geometric law of ratio $e^{-2\eta\cA}$, which can be
solved exactly and is an Euler scheme for $r$. The fast eigenvalue
contributes only $O(\omega^2)$ to the kernel and to the forcing. The proof has
five steps.
\begin{enumerate}
\item \emph{Eigenvalues.} For $\omega^2\le\omega_1^2\le\frac18$ the eigenvalues
of $F$ are real and given by
\[
z_\pm=\tfrac12\Bigl[(1+\beta-\lambda)\pm\varepsilon\sqrt{1-2\omega^2(1+\beta)+\omega^4\varepsilon^2}\Bigr].
\]
They are separated, $z_+-z_-\ge\varepsilon/2$, and they satisfy
$z_+=1-\eta\cA(1+O(\omega^2))$ and $z_-=\beta+O(\eta\cA)$.
\item \emph{Kernel.} By Lemma~\ref{lem:L2}, the normalized kernel
$h_n:=K_{n-1}/\sum_mK_m$ is proportional to $(z_+^n-z_-^n)^2$.
Expanding the square writes $h$ as a signed combination of normalized
geometric laws with ratios $z_+^2$, $z_+z_-$ and
$z_-^2$, in which the last two carry weights $O(\eta\cA/\varepsilon)=O(\omega^2)$
relative to the first. Indeed, the weights are proportional to
$z_+^2/(1-z_+^2)$, $-2z_+z_-/(1-z_+z_-)$ and
$z_-^2/(1-z_-^2)$, and $1-z_+^2\le3\eta\cA$, while
$1-z_+z_-=1-\det F=\varepsilon$ and
$1-z_-^2\ge\varepsilon/2$, since $z_+z_-=\beta\ge0$ and
$z_+>0$ give $0\le z_-\le\beta+O(\eta\cA)$ for every $\beta\in[0,1)$.
Let $\rho:=e^{-2\eta\cA}$ and $g_n:=(1-\rho)\rho^{n-1}$. Using
\[
\|\mathrm{Geom}(q_1)-\mathrm{Geom}(q_2)\|_{\ell^1}\le\frac{2|q_1-q_2|}{1-\max_iq_i},
\]
we obtain $\|h-g\|_{\ell^1}\le C\omega^2$.
\item \emph{Forcing.} For the cold start,
$a_k=R_0(c_+z_+^k+c_-z_-^k)^2$, where
$c_+=(z_+-\beta)/(z_+-z_-)=1+O(\omega^2)$ and $c_-=1-c_+$.
Comparing $z_+^{2k}$ with $\rho^k$ by means of
$\sup_{x\ge0}xe^{-x}<\infty$ gives $|a_k-R_0\rho^k|\le C\omega^2 R_0$.
\item \emph{Geometric renewal.} The renewal equation with kernel $g$ and
forcing $R_0\rho^k$ is solved by the recursion
\[
x_{k+1}=\rho x_k+(1-\rho)(\ut x_k+\pt),\qquad x_0=R_0,
\]
so that $x_k-R_\infty=\bigl(1-(1-\rho)(1-\ut)\bigr)^k(R_0-R_\infty)$.
To compare with $r(t_k)-R_\infty=e^{-2(1-\ut)k\eta\cA}(R_0-R_\infty)$, put
$s:=1-\ut$, $q:=1-s(1-\rho)$ and $a:=2s\eta\cA$. Then $|q-e^{-a}|\le6s(\eta\cA)^2$ and
$\max(q,e^{-a})\le e^{-s\eta\cA}$ for $\eta\cA\le\frac12$, so
$|q^k-e^{-ak}|\le6s(\eta\cA)^2ke^{-s(k-1)\eta\cA}\le4\eta\cA$; hence
$|x_k-r(t_k)|\le4\eta\cA|R_0-R_\infty|$, with an absolute constant.
\item \emph{Comparison.} The difference $D:=R-x$ satisfies
\[
D_k=(a_k-R_0\rho^k)+\ut\sum_{n=1}^kh_nD_{k-n}+\sum_{n=1}^k(h_n-g_n)(\ut x_{k-n}+\pt).
\]
Inverting the renewal operator as in Step~1 of the proof of
Theorem~\ref{thm:M}, and using Steps~2 and~3, we find
$\sup|D|\le C\omega^2(R_0+\sup x+\pt)/\delta$. Together with Step~4 and
$\eta\cA=\omega^2\varepsilon\le\omega^2$, this proves the corollary.
\end{enumerate}
\end{proof}

The same reduction can be seen inside the continuum system \eqref{eq:L3}.
When $\omega^2$ is small, the variables $C$ and $V$ relax on the retention
clock, while $R$ moves only on the slower clock $\omega^2\tau$. Setting
$C'=V'=0$ gives the quasi-steady values
\[
\hat C=(1-H)R-\varphi+O(\omega^2),\qquad \omega^2\hat V=HR+\varphi+O(\omega^2),
\]
where the errors are relative to the size of $R$ and $\varphi$, and these
values are attained after an initial layer of length $O(\log(1/\omega^2))$.
This is an instance of Tikhonov's theorem
\citep[Theorem~11.1]{Khalil2002}, which here can be checked directly.

Put $C_*:=((1-H)R-\varphi)/(1+\omega^2)$, $W:=\omega^2 V$ and
$W_*:=\omega^2 C_*+HR+\varphi$. The deviation $e:=(C-C_*,W-W_*)$ solves
\[
e'=\begin{pmatrix}-1&-1\\2\omega^2&-2\end{pmatrix}e-(C_*',W_*'),
\]
the matrix has eigenvalues with real part at most $-1$, and
$|C_*'|+|W_*'|\le C\omega^2|C|$ because $R'=-2\omega^2 C$. Hence
$|e(\tau)|\le Ce^{-\tau}|e(0)|+C\omega^2\sup|C|$, which is $O(\omega^2)$ once
$\tau\ge\log(1/\omega^2)$. Substituting $\hat C$ into $R'=-2\omega^2 C$ gives
$R'=-2\omega^2\bigl((1-H)R-\varphi\bigr)$, which is the equation for $r$ on the
clock $\omega^2\tau$.

When $\bar\omega^2$ stays bounded away from $0$ and $\infty$, no further reduction
occurs and the limit is \eqref{eq:L3} itself: Theorem~\ref{thm:M} combines with
the continuity of \eqref{eq:L3} in its parameters to give a genuinely
three-dimensional limit.

\begin{corollary}[Fixed $\omega^2$: the three-dimensional limit]\label{cor:regular}
Fix $\delta>0$ and a compact set $\mathcal K\subset(0,\infty)$. There are
$\ebar_0,C>0$ such that the following holds whenever $\ebar\le\ebar_0$,
$\bar\omega^2,\omega^{*2}\in\mathcal K$ and $\ut,H^*\le1-\delta$. Let $R^*$ be the
solution of \eqref{eq:L3} at the parameters $(\omega^{*2},H^*,\varphi^*)$ from
the initial condition $\Sigma_0^*$. Then
\[
\sup_k|R_k-R^*(k\ebar)|\le C\Bigl[\bigl(\ebar+|\bar\omega^2-\omega^{*2}|+|\ut-H^*|\bigr)(N_0+\pt)+|\pt-\varphi^*|+\|\Sh_0-\Sigma_0^*\|\Bigr].
\]
\end{corollary}

\noindent In words, along any sequence of chains whose matched parameters and
matched initial conditions converge, the error $R_k$ converges, uniformly in
time, to the continuum system at the limiting parameters, and the rate of
convergence is $\ebar$ plus the rate at which the parameters converge. This
is the form in which Theorem~\ref{thm:M} is used along co-scaling sequences.

Uniformity in time is possible because \eqref{eq:L3} is exponentially stable,
uniformly over the compact parameter set, so that perturbations of the
parameters do not accumulate.

\begin{proof}
Theorem~\ref{thm:M} reduces the claim to the Lipschitz dependence of
\eqref{eq:L3} on its parameters, uniformly in time, which we derive from the
renewal form of Lemma~\ref{lem:L3}(a).
Take $\ebar_0\le\log2$, so that $\beta\ge\frac12$.
For $\ebar$ small, the angle
$\theta^*=\ebar\sqrt{\bar\omega^2-\frac14}$ is small on $\mathcal K$, so the Nyquist
margin holds automatically and Theorem~\ref{thm:M} applies.

The matrix $M_2(\omega^2)$ is Hurwitz, uniformly for $\omega^2\in\mathcal K$, and
it is an entire function of $\omega^2$. Hence
\[
\|k_{\omega_1^2}-k_{\omega_2^2}\|_{L^1}\le C_{\mathcal K}|\omega_1^2-\omega_2^2|
\qquad\text{and}\qquad
\sup_\tau|\hat a_1-\hat a_2|\le C_{\mathcal K}\bigl(|\omega_1^2-\omega_2^2|N_0+\|\Sigma_{0,1}-\Sigma_{0,2}\|\bigr),
\]
where the index $1$ refers to the limit parameters $(\omega^{*2},H^*,\varphi^*)$
and initial condition $\Sigma_0^*$, and the index $2$ to the matched
parameters $(\bar\omega^2,\ut,\pt)$ and initial condition $\Sh_0$ of the chain, so
that $R_2=\Rh$.

The difference $D=R_1-R_2$ of the two solutions satisfies the renewal
equation
\[
D=(\hat a_1-\hat a_2)+H_1k_1*D+(H_1k_1-H_2k_2)*R_2+(\varphi_1k_1-\varphi_2k_2)*1 .
\]
Inverting the renewal operator $D\mapsto D-H_1k_1*D$ costs a factor
$(1-H_1)^{-1}\le\delta^{-1}$, as in Step~1 of the proof of
Theorem~\ref{thm:M}, and each of the remaining terms is bounded by the
displayed estimates together with the a priori bound
$\sup R_2=\sup\Rh\le(N_0+\pt)/\delta$ (Step~1 of the proof of
Theorem~\ref{thm:M}) and $\|\Sh_0\|\le C_{\mathcal K}N_0$.
\end{proof}

The first boundary of the phase diagram is the line on which all three
matched parameters of LS converge to positive constants. We use the
co-scaling~\eqref{eq:coscaling-main} of the main text, in which the sparsity,
batch size, retention rate and learning rate are powers of the dimension.
A batch size is an integer, so $B=B_*d^\sigma$ stands for any integer $B$ with
$B/(B_*d^\sigma)\to1$, such as $\lceil B_*d^\sigma\rceil$; when $\sigma=0$, $B_*$ is the
batch size itself. This changes no limiting constant in this section or the next, and the
rounding enters only the rate of Corollary~\ref{cor:resonance}.

\begin{corollary}[Boundary 1: the resonance line]\label{cor:resonance}
Take LS in the co-scaling~\eqref{eq:coscaling-main}, with $\varsigma$ fixed and
with an integer batch size $B$ satisfying $B/(B_*d^\sigma)\to1$, on
the line $\alpha=1-\sigma=\gamma-\kappa$ with
$\gamma>0$ (that is, $\gamma=1-\sigma+\kappa$). Then
\[
\bar\omega^2\to\omega^{*2}=\frac{\eta_*p_*}{\varepsilon_*},\qquad
\ut\to H^*=\frac{\eta_*}{2B_*},\qquad
\pt\to\varphi^*=\varsigma^2H^*,\qquad \lambda\to0 .
\]
Suppose $H^*<1$, that is, $\eta_*<2B_*$, and let $R^*$ solve \eqref{eq:L3}
at $(\omega^{*2},H^*,\varphi^*)$ from the cold start with $R_0>0$ or $\varsigma>0$.
\begin{enumerate}
\item[(i)] The chain started cold satisfies
$\sup_k|R_k-R^*(k\ebar)|=O\bigl(d^{-\min(\gamma,1)}+|B_*d^\sigma/B-1|\bigr)$.
\item[(ii)] $Q^*=R^*V^*-(C^*)^2>0$ for $\tau>0$.
\item[(iii)] As $\omega^{*2}\to0$, on the clock $\omega^{*2}\tau$, the limit
reduces to the equation of Corollary~\ref{cor:tikhonov}, and as
$\omega^{*2}\to\infty$ it reduces to the two-scale form of
Corollary~\ref{cor:window}.
\end{enumerate}
\end{corollary}

\noindent The last term in (i) is the relative rounding error of the batch
size. It vanishes when $B_*d^\sigma$ is an integer, and it is $O(d^{-\sigma})$
for $B=\lfloor B_*d^\sigma\rfloor$ with $\sigma>0$.

The rounding term cannot be dropped. Write $\rho:=B_*d^\sigma/B$. When
$\varsigma>0$, the floor $\pt/(1-\ut)$ of the chain and the floor
$\varphi^*/(1-H^*)$ of the limit differ by
$\varphi^*(\rho-1)/(1-H^*)^2$ to leading order. For $0<\sigma<1$ and
$B=\lfloor B_*d^\sigma\rfloor$, the fractional part of $B_*d^\sigma$ exceeds
$\frac12$ for infinitely many $d$, and for these $d$ one has
$\rho-1\ge cd^{-\sigma}$. Hence, for $\varsigma>0$ and
$\sigma<\min(\gamma,1)$, the error in (i) is of exact order $d^{-\sigma}$
along these $d$, and not $O(d^{-\min(\gamma,1)})$.

Part (ii) says that the limit is \emph{irreducibly
three-dimensional}: there is no square-root lift. Part (iii) says that it
contains both neighbouring cells, cell~2 and cell~6 of
Section~\ref{sec:dict}.

On the resonance line the learning rate and the retention rate are of the
same order ($\omega^{*2}$ is a positive constant), and the noise feedback $H^*$ and
the ambient temperature $\varphi^*$ are also of order one. The limit therefore
keeps the whole continuum system, with no reduction of dimension.

In the heavy-ball cell~3, where the feedback and the temperature vanish, the covariance of $(X,Y)$
started from a cold start stays of rank one and $R=X^2$ is the square of a
damped oscillator (Corollary~\ref{cor:lift}). On the resonance line the noise
makes the covariance of full rank at once (Lemma~\ref{lem:L3}(e)), and $R$ is
no longer the square of anything.

The line is a genuine boundary because the limit object changes across it:
below the line ($\omega^{*2}\to0$) it is the one-dimensional equation of
Corollary~\ref{cor:tikhonov}, and above it ($\omega^{*2}\to\infty$) it is the
two-scale form of Corollary~\ref{cor:window}. On the line itself the
envelope rate per step is $\ebar\,r_c(\omega^{*2},H^*)$. The middle panel of
Figure~\ref{fig:regimes} shows such a decay, at $\bar\omega^2\approx1$ and
$\ut=0.5$.

\begin{proof}
On the line the exponents of Section~\ref{sec:dict} are $e_{\omega^2}=e_H=0$ and
$e_\lambda=-2\gamma$. The parameter errors are as follows. First,
$|\bar\omega^2-\omega^{*2}|=O(\varepsilon)$.
Second, since $\eta=\eta_*d^{\sigma-1}$, we have
$H_n=H^*\rho\,(d+2-p)/d$ and $\varphi=\varphi^*\rho$, with $\rho=B_*d^\sigma/B$.
Hence $|\ut-H^*|$ and $|\pt-\varphi^*|$ are $O(d^{-1}+\lambda+|\rho-1|)$, where
$d^{-1}$ comes from replacing $d+2-p$ by $d$, $\lambda$ from the curvature feedback
$H_c=O(\lambda)$, and $|\rho-1|$ from the rounding of the batch size.
Third, for the cold start
$\Sigma_0=\operatorname{diag}(R_0,0)=\Sigma_0^*$ we have
$\|\Sh_0-\Sigma_0^*\|=O(\varepsilon)$, because $P^{-1}e_1=e_1+O(\varepsilon)$.
Indeed $P^{-1}e_1=(1,-\pi_{21}/\pi_{22})^\top$, and by the formulas in the proof
of Lemma~\ref{lem:match}, $\pi_{21}/\pi_{22}=-(1-\lambda-G_{11})/G_{12}$ with
$G_{11}=1-\frac12\bar\omega^2\ebar^2+O(\ebar^3)$ and $G_{12}=-\bar\omega^2\ebar(1+O(\ebar))$;
since $\bar\omega^2\ebar^2=\lambda(1+O(\varepsilon))$, the numerator is $O(\varepsilon^2)$
and $\pi_{21}/\pi_{22}=O(\varepsilon)$.
Since $\ebar\asymp\varepsilon\asymp d^{-\gamma}$ and $\lambda\asymp d^{-2\gamma}$,
Corollary~\ref{cor:regular} gives the error bound
$O(d^{-\gamma}+d^{-1}+|\rho-1|)$, which is (i).

Part (ii) follows from Lemma~\ref{lem:L3}(e), since $H^*>0$ and
$R^*\not\equiv0$ when $R_0>0$ or $\varsigma>0$.

For (iii), note that on the clock $t=\omega^{*2}\tau$ the kernel
$k_{\omega^{*2}}(t/\omega^{*2})/\omega^{*2}$ has Laplace transform
$\tilde k_{\omega^{*2}}(\omega^{*2}s)=4/\bigl((1+\omega^{*2}s)(\omega^{*2}s^2+2s+4)\bigr)\to2/(s+2)$,
which is the transform of $2e^{-2t}$, and the renewal equation with kernel
$2e^{-2t}$ is equivalent to $r'=-2(1-H^*)r+2\varphi^*$. The reduction
as $\omega^{*2}\to\infty$ is the averaging step in the proof of
Corollary~\ref{cor:window}.
\end{proof}

At the other extreme, $\bar\omega^2\to\infty$, the learning rate is large compared
with the retention rate, and the error oscillates many times within one
retention time. The error $R_k$ itself then has no limit, but the slow energy
$S$ and the amplitude $|Z|$ of the oscillation $Z$ of Lemma~\ref{lem:L3}(d) do, and part (ii) of
Theorem~\ref{thm:M} is designed for exactly this situation.

\begin{corollary}[Large $\omega^2$: the long-oscillation window]\label{cor:window}
Fix $\delta$, and assume the hypotheses of Theorem~\ref{thm:M} with
$\bar\omega^2\ge\omega_0^2(\delta)$. Let $S_k$ and $Z_k$ be the eigenfunctionals of
$\Sh_k$, let $\tau_k=k\ebar$, and set
\[
\mathcal E:=C_\delta\bigl(\ebar+\bar\omega^{-1}\bigr)(N_0+\pt).
\]
Let $\bar S$ solve $\bar S'=-(1-\ut)\bar S+2\pt$ with $\bar S(0)=S_0$. Then the
following hold.
\begin{enumerate}
\item[(a)] \emph{Slow energy.} $|S_k-\bar S(\tau_k)|\le\mathcal E$.
\item[(b)] \emph{Oscillation.}
$\bigl|Z_k-Z_0\,e^{-(1+\ut/2)\tau_k}e^{2ik\theta^*}\bigr|\le\mathcal E$.
\item[(c)] \emph{Two-scale form.}
\[
R_k=\tfrac12\bar S(\tau_k)+\tfrac12\operatorname{Re}\bigl(Z_0e^{-(1+\ut/2)\tau_k}e^{2ik\theta^*}\bigr)+O(\mathcal E).
\]
\item[(d)] \emph{Local averages.} For every $k_0$ and $N\ge1$,
\[
\Bigl|\frac1N\sum_{k=k_0}^{k_0+N-1}R_k-\tfrac12\bar S(\tau_{k_0})\Bigr|\le\mathcal E+C_\delta(N_0+\pt)\bigl(N\ebar+(N\theta^*)^{-1}\bigr).
\]
\end{enumerate}
\end{corollary}

\noindent By part (c), $R_k$ in general has no pointwise limit along a
scaling sequence unless $Z_0\to0$.

The error separates into two components that live on different time
scales. Half the slow energy, $\frac12\bar S$, obeys a one-dimensional
linear equation on the retention clock, with rate $1-\ut$ and floor
$\pt/(1-\ut)$. It is the analogue of Corollary~\ref{cor:tikhonov}, except
that the clock is now set by the retention rate rather than by the learning
rate. On top of it sits an oscillation that turns by the angle $2\theta^*$
per step and whose amplitude decays at rate $1+\ut/2$.

The noise feedback acts on the two components with opposite signs: it slows the
slow energy by $\ut$, and it speeds up each of the two oscillatory modes by
$\ut/2$. This is forced by the structure of \eqref{eq:L3}, whose generator
has trace $-3$ for every value of $H$, so that the noise can only
redistribute decay among the modes.

\begin{proof}
In the eigenfunctional coordinates, the generator of \eqref{eq:L3} is a
diagonal matrix plus a perturbation $E$ of bounded size, whose off-diagonal
entries couple modes with frequency gaps of at least $2\Omega$. One step of
averaging, that is, a near-identity change of variables, removes the
off-diagonal part to leading order and leaves the diagonal flow with real
corrections. The remainder is $O(1/\Omega)=O(\bar\omega^{-1})$, and it is the
source of that term in $\mathcal E$.

By Theorem~\ref{thm:M}(ii), it suffices to prove the statements for
\eqref{eq:L3} at $(\bar\omega^2,\ut,\pt)$, at a cost of $C_\delta\ebar(N_0+\pt)$. In
the continuum, let $x:=(S,Z,\bar Z)-x_\infty$. Then $x'=(D+E)x$, with $D$ and
$E$ as in the proof of Lemma~\ref{lem:modes}.

\emph{Averaging.} For $i\ne j$ let $K_{ij}:=-E_{ij}/(D_{ii}-D_{jj})$, so that
$|K|\le C/\Omega$, and set $y:=(I+K)^{-1}x$. Since $[D,K]=-E_{\mathrm{off}}$,
\[
(I+K)^{-1}(D+E)(I+K)=D+E_{\mathrm{diag}}+O(\|E\|\|K\|),
\]
and there is no term of order $\|D\|\|K\|^2$. Here
\[
E_{\mathrm{diag}}=\ut\operatorname{diag}\Bigl(\frac{\bar\omega^2}{\Omega^2},-\frac{\bar\omega^2}{2\Omega^2},-\frac{\bar\omega^2}{2\Omega^2}\Bigr)
\]
is real, and all the diagonal decay rates are at least $\delta/2$. Hence $y$
stays within $C(\delta\Omega)^{-1}\sup|y|$ of the diagonal flow.

\emph{Clean-up.} Replacing $\bar\omega^2/\Omega^2$ by $1$ changes the rates by
$O(\bar\omega^{-2})$, and hence the solutions by $O((\delta\bar\omega^2)^{-1})$. The fixed
point $S_\infty=2\pt/(1-\ut)$ is that of $\bar S$, and
$|Z_\infty|=O(\pt/(\delta\Omega))$. In discrete time,
$e^{z_+\tau_k}=\beta^ke^{2ik\theta^*}$.

\emph{Parts (c) and (d).} Part (c) follows from
$R=r_SS+2\operatorname{Re}(r_ZZ)$, with $r_S=\frac12+O(\bar\omega^{-2})$ and
$r_Z=\frac14+O(\bar\omega^{-1})$. For (d), the slow energy changes by
$O(N\ebar)$ over the window of averaging, and the oscillatory part sums to
at most $2/|1-q|$, where $q=e^{-(1+\ut/2)\ebar+2i\theta^*}$ is its multiplier.
Finally,
$|1-q|\ge|1-e^{2i\theta^*}|-(1-e^{-(1+\ut/2)\ebar})\ge2\sin\theta^*-\frac32\ebar\ge\frac12\theta^*$,
because $\sin\theta^*\ge\frac2\pi\theta^*$ for $\theta^*\le\frac\pi2$ and
$\theta^*=\Omega\ebar\ge2\ebar$ when $\bar\omega^2\ge\omega_0^2$.
\end{proof}

The two-scale form raises a natural question about what convergence means in
the window, since $R_k$ passes through deep troughs on every oscillation. The
next remark answers it.

\begin{remark}[What the window means for convergence]\label{rem:window}
By part (c), the envelope of $R$ over one oscillation is
$\frac12(\bar S+|Z|)+O(\mathcal E)$. Since $|Z|$ decays at least as fast as
$\bar S$ ($1+\ut/2\ge1-\ut$, with equality only when $\ut=0$), the envelope
decays at the rate of $\bar S$. When $\ut>0$ the asymptotic envelope is
$\frac12\bar S$, which is also the local average of part (d); when $\ut=0$
(cells~5 and~7) the two decay at the same rate
(Remark~\ref{rem:window-equal}). Averaging over periods therefore gains
nothing asymptotically, and the troughs of the oscillation are not a
mechanism of convergence. The convergence rate is the envelope rate,
$\ebar(1-\ut)$ per step. The right panel of Figure~\ref{fig:regimes} shows the
envelope, the troughs and the local average for the exact chain at
$\bar\omega^2=400$.

By Lemma~\ref{lem:L3}(c), the continuum rate satisfies
$r_c(\bar\omega^2,\ut)\le1-\ut$ for every $\bar\omega^2$ (strictly if $\ut>0$). Hence, to
the accuracy of Theorem~\ref{thm:M}, $\ebar(1-\ut)$ is a cap on the per-step
rate at a given value of $\ut$, and the window attains it, since
$z_0=-(1-\ut)+O(\bar\omega^{-1})$ (apply Gershgorin's theorem to the normal
form $D+E_{\mathrm{diag}}+O(\|E\|\|K\|)$ in the proof of
Corollary~\ref{cor:window}). But $\ut$ increases with $\eta$, so pushing $\eta$
into the window only lowers $1-\ut$ and raises the floor $\pt/(1-\ut)$. The
window is a phase of the dynamics, not a fast regime.
\end{remark}

\begin{remark}[Equal decay rates when $\ut=0$]\label{rem:window-equal}
When $\ut=0$, which is the limit in cells~5 and~7 of
Section~\ref{sec:dict}, the oscillation and the slow energy decay at the
same rate, $1+\ut/2=1-\ut=1$. The comparison of rates in
Remark~\ref{rem:window} then holds with equality, and the excess of the
envelope of $R$ over its limit $\frac12S_\infty$, where $S_\infty:=2\pt$, is
not asymptotic to $\frac12(\bar S-S_\infty)$ unless $Z_0=0$. The conclusion
of Remark~\ref{rem:window} is unchanged: to the accuracy of
Corollary~\ref{cor:window}, this excess decays at rate at least $1=1-\ut$ on
the clock $\tau$, and exactly $1$ when $S_0>S_\infty$ (in cells~5 and~7 also
$\pt\to0$, so this holds from any nonzero start). That is, the envelope rate
is $\ebar$ per step, and the troughs are not a mechanism of convergence.
\end{remark}

\begin{proof}
At $\ut=0$, $S_\infty$ is the fixed point of $\bar S$, so that
$\bar S(\tau)-S_\infty=(S_0-S_\infty)e^{-\tau}$. By
Corollary~\ref{cor:window}(c),
\[
R_k-\tfrac12S_\infty=\tfrac12e^{-\tau_k}\Bigl(S_0-S_\infty+\operatorname{Re}\bigl(Z_0e^{2ik\theta^*}\bigr)\Bigr)+O(\mathcal E),
\]
so over any stretch of steps the maximum of $R_k-\frac12S_\infty$ is at most
$\frac12e^{-\tau}\bigl(|S_0-S_\infty|+|Z_0|\bigr)+O(\mathcal E)$, with $\tau$
the start of the stretch. Over $N':=\lceil\pi/\theta^*\rceil+1$ consecutive
steps the angles $2k\theta^*+\arg Z_0$ advance by $2\theta^*<\pi$ and cover a
full turn, so one of them lies within $\theta^*$ of $0$ modulo $2\pi$; since
$\theta^*=\Omega\ebar$, we have
$N'\ebar\le(\pi/\theta^*+2)\ebar=(\pi+2\theta^*)/\Omega\to0$ in the window,
and the maximum is at least
$\frac12e^{-\tau}\bigl(S_0-S_\infty+|Z_0|\cos\theta^*\bigr)+o(e^{-\tau})+O(\mathcal E)$,
with $\cos\theta^*>0$ away from the Nyquist band. This is not
$\frac12(\bar S(\tau)-S_\infty)+o(e^{-\tau})$ when $Z_0\ne0$. By
Corollary~\ref{cor:window}(d), the maximum over $N$ consecutive steps is at
least their average, which is $\frac12\bar S$ up to
$\mathcal E+C_\delta(N_0+\pt)(N\ebar+(N\theta^*)^{-1})$; in the window
$\theta^*/\ebar\to\infty$, and $N$ can be chosen with $N\ebar\to0$ and
$N\theta^*\to\infty$. The upper bound decays at rate $1$ in $\tau$, and so
does the lower bound $\frac12(\bar S-S_\infty)$ when $S_0>S_\infty$.
\end{proof}

The next two results vary the feedback, the temperature and the curvature rather than $\bar\omega^2$.
When the feedback and the temperature both vanish, the noise disappears from \eqref{eq:L1}, and the
second-moment dynamics is the square of a deterministic linear dynamics.
The next corollary records this exact statement and quantifies how much the
noise can break it.

\begin{corollary}[Square-root lift and its defect]\label{cor:lift}
The following hold.
\begin{enumerate}
\item[(i)] If $H_n=\varphi=0$, then $\Sigma_k=F^k\Sigma_0F^{k\top}$ and
$\det\Sigma_k=\beta^{2k}\det\Sigma_0$.
A rank-one start $\Sigma_0=x_0x_0^\top$
stays of rank at most one (of rank one if $\beta>0$), and
$R_k=(e_1^\top F^kx_0)^2$ exactly. If $\beta>0$, then in matched
coordinates $R_k=X(k\ebar)^2$, where
\[
X''+X'+\bar\omega^2 X=0 .
\]
\item[(ii)] In general, let $a_k$ be the noise-free solution from the same
$\Sigma_0$, and suppose $\ut<1$. Then
\[
0\le R_k-a_k\le\frac{\ut\sup_ja_j+\pt}{1-\ut}.
\]
\end{enumerate}
\end{corollary}

\noindent The cold start $m_0=0$ is a rank-one start. Part (i) says that
without noise the error of heavy-ball SGD is literally the square of a
damped oscillator on the retention clock, with frequency set by $\bar\omega^2$; the
identity $\det\Sigma_k=\beta^{2k}\det\Sigma_0$ is the discrete counterpart of
$Q'=-2Q$ in Lemma~\ref{lem:L3}(e).

Part (ii) measures how the noise breaks this picture. The defect is linear
in the feedback $\ut$ and the temperature $\pt$, it is uniform in time, and it holds for every
$\ut<1$, with no smallness of $\ebar$ required. In cells~3 and~5 of
Section~\ref{sec:dict}, where $\ut,\pt\to0$, it therefore gives, from a
rank-one start, the square-root lift $R_k\approx X(k\ebar)^2$ at every stable
learning rate $\eta_*$, with an explicit defect and with no further
hypothesis. This is in sharp contrast with the resonance line, where the feedback and the temperature stay positive and the lift fails (Corollary~\ref{cor:resonance}).

\begin{proof}
For (i), $\det(F\Sigma F^\top)=(\det F)^2\det\Sigma=\beta^2\det\Sigma$.
Since $F x_0x_0^\top F^\top=(Fx_0)(Fx_0)^\top$, a rank-one matrix stays of
rank at most one, and of rank one when $\beta>0$, since $F$ is then
invertible. For $\beta>0$ the matched form
follows from Lemma~\ref{lem:embed}: with $H_n=\varphi=0$ the chain is the free
continuum flow sampled on the grid, the free flow of a rank-one initial
condition is the second moment of the deterministic pair $(X,Y)$ solving
$X'=-\bar\omega^2 Y$ and $Y'=X-Y$, and eliminating $Y$ gives the damped oscillator.

For (ii), the renewal equation of Lemma~\ref{lem:L2} gives
$R_k-a_k=\sum_{j<k}K_{k-1-j}(\vinc R_j+\vadd)\ge0$, and by the mass
identities of that lemma this is at most $\ut\sup R+\pt$. By
Corollary~\ref{cor:stab}(ii),
\[
\ut\sup R+\pt\le\frac{\ut(\sup a+\pt)}{1-\ut}+\pt=\frac{\ut\sup a+\pt}{1-\ut}.
\]
\end{proof}

The lower bound $R_k\ge a_k$ in part (ii) also gives an exact form of the cap
of Remark~\ref{rem:window}: no choice of the learning rate, the batch or the
sparsity makes the chain decay faster than $\beta^k$, the decay of the
momentum buffer itself.

\begin{lemma}[The retention cap is exact]\label{lem:retention-cap}\label{rem:retention-cap}
Let $\lambda>0$ and $H_n\ge0$, and let
$\mathcal L(\Sigma):=F\Sigma F^\top+2\lambda\varepsilon H_n\Sigma_{11}\,bb^\top$ be
the linear part of~\eqref{eq:L1}.
Its spectral radius satisfies $\rho(\mathcal L)\ge\rho(F)^2\ge\beta$,
with $\rho(F)^2=\beta$ exactly when $(1+\beta-\lambda)^2\le4\beta$, that is, for
$\beta>0$, when $|\cos\theta|\le1$.
Hence the asymptotic decay rate per step of the linear part, $-\log\rho(\mathcal L)$,
is at most
$-\log\beta$ (a vacuous bound at $\beta=0$),
with no approximation and for every value of $(\lambda,H_n)$.
\end{lemma}

\begin{proof}
Let $x_k:=e_1^\top F^ke_1$. Since
$I-zF=\bigl(\begin{smallmatrix}1-z(1-\lambda)&z\beta\\-z\lambda&1-z\beta\end{smallmatrix}\bigr)$,
\[
\sum_{k\ge0}x_kz^k=\bigl((I-zF)^{-1}\bigr)_{11}=\frac{1-\beta z}{1-(1+\beta-\lambda)z+\beta z^2}.
\]
The poles are the reciprocals of the nonzero eigenvalues of $F$. If
$\beta>0$, the numerator vanishes only at $z=1/\beta$, which is not a pole,
because the characteristic polynomial of $F$ takes the value $\beta \lambda\ne0$
at $\beta$; if $\beta=0$, the numerator is $1$. So the radius of
convergence is $1/\rho(F)$, infinite when $\rho(F)=0$, and
$\limsup_k|x_k|^{1/k}=\rho(F)$ by the Cauchy--Hadamard formula. From the cold
start, $a_k=R_0x_k^2$, so $\limsup_ka_k^{1/k}=\rho(F)^2$. The eigenvalues
of $F$ have product $\det F=\beta$, so $\rho(F)^2\ge\beta$, with equality
exactly when they have equal moduli, that is, when they are non-real or
equal, which is the case $(1+\beta-\lambda)^2\le4\beta$.

With $\varphi=0$, unrolling~\eqref{eq:L1} as in the proof of
Lemma~\ref{lem:L2} writes $R_k$ as $a_k$ plus a sum of nonnegative terms,
so $R_k\ge a_k$, and this step uses no hypothesis on $\lambda$ or on $\ut$.
The same bound therefore applies to $\mathcal L^k(\Sigma_0)_{11}$, the risk
of the chain with $\varphi=0$. Fix any norm on symmetric $2\times2$ matrices; since all
norms on a finite-dimensional space are equivalent, there is $c>0$ with
$|M_{11}|\le\|M\|/c$ for every $M$. Then
$\|\mathcal L^k\|\ge c\,a_k/\|\Sigma_0\|$ in the induced operator norm, the
constant $c$ disappears under the $k$-th root, and Gelfand's formula
$\rho(\mathcal L)=\lim_k\|\mathcal L^k\|^{1/k}$
\citep[Corollary~5.6.14]{HornJohnson2013} gives
$\rho(\mathcal L)\ge\limsup_ka_k^{1/k}=\rho(F)^2$.
\end{proof}

\begin{remark}[The cap along the cold start]\label{rem:retention-cold}
The same bound holds for the risk itself from the cold start. If $\varphi=0$
and $\Sigma_0=\operatorname{diag}(R_0,0)$ with $R_0>0$, the proof of
Lemma~\ref{lem:retention-cap} gives $R_k\ge a_k$ and
$\limsup_ka_k^{1/k}=\rho(F)^2$, so
\[
\limsup_{k\to\infty}R_k^{1/k}\ \ge\ \rho(F)^2\ \ge\ \beta ,
\]
and the upper envelope of $R_k$ decays at most at the rate $-\log\beta$ per step.
\end{remark}

\noindent In the underdamped case $|\cos\theta|\le1$, that is $\bar\omega^2\ge\frac14$,
the free map alone decays at exactly $\beta^k$, and the noise can only slow
it. This case contains the windows, cells~5 to~8 of Section~\ref{sec:dict},
and cells~3 and~4 when $\omega^{*2}>\frac14$. In the overdamped case
$\rho(F)^2>\beta$; in cell~2, for instance, the rate is of order
$\eta p\ll\ebar$.

The window of Corollary~\ref{cor:window} is reached in two different ways:
by letting $\omega^2\to\infty$ with $\lambda\to0$, or by keeping the per-step
curvature $\lambda$ of order one, up to its ceiling $2(1+\beta)\to4$, while
$\beta\to1$. In the second case the curvature feedback $H_c$ stays positive, and
it competes with the noise feedback $H_n$.

\begin{corollary}[Curvature ceiling and boundary 2: the noise/curvature switch]\label{cor:curv}
Let $\beta\to1$ and $\lambda\to \lambda^*\in(0,4)$ with $|\lambda^*-2|\ge2\sin\delta$, and let
$H_n\to H_n^*$ and $H_c\to H_c^*=\lambda^*/4$. Suppose that
$\ut^*=H_n^*/(1-H_c^*)\le1-\delta$. Then
$\theta^*\to\arccos|1-\lambda^*/2|\in(0,\frac\pi2-\delta]$ and
$\bar\omega^2\asymp\ebar^{-2}$, and Corollary~\ref{cor:window} holds with
$\mathcal E=O\bigl(\ebar\,(N_0+\pt)\bigr)$, where $N_0$ is the initial size of
Theorem~\ref{thm:M}. In units of $\ebar$ per step, the envelope rate is
$1-\ut^*$ and the oscillation decays at rate $1+\ut^*/2$ while turning by the
angle $2\theta^*$ per step. In particular:
\begin{enumerate}
\item[(a)] on the \emph{curvature side} ($H_n^*=0$) the envelope rate is $1$;
\item[(b)] at the \emph{switch} ($H_n^*,H_c^*>0$) it is $1-H_n^*/(1-H_c^*)$,
which lies strictly between $1-H_n^*-H_c^*$ and $1$;
\end{enumerate}
\end{corollary}

\noindent The condition $|\lambda^*-2|\ge2\sin\delta$ is the Nyquist margin of
Theorem~\ref{thm:M} in the limit. On the curvature side the envelope decays
at the full retention rate, although the stability margin is only
$1-H_c^*<1$. At the switch, stability requires $H_n^*+H_c^*<1$, but the rate
is governed by $\ut^*$ and not by the total feedback $H_n^*+H_c^*$. On the noise
side, $\lambda\to0$, the curvature feedback vanishes and Corollary~\ref{cor:window}
gives the envelope rate $1-H^*$, the limit of $1-\ut^*$ as $\lambda^*\to0$.

The noise and curvature feedback enter stability symmetrically, through $H=H_n+H_c<1$, but
they enter the convergence rate asymmetrically, through
$1-\ut=1-H_n/(1-H_c)$. Curvature alone does not slow the envelope at all:
even close to the curvature ceiling, where the chain turns by an angle of
order one at every step and the stability margin $1-H_c$ is small, the
envelope still decays at the full retention rate. What the curvature does is
amplify the effect of the noise, by the factor $1/(1-H_c)$.

Thus the stability margin $1-H$ and the convergence rate $1-\ut$ are
different quantities, and they coincide only on the noise side. The second
boundary of the phase diagram is the switch between the two sides, where it
changes which feedback sets the rate.

\begin{proof}
The limits of $\theta^*$ and $H_c$ follow from
$\cos\theta=(1+\beta-\lambda)/(2\sqrt\beta)\to1-\lambda^*/2$ and $H_c=\lambda/(2(1+\beta))$.
Since the margins hold with $\delta$ in the limit, for all sufficiently
small $\ebar$ the actual parameters satisfy the hypotheses of
Theorem~\ref{thm:M} with $\delta/2$ in place of $\delta$:
$\beta\ge\frac12$, the Nyquist margin
$|\lambda-(1+\beta)|\ge2\sqrt\beta\sin(\delta/2)$, the stability margin
$\ut\le1-\delta/2$, and $\bar\omega^2\ge\omega_0^2(\delta/2)$. Then
$\bar\omega^2=\frac14+(\theta^*/\ebar)^2$ with $\theta^*$ bounded away from zero
gives $\bar\omega^2\asymp\ebar^{-2}$, so that the error $\mathcal E$ of
Corollary~\ref{cor:window} is $C_\delta(\ebar+\bar\omega^{-1})(N_0+\pt)=O(\ebar\,(N_0+\pt))$, and the rates are those of
Corollary~\ref{cor:window} with $\ut\to\ut^*$.

For (b), the inequality $1-H_n^*-H_c^*<1-\ut^*<1$ follows from
\[
H-\ut=\frac{H_c(1-H)}{1-H_c}>0
\]
when $0<H_c<1$ and $H<1$, together with $\ut^*>0$.
\end{proof}

\section{The phase dictionary}\label{sec:dict}

This section turns the corollaries of Section~\ref{sec:cor} into a phase
diagram for least squares. Under the co-scaling~\eqref{eq:coscaling-main}, the ratio
$\omega^2$, the noise feedback $H_n$, the ambient temperature $\varphi$, and the curvature $\lambda$ are powers of
$d$ whose exponents are affine in $(\kappa,\sigma,\gamma,\alpha)$. A cell is a
sign pattern of these exponents, and the limit in each cell is given by the corollaries of Section~\ref{sec:cor}.
We first compute the exponents and the stability constraint
$\alpha\ge\alpha_c$, then list the nine cells with their limits and rates,
then describe the phase diagram at the critical learning rate $\alpha=\alpha_c$
(Figure~\ref{fig:phase}), and then convert per-step rates into counts of
active updates and of samples; Remark~\ref{rem:regions} relates the cells to a
region-by-region classification, and Remark~\ref{rem:cost} computes the cost
and floor exponents at the rate-optimal learning rates. Lemma~\ref{lem:nonboundary} closes the
section by showing that the three lines along which the bookkeeping of active
updates changes are not phase boundaries. Section~\ref{sec:lr} reuses the same dictionary for logistic
regression, with the noise feedback replaced by zero.

\begin{figure}[t]
\centering
\includegraphics[width=0.72\textwidth]{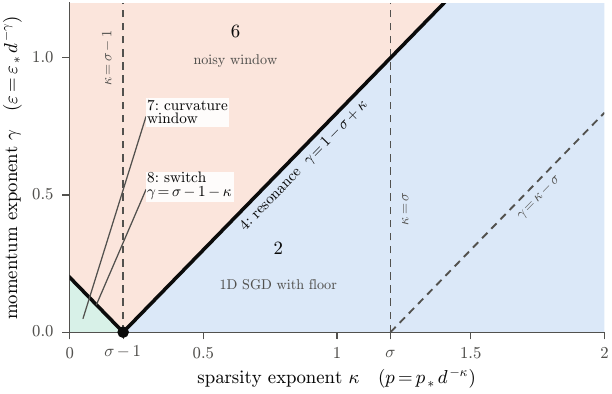}
\caption{Phase diagram of least squares at the critical learning rate
$\alpha=\alpha_c=\max(1-\sigma,-\kappa-\gamma)$, in the plane of the sparsity
exponent $\kappa$ and the momentum exponent $\gamma$, for $\sigma=\appphasesigma$. The
shaded regions are the three open phases: cell~2, one-dimensional SGD with a
floor (Corollary~\ref{cor:tikhonov}); cell~6, the noisy window
(Corollary~\ref{cor:window}); and cell~7, the curvature window
(Corollary~\ref{cor:curv}). The solid lines are the only phase boundaries: the
resonance line $\gamma=1-\sigma+\kappa$ (cell~4,
Corollary~\ref{cor:resonance}), on which the limit is the three-dimensional
system~\eqref{eq:L3}, and the noise/curvature switch $\gamma=\sigma-1-\kappa$
(cell~8, Corollary~\ref{cor:curv}). Above the switch line the learning rate is
noise-limited, $\alpha_c=1-\sigma$; below it it is curvature-limited,
$\alpha_c=-\kappa-\gamma$. The two boundaries meet only on the edge $\gamma=0$
(cell~9), at $\kappa=\sigma-1$. The dashed lines $\kappa=\sigma-1$,
$\kappa=\sigma$ and $\gamma=\kappa-\sigma$ are not boundaries
(Lemma~\ref{lem:nonboundary}).}
\label{fig:phase}
\end{figure}

By Lemma~\ref{lem:L1}, the whole co-scaling enters the chain through the four
numbers $(\beta,\lambda,H_n,\varphi)$, so the phase of a co-scaling sequence is
decided by how these numbers scale with $d$. We record the exponents first.

\paragraph{Exponents.} Under the co-scaling~\eqref{eq:coscaling-main},
$\omega^2\asymp d^{e_{\omega^2}}$, $H_n,\varphi\asymp d^{e_H}$ and $\lambda,H_c\asymp d^{e_\lambda}$,
where
\[
e_{\omega^2}=\gamma-\kappa-\alpha,\qquad e_H=1-\sigma-\alpha,\qquad e_\lambda=-\alpha-\kappa-\gamma=e_{\omega^2}-2\gamma .
\]
The limiting constants on the corresponding hyperplanes are
\[
\omega^{*2}=\frac{\eta_*p_*}{\varepsilon_*},\qquad H^*=\frac{\eta_*}{2B_*},\qquad
\varphi^*=\varsigma^2H^*,\qquad \lambda^*=\eta_*\varepsilon_*p_* .
\]

By Corollary~\ref{cor:stab}, stability requires $H=H_n+H_c<1$, which forces
$e_H\le0$ and $e_\lambda\le0$, that is,
\[
\alpha\ \ge\ \alpha_c:=\max(1-\sigma,\,-\kappa-\gamma).
\]
On the critical line $\alpha=\alpha_c$ it also requires the limiting feedback to
satisfy $H^*_n+H^*_c<1$, that is, $H_c^*<1$ and $\ut^*<1$. We call
$\alpha_c$ the critical learning-rate exponent.

\paragraph{Cells.} For $\gamma>0$, the signs of $(e_{\omega^2},e_H,e_\lambda)$ divide
the stable range into the cells of Table~\ref{tab:cells}, and each cell is
covered by one of the corollaries of Section~\ref{sec:cor}. The rates in the
table are envelope rates per minibatch step.

\begin{table}[ht]
\centering
\small
\renewcommand{\arraystretch}{1.25}
\begin{tabular}{@{}clcccll@{}}
\toprule
cell & name & $e_{\omega^2}$ & $e_H$ & $e_\lambda$ & limit (result) & rate/step\\
\midrule
1 & gradient flow & $<0$ & $<0$ & $<0$ & $r'=-2r$ (Cor.~\ref{cor:tikhonov}) & $2\eta\cA$\\
2 & 1D SGD with floor & $<0$ & $0$ & $<0$ & $r'=-2(1-H^*)r+2\varphi^*$ (Cor.~\ref{cor:tikhonov}) & $2\eta\cA(1-\ut)$\\
3 & heavy ball & $0$ & $<0$ & $<0$ & $R=X^2$, $X''+X'+\omega^{*2}X=0$ (Cors.~\ref{cor:regular},\ref{cor:lift}) & $\ebar r_c(\omega^{*2},0)$\\
4 & 3D SDE (resonance) & $0$ & $0$ & $<0$ & \eqref{eq:L3} at $(\omega^{*2},H^*,\varphi^*)$ (Cor.~\ref{cor:resonance}) & $\ebar r_c(\omega^{*2},H^*)$\\
5 & deterministic window & $>0$ & $<0$ & $<0$ & $\bar S'=-\bar S$, $Z$ rotates (Cor.~\ref{cor:window}) & $\ebar$\\
6 & noisy window & $>0$ & $0$ & $<0$ & $\bar S'=-(1-H^*)\bar S+2\varphi^*$ (Cor.~\ref{cor:window}) & $\ebar(1-H^*)$\\
7 & curvature window & $>0$ & $<0$ & $0$ & as 5, per-step angle $\theta^*\asymp1$ (Cor.~\ref{cor:curv}) & $\ebar$\\
8 & switch & $>0$ & $0$ & $0$ & as 6 with $(\ut^*,\pt^*)=(H_n^*,\varphi^*)/(1-H_c^*)$ (Cor.~\ref{cor:curv}) & $\ebar(1-\ut^*)$\\
9 & $\gamma=0$ edge & \multicolumn{3}{c}{$\varepsilon$ fixed} & exact chain; Cor.~\ref{cor:tikhonov} if $e_{\omega^2}<0$ & ---\\
\bottomrule
\end{tabular}
\caption{The cells of the LS phase dictionary, the limit in each cell with
the result that proves it, and the envelope rate per minibatch step. The
square-root form in cell~3 holds from a cold start.}
\label{tab:cells}
\end{table}

\noindent Since $e_\lambda=e_{\omega^2}-2\gamma$, the condition $e_\lambda=0$ forces
$e_{\omega^2}=2\gamma>0$, so cells~7 and~8 are automatically windows. Every other
sign pattern is either unstable or empty. The floor in every cell is
$\pt/(1-\ut)=\varphi/(1-H)$, by Corollary~\ref{cor:stab}.

The sign of $e_{\omega^2}$ selects the clock and the shape of the
limit: the learning clock and a one-dimensional equation when $e_{\omega^2}<0$,
the retention clock and the full continuum system when $e_{\omega^2}=0$, and the
retention clock with a two-scale oscillation when $e_{\omega^2}>0$. The sign of
$e_H$ decides whether the noise feedback survives in the limit, and the sign of
$e_\lambda$ whether the curvature feedback does. The rates in the last column are
envelope rates, which are the relevant notion of convergence
(Remark~\ref{rem:window}).

\paragraph{The phase diagram at the critical learning rate.} Now take
$\alpha=\alpha_c$ in the $(\kappa,\gamma)$ plane at fixed $\sigma$;
Figure~\ref{fig:phase} draws the result for $\sigma=\appphasesigma$. Three
cases arise, according to which of the two constraints defining $\alpha_c$
is active.

In the \emph{noise-limited} region $\gamma>\sigma-1-\kappa$ we
have $\alpha_c=1-\sigma$, so $e_H=0$ and
$e_{\omega^2}=\gamma-(1-\sigma+\kappa)$. This gives cell~2 below the resonance
line $\gamma=1-\sigma+\kappa$, cell~4 on it, and cell~6 above it.

On the \emph{switch} line $\gamma=\sigma-1-\kappa$,
which requires $\kappa<\sigma-1$, both constraints are active and we are in
cell~8. In the \emph{curvature-limited} region $\gamma<\sigma-1-\kappa$ we
have $\alpha_c=-\kappa-\gamma$, so $e_\lambda=0$ and $e_H<0$, which is cell~7.

The phase diagram at $\alpha_c$ therefore has three open phases, $\{2,6,7\}$,
and two boundaries, $\{4,8\}$. The two boundaries meet only at the point
$(\kappa,\gamma)=(\sigma-1,0)$ of the edge $\gamma=0$.

The choice of learning rate matters here. Above the resonance line, the
exponent $\alpha=\gamma-\kappa$ exceeds both $1-\sigma$ and $-\kappa-\gamma$,
hence $\alpha_c$, so this choice
is a subcritical learning rate, and it lands in the heavy-ball cell~3 rather
than in cell~6. The noisy window is seen only when the learning rate is pushed
to its critical value. Even then it is not a faster regime: by
Remark~\ref{rem:window} its envelope rate $\ebar(1-H^*)$ is a cap for the
given noise feedback, and the larger learning rate raises the floor.

\paragraph{Clocks.} One unit of retention time, $\tau=1$, is
$\ebar^{-1}\asymp d^\gamma$ steps, and one unit of learning time is $(\eta\cA)^{-1}$
steps. To count active updates instead of steps, multiply step counts by
$P_{\mathrm{batch}}=1-(1-p)^B$; to count samples, multiply by $B$. Time
exponents measured in active updates or in samples are therefore
bookkeeping, obtained from the per-step clock by these conversions, and not
part of the limit.

\begin{remark}[Comparison with a region-by-region analysis]\label{rem:regions}
An analysis region by region splits the $(\kappa,\gamma)$ plane also along
$\kappa=\sigma$ (dense versus sparse batches), along $\kappa=\sigma-1$, and
along $\gamma=\kappa-\sigma$ (memoryless momentum). By
Lemma~\ref{lem:nonboundary} none of these lines is a phase boundary, and the
regions it produces fall into the cells above as follows.

Every region below resonance, $\gamma<1-\sigma+\kappa$, taken with
$\alpha=1-\sigma$, is cell~2. This covers the dense region
$\sigma-1<\kappa<\sigma$, the sparse region $\kappa>\sigma$ with
$\gamma>\kappa-\sigma$, the memoryless region $\kappa>\sigma$ with
$\gamma\le\kappa-\sigma$, and the line $\kappa=\sigma$ with $0<\gamma<1$.

Every region above resonance, $\gamma>1-\sigma+\kappa$, taken with
$\alpha=\gamma-\kappa$, is cell~3, including $\kappa\le\sigma-1$ and the line
$\kappa=\sigma$ with $\gamma>1$. Every region on the resonance line,
including the point $\kappa=\sigma$, $\gamma=1$, is cell~4. Cells~6, 7 and~8,
at $\alpha_c$ above resonance and at $\alpha_c$ with $\kappa<\sigma-1$, have
no counterpart in such an analysis.
\end{remark}

\paragraph{Limiting equations by cell.} Table~\ref{tab:limits} lists, for the
cells discussed in Section~\ref{sec:limits_main} of the main text, where each cell lies in
the $(\kappa,\gamma)$ plane and at which learning rate, the clock on which its
limit lives, the limiting equation of LS with the result that proves it, and
the corresponding limit for LR. Throughout,
$\gamma_{\mathrm{res}}:=1-\sigma+\kappa$, the limiting constants
$\omega^{*2},H^*,\varphi^*,\lambda^*$ are those of the exponents paragraph above, and
the start is cold, $\Sigma_0=\operatorname{diag}(R_0,0)$.

\begin{table}[ht]
\centering
\small
\renewcommand{\arraystretch}{1.35}
\setlength{\tabcolsep}{4pt}
\begin{tabular}{@{}c >{\raggedright\arraybackslash}p{3.0cm} >{\raggedright\arraybackslash}p{1.7cm} >{\raggedright\arraybackslash}p{6.1cm} >{\raggedright\arraybackslash}p{4.3cm}@{}}
\toprule
cell & where; learning rate & clock \newline (steps/unit) & LS limit (result) & LR limit of the drift recursion\\
\midrule
2 & $0<\gamma<\gamma_{\mathrm{res}}$; \newline critical $\alpha=1-\sigma$
  & $t=\eta p\,k$ \newline ($d^{\gamma_{\mathrm{res}}}$)
  & $dR/dt=-2(1-H^*)\,R+2\varphi^*$, \ $R(0)=R_0$ \newline
    (Cor.~\ref{cor:tikhonov}, where this limit is called $r$)
  & $d\theta_\parallel/dt=r-\varrho\theta_\parallel$, \newline
    $dR_\perp/dt=2\Phi^*-2\varrho R_\perp$; \newline
    converges to its equilibrium (Prop.~\ref{prop:LR2})\\
3 & $\gamma>\gamma_{\mathrm{res}}$; \newline subcritical $\alpha=\gamma-\kappa$
  & $\tau=\ebar k$ \newline ($d^{\gamma}$)
  & $dx/d\tau=-\omega^{*2}y$, \ $dy/d\tau=x-y$, \ $R=x^2$, \newline
    $x(0)=\sqrt{R_0}$, $y(0)=0$; \newline equivalently $x''+x'+\omega^{*2}x=0$
    \newline (Cors.~\ref{cor:regular}, \ref{cor:lift})
  & \eqref{eq:LR5} with $\Phi^*=0$ \newline (Prop.~\ref{prop:LR34}); \newline
    long times: Prop.~\ref{prop:S}, Cor.~\ref{cor:recursion}\\
4 & $\gamma=\gamma_{\mathrm{res}}>0$; \newline critical $\alpha=1-\sigma$
  & $\tau=\ebar k$ \newline ($d^{\gamma}$)
  & \eqref{eq:L3} at $(\omega^{*2},H^*,\varphi^*)$ from $(R,V,C)=(R_0,0,0)$,
    that is, $R=\E|X|^2$ for~\eqref{eq:SHB}; irreducibly three-dimensional
    \newline (Cor.~\ref{cor:resonance})
  & \eqref{eq:LR5} with $\Phi^*=\eta_*/(2B_*)$ \newline (Prop.~\ref{prop:LR34}); \newline
    long times: Prop.~\ref{prop:S}, Cor.~\ref{cor:recursion}\\
6 & $\gamma>\gamma_{\mathrm{res}}$, $\gamma>\sigma-1-\kappa$; \newline critical $\alpha=1-\sigma$
  & $\tau=\ebar k$ \newline ($d^{\gamma}$)
  & $d\bar S/d\tau=-(1-H^*)\,\bar S+2\varphi^*$, \ $R\approx\frac12\bar S$ plus a
    decaying oscillation \newline (Cor.~\ref{cor:window})
  & \eqref{eq:LR5} with $\omega^{*2}\to\infty$; rates of Prop.~\ref{prop:W1},
    slow energy $1+\psi/(2+\psi)$, risk envelope
    $\min(-z_0,-\mu_1,-\mu_2)<\frac34$; \newline recursion: Cor.~\ref{cor:recursion}
    at fixed $\omega^{*2}$\\
7 & $0<\gamma<\sigma-1-\kappa$; \newline critical $\alpha=-\kappa-\gamma$
  & $\tau=\ebar k$ \newline ($d^{\gamma}$)
  & $d\bar S/d\tau=-\bar S$; oscillation turning by an angle $2\theta^*$ of
    order one per step, $\theta^*\to\arccos|1-\lambda^*/2|$ \newline (Cor.~\ref{cor:curv}(a))
  & zero bulk floor; for $\lambda^*<\min(2,4/(1+\theta^{*2}))$
    (Prop.~\ref{prop:W2}) \newline and warm starts: Ass.~\ref{ass:W2}\\
8 & $\gamma=\sigma-1-\kappa>0$; \newline critical $\alpha=1-\sigma$ \newline ${}=-\kappa-\gamma$
  & $\tau=\ebar k$ \newline ($d^{\gamma}$)
  & $d\bar S/d\tau=-(1-\ut^*)\,\bar S+2\pt^*$, \newline
    $(\ut^*,\pt^*)=(H^*,\varphi^*)/(1-H_c^*)$, \ $H_c^*=\lambda^*/4$
    \newline (Cor.~\ref{cor:curv}(b))
  & bulk floor $\varphi/(1-H_c)$; for $\lambda^*<\lambda_c<2$ \newline
    (Prop.~\ref{prop:W2}) and warm starts: Ass.~\ref{ass:W2}\\
\bottomrule
\end{tabular}
\caption{Limiting equations by cell. The clock column gives the number
of steps in one unit of the clock, up to a constant. In cells~6 to~8,
$R_k$ has in general no pointwise limit: $\bar S$ is the slow energy, and
$R_k$ is $\frac12\bar S$ plus an oscillation that turns by $2\theta^*$ per
step and whose amplitude decays at rate $1+\ut^*/2$ on the clock $\tau$,
at least as fast as the rate $1-\ut^*$ of $\bar S$ (equally fast in cell~7,
where $\ut^*=0$). The local average of $R_k$ is $\frac12\bar S$. Logistic
regression: limits of the drift recursion, in the variables
of~\eqref{eq:LR5}, on the clock $\tau=\varepsilon k$ in cells~3 and~4; the
random chain follows the drift recursion over finite horizons in cells~2
to~4 (Corollary~\ref{cor:fluid}); details in Remark~\ref{rem:limits}(d).
In cell~6, $\psi:=R^*/(2(1+\theta^{*2}))$ at the equilibrium of
Proposition~\ref{prop:D}; the rates are those of~\eqref{eq:LR5} near its
equilibrium as $\omega^{*2}\to\infty$, and Corollary~\ref{cor:recursion} is not
shown to hold uniformly in $\omega^{*2}$. In cells~7 and~8, the LR
$\lambda^*$ is $\lim\eta\varepsilon p\varrho^*$, and $\lambda_c$ is the LR curvature ceiling
of Proposition~\ref{prop:W2}, below the LS ceiling $4$; in cell~7, where
$\Phi^*=0$, its limit is $\min(2,4/(1+\theta^{*2}))$.
The window rows are discussed further in Remark~\ref{rem:window-equal}.}
\label{tab:limits}
\end{table}

\begin{remark}[Limiting equations by cell]\label{rem:limits}
Under the co-scaling~\eqref{eq:coscaling-main}, with prefactors satisfying
the stability margin and, in cells~7 and~8, the Nyquist margin
$|\lambda^*-2|\ge2\sin\delta$, the following hold.
\begin{enumerate}
\item[(a)] \emph{Locations and clocks.} The locations in
Table~\ref{tab:limits} are those of the paragraph on the phase diagram at
the critical learning rate, and cell~3 is the subcritical choice
$\alpha=\gamma-\kappa$ above the resonance line. One unit of $t$ is
$1/(\eta p)\asymp d^{\alpha+\kappa}=d^{\gamma_{\mathrm{res}}}$ steps at
$\alpha=1-\sigma$, and one unit of $\tau$ is $1/\ebar\asymp d^\gamma$ steps.
On the resonance line the two coincide.
\item[(b)] \emph{Least squares.} In each row, the limit is the one of the
cited result, with the matched parameters replaced by their limits:
$(\ut,\pt)\to(H^*,\varphi^*)$ in cells~2, 4 and~6, where $\lambda\to0$;
$(\ut,\pt)\to(0,0)$ in cells~3 and~7, where the noise exponent is negative;
and $(\ut,\pt)\to(H^*,\varphi^*)/(1-H_c^*)$ in cell~8. Row~8 of
Table~\ref{tab:cells} is to be read with this renormalized ambient temperature
$\pt^*=\varphi^*/(1-H_c^*)$ as well as with $\ut^*$, as in
Corollary~\ref{cor:window}, so that the floor is
$\pt^*/(1-\ut^*)=\varphi^*/(1-H^*-H_c^*)$; here $H^*=\eta_*/(2B_*)$ is the
limit of the noise feedback $H_n$, and $H^*+H_c^*$ is the limit of the total feedback.
\item[(c)] \emph{The stochastic heavy ball.} The limits in cells~2 to~8 are
all limits of the moments of~\eqref{eq:SHB}: cell~4 is~\eqref{eq:L3} itself;
cell~3 is its noiseless, rank-one case, the deterministic heavy ball of
Remark~\ref{rem:SHB}(b); cell~2 is its reduction as $\omega^{*2}\to0$ on the
clock $\omega^{*2}\tau$ (Corollary~\ref{cor:resonance}(iii)); and cells~6 to~8
are its two-scale form as $\bar\omega^2\to\infty$.
\item[(d)] \emph{Logistic regression.} All LR limits are limits of the drift
recursion. In cells~2 to~4, where $\omega^2$ stays bounded,
Corollary~\ref{cor:fluid} transfers them to the random chain over finite
horizons, at the scale $\sqrt{\log d/d}$, when $\alpha_\eta>-\frac12$. In the
windows, where $\omega^2\to\infty$, no such transfer is proved.
In cell~2, on $t=\eta pk$ (Proposition~\ref{prop:LR2}),
\[
\frac{d\theta_\parallel}{dt}=r-\varrho\theta_\parallel,\qquad
\frac{dR_\perp}{dt}=2\Phi^*-2\varrho R_\perp,\qquad
\varrho=e^{(\theta_\parallel^2+R_\perp-r^2)/2},
\]
with $\Phi^*=\eta_*/(2B_*)$, and every trajectory of this planar system
converges to the equilibrium of Proposition~\ref{prop:D}. In cells~3 and~4,
on $\tau=\varepsilon k$, the limit is the dynamic-$\varrho$
system~\eqref{eq:LR5} (Proposition~\ref{prop:LR34}), with $\Phi^*=0$ in
cell~3, where the bulk stays of rank one, and $\Phi^*=\eta_*/(2B_*)$ in
cell~4; it converges to its equilibrium from every physical state
(Proposition~\ref{prop:S}), and so does the drift recursion
(Corollary~\ref{cor:recursion}).
In the noisy window, cell~6, the limit is~\eqref{eq:LR5} with
$\omega^{*2}\to\infty$, whose rates near the equilibrium are those of
Proposition~\ref{prop:W1}; at each fixed $\omega^{*2}$ the drift recursion
converges (Corollary~\ref{cor:recursion}); uniformity in $\omega^{*2}$ is not
proved. In the curvature window and on the switch, cells~7 and~8, the fixed
point of the drift recursion is linearly unstable just above the curvature
ceiling $\lambda_c$ (Proposition~\ref{prop:W2}), which is below $2$ in cell~8
and tends to $\min(2,4/(1+\theta^{*2}))$ in cell~7, where $\Phi^*=0$;
convergence from a warm start below the ceiling, with the bulk floor
$\varphi/(1-H_c)$, is Assumption~\ref{ass:W2}.
\end{enumerate}
\end{remark}

\begin{proof}
(a) For $\gamma>0$, the inequality $-\kappa-\gamma>1-\sigma$ means
$\gamma<\sigma-1-\kappa$, so $\alpha_c=1-\sigma$ above the switch line and
$\alpha_c=-\kappa-\gamma$ below it. At $\alpha_c=1-\sigma$,
$e_{\omega^2}=\gamma-\gamma_{\mathrm{res}}$, $e_H=0$ and $e_\lambda=-2\gamma+e_{\omega^2}$,
which gives cells~2, 4 and~6; below the switch, $e_\lambda=0$ and $e_H<0$, which
is cell~7; on the switch, cell~8. Cell~2 needs no condition relating
$\gamma$ to the switch line: $0<\gamma<\gamma_{\mathrm{res}}$ forces
$\gamma_{\mathrm{res}}>0$, hence $\sigma-1-\kappa<0<\gamma$. At
$\alpha=\gamma-\kappa$ above the resonance line, $e_{\omega^2}=0$,
$e_H=\gamma_{\mathrm{res}}-\gamma<0$ and $e_\lambda=-2\gamma<0$, which is cell~3.
The clocks follow from $\eta p=\eta_*p_*d^{-\alpha-\kappa}$ and
$\ebar=\varepsilon(1+O(\varepsilon))$ with $\varepsilon=\varepsilon_*d^{-\gamma}$.

(b) The parameter limits follow from the exponents, from $H_c\asymp \lambda$, and
from $H_n\to H^*$, $\varphi\to\varphi^*$ when $e_H=0$; in cell~8 also
$H_c\to \lambda^*/4$, since $\beta\to1$. The rows are then
Corollary~\ref{cor:tikhonov} (cell~2), Corollaries~\ref{cor:regular}
and~\ref{cor:lift}(i) with the free flow of the rank-one start (cell~3),
Corollary~\ref{cor:resonance}(i) (cell~4), Corollary~\ref{cor:window}(a)
and~(c) (cell~6), and Corollary~\ref{cor:curv} (cells~7 and~8), in which
Corollary~\ref{cor:window} holds with $(\ut,\pt)$, whose additive term
is $2\pt$. The floor identity is
$\pt/(1-\ut)=\varphi/(1-H_c-H_n)$.

(c) This is Remark~\ref{rem:SHB} together with
Corollary~\ref{cor:resonance}(iii) and the two-scale form of
Corollary~\ref{cor:window}(c).

(d) Cell~2 is Proposition~\ref{prop:LR2} with $\Phi=\eta(d-1)/(2B)\to\eta_*/(2B_*)$
at $\alpha_\eta=1-\sigma$. In cells~3 and~4,
$\omega_{\mathrm{nom}}^2=\eta p/\varepsilon\to\omega^{*2}$, and $\Phi\to0$ at
$\alpha_\eta=\gamma-\kappa>1-\sigma$, while $\Phi\to\eta_*/(2B_*)$ on the
resonance line; Proposition~\ref{prop:LR34} applies, and the long-time
statements are Proposition~\ref{prop:S}(i) and
Corollary~\ref{cor:recursion}. In cell~6, $\omega_{\mathrm{nom}}^2\to\infty$
and $\Phi\to\eta_*/(2B_*)>0$, so Proposition~\ref{prop:W1} applies to the
limit; in cells~7 and~8, $\lambda$ stays of order one, and the statements are
Proposition~\ref{prop:W2} and Assumption~\ref{ass:W2}.
\end{proof}

The dictionary also gives the cost of reaching a given accuracy. Write
$\gamma_{\mathrm{res}}:=1-\sigma+\kappa$ for the height of the resonance
line, and let $T_{\mathrm{steps}}$ be the decay time in steps of the
excess risk $R_k-R_\infty$, that is, the inverse of the envelope rate in the
last column of Table~\ref{tab:cells}. For LS, the per-sample gradient
$x(\langle x,e_k\rangle-\zeta)$ vanishes when $s=0$, so only the active
samples ($s=1$) carry nonzero gradients, and a step processes $pB$ of them on
average.

\begin{remark}[Cost and floor exponents]\label{rem:cost}
Take LS with $\gamma>0$ and $\varsigma>0$, an exponent $\alpha\ge\alpha_c$,
and prefactors for which the limiting parameters satisfy the stability and
Nyquist margins of the corollaries of Section~\ref{sec:cor}.
\begin{enumerate}
\item[(i)] \emph{Steps.} $T_{\mathrm{steps}}\asymp d^{T(\alpha)}$ with
$T(\alpha):=\max(\alpha+\kappa,\gamma)$. The minimum of $T$ over
$\alpha\ge\alpha_c$ is $T^*:=\max(\gamma,\gamma_{\mathrm{res}})$, and it is
attained exactly for $\alpha_c\le\alpha\le\max(1-\sigma,\gamma-\kappa)$; we
call these exponents rate-optimal.
\item[(ii)] \emph{Active samples.} At every rate-optimal $\alpha$, the
expected number of active samples processed per decay time is
$pB\,T_{\mathrm{steps}}\asymp d^{\,1+(\gamma-\gamma_{\mathrm{res}})_+}$. Below the
resonance line the only rate-optimal exponent is $\alpha=1-\sigma$, in
cell~2, and there this number is
$(d+2-p)/(4H_n(1-\ut))=\frac{d}{4H^*(1-H^*)}(1+o(1))\ge d\,(1+o(1))$, whatever
$\kappa$ and $\sigma$.
\item[(iii)] \emph{Floor.} The floor is $\varphi/(1-H)\asymp d^{1-\sigma-\alpha}$.
Among the rate-optimal exponents it is lowest at the smallest rate-optimal
learning rate, $\alpha=\max(1-\sigma,\gamma-\kappa)$, where it is
$\asymp d^{-(\gamma-\gamma_{\mathrm{res}})_+}$.
\end{enumerate}
\end{remark}

\begin{proof}
(i) By Table~\ref{tab:cells}, the envelope rate per step is of order
$\eta p\asymp d^{-\alpha-\kappa}$ in cells~1 and~2 and of order
$\ebar\asymp d^{-\gamma}$ in cells~3 to~8. Indeed, the factors $1-\ut$,
$1-H^*$ and $1-\ut^*$ are bounded away from zero by the stability margin,
and in cells~3 and~4, where $\omega^{*2}\in(0,\infty)$ and $H^*<1$ are fixed,
$r_c(\omega^{*2},H^*)>0$ because~\eqref{eq:L3} is exponentially stable for
$H^*<1$ (Lemma~\ref{lem:L3}(c), first sentence). Cells~1 and~2 are those with
$e_{\omega^2}<0$, that is, $\alpha+\kappa>\gamma$, and this gives $T(\alpha)$ in
every cell. Since $T$ is nondecreasing in $\alpha$, its minimum is
$T(\alpha_c)=\max(1-\sigma+\kappa,-\gamma,\gamma)=\max(\gamma,\gamma_{\mathrm{res}})$,
and $T(\alpha)=T^*$ exactly when $\alpha+\kappa\le T^*$, that is,
$\alpha\le\max(1-\sigma,\gamma-\kappa)$.

(ii) The first claim is $\sigma-\kappa+T^*=\max(1,\sigma-\kappa+\gamma)=1+(\gamma-\gamma_{\mathrm{res}})_+$.
Below the line, $\gamma<\gamma_{\mathrm{res}}$ gives
$\max(1-\sigma,\gamma-\kappa)=1-\sigma$. Also $\alpha_c=1-\sigma$, since
$-\kappa-\gamma>1-\sigma$ would mean $\gamma<-\gamma_{\mathrm{res}}$, while
$0<\gamma<\gamma_{\mathrm{res}}$. So $\alpha=1-\sigma$ is the only
rate-optimal exponent, and $e_{\omega^2}<0$, $e_H=0$, $e_\lambda<0$ put it in cell~2. By
Corollary~\ref{cor:tikhonov}, the decay time is $1/(2\eta p(1-\ut))$ steps, hence
$pB/(2\eta p(1-\ut))=B/(2\eta(1-\ut))$ active samples, and
$B/\eta=(d+2-p)/(2H_n)$ by \eqref{eq:params} and \eqref{eq:LSoracle}. In
cell~2, $H_n\to H^*$ and $\ut\to H^*$, and $4H^*(1-H^*)\le1$.

(iii) By Corollary~\ref{cor:stab}(ii) the floor is
$\varphi/(1-H)$, with $1-H$ bounded away from zero, and
$\varphi=\eta\varsigma^2d/(2B)\asymp d^{1-\sigma-\alpha}$. This decreases with
$\alpha$, and at $\alpha=\max(1-\sigma,\gamma-\kappa)$ the exponent is
$\min(0,\gamma_{\mathrm{res}}-\gamma)=-(\gamma-\gamma_{\mathrm{res}})_+$.
\end{proof}

Above the resonance line, $\gamma=\gamma_{\mathrm{res}}+x$ with $x>0$, every
rate-optimal learning rate costs $d^{x}$ times more steps and active samples
than at the line, and the smallest one lowers the floor to $d^{-x}$. The same
three exponents are available without the extra retention: for
$0<\gamma\le\gamma_{\mathrm{res}}$, the subcritical exponent
$\alpha=1-\sigma+x$ gives $T(\alpha)=\gamma_{\mathrm{res}}+x$, the active-sample
exponent $1+x$ and the floor exponent $-x$. Choosing a longer momentum
memory beyond the resonance line thus amounts to choosing a smaller learning
rate, and it is the learning rate that sets the floor.

For LR the floor exponent is the same. By Proposition~\ref{prop:D}(v) the
equilibrium bulk second moment is $R_\perp^*\asymp d^{1-\sigma-\alpha_\eta}$,
and the scalar equation of Proposition~\ref{prop:D} gives
$r/\theta^*\le e^{\Phi/2}$, hence $0<r-\theta^*\le r\Phi/2$, so that the
signal error $(r-\theta^*)^2$ is at most of the same order. The step exponent
$T^*$ carries over if the LR envelope rates are of order $\eta p$ per step in
cell~2 and of order $\ebar$ in cells~3 to~8. For cells~3 and~4 this is
proved: the limit system~\eqref{eq:LR5} converges exponentially
(Proposition~\ref{prop:S}(iii)), and the drift recursion converges at a rate
of order $\varepsilon$ per step (Corollary~\ref{cor:recursion}). For
cells~5 and~6 the candidate rate is the envelope rate of~\eqref{eq:LR5}
near its equilibrium as $\omega^{*2}\to\infty$. In cell~5, where $\Phi^*=0$,
it is $\ebar$, by the block-triangular linearization (discussion after
Proposition~\ref{prop:W1}). In cell~6, where $\Phi^*>0$,
it is $\ebar\min(-z_0,-\mu_1,-\mu_2)$ by
Proposition~\ref{prop:W1} and the discussion after it. By the trace
identity this is below $\frac34\ebar$ up to $O(\omega^{*-1})$, strictly
below the LS window rate $\ebar$. The rate reaches the drift recursion at each
fixed $\omega^{*2}$ through Corollary~\ref{cor:recursion}, but not uniformly as
$\omega^{*2}\to\infty$. For cells~7 and~8 a rate of order $\ebar$ is part of
Assumption~\ref{ass:W2}. For cell~2 it is the rate of the linearization in
Proposition~\ref{prop:D}(vi), which we have not transferred to the drift
recursion over unbounded horizons. All of these are statements about the
drift recursion. Corollary~\ref{cor:fluid} transfers them to the random
chain over finite horizons, under the condition $\alpha_\eta>-\frac12$,
which fails at the critical exponent $\alpha_\eta=1-\sigma$ when
$\sigma\ge\frac32$.

It remains to explain why the lines along which the bookkeeping of active
updates changes do not appear in Table~\ref{tab:cells}. Three lines in the
$(\kappa,\gamma)$ plane look as if they should separate phases, because a
natural physical quantity changes its order of magnitude across each of
them. The exact reduction shows that none of them does.

\begin{lemma}[Non-boundaries]\label{lem:nonboundary}
Across each of the lines $\kappa=\sigma-1$, $\kappa=\sigma$ and
$\gamma=\kappa-\sigma$, the exponents $(e_{\omega^2},e_H,e_\lambda)$ are given by the
same affine formulas and the limiting constants
$(\omega^{*2},H^*,\varphi^*,\lambda^*)$ by the same monomials in
$(p_*,B_*,\varepsilon_*,\eta_*,\varsigma)$.
Hence these lines are not boundaries: crossing one of them changes neither the cell nor the
limit, except where it meets a boundary of Table~\ref{tab:cells}. For example,
$\kappa=\sigma$ meets the resonance line $\gamma=1-\sigma+\kappa$ at $\gamma=1$.
\end{lemma}

\noindent What does change across these lines is the following. At
$\kappa=\sigma$, the batch-activity probability changes from
$P_{\mathrm{batch}}\approx Bp$ (for $\kappa>\sigma$) to
$P_{\mathrm{batch}}\approx1$ (for $\kappa<\sigma$). At $\kappa=\sigma-1$, the
single-step gradient changes from noise-dominated (for $\kappa>\sigma-1$) to
signal-dominated, $p^2>\vinc$ (for $\kappa<\sigma-1$). At
$\gamma=\kappa-\sigma$, the memory $1/\varepsilon$ of the momentum changes
from shorter than the gap $1/P_{\mathrm{batch}}$ between active updates (for
$\gamma<\kappa-\sigma$) to longer than it (for $\gamma>\kappa-\sigma$). These
quantities enter only the conversion from steps to active updates or to
samples.

\begin{proof}
This is immediate from Lemma~\ref{lem:L1} and \eqref{eq:LSoracle}. The
quantities $p$, $B$ and $d$ enter the chain only through
\[
\lambda=\eta\varepsilon p,\qquad H_n=\frac{\eta(d+2-p)}{2B},\qquad \varphi=\frac{\eta\varsigma^2d}{2B},
\]
and $d+2-p=d\,(1+O(d^{-1}))$ for every $\kappa\ge0$. None of these
expressions changes form across the three lines, so the per-step chain is
the same function of $(\beta,\lambda,H_n,\varphi)$ on both sides, and
Theorem~\ref{thm:M} and its corollaries apply unchanged.
\end{proof}

The lesson is that a quantity which changes order of magnitude is
a phase boundary only if it enters the four numbers $(\beta,\lambda,H_n,\varphi)$.
The batch-activity probability and the ratio $p^2/\vinc$ do not, so they
affect only how one converts steps into active updates or into samples.

The quantity that does separate signal-dominated from noise-dominated
dynamics is the ratio of the noise and curvature feedback,
\[
\frac{H_c}{H_n}\asymp\frac{\varepsilon pB}{d},
\]
which compares the signal accumulated over one retention time with the
noise. It changes order at the switch $\gamma=\sigma-1-\kappa$
(Corollary~\ref{cor:curv}), not at $\kappa=\sigma-1$.

\section{Logistic regression}\label{sec:lr}

In this section $\varrho:=e^{(q-r^2)/2}$, with $q:=|\theta|^2$, denotes the LR
curvature factor. The learning-rate exponent is written $\alpha_\eta$.

We now extend the analysis to logistic regression. The plan follows the
least-squares argument, with two additions. We first show that the
$d$-dimensional iteration reduces exactly to a five-dimensional Markov chain
(Lemma~\ref{lem:LR0}), and that the bulk part of its conditional drift is
exactly the moment recursion~\eqref{eq:L1}, with parameters that depend on the
state through $\varrho$ (Lemma~\ref{lem:LRdrift}). In the tame regime of
Lemma~\ref{lem:B} the noise feedback vanishes, so Sections~\ref{sec:renewal}
to~\ref{sec:dict} apply to the bulk, conditionally on the state, with
$\ut\approx0$.

The first addition is
a fluid limit, needed because Lemma~\ref{lem:LRdrift} holds only
conditionally on the state: Theorem~\ref{thm:A}, a general concentration
estimate for a Markov chain around its drift recursion, together with the
increment bounds of Lemma~\ref{lem:LRinc}, shows that the chain follows its
drift recursion in cells~2 to~4 (Corollary~\ref{cor:fluid}). The second addition is the signal
subsystem, which is coupled to the bulk through $\varrho$. After the LR version
of the dictionary, Proposition~\ref{prop:LR34} gives the five-dimensional
limit with dynamic $\varrho$ in cells~3 and~4, Proposition~\ref{prop:LR2} gives
the planar slow system of cell~2 and its global stability
(Figure~\ref{fig:lrslow}), and
Proposition~\ref{prop:S} proves global
stability in cells~3 and~4. Proposition~\ref{prop:W1} and
Corollary~\ref{cor:recursion} treat the windows of cells~5 and~6, and
Proposition~\ref{prop:W2} and Assumption~\ref{ass:W2} the curvature window
of cells~7 and~8.
Proposition~\ref{prop:D}, which identifies the common equilibrium and floor of
all cells, comes last, because its statement about the Jacobian refers to the
slow system of Proposition~\ref{prop:LR2}.

\begin{figure}[t]
\centering
\includegraphics[width=0.55\textwidth]{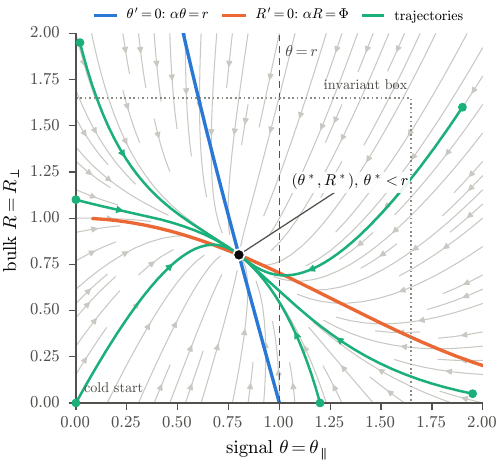}
\caption{Phase portrait of the cell-2 slow system of
Proposition~\ref{prop:LR2}, $\theta'=r-\varrho\theta$, $R'=2\Phi-2\varrho R$ with
$\varrho=e^{(\theta^2+R-r^2)/2}$, for $r=1$ and $\Phi=1$, on the learning clock
$t=\eta pk$. Here $\theta$ is the signal overlap $\theta_\parallel$ and $R$ is
the bulk second moment $R_\perp$. The nullclines $\varrho\theta=r$ and
$\varrho R=\Phi$ cross exactly once, at the equilibrium $(\theta^*,R^*)$ of
Proposition~\ref{prop:D}, and every trajectory, including the one from the
cold start $(0,0)$, converges to it, as the Bendixson--Dulac and
Poincar\'e--Bendixson argument in the proof of Proposition~\ref{prop:LR2}
guarantees. The dotted box is forward invariant. The equilibrium lies to the
left of the dashed line $\theta=r$: the bulk second moment raises $\varrho$,
and the equilibrium condition $\varrho\theta=r$ then forces $\theta^*<r$, so
the equilibrium is biased against alignment (Proposition~\ref{prop:D}(i)).}
\label{fig:lrslow}
\end{figure}

\paragraph{Model.} A sample $(x,y)$ has a label $y\sim\mathrm{Ber}(p)$, where
$y=1$ is the rare class, and features $x=y\mu+z$, where
$z\sim\mathcal N(0,I_d)$ is independent of $y$ and $\mu\in\R^d$ is a fixed
signal vector of norm $|\mu|=r$. We fit a logistic model in which only the
weight $\theta$ is trained and the bias is held at its Bayes-optimal value
\[
b^*=\log\frac p{1-p}-\frac{r^2}2 .
\]
With this bias, the Bayes posterior $\mathbb P(y=1\mid x)$ is
$\sigma(\langle\mu,x\rangle+b^*)$, where $\sigma(t)=e^t/(1+e^t)$ is the
logistic function, so the target is $\theta_\star=\mu$. The per-sample gradient
of the logistic loss is $(\sigma(\langle\theta,x\rangle+b^*)-y)\,x$, and the
gradient estimate $g$ is its average over a fresh batch of $B$ independent
samples.

We decompose $\theta=\theta_\parallel\hat\mu+\theta_\perp$ along
$\hat\mu:=\mu/r$, write $q:=|\theta|^2$, and decompose the momentum in the same
way, $m=m_\parallel\hat\mu+m_\perp$. The state of the algorithm is summarized
by
\[
X:=(\theta_\parallel,m_\parallel,R_\perp,V_\perp,C_\perp),\qquad
R_\perp=|\theta_\perp|^2,\quad V_\perp=|m_\perp|^2,\quad C_\perp=\langle\theta_\perp,m_\perp\rangle .
\]

The conditional moments of the gradient are Gaussian integrals of the two
class-conditional logits
\[
G_1:=\langle\theta,z\rangle+b^*\sim\mathcal N(b^*,q),\qquad
G_2:=\langle\theta,\mu+z\rangle+b^*\sim\mathcal N(\theta_\parallel r+b^*,q),
\]
which depend on $\theta$ only through $(\theta_\parallel,q)$. We use the
\emph{Stein coefficients}
\[
\begin{aligned}
\cA&:=(1-p)\,\E\sigma'(G_1)+p\,\E\sigma'(G_2), &
\mathcal B&:=p\,\bigl(\E\sigma(G_2)-1\bigr),\\
\mathcal D_0&:=(1-p)\,\E\sigma(G_1)^2+p\,\E(1-\sigma(G_2))^2, &
\mathcal D_\theta&:=(1-p)\,\E(\sigma^2)''(G_1)+p\,\E\bigl((1-\sigma)^2\bigr)''(G_2),
\end{aligned}
\]
all functions of $(\theta_\parallel,q)$. We also write
$f:=\cA\theta_\parallel+\mathcal Br$ for the mean gradient along the signal,
and $\Phi:=\eta(d-1)/(2B)$ for the additive noise level of the bulk.

The first observation is that the $d$-dimensional iteration is, exactly, a
Markov chain in five variables.

\begin{lemma}[Exact Markov reduction]\label{lem:LR0}
The sequence $(X_k)$ is a Markov chain.
\end{lemma}

\begin{proof}
The reduction is a consequence of symmetry. The law of $(x,y)$ is invariant
under the rotations $O\in O(d)$ that fix $\mu$, and the gradient is
equivariant, $g(O\theta;Ox,y)=Og(\theta;x,y)$, so the transition kernel of
$(\theta,m)$ is equivariant. The vector $X$ is a maximal invariant of the
action of $O(d-1)$ on $(\theta_\perp,m_\perp)$, since it consists of their
Gram matrix together with the components along $\hat\mu$, and therefore $X$
is itself a Markov chain. Indeed, if $X(\theta',m')=X(\theta,m)$, then
$(\theta',m')=O(\theta,m)$ for some $O$ fixing $\mu$, and equivariance gives
$\mathbb P_{(\theta',m')}(X_1\in A)=\mathbb P_{(\theta,m)}(X_1\circ O\in A)=\mathbb P_{(\theta,m)}(X_1\in A)$,
because $X$ is $O$-invariant; so the conditional law of $X_{k+1}$ given the
past depends only on $X_k$.
\end{proof}

The conditional drift of this chain is explicit, and its bulk part is exactly
the moment recursion of Section~\ref{sec:L1}.

\begin{lemma}[Exact conditional drift: the bulk is the moment recursion]\label{lem:LRdrift}
Let $\Psi(x):=\E[X_{k+1}\mid X_k=x]$. In the signal coordinates, $\Psi$ is the
heavy-ball map
\[
m_\parallel\mapsto\beta m_\parallel+\varepsilon f,\qquad
\theta_\parallel\mapsto\theta_\parallel-\eta(\beta m_\parallel+\varepsilon f).
\]
In the bulk coordinates, write
$\Sigma=\bigl(\begin{smallmatrix}R_\perp&\eta C_\perp\\ \eta C_\perp&\eta^2V_\perp\end{smallmatrix}\bigr)$.
Then $\Psi$ acts on $\Sigma$ exactly as the moment recursion \eqref{eq:L1}, with
the state-dependent parameters
\[
\lambda=\eta\varepsilon\cA,\qquad H_n=\frac{\eta(\mathcal D_\theta-\cA^2)}{2B\cA},\qquad
\varphi=\frac{\eta(d-1)\mathcal D_0}{2B\cA},
\]
all evaluated at the current $(\theta_\parallel,q)$.
\end{lemma}

\begin{proof}
It suffices to check Assumption~\ref{ass:G} for the bulk, conditionally on
the current state, and then to apply the proof of Lemma~\ref{lem:L1}. For a
single sample, Stein's lemma \citep{Stein1981} in the form
$\E[h(\langle\theta,z\rangle)z]=\E[h'(\langle\theta,z\rangle)]\,\theta$,
applied to each class, gives the conditional mean $\cA\theta+\mathcal B\mu$;
its component along $\hat\mu$ is $f$, which gives the signal map.

Stein's second identity,
$\E[h(\langle\theta,z\rangle)zz^\top]=\E[h]\,I+\E[h'']\,\theta\theta^\top$,
gives the second-moment matrix of a single per-sample gradient as
$\mathcal D_0I+\mathcal D_\theta\theta\theta^\top$ plus terms in
$\mu\mu^\top$ and $\theta\mu^\top+\mu\theta^\top$, which vanish when traced
against the projection onto $\hat\mu^\perp$. Averaging over the batch,
\[
\E\bigl[|g_\perp|^2\bigm|\theta\bigr]=\frac1B\bigl[(d-1)\mathcal D_0+R_\perp\mathcal D_\theta\bigr]+\frac{B-1}B\cA^2R_\perp .
\]
Hence $\E[|g_\perp-\cA\theta_\perp|^2\mid\theta]=\vinc R_\perp+\vadd$, with
$\vinc=(\mathcal D_\theta-\cA^2)/B$ and $\vadd=(d-1)\mathcal D_0/B$. This is
Assumption~\ref{ass:G} for the bulk, with target $0$ and with coefficients
depending on the current state, except that $\vinc$ may be negative; the sign
of $\vinc$ is not used in the proof of Lemma~\ref{lem:L1}, and the computation
there applies verbatim to one step.
\end{proof}

Lemma~\ref{lem:LRdrift} says that the bulk of logistic regression is the
moment recursion of least squares with label noise, and that the only difference
is that its parameters $(\lambda,H_n,\varphi)$ move with the state, through the
Stein coefficients. In the tame regime of Lemma~\ref{lem:B} they depend on
the state, to leading order, only through the curvature factor $\varrho$.

This difference is not cosmetic. For LS the parameters are constants, the
conditional identity can be averaged, and Lemma~\ref{lem:L1} is a closed
recursion for deterministic second moments. For LR the moments do not
close: the identity holds conditionally, realization by realization, but
its parameters are random through the state. To pass from the random chain
to its drift recursion we need a concentration statement,
Theorem~\ref{thm:A}, which LS does not need.

Lemma~\ref{lem:LRdrift} expresses the bulk parameters through the Stein
coefficients; the next lemma evaluates them. These coefficients are
Gaussian integrals of the logistic function and have no closed form. When
the rare class is rare and the logits are not too spread out, however, the
logistic function is effectively exponential on the relevant range, and the
coefficients become explicit up to a small relative error.

Recall the curvature factor
\[
\varrho=e^{(q-r^2)/2}=e^{(\theta_\parallel^2+R_\perp-r^2)/2},
\]
and note that $(1-p)e^{b^*+q/2}=p\varrho$. We measure the error by
\[
\epsilon_B:=p\Bigl[e^{2q-r^2}+e^{(3q-r^2)/2}+(1+\varrho)\,e^{r|\theta_\parallel|}\Bigr],
\]
and we say that the state is in the \emph{tame regime} when $\epsilon_B$ is
small.

\begin{lemma}[Tame coefficients]\label{lem:B}
If $\epsilon_B\le\frac12$, then
\[
\cA=p\varrho\,(1+O(\epsilon_B)),\qquad \mathcal B=-p\,(1+O(\epsilon_B)),\qquad
\mathcal D_0=p\,(1+O(\epsilon_B)),\qquad |\mathcal D_\theta|=O(p\,\epsilon_B),
\]
with absolute constants. For the first derivatives in $(\theta_\parallel,q)$
one has $\partial_q\cA=\tfrac12p\varrho\,(1+O(\epsilon_B))$ and
\[
|\partial_{\theta_\parallel}\cA|+|\partial(\mathcal B+p)|+|\partial(\mathcal D_0-p)|+|\partial\mathcal D_\theta|\le C(1+r)\,p\,\epsilon_B,\qquad\partial\in\{\partial_{\theta_\parallel},\partial_q\},
\]
with an absolute constant $C$; the second derivatives obey the same bounds
with $(1+r)^2$ in place of $(1+r)$.
The pure second derivative in $q$ keeps its main term,
$\partial_q^2\cA=\tfrac14p\varrho\,(1+O(\epsilon_B))$. At $q=0$, where the Gaussian
variance vanishes, the derivatives in $q$ are one-sided, on $q\ge0$.
\end{lemma}

\begin{proof}
Each coefficient is a Gaussian expectation of a function that equals
$e^t(1+O(e^t))$, so it suffices to compute exponential moments. We write
$\bar\mu$ for the mean of a Gaussian $G$ with variance $q$, so that
$\E e^{kG}=e^{k\bar\mu+k^2q/2}$, and we use the identity
$(1-p)e^{b^*+q/2}=p\varrho$ noted above. The four estimates follow from elementary bounds
on the logistic function.

For $\cA$, the bounds $e^t(1-2e^t)\le\sigma'(t)\le e^t$ give the main term
$p\varrho$ from class~1, with a correction of relative size
$e^{b^*+3q/2}\asymp p\,e^{(3q-r^2)/2}$, while class~2 contributes at most
$p\,\E\sigma'(G_2)\le p\varrho\cdot p\,e^{\theta_\parallel r}/(1-p)$. For
$\mathcal B$ and the class-2 part of $\mathcal D_0$, we use $0\le1-\sigma\le1$
and $\E\sigma(G_2)\le p\varrho e^{\theta_\parallel r}/(1-p)$. For the class-1
part of $\mathcal D_0$, we use $\E\sigma(G_1)^2\le\E e^{2G_1}\asymp p^2e^{2q-r^2}$.
For $\mathcal D_\theta$, we use $(\sigma^2)''=2\sigma\sigma'(2-3\sigma)$ and
$((1-\sigma)^2)''=-2\sigma'(1-\sigma)(1-3\sigma)$, which give
$|(\sigma^2)''|\le4\sigma^2$ and $|((1-\sigma)^2)''|\le2\sigma'$.

For the derivatives, we use
$\partial_q\E h(G)=\frac12\E h''(G)$ and $\partial_{\bar\mu}\E h(G)=\E h'(G)$,
together with the facts that $|\sigma'''(t)-e^t|+|\sigma^{(5)}(t)-e^t|\le Ce^{2t}$ (the
second gives the main term of $\partial_q^2\cA$), that every
derivative of $\sigma$ of order at least one is bounded by $C\sigma'\le Ce^t$,
and that every derivative of $\sigma^2$ is bounded by $C\sigma^2\le Ce^{2t}$; a
derivative in $\theta_\parallel$ produces the factor $r$ from the mean
$\theta_\parallel r+b^*$ of $G_2$.
\end{proof}

Substituting Lemma~\ref{lem:B} into Lemma~\ref{lem:LRdrift} gives the
parameters of the bulk chain in the tame regime. The curvature is
$\lambda=\eta\varepsilon p\varrho\,(1+O(\epsilon_B))$, the ambient temperature is
$\varphi=\frac{\Phi}{\varrho}(1+O(\epsilon_B))$, and the noise feedback is
\[
H_n=\frac{\eta(\mathcal D_\theta-\cA^2)}{2B\cA}
=O\Bigl(\frac\Phi d\Bigl(\frac{\epsilon_B}{\varrho}+p\varrho\Bigr)\Bigr)\to0,
\]
since $\eta/B=2\Phi/(d-1)$.

\emph{Logistic regression has no noise feedback.} The multiplicative part of the
bulk noise, $\mathcal D_\theta-\cA^2$, is too small to matter, so
$\ut\approx0$, and the curvature ceiling $H_c<1$ is the only stability
constraint on the bulk chain at frozen $\varrho$.
For the drift recursion with dynamic $\varrho$ the constraint is stronger:
its fixed point loses stability at a curvature ceiling $\lambda_c$, below $2$
when the bulk floor is positive (Proposition~\ref{prop:W2}), because the
signal and the bulk couple through $\varrho$.

The additive channel $\varphi$ remains, and it plays exactly the role of the
label noise of LS, with $\varsigma^2=1$ in $\vadd$ and with the curvature
$\cA=p$ replaced by $p\varrho$, so that $\varphi=\Phi/\varrho$. The state
enters the bulk parameters only through the curvature factor $\varrho$, which
multiplies the curvature and divides the ambient temperature. This is the
coupling between the signal and the bulk that the rest of the section
analyzes.

The concentration statement announced after Lemma~\ref{lem:LRdrift} is the
following general estimate: it applies
to any Markov chain whose drift map is smooth near a deterministic path,
whose linearization along that path does not amplify perturbations too much,
and whose increments satisfy a conditional Bernstein condition.

\begin{theorem}[Fluid limit]\label{thm:A}
Let $(X_k)$ be a Markov chain with values in a closed convex set
$\mathcal S\subset\R^n$, started from a deterministic state $X_0$, with drift map
$\Psi(x)=\E[X_{k+1}\mid X_k=x]$ and increments
$\xi_{k+1}:=X_{k+1}-\Psi(X_k)$, and let $x_{k+1}=\Psi(x_k)$ be the deterministic
recursion with $x_0=X_0$.
Fix a norm $\|\cdot\|$, a finite set $\mathcal L$ of
dual unit vectors such that $\|y\|\le\max_{\ell\in\mathcal L}|\ell(y)|$, a
horizon $K$, and a radius $\rho$. Write $\bar B(x,\rho)$ for the closed ball,
and let $U:=\mathcal S\cap\bigcup_{k\le K}\bar B(x_k,\rho)$.
Assume the following.
\begin{enumerate}
\item[(A1)] \emph{Smoothness.} There are linear maps $J_k$ such that
$\|\Psi(y)-\Psi(x_k)-J_k(y-x_k)\|\le\frac{L_2}2\|y-x_k\|^2$ for
$y\in\mathcal S\cap\bar B(x_k,\rho)$.
\item[(A2)] \emph{Stability of the linearization.}
$\|J_{k-1}\cdots J_j\|\le\Gamma$ for all $0\le j\le k\le K$.
\item[(A3)] \emph{Bernstein increments.} On the event $\{X_k\in U\}$, for
every dual unit vector $\ell$ and every $0\le s<1/M$,
\[
\E\bigl[e^{s\ell(\xi_{k+1})}\bigm|\mathcal F_k\bigr]\le\exp\Bigl(\frac{s^2v}{2(1-s M)}\Bigr).
\]
\end{enumerate}
Let $\mathsf L:=\log(2|\mathcal L|K/\delta)$ and
$s_*:=\Gamma\bigl(\sqrt{2Kv\mathsf L}+2M\mathsf L\bigr)$. If
$2s_*\le\rho$ and $2K\Gamma L_2s_*\le1$, then
\[
\mathbb P\Bigl(\max_{k\le K}\|X_k-x_k\|\le2s_*\Bigr)\ge1-\delta .
\]
\end{theorem}

\noindent The deviation of the chain from its drift recursion is of the
order of $\Gamma\sqrt{Kv\log(|\mathcal L|K/\delta)}$, the size of a
martingale with $K$ increments of variance $v$, amplified by at most
$\Gamma$.

The constant that matters is the norm $\Gamma$ of the products of
linearizations along the deterministic path, not a global Lipschitz
constant. When the drift recursion is contracting or neutral along the
path, $\Gamma$ stays bounded, and the bound does not grow exponentially with
the horizon as a Gr\"onwall argument would, although the condition
$2K\Gamma L_2s_*\le1$ still limits the horizon polynomially. The price
is the hypothesis (A2), which must be verified separately in each cell; this
is where the dynamics of the drift enters.

Since $\mathcal S$ is closed and convex, $\Psi$ maps $\mathcal S$ into itself,
so the path $(x_k)$ stays in $\mathcal S$. Hypothesis (A1) is asked only on
$\mathcal S$. For LR this matters: there $\mathcal S$ is the set of physical
states, on which the Gaussian variance $q=\theta_\parallel^2+R_\perp$ is
nonnegative, and near the cold start $J_k$ is a one-sided derivative. The
balls are closed because the induction below reaches distance exactly $\rho$.

\begin{proof}
We linearize around the deterministic path. The error is then the sum of the
martingale increments and of the quadratic remainders, each propagated by
the linearization. The martingale part is controlled by Bernstein's
inequality and a union bound, and the quadratic part is absorbed by an
induction, as long as the error stays in the region where (A1) holds.

Let $\mathcal T$ be the exit time from $U$, and set $D_k:=X_k-x_k$ and
$Q_j:=\Psi(X_j)-\Psi(x_j)-J_jD_j$. Since $D_{j+1}=J_jD_j+Q_j+\xi_{j+1}$ and
$D_0=0$, unrolling gives
\[
D_k=\sum_{j<k}J_{k-1}\cdots J_{j+1}\,(\xi_{j+1}+Q_j).
\]

Fix $k$ and $\ell\in\mathcal L$. The sum
\[
S^{k,\ell}_i:=\sum_{j<i}\ell\bigl(J_{k-1}\cdots J_{j+1}\xi_{j+1}\bigr)\mathbf 1\{j<\mathcal T\}
\]
is a martingale in $i$, with deterministic weights of dual norm at most
$\Gamma$. Here $\{j<\mathcal T\}\in\mathcal F_j$, and the functional
$\ell\circ J_{k-1}\cdots J_{j+1}$ is $\gamma_j\ell_j$ with $\ell_j$ a dual unit
vector and $0\le\gamma_j\le\Gamma$, so (A3) can be applied with
$s\gamma_j$ in place of $s$ whenever $0\le s<1/(\Gamma M)$.
By (A3), the process
$\exp\bigl(s S_i-s^2i\Gamma^2v/(2(1-s\Gamma M))\bigr)$ is a
supermartingale. Optimizing over $s$ gives Bernstein's bound
\citep[Section~2.4 and Corollary~2.11]{BLM2013} (for martingales, compare
\citealp{Freedman1975}),
\[
\mathbb P\bigl(|S^{k,\ell}_k|\ge t\bigr)\le2\exp\Bigl(-\frac{t^2}{2(K\Gamma^2v+\Gamma Mt)}\Bigr),
\]
which is below $\delta/(K|\mathcal L|)$ at $t=s_*$. A union bound over
$(k,\ell)$ gives an event of probability at least $1-\delta$ on which all
the stopped sums are at most $s_*$.

On this event we argue by induction on $k$. Suppose that
$\|D_j\|\le2s_*\le\rho$ for all $j<k$. Then $\mathcal T\ge k$, so the
stopped and unstopped sums agree, and by (A1) and (A2),
\[
\|D_k\|\le s_*+\sum_{j<k}\Gamma\,\tfrac{L_2}2(2s_*)^2\le s_*+2K\Gamma L_2s_*^2\le2s_* .
\]
\end{proof}

To apply Theorem~\ref{thm:A} to logistic regression we must choose
coordinates in which the increments are of comparable size and check the
Bernstein condition (A3).

\begin{lemma}[LR increments]\label{lem:LRinc}
Use the matched coordinates
\[
\zeta=\Bigl(\theta_\parallel,\ \frac{\eta m_\parallel}{\varepsilon},\ R_\perp,\ \frac{\eta C_\perp}{\varepsilon},\ \frac{\eta^2V_\perp}{\varepsilon^2}\Bigr)
\]
with the sup norm, and suppose that $\epsilon_B\le\frac12$ and $\eta\cA\le a$ on $U$, for a
constant $a$. Then (A3) holds, with a constant $C$ that depends only on $a$, with
\[
v\le C\Bigl[\frac{\eta^2p}{B}\Bigl(1+r^2+R_\perp+\frac{\eta^2V_\perp}{\varepsilon^2}\Bigr)+\Phi^2\eta^2p^2\Bigl(\frac1{Bp}+\frac2d\Bigr)\Bigr],
\]
\[
M\le C\,\frac\eta B\Bigl(1+r+\sqrt{R_\perp}+\frac{\eta\sqrt{V_\perp}}{\varepsilon}\Bigr)(1+\Phi+\eta).
\]
\end{lemma}

\begin{proof}
Write $\ell_i:=\langle\theta,x_i\rangle+b^*$ for the logit of the $i$-th
sample of the batch and $c_i:=\sigma(\ell_i)-y_i$ for its residual, so that
$g=B^{-1}\sum_ic_ix_i$. Fix an orthonormal frame of the span of $\mu$,
$\theta_\perp$ and $m_\perp$ (for $d\ge3$); the logit $\ell_i$ involves only
the label and the three coordinates of the noise vector $z_i$ in this frame,
and the component $g_{\mathrm{rest}}$ of $g$ orthogonal to the frame is,
given $(c_i)$, distributed as $\mathcal N(0,B^{-2}\sum_ic_i^2\,I_{d-3})$.

Each coordinate of $\xi$ is a sum of two kinds of term. The first kind is a
batch average of independent centred terms of the form
$\frac\eta Bc_i\times(\text{Gaussian})$, with $|c_i|\le1$ and
$\E c_i^2=\mathcal D_0\le2p$; by Bernstein's moment condition \citep[Theorem~2.10]{BLM2013} their batch
average is sub-gamma with variance factor $Cp(1+r^2)/B$ and scale
$C(1+r)/B$ (for class~$0$ one uses $c_i^2\le e^{2\ell_i}$ and an exponential
tilt, which costs the factor $p\,\epsilon_B$).

The second kind is a multiple of $|g_\perp|^2-\E|g_\perp|^2$. Given $(c_i)$,
$|g_{\mathrm{rest}}|^2$ is a $\chi^2_{d-3}$ variable scaled by
$\sum_ic_i^2/B^2$, so this term is a centred scaled $\chi^2_{d-3}$ variable, plus a
multiple of $\sum_ic_i^2-\E\sum_ic_i^2$. Both are sub-gamma (a $\chi^2_{d-3}$
variable is a Gamma variable with shape $(d-3)/2$ and scale $2$, see
\citep[Section~2.4]{BLM2013}; its scale parameter here is
$2\sum_ic_i^2/B^2\le2/B$), and
\[
\operatorname{Var}|g_\perp|^2\le C\Bigl(\frac{dp}B\Bigr)^2\Bigl(\frac1{Bp}+\frac2d\Bigr).
\]

The square $|g_\perp|^2$ also contains the cross term
$2\langle\cA\theta_\perp,\,g_\perp-\cA\theta_\perp\rangle$ between the mean
$\E[g_\perp\mid\mathcal F_k]=\cA\theta_\perp$ of the bulk gradient and its fluctuation. It
is of the first kind, with a standard Gaussian and the extra factor
$2\cA|\theta_\perp|=2\cA\sqrt{R_\perp}$. With a weight of at most $\eta^2$, it adds at most
$C(\eta\cA)^2\,\eta^2pR_\perp/B$ to $v$ and $C\eta\cA\cdot\eta\sqrt{R_\perp}/B$ to $M$, which
the displayed bounds absorb because $\eta\cA\le a$.

In the coordinates $\zeta$, the weights of the first kind of term are
$(\eta\varepsilon,\eta,2\eta\varepsilon,\eta,2\beta\eta^2/\varepsilon)$, and the
weights of the second kind are $(0,0,\eta^2\varepsilon^2,\eta^2\varepsilon,\eta^2)$.
The largest contribution of the second kind is therefore
$\eta^4\operatorname{Var}|g_\perp|^2$, and
$\eta^4(dp/B)^2\le4\Phi^2\eta^2p^2(1+O(d^{-1}))$ because $\eta d/B=2\Phi\,d/(d-1)$.
\end{proof}

\begin{corollary}[Fluid limit in cells 2--4]\label{cor:fluid}
Fix $T,\bar r,\bar\Phi,\bar\omega^2,W,C_0>0$. Take the horizon $K=T/\vartheta$
on the clock $\vartheta:=\min(\eta p,\varepsilon)$, and let
$\zeta^{\mathrm{det}}$ be the drift recursion from $\zeta_0$. Suppose that
$r\le\bar r$, $\Phi\le\bar\Phi$, $\omega^2\le\bar\omega^2$ and $K\le d^{C_0}$,
that $\|\zeta^{\mathrm{det}}_k\|_\infty\le W$ for $k\le K$, that
$\epsilon_B\le\frac12$ at every physical state within distance $1$ of this
path, and that $\eta=o(\sqrt{d/\log d})$. Then there is $C$, depending only on
$T,\bar r,\bar\Phi,\bar\omega^2,W,C_0$, such that for $d$ large enough, with
probability at least $1-d^{-10}$,
\[
\max_{k\le K}\|\zeta_k-\zeta^{\mathrm{det}}_k\|_\infty\le C\sqrt{\frac{\log d}{d}} .
\]
\end{corollary}

\noindent The clock $\vartheta$ is the learning clock in cell~2 and the
retention clock in cells~3 and~4.

The constant depends on the path and not only on $T$, because tameness alone
does not bound the state: $\epsilon_B\le\frac12$ allows $q$ of order
$\log(1/p)$, and the variance factor $v$ of Lemma~\ref{lem:LRinc} grows with
$R_\perp$ and $\eta^2V_\perp/\varepsilon^2$.

Under the co-scaling the horizon condition $K\le d^{C_0}$ holds for some
$C_0$, since $\vartheta$ is a power of $d$. The condition on $\eta$ holds exactly when
$\alpha_\eta>-\frac12$. Since admissibility gives $\alpha_\eta\ge1-\sigma$, it
can fail only for batches with $\sigma\ge\frac32$.

Corollary~\ref{cor:fluid} reduces the study of LR over finite horizons in
cells~2--4 to the deterministic recursion $\zeta^{\mathrm{det}}$, with an error
of order $\sqrt{\log d/d}$, which is the fluctuation scale of a
high-dimensional statistic. The matched coordinates $\zeta$ are essential
here: they measure each variable through its effect on the error, and in
them all five increments have the same order.

\begin{proof}
We verify the hypotheses of Theorem~\ref{thm:A} with $\mathcal S$ the set of
physical states, $\rho=1$, $\delta=d^{-10}$, and $\mathcal L$ the ten
functionals $\pm\zeta\mapsto\zeta_i$. Throughout, $C$ denotes a constant
depending only on $T,\bar r,\bar\Phi,\bar\omega^2,W,C_0$. On $U$ we have
$R_\perp+\eta^2V_\perp/\varepsilon^2\le2(W+1)$, and
$\eta p\le C$ because $\omega^2\le\bar\omega^2$. Also $|\theta|^2=\theta_\parallel^2+R_\perp\le C$
on $U$, so $\varrho\le C$, and $\eta\cA\le C$ by Lemma~\ref{lem:B}. By Lemma~\ref{lem:LRinc},
together with $\eta/B=2\Phi/(d-1)$ and $\vartheta\ge\eta p\min(1,\omega^{-2})$,
\[
Kv\le CT\max(1,\omega^2)\,\frac{(1+\Phi)^3(1+r^2+W)}{d}\le\frac Cd,\qquad M\le\frac{C(1+\eta)}{d}.
\]
Since $K\le d^{C_0}$, we have $\mathsf L=\log(20Kd^{10})\le C\log d$. Hence
\[
s_*\le C\Gamma\Bigl(\sqrt{\frac{\log d}d}+\frac{(1+\eta)\log d}{d}\Bigr)\le C\Gamma\sqrt{\frac{\log d}d}
\]
for $d$ large, because $\eta=o(\sqrt{d/\log d})$. It remains to show that
$KL_2\le C$ and $\Gamma\le C$. Then $2s_*\le\rho$ and
$2K\Gamma L_2s_*\le1$ for $d$ large, and Theorem~\ref{thm:A} gives the
claim.

In cells~3--4, the drift map has the form
$\zeta\mapsto\zeta+\varepsilon b(\zeta)+O(\varepsilon^2)$ with $b$ smooth (see
the proof of Proposition~\ref{prop:LR34}), so (A1) and (A2) hold with
$L_2\le C\varepsilon$ and $\Gamma\le e^{CT}$.

In cell~2, the state-dependent coefficients enter the drift map multiplied
by $\eta$, so its Hessian on $U$ is $O(\eta p(1+\Phi)(1+r)^2)$, and (A1) holds
with $L_2\le C\eta p$; since $K=T/(\eta p)$, $KL_2\le CT$.

For (A2) in cell~2, the point is that the fast coordinates are contracting
and transfer perturbations to the slow ones with gain of order one; for
instance, a kick $\delta$ to $\eta m_\parallel/\varepsilon$ moves
$\theta_\parallel$ by $\sum_j\beta^j\varepsilon\delta=\delta$. To make this
precise, let $z:=\eta p$, so that $K=T/z$ and $\omega^2=z/\varepsilon\to0$. In
the coordinates $\zeta$ the Jacobian of the drift map is
$J_k=\mathcal P_k\Lambda_k\mathcal P_k^{-1}+E_k$. The first term is block
diagonal: a signal block
$\bigl(\begin{smallmatrix}1-\varepsilon a_k&-\varepsilon\beta\\a_k&\beta\end{smallmatrix}\bigr)$
with $a_k:=\eta\,\partial_{\theta_\parallel}f(\zeta_k)=O(z)$, and a bulk block,
the linear part of the moment recursion \eqref{eq:L1} at the frozen parameters of
step $k$, written in $(R_\perp,\eta C_\perp/\varepsilon,\eta^2V_\perp/\varepsilon^2)$.
The remainder $E_k$ collects the derivatives of the state-dependent
parameters and satisfies $\|E_k\|\le Cz$.

The signal block has eigenvalues
$1-a_k+O(za_k)$ and $\beta+O(z)$ with eigenvectors $(1,O(\omega^2))$ and
$(\beta+O(\omega^2),1)$, whose determinant is $1-O(\omega^2)$, and the bulk block has the same structure with $a_k$ replaced by
$\eta\cA_k=O(z)$: its eigenvalues are the products of the eigenvalues of
$\bigl(\begin{smallmatrix}1-\varepsilon\eta\cA_k&-\varepsilon\beta\\\eta\cA_k&\beta\end{smallmatrix}\bigr)$,
and its eigenvectors are the symmetric tensor products of that matrix's
eigenvectors. Hence
$\|\mathcal P_k\|,\|\mathcal P_k^{-1}\|\le C$ and $\|\Lambda_k\|\le1+Cz$.
The eigenvectors depend on $a_k$ with derivative $O(1/\varepsilon)$, and
$|a_{i+1}-a_i|\le Cz\|\zeta_{i+1}-\zeta_i\|$, so
$\|\mathcal P_{i+1}-\mathcal P_i\|\le C\omega^2\|\zeta_{i+1}-\zeta_i\|$.

The path does not move by $O(z)$ per step when the initial momentum is
warm: the momentum-type coordinates relax at rate $\varepsilon$, and during
this initial layer $\theta_\parallel$ moves by $O(\varepsilon)$ per step. The
total variation of the path is nevertheless bounded. Each momentum-type
coordinate is within $O(\beta^iW+\omega^2)$ of its quasi-steady value, determined by the slow
coordinates, so each step moves $\zeta$ by $O(\varepsilon\beta^iW+z)$. Summing
over $i<K=T/z$,
\[
\sum_{i<K}\|\zeta_{i+1}-\zeta_i\|\le C(W+T),\qquad
\sum_{i<K}\|\mathcal P_{i+1}-\mathcal P_i\|\le C\omega^2(W+T).
\]
Writing the product $J_{k-1}\cdots J_j$ in the moving bases, the $i$-th
factor has norm at most $1+Cz+C\|\mathcal P_{i+1}-\mathcal P_i\|$, and
therefore
\[
\|J_{k-1}\cdots J_j\|\le C^2\exp\bigl(C(k-j)z+C\omega^2(W+T)\bigr)\le C^2e^{C(1+\omega^2)(W+T)}=:\Gamma .
\]
\end{proof}

The following remark records what the corollary does not give.

\begin{remark}[What the fluid limit does and does not say]
The bound is uniform over finite horizons, and in cells~3 and~4 it extends
to much longer ones. There Corollary~\ref{cor:recursion} gives the decay of
the products in (A2), $\|J_{k-1}\cdots J_j\|\le\Gamma e^{-c(k-j)\vartheta}$
(the coordinates $y$ of that corollary and the matched coordinates $\zeta$
differ by a bounded diagonal change of scale). With this decay, the proof of
Theorem~\ref{thm:A} goes through with $K\Gamma^2v$ replaced by
$\Gamma^2v/(c\vartheta)$ and $K\Gamma L_2$ by $\Gamma L_2/(c\vartheta)$, and
gives a bound of the same order for $K=d^{C'}$ steps, with $\log K$ in place
of $\log d$.

Uniformity on all of $[0,\infty)$ is false: over an infinite horizon the
ergodic chain eventually makes every deviation that has positive
probability under its stationary law.

In the memoryless sub-case $Bp\ll\varepsilon$, the momentum is spiky. Then
$V_\perp$ does not concentrate in its natural units, but it does concentrate
in the matched units of $\zeta$ (weight $\eta^2/\varepsilon^2$), which
measure its effect on $R_\perp$. Numerically, in this sub-case the
run-to-run spread of $V_\perp$ is of the order of its mean, while that of
$R_\perp$ is small.

In the windows, the free rotation is an isometry in the energy norm.
Heuristically, however, a perturbation $\delta\varrho$ of the curvature factor
shifts the phase by about $\bar\omega\,\delta\varrho\,\tau$, so that
$\Gamma\lesssim1+\bar\omega\sup_\tau\tau|Z(\tau)|$. Tracking the phase
accurately therefore requires $\bar\omega^2\ll d/\log d$. The slow energies are
insensitive to the phase.
\end{remark}

With the fluid limit in hand, the LR question becomes a question about the
deterministic drift recursion, which we now classify by cell.

\paragraph{The LR dictionary.} Under the co-scaling~\eqref{eq:coscaling-main},
with the learning-rate exponent written $\alpha_\eta$ (so that
$\eta=\eta_*d^{-\alpha_\eta}$) and with $r$ fixed, the parameters of the bulk
chain in the tame regime scale as
\[
\begin{aligned}
&\omega^2=\frac{\eta p\varrho}{\varepsilon}\asymp d^{\gamma-\kappa-\alpha_\eta},\qquad
\lambda=\eta\varepsilon p\varrho\asymp d^{-\alpha_\eta-\kappa-\gamma},\\
&\varphi=\frac\Phi\varrho\asymp d^{1-\sigma-\alpha_\eta},\qquad
H_n=O(d^{-\alpha_\eta-\kappa-\sigma})\to0 .
\end{aligned}
\]
These are the exponents of LS, with $e_\varphi=1-\sigma-\alpha_\eta$ in place
of $e_H$. Admissibility is again $\alpha_\eta\ge\alpha_c=\max(1-\sigma,-\kappa-\gamma)$,
but the meaning of $e_\varphi=0$ has changed: it is a constraint on the
floor, not on stability. If $\varphi\to\infty$, the bulk equilibrium
$R_\perp^*$ tends to infinity and the state leaves the tame regime, while
the only stability constraint is $\lambda\lesssim1$.

The cells are those of Section~\ref{sec:dict} with $\ut\equiv0$. In cell~2
the limit is the planar slow system of Proposition~\ref{prop:LR2}; in
cells~3 and~4 it is the five-dimensional system with dynamic $\varrho$ of
Proposition~\ref{prop:LR34}, with a rank-one bulk in cell~3, and this system
is globally stable (Proposition~\ref{prop:S}).

In the windows reached with vanishing curvature, cells~5 and~6, the limit
is the same system with $\omega^{*2}\to\infty$. Its rates near the equilibrium
are those of Proposition~\ref{prop:W1}, and at each fixed $\omega^{*2}$ the drift
recursion converges (Corollary~\ref{cor:recursion}). Whether this
convergence is uniform as $\omega^{*2}\to\infty$ is not proved; this is the one
gap in cells~5 and~6.

In the windows with curvature of order one, cells~7 and~8, the fixed point
loses stability at a curvature ceiling $\lambda_c$ (Proposition~\ref{prop:W2}).
On the switch, cell~8, the bulk floor is positive and $\lambda_c<2$. In cell~7,
where $\Phi^*=0$, the floor vanishes in the limit and $\lambda_c$ tends to
$\min(2,4/(1+\theta^{*2}))$. Convergence below the ceiling from a warm
start, one close to the fixed point, is Assumption~\ref{ass:W2}. Since $\ut\approx0$, the switch changes
the window only through the floor, $\varphi\mapsto\varphi/(1-H_c)$, and
through the ceiling, which depends on $R^*$.

\begin{remark}[Comparison with a region-by-region analysis of LR]\label{rem:lrregions}
A region-by-region analysis of LR splits the plane at the resonance line
$\gamma=1-\sigma+\kappa$ and at $\kappa=\sigma-1$, and uses
$\alpha_\eta=\gamma-\kappa$ above resonance and $\alpha_\eta=1-\sigma$ below
it. In the present classification, the region below resonance with
$\kappa>\sigma-1$ is cell~2. The two regions above resonance, on either
side of $\kappa=\sigma-1$, are the same cell~3: their split compares the
single-step signal $p^2\varrho^2R_\perp$ with the noise $dp/B$, a ratio that
does not enter the per-step chain (Lemma~\ref{lem:LRdrift}), so
$\kappa=\sigma-1$ is not a boundary (Lemma~\ref{lem:nonboundary}). The
resonance line is cell~4, and the windows, cells~5--8, have no counterpart.
The floor above resonance is $\varphi\asymp d^{1-\sigma+\kappa-\gamma}$, which
also vanishes as $d\to\infty$.
\end{remark}

In cells~3 and~4 the learning and retention clocks coincide, and the limit
couples the bulk, which is the LS continuum system with no noise feedback, to a
two-dimensional heavy-ball signal through the curvature factor $\varrho$.

\begin{proposition}[Cells 3--4: dynamic-$\varrho$ limit]\label{prop:LR34}
Let $\omega_{\mathrm{nom}}^2:=\eta p/\varepsilon\to\omega^{*2}\in(0,\infty)$,
$\Phi\to\Phi^*$, $\varepsilon\to0$ and $\epsilon_B\to0$ along the path. In the
units $y=(\theta_\parallel,m_\parallel/p,R_\perp,V_\perp/p^2,C_\perp/p)$ on
the clock $\tau=\varepsilon k$, the drift recursion $x_{k+1}=\Psi(x_k)$
satisfies, for every $T>0$,
\[
\sup_{k\le T/\varepsilon}|y_k-y(\tau_k)|\le C_T\Bigl(\varepsilon+|\omega_{\mathrm{nom}}^2-\omega^{*2}|+|\Phi-\Phi^*|+\sup\epsilon_B\Bigr),
\]
where $y=(\theta,Y,R,V,C)$ starts from the same state, $y(0)=y_0$, and solves
\begin{equation}\label{eq:LR5}
\theta'=-\omega^{*2}Y,\quad Y'=-Y+\varrho\theta-r,\quad
R'=-2\omega^{*2}C,\quad V'=-2V+2\varrho C+\tfrac{2\Phi^*}{\omega^{*2}},\quad C'=-C+\varrho R-\omega^{*2}V,
\end{equation}
with $\varrho=e^{(\theta^2+R-r^2)/2}$. In cell~3 ($\Phi^*=0$), a rank-one bulk
start stays of rank one, and $R=X^2$ with $X'=-\omega^{*2}Y_\perp$ and
$Y_\perp'=-Y_\perp+\varrho(\tau)X$.
\end{proposition}

\noindent For $\varrho\equiv1$ the bulk equations of \eqref{eq:LR5} are
\eqref{eq:L3} at $(\omega^{*2},0,\Phi^*)$.

The coupling through $\varrho$ runs in both directions. The bulk second
moment $R$ raises $\varrho$, which increases the curvature seen by both the
signal and the bulk, and the signal raises $\varrho$ through $\theta^2$.
Because $\varrho$ moves on the same clock as everything else, it cannot be
replaced by a constant (Remark~\ref{rem:frozen}).

\begin{proof}
This is a consistency-plus-stability argument for a one-step map. Expanding
Lemma~\ref{lem:LRdrift} with Lemma~\ref{lem:B} shows that, in these units,
the drift map is
\[
y\mapsto y+\varepsilon b(y)+\varepsilon^2c(y)+O(\varepsilon\,\epsilon_B),
\]
where $b$ is the vector field of \eqref{eq:LR5} and $c$ is smooth. For
instance, the noise term in the equation for $V_\perp/p^2$ is
$\varepsilon^2\E[|g_\perp|^2\mid\theta]/p^2=\varepsilon\cdot2\Phi/\omega_{\mathrm{nom}}^2+O(\varepsilon^2)$,
since $2\Phi/\omega_{\mathrm{nom}}^2=(d-1)\varepsilon/(Bp)$. The local error is
therefore $O(\varepsilon^2)$ per step, and the one-step map is Lipschitz on
compact sets, so the discrete Gr\"onwall inequality gives the bound (the
standard convergence argument for one-step methods,
\citep[Section~II.3]{HNW1993}), with $C_T$ depending on the solution of
\eqref{eq:LR5} on $[0,T]$.

The bulk stays of rank one in cell~3 because,
without noise, the bulk is the second moment of a linear, time-varying
system: with $Q:=RV-C^2$ one finds $Q'=-2Q$ from \eqref{eq:LR5} with
$\Phi^*=0$, since the terms in $\varrho$ cancel.
\end{proof}

\begin{remark}[$\varrho$ must not be frozen]\label{rem:frozen}
Freezing $\varrho$ at its equilibrium value $\varrho^*$ is a natural
simplification on the resonance line, since the bulk then becomes the LS
system \eqref{eq:L3} with constant parameters. Numerically it is wrong by
an amount of order one: on the resonance line, the signal transient of
\eqref{eq:LR5} and that of the same system with $\varrho$ frozen at
$\varrho^*$ differ by a fixed fraction of $r$, uniformly as
$\varepsilon\to0$.
\end{remark}

In cell~2 the momentum is fast, as in Corollary~\ref{cor:tikhonov}, and the
five-dimensional recursion collapses to a planar system on the learning
clock.

\begin{proposition}[Cell 2: slow system, globally stable]\label{prop:LR2}
Let $\omega^2=\eta p/\varepsilon\to0$, $\eta p\to0$, $\Phi$ fixed and
$\epsilon_B\to0$, with a cold start, and consider the drift recursion
$x_{k+1}=\Psi(x_k)$. On the clock $t=\eta p\,k$,
\[
\sup_{k\le T/(\eta p)}\Bigl(|\theta_{\parallel,k}-\theta(t_k)|+|R_{\perp,k}-R(t_k)|\Bigr)\le C_T(\omega^2+\eta p+\sup\epsilon_B),
\]
where
\[
\theta'=r-\varrho\theta,\qquad R'=2\Phi-2\varrho R .
\]
Every trajectory of this slow system converges to the unique equilibrium of
Proposition~\ref{prop:D}.
\end{proposition}
\begin{figure}[t]
\centering
\includegraphics[width=0.55\textwidth]{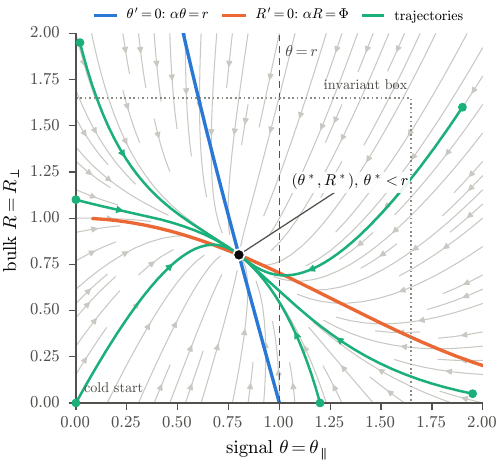}
\caption{Phase portrait of the cell-2 slow system of
Proposition~\ref{prop:LR2}, $\theta'=r-\varrho\theta$, $R'=2\Phi-2\varrho R$ with
$\varrho=e^{(\theta^2+R-r^2)/2}$, for $r=1$ and $\Phi=1$, on the learning clock
$t=\eta pk$. Here $\theta$ is the signal overlap $\theta_\parallel$ and $R$ is
the bulk second moment $R_\perp$. The nullclines $\varrho\theta=r$ and
$\varrho R=\Phi$ cross exactly once, at the equilibrium $(\theta^*,R^*)$ of
Proposition~\ref{prop:D}, and every trajectory, including the one from the
cold start $(0,0)$, converges to it, as the Bendixson--Dulac and
Poincar\'e--Bendixson argument in the proof of Proposition~\ref{prop:LR2}
guarantees. The dotted box is forward invariant. The equilibrium lies to the
left of the dashed line $\theta=r$: the bulk second moment raises $\varrho$,
and the equilibrium condition $\varrho\theta=r$ then forces $\theta^*<r$, so
the equilibrium is biased against alignment (Proposition~\ref{prop:D}(i)).}
\label{fig:lrslow}
\end{figure}

\noindent The hypotheses do not require $\varepsilon\to0$: only the ratio
$\omega^2$ and the per-step learning rate $\eta p$ must be small.
Along the co-scaling, $\Phi=\eta(d-1)/(2B)$ is not fixed but converges to a
limit $\Phi^*$. The bound of Proposition~\ref{prop:LR2} then holds for the slow
system with $\Phi^*$ in place of $\Phi$, with $|\Phi-\Phi^*|$ added to its
right-hand side: the slow vector field is affine in $\Phi$, and its solutions
stay in a bounded set on $[0,T]$, so Gronwall's inequality bounds the effect of
replacing $\Phi$ by $\Phi^*$ by $C_T|\Phi-\Phi^*|$.

The slow system says that, in cell~2, the signal relaxes toward $r/\varrho$
and the bulk toward $\Phi/\varrho$, at rates $\varrho$ and $2\varrho$ on the
learning clock, and that the two are coupled only through $\varrho$.
Figure~\ref{fig:lrslow} shows the phase portrait of the slow system, its two
nullclines and the equilibrium of Proposition~\ref{prop:D}.

\begin{proof}
For the signal, the mean momentum is an exponential average of past mean
gradients, and since $\omega^2$ is small this average is close to the current
gradient. For the bulk, the fast pair formed by the cross moment and the
momentum moment relaxes, on the retention clock, to a fixed point
determined by $R_\perp$, and at that fixed point the increment of $R_\perp$ is an
Euler step of the slow equation. Global stability is a planar argument.

\emph{Signal.} Unrolling the momentum gives
\[
\tilde m_{\parallel,k+1}:=\eta m_{\parallel,k+1}=\eta f_k(1-\beta^{k+1})+\beta^{k+1}\tilde m_{\parallel,0}+\eta\varepsilon\sum_n\beta^n(f_{k-n}-f_k).
\]
Since $|f_{k-n}-f_k|\le Cnpz$, the last sum is $O(\eta p\omega^2)$. After an
initial layer of $O(1/\varepsilon)$ steps, during which the total
displacement is $O(\eta p/\varepsilon)=O(\omega^2)$, we therefore have
\[
\theta_{\parallel,k+1}=\theta_{\parallel,k}-\eta p\bigl[(\varrho\theta_\parallel-r)+O(\omega^2+\epsilon_B)\bigr],
\]
which is an Euler scheme with step $\eta p$ for the slow system.

\emph{Bulk.} Write $a_k:=\eta\cA_k=\eta p\,\varrho_k(1+O(\epsilon_B))$,
$c_k:=\eta C_{\perp,k}/\varepsilon$, $v_k:=\eta^2V_{\perp,k}/\varepsilon^2$ and
$n_k:=\eta^2\E[|g_\perp|^2\mid x_k]=2a_k(\varphi_k+H_{n,k}R_k)+a_k^2R_k$.
The bulk part of Lemma~\ref{lem:LRdrift} reads, exactly,
\[
\begin{aligned}
v_{k+1}&=\beta^2v_k+2\beta a_kc_k+n_k,\qquad
c_{k+1}=\beta c_k+a_kR_k-\varepsilon v_{k+1},\\
R_{k+1}&=R_k-2\varepsilon\beta c_k-2\varepsilon a_kR_k+\varepsilon^2v_{k+1}.
\end{aligned}
\]
For frozen $(R,a,n)$ the fast pair $(c,v)$ has the fixed point
\[
v_*=\frac{n+2\beta a^2R/\varepsilon}{1-\beta^2+2\beta a},\qquad c_*=\frac{aR}{\varepsilon}-v_*,
\]
and there the increment of $R$ is exactly
$-2aR+\varepsilon(1+\beta)v_*=2a(\varphi-R)+O\bigl(a(\omega^2+H_n+a)(R+\varphi)\bigr)$,
because $a/\varepsilon=O(\omega^2)$.

The linear part of the fast map has
determinant $\beta^3$ and real eigenvalues $\beta+O(a)$ and $\beta^2+O(a)$,
separated by about $\beta\varepsilon$, so its powers are bounded by
$C\beta^k$. Along the recursion, $(c_*,v_*)$ moves by $O(\eta p\omega^2)$ per step
(since $R$ and $a/(\eta p)$ move by $O(\eta p)$ and $\partial c_*/\partial R=O(\omega^2)$,
$\partial c_*/\partial a=O(1/\varepsilon)$), and from the cold start the
initial deviation is $O(\omega^2)$. Hence the deviation $\delta_k$ of
$(c_k,v_k)$ from $(c_*,v_*)$ satisfies $|\delta_k|\le C\omega^2\beta^k+C\omega^4$,
and it changes the increment of $R$ by $O(\varepsilon|\delta_k|)$, whose sum
over $k\le T/(\eta p)$ is $O(\omega^2(1+T))$. Since $2a_k\varphi_k=2\eta p\Phi(1+O(\epsilon_B))$,
the recursion for $R_\perp$ is an Euler scheme with step $\eta p$ for
$R'=2\Phi-2\varrho R$, with local errors $O(\eta p(\omega^2+\eta p+\epsilon_B))$ and a
perturbation of total size $O(\omega^2(1+T))$; the discrete Gr\"onwall
inequality \citep[Section~II.3]{HNW1993}, applied to the signal and the bulk
together, gives the bound.

\emph{Global stability.} The divergence of the vector field is
$-\varrho(3+\theta^2+R)<0$, so by Bendixson's criterion \citep[Section~3.9, Theorem~1]{Perko2001} there
are no periodic orbits. The box
\[
\bigl\{0\le\theta\le\max(\theta_0,re^{r^2/2}),\ 0\le R\le\max(R_0,\Phi e^{r^2/2})\bigr\}
\]
is forward invariant, since $\varrho\ge e^{-r^2/2}$ and $\theta'=r>0$ at
$\theta=0$.
An orbit with $\theta_0<0$ reaches $\theta=0$ in finite time, because
$\theta'=r-\varrho\theta\ge r$ while $\theta\le0$; from then on it stays in the
box built from its state at that time (and $R\ge0$ throughout), so it suffices
to consider orbits that start in such a box.
The equilibrium is unique, and it is a sink: the Jacobian there
has determinant $\varrho^2(2+R+2\theta^2)>0$ and negative trace. By the generalized Poincar\'e--Bendixson theorem
\citep[Section~3.7, Theorem~2]{Perko2001}, the $\omega$-limit set of every
orbit in the box is the equilibrium, a periodic orbit, or the equilibrium
together with orbits whose backward and forward limit sets are the
equilibrium. The second is excluded by Bendixson's criterion, and the third
because no nonconstant orbit has a sink as its backward limit set. Hence every
$\omega$-limit set is this equilibrium.
\end{proof}

The global convergence of cell~2 is a planar argument. In cells~3 and~4 the
limit \eqref{eq:LR5} is five-dimensional, and the Poincar\'e--Bendixson
argument is not available, but a Lyapunov function is.

\begin{proposition}[Cells 3--4: global stability]\label{prop:S}
Let $\omega^{*2}>0$, $\Phi^*\ge0$ and $r>0$. Call a state $(\theta,Y,R,V,C)$
physical if $R\ge0$, $V\ge0$ and $RV\ge C^2$.
\begin{enumerate}
\item[(i)] From every physical state, the solution of \eqref{eq:LR5} exists
for all $\tau\ge0$, stays physical, and converges to the point
$y^*=(\theta^*,0,R^*,\Phi^*/\omega^{*2},0)$, with $(\theta^*,R^*)$ the
equilibrium of Proposition~\ref{prop:D}.
\item[(ii)] The Jacobian of \eqref{eq:LR5} at $y^*$ is Hurwitz.
\item[(iii)] For every compact set $K$ of physical states and every compact
$\Pi\subset(0,\infty)\times[0,\infty)\times(0,\infty)$ there are $M,c>0$ such
that $|y(\tau)-y^*|\le Me^{-c\tau}$ for all $\tau\ge0$, all $y(0)\in K$ and
all $(\omega^{*2},\Phi^*,r)\in\Pi$.
\end{enumerate}
\end{proposition}

\noindent A state is physical exactly when $R$, $C$ and $V$ are the entries
of a positive semidefinite $2\times2$ matrix, as second moments must be.

\noindent Part (i) gives global existence, which is why the finite-horizon
bound of Proposition~\ref{prop:LR34} holds for every $T$. Parts (ii) and
(iii) are the input to Corollary~\ref{cor:recursion} below, which transfers
the long-time behaviour to the drift recursion itself.

The Lyapunov function has three parts: a potential in $(\theta,R)$, the
kinetic terms $Y^2/2$ and $V/2$, and a logarithmic barrier in the bulk
second moments. For the potential, write
\[
U(\theta,R):=\varrho-r\theta=e^{(\theta^2+R-r^2)/2}-r\theta ,
\]
so that $\partial_\theta U=\varrho\theta-r$ is the mean signal gradient of
Lemma~\ref{lem:B} divided by $p$, and $\partial_RU=\varrho/2$ is half the
curvature factor. In the tame regime, $pU$ plays the role of the population
loss $L$ of Proposition~\ref{prop:D}(iv): the two gradients agree up to
relative errors $O(\epsilon_B)$, because $\nabla L=\cA\theta+\mathcal B\mu$
(proof of Proposition~\ref{prop:D}(iv)) and Lemma~\ref{lem:B} gives $\cA$
and $\mathcal B$. The function $U$ is strictly convex, since its Hessian
\[
\varrho\begin{pmatrix}1+\theta^2&\theta/2\\\theta/2&1/4\end{pmatrix}
\]
has determinant $\varrho^2/4$. With $b_*:=\Phi^*/\omega^{*2}$ and $Q:=RV-C^2$, set
\begin{equation}\label{eq:LR5-lyap}
\mathcal V:=\frac{U(\theta,R)}{\omega^{*2}}+\frac{Y^2}2+\frac V2-\frac{b_*}2\log Q ,
\end{equation}
where the last term is absent when $\Phi^*=0$.

The form of $\mathcal V$ comes from a diffusion. The bulk variables are the
second moments of the linear diffusion
\[
X'=-\omega^{*2}Y_\perp,\qquad Y_\perp'=-Y_\perp+\varrho X+\sqrt{2b_*}\,\dot W ,
\]
whose noise strength $b_*$ plays the role of a temperature. The free energy
of a diffusion is its mean energy minus the temperature times its entropy,
and the entropy of a Gaussian with these second moments is $\frac12\log Q$
up to a constant. Thus \eqref{eq:LR5-lyap} is the free energy of this
diffusion, with the potential $U$ in place of a quadratic one, and
$-\frac{b_*}2\log Q$ is its entropy term. On physical states with $Q>0$,
where $R$ and $V$ are positive, the derivative of $\mathcal V$ along
\eqref{eq:LR5} is
\begin{equation}\label{eq:LR5-dissipation}
\mathcal V'=-Y^2-\frac{(V-b_*)^2}{V}-\frac{b_*^2C^2}{QV}\le0 .
\end{equation}

\begin{proof}
The proof rests on the dissipation identity \eqref{eq:LR5-dissipation}. We
first check that physical states stay physical and prove that identity. For
(i), the sublevel sets of $\mathcal V$ are then compact and forward
invariant, and LaSalle's principle reduces the $\omega$-limit set to $y^*$.
For (ii) with $\Phi^*>0$, the same identity, expanded to second order at
$y^*$, is a Lyapunov equation for the Jacobian; with $\Phi^*=0$ the Jacobian
is block triangular. Part (iii) combines a neighbourhood of $y^*$ in which
convergence is exponential, uniformly in the parameters, with a compactness
bound on the time needed to reach it.

\emph{Physical states and the dissipation identity.} The bulk equations of
\eqref{eq:LR5} say that $\Sigma:=\bigl(\begin{smallmatrix}R&C\\C&V\end{smallmatrix}\bigr)$
solves the linear matrix equation $\Sigma'=M\Sigma+\Sigma M^\top+N$ with
$M=\bigl(\begin{smallmatrix}0&-\omega^{*2}\\\varrho(\tau)&-1\end{smallmatrix}\bigr)$
and $N=\operatorname{diag}(0,2b_*)$, so that
\[
\Sigma(\tau)=\mathcal U(\tau,0)\Sigma(0)\mathcal U(\tau,0)^\top
+\int_0^\tau\mathcal U(\tau,s)\,N\,\mathcal U(\tau,s)^\top\,ds ,
\]
with $\mathcal U$ the propagator of $M$. Both terms are positive
semidefinite, so the state stays physical. When $\Phi^*>0$ more is true:
$\Sigma(\tau)\succ0$ for every $\tau>0$. Indeed, a vector $v$ in the kernel
of the integral satisfies $e_2^\top\mathcal U(\tau,s)^\top v=0$ for
$s\le\tau$; at $s=\tau$ this gives $v_2=0$, and differentiating in $s$ at
$s=\tau$ gives $(M^\top v)_2=-\omega^{*2}v_1=0$.

For the dissipation identity, differentiate the terms of $\mathcal V$ one at
a time. The terms in $\varrho$ cancel in $Q'$, which gives $Q'=-2Q+2b_*R$, and
\[
\begin{aligned}
(U/\omega^{*2})'&=-(\varrho\theta-r)Y-\varrho C, & (Y^2/2)'&=-Y^2+(\varrho\theta-r)Y,\\
(V/2)'&=-V+\varrho C+b_*, & -\tfrac{b_*}2(\log Q)'&=b_*-b_*^2R/Q .
\end{aligned}
\]
The cross terms cancel in the first three, whose sum is $-Y^2-V+b_*$, and
since $R/Q=1/V+C^2/(QV)$, adding the last gives
\eqref{eq:LR5-dissipation}.

\emph{(i), $\Phi^*>0$.} By the first step we may start at a state with
$Q>0$. On the open convex set of physical states with $Q>0$, that is, with
$\Sigma$ positive definite, $\mathcal V$ is strictly convex: $U$ is,
$Y^2/2$ is, $V/2$ is affine, and $-\log Q=-\log\det\Sigma$ is strictly convex
on positive definite matrices. It is also coercive, for four reasons. The
term $U\ge e^{(\theta^2+R-r^2)/2}-r\theta$ grows exponentially in
$\theta^2+R$, while $-\frac{b_*}2\log Q\ge-\frac{b_*}2\log(RV)$ loses only a
logarithm; $V/2-\frac{b_*}2\log V\to\infty$ as $V\to\infty$;
$\mathcal V\to\infty$ as $Q\to0$; and $|C|\le\sqrt{RV}$. The sublevel sets
of $\mathcal V$ are compact subsets of $\{Q>0\}$, and by
\eqref{eq:LR5-dissipation} they are forward invariant: the solution exists
for all time and is bounded. By LaSalle's invariance principle
\citep[Theorem~4.4]{Khalil2002}, the
$\omega$-limit set lies in the largest invariant subset of
$\{\mathcal V'=0\}=\{Y=0,\ V=b_*,\ C=0\}$. On an invariant subset,
$Y'=0$ gives $\varrho\theta=r$, $C'=0$ gives $\varrho R=\omega^{*2}V=\Phi^*$, and
then $R'=V'=0$: the set is the single point $y^*$, which is the equilibrium of
Proposition~\ref{prop:D} by its scalar equation. The solution therefore
converges to $y^*$.

\emph{(i), $\Phi^*=0$.} Now $\mathcal V=U/\omega^{*2}+Y^2/2+V/2$ is defined and
coercive on the closed set of physical states, since bounded $\mathcal V$
bounds $\theta,R,Y,V$ and then $C^2\le RV$, and the first step with $b_*=0$
gives $\mathcal V'=-Y^2-V\le0$ on all physical states. The same argument applies: on an invariant subset of
$\{Y=0,V=0\}$, $V'=0$ gives $C=0$, $C'=0$ gives $\varrho R=0$, so $R=0$, and
$Y'=0$ gives $\varrho\theta=r$; this is $y^*=(\theta^*,0,0,0,0)$.

\emph{(ii), $\Phi^*>0$.} Let $A$ be the Jacobian at $y^*$, $\mathbf H$ the Hessian of
$\mathcal V$ at $y^*$ and $S$ the Hessian at $y^*$ of the right side of
\eqref{eq:LR5-dissipation} with its sign reversed. The point is that the
dissipation identity, expanded to second order at $y^*$, is a Lyapunov
equation for $A$. With $f$ the vector field of \eqref{eq:LR5}, expanding
$\mathcal V'=\nabla\mathcal V\cdot f$ to second order at $y^*$, where both
$\nabla\mathcal V$ and $f$ vanish, gives
\[
A^\top \mathbf H+\mathbf HA=-S .
\]
Here $\mathbf H\succ0$:
it is the sum of the Hessians of $U/\omega^{*2}$, positive definite in
$(\theta,R)$, of $Y^2/2$, and of $-\frac{b_*}2\log\det\Sigma$, positive definite
in $(R,V,C)$ on positive definite matrices, so $\bar v^\top Hv=0$ forces
$v=0$. Also $S\succeq0$ is diagonal in the coordinates
$(Y,V,C)$ with entries $2$, $2/b_*$ and $2b_*/Q^*$, and vanishes on $(\theta,R)$.
If $Av=z v$ with $v\ne0$, then
\[
2\operatorname{Re}z\;\bar v^\top Hv=-\bar v^\top Sv\le0,
\]
so $\operatorname{Re}z\le0$, with equality only if $Sv=0$, that is, if
$v$ is supported on $(\theta,R)$. For such $v$, the $Y$- and $C$-components of
$Av$ are
\[
\varrho^*(1+\theta^{*2})\,v_\theta+\tfrac12\varrho^*\theta^*\,v_R
\quad\text{and}\quad
\varrho^*\theta^*R^*\,v_\theta+\varrho^*(1+\tfrac12R^*)\,v_R ,
\]
and they must vanish since $v_Y=v_C=0$. As linear forms in
$(v_\theta,v_R)$ they have determinant
$\frac12\varrho^{*2}(2+2\theta^{*2}+R^*)>0$, so $v=0$. No eigenvalue of $A$
lies on the imaginary axis, and $A$ is Hurwitz.

\emph{(ii), $\Phi^*=0$.} Then $R^*=C^*=V^*=0$, and the linearized bulk
equations
\[
\delta R'=-2\omega^{*2}\delta C,\qquad \delta V'=-2\delta V+2\varrho^*\delta C,\qquad \delta C'=-\delta C+\varrho^*\delta R-\omega^{*2}\delta V
\]
do not involve $(\delta\theta,\delta Y)$, so $A$ is block triangular. The
bulk block is the second-moment system of the matrix
$\bigl(\begin{smallmatrix}0&-\omega^{*2}\\\varrho^*&-1\end{smallmatrix}\bigr)$,
whose eigenvalues $z$ solve $z^2+z+\omega^{*2}\varrho^*=0$ and
have negative real part, and the eigenvalues of the bulk block are the
pairwise sums of these. The signal block is the linearized heavy-ball
oscillator in $(\delta\theta,\delta Y)$, with curvature
$\varrho^*(1+\theta^{*2})>0$. Both blocks are Hurwitz.

\emph{(iii).} The argument has two stages: a neighbourhood of $y^*$ in
which convergence is exponential, uniformly in the parameters, and a
uniform bound on the time needed to enter it. For the first, note that
$y^*$ depends continuously on $(\omega^{*2},\Phi^*,r)$, because the positive
root of the scalar equation of Proposition~\ref{prop:D} is continuous in the
parameters by strict monotonicity; so does the Jacobian, which is Hurwitz
by (ii). A quadratic Lyapunov matrix for the
Jacobian at one parameter remains one, with half the rate, at nearby
parameters, and the field is $C^1$ jointly in parameters and state, so the
remainder is $o(|y-y^*|)$ locally uniformly. Covering $\Pi$ by finitely many
such parameter neighbourhoods gives a radius and $c>0$, uniform over $\Pi$,
such that solutions starting in the open ball $\mathcal N$ of that radius
about $y^*$ converge to $y^*$ at rate $c$.

It remains to bound the entry time uniformly. For each
$(y_0,\pi)\in K\times\Pi$ the solution
enters $\mathcal N$ at some finite time $T(y_0,\pi)$ by (i), and by continuous
dependence on initial data and parameters the solutions from a neighbourhood
of $(y_0,\pi)$ are in $\mathcal N$ at the same time $T(y_0,\pi)$; later times
then follow from local stability. By compactness, finitely many such
neighbourhoods cover $K\times\Pi$, so the entry time is bounded by some
$T_0$, and $|y(\tau)-y^*|\le Me^{-c\tau}$ follows with $M$ accounting for
$[0,T_0]$.
\end{proof}

\begin{remark}[Rates]
For large $\tau$, the rate $c$ in (iii) can be taken arbitrarily close to
the smallest distance, over $\Pi$, from the spectrum of the Jacobian at
$y^*$ to the imaginary axis, while
the entry time $T_0$ in the proof comes from compactness and has no explicit
bound.
\end{remark}

\begin{remark}[The limit as a stochastic heavy ball]\label{rem:LR-SHB}
Let $y=(\theta,Y,R,V,C)$ solve \eqref{eq:LR5} from a physical state, with
curvature factor $\varrho(\tau)=e^{(\theta^2+R-r^2)/2}$. Write $P_\perp$ for the
orthogonal projection onto $\mu^\perp$ in $\R^d$, and let $(\mathbf X,\mathbf Y)$ solve the linear SDE
\[
d\mathbf X=-\omega^{*2}\mathbf Y\,d\tau,\qquad d\mathbf Y=\bigl(\varrho(\tau)\mathbf X-\mu-\mathbf Y\bigr)\,d\tau+dM,
\]
where $M$ is a continuous martingale with
$d\langle M\rangle=\frac{2\Phi^*}{\omega^{*2}(d-1)}P_\perp\,d\tau$, from an initial
state with $\langle \mathbf X_0,\mu\rangle=r\theta(0)$, $\langle \mathbf Y_0,\mu\rangle=rY(0)$ and
bulk second moments $\E|P_\perp \mathbf X_0|^2=R(0)$, $\E|P_\perp \mathbf Y_0|^2=V(0)$,
$\E\langle P_\perp \mathbf X_0,P_\perp \mathbf Y_0\rangle=C(0)$. Then for all $\tau\ge0$,
\[
\tfrac1r\langle\mathbf X,\mu\rangle=\theta,\quad
\tfrac1r\langle\mathbf Y,\mu\rangle=Y,\quad
\E|P_\perp\mathbf X|^2=R,\quad \E|P_\perp\mathbf Y|^2=V,\quad \E\langle P_\perp\mathbf X,P_\perp\mathbf Y\rangle=C .
\]

Indeed, $M$ has no component along $\mu$, so the signal coordinates solve
the first two equations of \eqref{eq:LR5} without noise. In the bulk,
$dP_\perp\mathbf X=-\omega^{*2}P_\perp\mathbf Y\,d\tau$ and $dP_\perp\mathbf Y=(\varrho P_\perp\mathbf X-P_\perp\mathbf Y)\,d\tau+dM$,
because $P_\perp\mu=0$. Since $\mathbf X$ carries no noise, It\^o's formula gives
$d\E|P_\perp\mathbf X|^2=-2\omega^{*2}\E\langle P_\perp\mathbf X,P_\perp\mathbf Y\rangle\,d\tau$,
$d\E\langle P_\perp\mathbf X,P_\perp\mathbf Y\rangle=(\varrho\E|P_\perp\mathbf X|^2-\E\langle P_\perp\mathbf X,P_\perp\mathbf Y\rangle-\omega^{*2}\E|P_\perp\mathbf Y|^2)\,d\tau$ and
$d\E|P_\perp\mathbf Y|^2=(2\varrho\E\langle P_\perp\mathbf X,P_\perp\mathbf Y\rangle-2\E|P_\perp\mathbf Y|^2+2\Phi^*/\omega^{*2})\,d\tau$,
the last term being the trace of $d\langle M\rangle/d\tau$. These are the bulk
equations of \eqref{eq:LR5}, a linear system for the three bulk moments
once $\varrho(\tau)$ is given, so the two sets of moments agree for all time.

Since $|\mu|=r$, the drift $\varrho\mathbf X-\mu$ is the gradient of
$U(x)=e^{(|x|^2-r^2)/2}-\langle\mu,x\rangle$ at $\mathbf X$ whenever
$|\mathbf X|^2=\theta^2+R$, and $U(\theta\mu/r+x_\perp)=U(\theta,|x_\perp|^2)$ is the
potential of Proposition~\ref{prop:S}. The noise is spread evenly over the
$d-1$ bulk directions, so from a deterministic initial state $|P_\perp\mathbf X|^2$
concentrates around $R$ as $d\to\infty$, and $(\mathbf X,\mathbf Y)$ is then a heavy ball in the potential $U$, with
unit friction and stiffness $\omega^{*2}$, driven by noise at
temperature $\omega^{*2}\Phi^*$ and with no feedback. This is the reading of
\eqref{eq:LR5} used in Section~\ref{sec:logistic}.
\end{remark}

The windows are reached in two ways, as for LS (Corollary~\ref{cor:curv}):
with a vanishing per-step curvature $\lambda\to0$ and $\bar\omega^2\to\infty$ (cells~5
and~6), where the limit is \eqref{eq:LR5} with $\omega^{*2}\to\infty$, or with
$\lambda$ of order one (cells~7 and~8), where the recursion turns by an angle of
order one at every step and has no continuum limit on the retention clock.
The two cases behave differently for LR, and we treat them separately.

\paragraph{Large $\omega^{*2}$: cells~5 and~6.}

As $\omega^{*2}\to\infty$ the solutions oscillate with frequency of order
$\omega^*$, while the damping stays of order one. To separate the
two, rescale the coordinates that carry the momentum: in the fast variables
\[
\tilde y:=\omega^*\,Y,\qquad \tilde C:=\omega^*\,C,\qquad \tilde V:=\omega^{*2}V,
\]
the system \eqref{eq:LR5} reads
\[
x'=\omega^*\,F(x)+G(x),\qquad x=(\theta,\tilde y,R,\tilde C,\tilde V),
\]
with
\[
\begin{aligned}
F(x)&=\bigl(-\tilde y,\ \varrho\theta-r,\ -2\tilde C,\ \varrho R-\tilde V,\ 2\varrho\tilde C\bigr),\\
G(x)&=\bigl(0,\ -\tilde y,\ 0,\ -\tilde C,\ -2\tilde V+2\Phi^*\bigr).
\end{aligned}
\]
The fast field $F$ has no damping: it conserves an energy, written out in
the remark after Corollary~\ref{cor:recursion}. The slow field $G$ carries
the damping and the noise; $\omega^{*2}$ enters only through the ratio of their
sizes. The equilibrium is $x^*=(\theta^*,0,R^*,0,\Phi^*)$ with
$(\theta^*,R^*)$ from Proposition~\ref{prop:D}, and we write
$a:=\varrho^*=r/\theta^*$ and
\[
\psi:=\frac{R^*}{2(1+\theta^{*2})}.
\]

\begin{proposition}[Cell 6: rates in the noisy window]\label{prop:W1}
Let $\Phi^*>0$ and $r>0$, so that $\theta^*>0$. As $\omega^{*2}\to\infty$, the eigenvalues of the
linearization of \eqref{eq:LR5} at $x^*$ are
\[
\begin{aligned}
z_0&=-\frac{2+2\psi}{2+\psi}+O(\omega^{*-1}),\\
z_{j}^{\pm}&=\pm i\omega^*\,\nu_j+\mu_j+O(\omega^{*-1}),\qquad j=1,2,
\end{aligned}
\]
where $\nu_1^2<\nu_2^2$ are the roots of
\[
\nu^4-a\,(5+\theta^{*2}+R^*)\,\nu^2+a^2\,(4+4\theta^{*2}+R^*)=0,
\]
and, with $s_j:=-\tfrac12a\theta^*/\bigl(a(1+\theta^{*2})-\nu_j^2\bigr)$,
\[
\begin{aligned}
\mu_j&=-\frac{N_j}{2M_j}<0,\qquad
N_j:=2\nu_j^2s_j^2+\frac{\nu_j^2+4a}{2R^*},\\
M_j&:=\bigl(a(1+\theta^{*2})+\nu_j^2\bigr)s_j^2+a\theta^*s_j+\frac{a(R^*+4)+\nu_j^2}{4R^*}.
\end{aligned}
\]
The real parts satisfy $z_0+2\mu_1+2\mu_2=-4+O(\omega^{*-1})$.
\end{proposition}

\noindent Compare Corollary~\ref{cor:window} with $\ut=0$: there the slow
energy decays at rate exactly $1$ and both oscillatory modes at rate $1$,
while the signal of a frozen-$\varrho$ heavy ball would decay at rate
$\frac12$. For LR the non-oscillatory mode, which plays the role of the
slow energy, decays at the rate $1+\psi/(2+\psi)>1$. Each oscillatory mode
has components of order one in both the signal and the bulk coordinates,
and its rate $-\mu_j$ is neither $\frac12$ nor $1$. The reason is that the coupling through $\varrho$ and the gap
between the two fast frequencies are both of order $\omega^*$, so
the fast modes mix by an amount of order one that does not vanish as
$\omega^{*2}\to\infty$. At $R^*=0$ the roots are $\nu^2=a(1+\theta^{*2})$ and
$\nu^2=4a$, the signal and bulk frequencies, which coincide when
$1+\theta^{*2}=4$. This coincidence of the uncoupled signal and bulk
frequencies is the internal resonance; for $R^*>0$ the
discriminant is $(\theta^{*2}-3)^2+R^*(R^*+6+2\theta^{*2})>0$, so the
frequencies never cross, and the resonance is the point of strongest
mixing, not a degeneracy.

Which of these rates does the excess risk see? Its gradient agrees with
that of $p\,U$ in the tame regime (discussion before the proof of
Proposition~\ref{prop:S}; the excess risk $\mathrm{KL}$ of
Proposition~\ref{prop:D}(iv) is $L-L(\mu)$), so near $x^*$ the deviation of
$\mathrm{KL}$ from its value at $x^*$ is $p\,(U-U(x^*))$ to first order, up
to relative errors $O(\epsilon_B)$. At $x^*$ the
$\theta$-derivative of $U$ vanishes but its $R$-derivative, $\varrho^*/2$, does
not, and every eigenvector of the linearization has a nonzero $R$-component
(see the proof). The deviation of the excess risk from its equilibrium value
is therefore, to first order, a combination of all five modes.

From a generic start near $x^*$, the excess risk of \eqref{eq:LR5}
decays at the rate
$\min(-z_0,-\mu_1,-\mu_2)+O(\omega^{*-1})$ on the clock $\tau$. The
drift recursion at fixed $\omega^{*2}$ decays at $\ebar$ times this rate per
step, up to $O(\varepsilon+\delta_\varepsilon)$, where $\delta_\varepsilon$
is the parameter error defined in Corollary~\ref{cor:recursion} below. The
reason is that, by the proof of that corollary, the Jacobian of the drift
map at its fixed point is
$I+\varepsilon A+O(\varepsilon(\varepsilon+\delta_\varepsilon))$, with $A$
the Jacobian of \eqref{eq:LR5} at $x^*$.

The slowest mode is always oscillatory. Indeed
$-z_0=1+\psi/(2+\psi)+O(\omega^{*-1})$, and the trace identity gives $-\mu_1-\mu_2=\frac12(4+z_0)+O(\omega^{*-1})<\frac32$,
so that
\[
\min(-z_0,-\mu_1,-\mu_2)=\min(-\mu_1,-\mu_2)<\tfrac34+O(\omega^{*-1}).
\]
The per-step decay rate of the envelope of the excess risk is therefore
strictly below the LS window rate $\ebar$.

This differs from LS, where the signal enters the risk quadratically, so
that its amplitude rate $\frac12$ becomes an energy rate $1$. In LR the
dynamic $\varrho$ feeds the signal amplitude into the bulk moment $R$ at
first order, with a weight that vanishes as $R^*\to0$, so the excess risk
decays at the oscillatory rates $-\mu_j$ themselves rather than at twice
those rates.

In cell~5, $\Phi^*=0$, the bulk at the fixed point is zero and the
linearization is block triangular (proof of Proposition~\ref{prop:S}(ii)).
Its rates are those of LS: $\frac12$ for the signal amplitude, which enters
the excess risk quadratically, and $1$ for the bulk.

\begin{proof}
The idea is to treat the damping as a small perturbation of the undamped
fast part. After division by $\omega^*$, the linearization is the
matrix $\Xi$ of the fast field plus $\omega^{*-1}$ times the matrix $D$
of the slow field. The eigenvalues of $\Xi$ are $0$ and $\pm i\nu_j$, all
simple, so first-order perturbation theory gives each eigenvalue to
$O(\omega^{*-1})$. The correction is explicit for the eigenvalue $0$; for
the oscillatory ones, its real part comes from the Lyapunov function, whose
Hessian turns right eigenvectors into left ones.

Let $A=\omega^*\,\Xi+D$ be the linearization at $x^*$, with
$\Xi$ and $D$ the Jacobians of $F$ and $G$ at $x^*$, so that
$D=\operatorname{diag}(0,-1,0,-1,-2)$ and
\[
\Xi=\begin{pmatrix}
0&-1&0&0&0\\ a(1+\theta^{*2})&0&\tfrac12a\theta^*&0&0\\ 0&0&0&-2&0\\
a\theta^*R^*&0&a(1+\tfrac12R^*)&0&-1\\ 0&0&0&2a&0\end{pmatrix}.
\]
At $z=i\nu$, the characteristic polynomial $\det(z I-\Xi)$
equals $i\nu$ times the quadratic in $\nu^2$ of the statement, whose roots
are positive and distinct by the discriminant computation after the
statement. So $\Xi$ has the five simple eigenvalues
$0,\pm i\nu_1,\pm i\nu_2$. The matrix $\Xi+sD$ is analytic in
$s=\omega^{*-1}$, so first-order perturbation of simple eigenvalues
\citep[Theorem~6.3.12]{HornJohnson2013} holds for it with remainder
$O(s^2)$. Multiplying by $\omega^*$, each eigenvalue $\iota$ of
$\Xi$, with right and left eigenvectors $v$ and $\ell$, yields an
eigenvalue of $A=\omega^*(\Xi+\omega^{*-1}D)$,
\[
z=\omega^*\,\iota+\langle \ell,Dv\rangle/\langle \ell,v\rangle+O(\omega^{*-1}).
\]
It remains to compute this correction for $\iota=0$ and for
$\iota=\pm i\nu_j$.

\emph{The real eigenvalue.} For $\iota=0$, solving $\Xi v=0$ gives, for
any scalar $\delta\ne0$,
$v=(-\theta^*\delta/(2(1+\theta^{*2})),\,0,\,\delta,\,0,\,a(1+\psi)\delta)$
and $\Xi^\top\ell=0$ gives $\ell=(0,0,a,0,1)$. Then
$\langle \ell,v\rangle=a\delta(2+\psi)$ and
$\langle \ell,Dv\rangle=-2a(1+\psi)\delta$, which is the formula for $z_0$.

\emph{The oscillatory eigenvalues.} Here the Lyapunov function does the
work, as in the proof of Proposition~\ref{prop:S}(ii): its Hessian converts
right eigenvectors of $\Xi$ into left ones, and its dissipation fixes the
sign of the real parts. In the fast variables, \eqref{eq:LR5-lyap} is
$\omega^{*-2}$ times
\[
\tilde{\mathcal V}:=U(\theta,R)+\tfrac12\tilde y^2+\tfrac12\tilde V-\tfrac12\Phi^*\log(R\tilde V-\tilde C^2)
\]
plus a constant, and the dissipation identity \eqref{eq:LR5-dissipation}
becomes
\[
\tilde{\mathcal V}'=-\tilde y^2-(\tilde V-\Phi^*)^2/\tilde V-\Phi^{*2}\tilde C^2/((R\tilde V-\tilde C^2)\tilde V),
\]
with no $\omega^{*2}$. Let $\mathbf H$ be the Hessian of $\tilde{\mathcal V}$ at $x^*$
and $S$ the Hessian at $x^*$ of the dissipation with its sign reversed.
Expanding $\nabla\tilde{\mathcal V}\cdot(\omega^*F+G)$ to second
order and separating the powers of $\omega^{*2}$ gives
\[
\Xi^\top \mathbf H+\mathbf H\Xi=0,\qquad D^\top \mathbf H+\mathbf HD=-S,
\]
with $\mathbf H\succ0$ for the reason given in the proof of
Proposition~\ref{prop:S}(ii) (it is also visible in the explicit form below), and
\[
S=\operatorname{diag}\Bigl(0,\,2,\,0,\,\frac2{R^*},\,\frac2{aR^*}\Bigr)\succeq0,
\]
using $\tilde V^*=\Phi^*=aR^*$. If $\Xi v=i\nu v$ then
$\Xi^\top(\mathbf H\bar v)=-\mathbf H\Xi\bar v=i\nu\,\mathbf H\bar v$, so $\ell=\mathbf H\bar v$ is a
left eigenvector, and
\[
\begin{aligned}
\operatorname{Re}\frac{\langle \ell,Dv\rangle}{\langle \ell,v\rangle}
&=\frac{\operatorname{Re}\,v^*\mathbf HDv}{v^*Hv}
=\frac{v^*(\mathbf HD+D^\top \mathbf H)v}{2\,v^*Hv}\\
&=-\frac{v^*Sv}{2\,v^*Hv}=:\mu .
\end{aligned}
\]
Solving $\Xi v=i\nu v$ row by row gives $v_{\tilde y}=-i\nu v_\theta$,
$v_{\tilde C}=-\tfrac12i\nu v_R$, $v_{\tilde V}=-av_R$ and
$v_\theta=s\,v_R$ with $s$ as in the statement. Here $s$ is well defined:
$\nu^2=a(1+\theta^{*2})$ is not a root of the quadratic in $\nu^2$, which
equals $-a^2R^*\theta^{*2}\ne0$ there. Substituting into
$v^*Sv$ and $v^*Hv$, where
\[
\mathbf H=\begin{pmatrix}a(1+\theta^{*2})&0&\tfrac12a\theta^*&0&0\\0&1&0&0&0\\
\tfrac12a\theta^*&0&\tfrac{a(R^*+2)}{4R^*}&0&0\\0&0&0&\tfrac1{R^*}&0\\0&0&0&0&\tfrac1{2aR^*}\end{pmatrix},
\]
gives the formula for $\mu_j$, which is negative because $S\succeq0$ and
$v\notin\ker S$ (an eigenvector with $v_{\tilde y}=0$ would have
$v_\theta=0$ and then $v=0$).

Two facts used in the discussion before the proof remain. Every eigenvector
of $\Xi$ has a nonzero $R$-component: for the oscillatory ones,
$v_\theta=s\,v_R$ with $v_\theta\ne0$ forces $v_R\ne0$, and the null vector
has $v_R=\delta\ne0$. By continuity of simple eigenvectors, the same holds
for the eigenvectors of $A$ when $\omega^{*2}$ is large. Finally, the trace
identity is $\operatorname{tr}A=\operatorname{tr}D=-4$.
\end{proof}

\begin{remark}[Numerical values, and what is not asserted]\label{rem:W1}
In numerical evaluations over a range of $(\Phi^*,r)$, the rates $-\mu_j$
lie between $\frac12$ and $1$: the signal-like one is close to $\frac12$
for small $\Phi^*$ and the bulk-like one close to $1$, and they are closest
to each other at the internal resonance. The slow-energy rate $-z_0$
exceeds $1$ by an amount that grows with $\Phi^*$. Integrations of
\eqref{eq:LR5} near the equilibrium at large $\omega^{*2}$ show the excess
risk decaying at the rate of the slower oscillatory mode, and the limits of
the rates are attained already at moderate $\omega^{*2}$.

The proposition describes the neighbourhood of the equilibrium. It does not
describe the transient from a cold start, where the signal oscillates with
an amplitude of order $r$ and $\varrho$ oscillates by a factor of order one
within each period, so that no closed equation for the slow energy alone,
of the kind Corollary~\ref{cor:window} gives for LS, with $\varrho$ frozen or
slowly varying, is expected to hold uniformly.
\end{remark}

At each fixed $\omega^{*2}$, the long-time behaviour of the drift recursion
itself follows from Proposition~\ref{prop:S}, because a one-step map that
follows a vector field to first order in the step inherits, for small
steps, an attracting fixed point from a globally attracting equilibrium of
that field at which the Jacobian is Hurwitz.

\begin{corollary}[The drift recursion converges]\label{cor:recursion}
Assume the hypotheses of Proposition~\ref{prop:LR34}, with
$\omega_{\mathrm{nom}}^2\to\omega^{*2}\in(0,\infty)$, $\Phi\to\Phi^*\ge0$ and
$\sup\epsilon_B\to0$, and fix a compact set $K$ of physical initial states.
Write $\delta_\varepsilon:=\varepsilon+|\omega_{\mathrm{nom}}^2-\omega^{*2}|+|\Phi-\Phi^*|+p$.
Then, for $\varepsilon$ small enough along the path, the drift recursion has a
fixed point $y^*_\varepsilon$ within distance $M\delta_\varepsilon$ of $y^*$,
unique in a neighbourhood of $y^*$ that does not depend on $\varepsilon$,
and from every start in $K$,
\[
|y_k-y^*_\varepsilon|\le Me^{-c\,\varepsilon k}\quad\text{and}\quad
\sup_{j\ge0}|y_j-y(\tau_j)|\le M\delta_\varepsilon
\]
for all $k\ge0$, with $M,c$ depending only on $K$ and on
$(\omega^{*2},\Phi^*,r)$. The Jacobians $J_k$ of the drift map along the path
satisfy $\|J_{k-1}\cdots J_j\|\le\Gamma e^{-c\varepsilon(k-j)}$ for
$0\le j\le k$, with $\Gamma$ depending on the same quantities.
\end{corollary}

\noindent The corollary gives, in cells~3 and~4, convergence of the drift
recursion itself, uniformly in time, and the decay of the Jacobian products
along its path. The last property is what extends the fluid limit to long
horizons (remark after Corollary~\ref{cor:fluid}).

In the windows of cells~5 and~6 the corollary applies at each fixed
$\omega^{*2}$. Near its fixed point the recursion then contracts at the rates
of Proposition~\ref{prop:W1}, up to $O(\omega^{*-1})$ and
$O(\delta_\varepsilon)$ and with constants depending on $\omega^{*2}$; this
follows from the Jacobian expansion in the proof of the corollary, as
explained after Proposition~\ref{prop:W1}. It does not give uniformity as $\omega^{*2}\to\infty$: the constants $M,c$ and the
threshold on $\varepsilon$ depend on $\omega^{*2}$, and we have not tracked
them. This is the one gap left in cells~5 and~6.

\begin{proof}
The idea is to compare the drift map with the time-$\varepsilon$ flow of
\eqref{eq:LR5}. Near $y^*$ the flow contracts in the norm of a quadratic
Lyapunov matrix, and the drift map, which is close to the flow together
with its derivative, inherits the contraction and with it a fixed point.
Proposition~\ref{prop:S}(iii) bounds the time the flow needs to come near
$y^*$, and the finite-horizon bound of Proposition~\ref{prop:LR34} carries
the recursion along up to that time. We first set up the comparison.

Write the one-step map as
$\Psi_\varepsilon(y)=y+\varepsilon b(y)+\varepsilon e_\varepsilon(y)$, where
$b$ is the field of \eqref{eq:LR5} at $(\omega^{*2},\Phi^*)$ and
$\varepsilon e_\varepsilon$ collects the $O(\varepsilon^2)$ terms and the
$O(\varepsilon\delta_\varepsilon)$ parameter errors, as in the proof of
Proposition~\ref{prop:LR34}. The parameter $p$ enters $\delta_\varepsilon$
for the following reason. Since the bracket in the definition of
$\epsilon_B$ is at least $1$ and is bounded on compact sets,
$p\le\epsilon_B\le C_{K'}p$ on every compact set $K'$; so the hypothesis
$\sup\epsilon_B\to0$ means $p\to0$, and $\epsilon_B=O(p)$ uniformly on
$K'$. On a fixed compact set, $e_\varepsilon$ is then $O(\delta_\varepsilon)$
by the proof of Proposition~\ref{prop:LR34}, and so are its first
derivatives, by the derivative bounds of Lemma~\ref{lem:B}. Write $\phi_\varepsilon$ for the
time-$\varepsilon$ flow of $b$, so that
$\Psi_\varepsilon-\phi_\varepsilon$ and its first derivatives are
$O(\varepsilon\delta_\varepsilon)$ on compact sets.

\emph{Fixed point.} Let $A=Db(y^*)$, which is Hurwitz by
Proposition~\ref{prop:S}(ii), and let $P\succ0$ solve $A^\top P+PA=-I$.
Choose $s_0>0$ so that $Db(y)^\top P+PDb(y)\preceq-\tfrac12I$ on the
ball $\mathcal N:=\{|y-y^*|_P\le s_0\}$. Then $\mathcal N$ is forward invariant under the
flow, and there is $c_0>0$ such that
$|\phi_\varepsilon(y)-\phi_\varepsilon(y')|_P\le(1-c_0\varepsilon)|y-y'|_P$
for $y,y'\in\mathcal N$ and $\varepsilon$ small. Since $\phi_\varepsilon(y^*)=y^*$,
\[
|\Psi_\varepsilon(y)-y^*|_P\le(1-c_0\varepsilon)|y-y^*|_P+M\varepsilon\delta_\varepsilon
\qquad(y\in\mathcal N),
\]
so $\Psi_\varepsilon$ maps $\mathcal N$ into itself once $M\delta_\varepsilon\le c_0s_0$, and
it is a contraction there with factor $1-\tfrac12c_0\varepsilon$, because
its derivative differs from that of $\phi_\varepsilon$ by
$O(\varepsilon\delta_\varepsilon)$. Its unique fixed point $y^*_\varepsilon$ in $\mathcal N$
satisfies $c_0\varepsilon|y^*_\varepsilon-y^*|_P\le M\varepsilon\delta_\varepsilon$ by
the display.

\emph{Entry and convergence.} By Proposition~\ref{prop:S}(iii) there is
$T_0$, depending on $K$, such that every solution of \eqref{eq:LR5} from
$K$ is in $\tfrac12\mathcal N$ at time $T_0$. By Proposition~\ref{prop:LR34} the
recursion is within $C_{T_0}\delta_\varepsilon$ of the solution on $[0,T_0]$, so for
$\varepsilon$ small it is in $\mathcal N$ at step $k_0:=T_0/\varepsilon$, and from
then on it contracts to $y^*_\varepsilon$ by the factor
$1-\tfrac12c_0\varepsilon$ per step. This gives the exponential bound.

\emph{Uniform-in-time comparison.} For $k\ge k_0$ both $y_k$ and $y(\tau_k)$ lie in
$\mathcal N$, and
\[
y_{k+1}-y(\tau_{k+1})=\bigl[\Psi_\varepsilon(y_k)-\phi_\varepsilon(y_k)\bigr]
+\bigl[\phi_\varepsilon(y_k)-\phi_\varepsilon(y(\tau_k))\bigr],
\]
so that
\[
|y_{k+1}-y(\tau_{k+1})|_P\le(1-c_0\varepsilon)|y_k-y(\tau_k)|_P+M\varepsilon\delta_\varepsilon .
\]
Iterating this inequality gives
$|y_k-y(\tau_k)|_P\le\max\bigl(|y_{k_0}-y(\tau_{k_0})|_P,\,M\delta_\varepsilon/c_0\bigr)$
for $k\ge k_0$, and together with the finite-horizon bound on $[0,T_0]$ this
is the uniform-in-time bound.

\emph{Products.} For $k\ge k_0$ the Jacobians $J_k=D\Psi_\varepsilon(y_k)$
satisfy $J_k^\top PJ_k\preceq(1-c_0\varepsilon)P$, after decreasing $c_0$
if necessary, since
$D\Psi_\varepsilon=I+\varepsilon Db+O(\varepsilon^2+\varepsilon\delta_\varepsilon)$ and
$Db^\top P+PDb\preceq-\tfrac12I$ on $\mathcal N$. Before step $k_0$ there are at most
$T_0/\varepsilon$ factors, each of norm at most $1+M\varepsilon$, which
contribute a factor $e^{MT_0}$ to $\Gamma$.
The matrix $P$, the radius $s_0$, the rate $c_0$ and the entry time $T_0$
depend only on $K$ and on $(\omega^{*2},\Phi^*,r)$, and hence so do $M$, $c$
and $\Gamma$.
\end{proof}

\begin{remark}[Why uniformity in $\omega^{*2}$ is a separate question]
The constants given by the proof of Corollary~\ref{cor:recursion} degenerate
as $\omega^{*2}\to\infty$ for two reasons, one in the fixed-point step of its
proof and one in the entry step.

The fixed-point step compares the one-step map with the time-$\varepsilon$
flow. In the fast variables of Proposition~\ref{prop:W1} the flow rotates
at a frequency of order $\omega^*$, so the comparison error is of
order $\varepsilon^2\omega^{*2}$ per step, against a contraction of order
$\varepsilon$, and the argument needs $\varepsilon\omega^{*2}$ small. In
cells~5 and~6 only the per-step rotation angle, of order
$\varepsilon\omega^*$, tends to zero. A uniform argument has to treat
the rotation exactly, as the least-squares analysis does through the exact
embedding of Lemmas~\ref{lem:match} and~\ref{lem:embed}, and only the damping and the variation
of $\varrho$ as perturbations.

The entry step uses the entry time $T_0$ in the proof of
Proposition~\ref{prop:S}(iii), which
comes from compactness and carries no information on $\omega^{*2}$. During the
transient $\varrho$ changes by a factor of order one within each fast period,
and this does not improve as $\omega^{*2}$ grows, so neither a frozen nor a
slowly varying $\varrho$ describes it (Remark~\ref{rem:W1}). Any averaging
must be over the nonlinear fast motion.

That motion is Hamiltonian. The fast field $F$ defined before
Proposition~\ref{prop:W1} conserves $\tilde Q:=R\tilde V-\tilde C^2=\omega^{*2}Q$,
and on each level set of $\tilde Q$, with $R=\xi^2$ and $p_\xi:=-\tilde C/\xi$, it is the
Hamiltonian flow, in the canonical pairs $(\theta,-\tilde y)$ and
$(\xi,p_\xi)$, of
\[
\mathcal H=\tfrac12\tilde y^2+\tfrac12p_\xi^2+U(\theta,\xi^2)+\frac{\tilde Q}{2\xi^2},
\]
with $U$ as in \eqref{eq:LR5-lyap}. Since $\tilde V=p_\xi^2+\tilde Q/\xi^2$, $\mathcal H$
is the energy part of the Lyapunov function in the fast variables.

Near the equilibrium this two-degree-of-freedom motion is close to the
quasi-periodic motion of its linearization, with the frequencies
$\nu_1,\nu_2$ of Proposition~\ref{prop:W1}, and averaging over the
corresponding 2-tori, with the nonlinear $\varrho$, is what a uniform proof
needs. At the energy of the cold start the fast motion looks quasi-periodic
in our integrations for small $r$, but for larger $r$ it has a clearly
positive Lyapunov exponent, and there averaging over tori is not available.

For the limit system the entry time may be uniform for a simpler reason.
In the fast variables the dissipation identity \eqref{eq:LR5-dissipation}
contains no $\omega^{*2}$, so the Lyapunov function decreases at a rate
independent of $\omega^{*2}$, provided the dissipation cannot stay small over a
fast time window of fixed length away from the equilibrium. This is an
observability property of the fast flow, and it would give a $T_0$
independent of $\omega^{*2}$. The recursion would need a discrete version.

Numerically, we compared the drift recursion with the tame coefficients of
Lemma~\ref{lem:B} and the limit system, from a representative cold start
with $\Phi^*>0$. In these comparisons the two parts of
Corollary~\ref{cor:recursion} behave differently. The entry time and the
late decay rate of the limit system do not change as $\omega^{*2}$ grows over
several orders of magnitude, and those of the recursion stay close to them
as long as the per-step curvature $\varepsilon^2\omega^{*2}$ is small, as it is
in cells~5 and~6. In that range the convergence statement appears to be
uniform.

In the same runs, by contrast, the comparison error over a fixed horizon is
proportional to $\varepsilon$ at fixed $\omega_{\mathrm{nom}}^2$, with a
constant that grows at least like $\omega_{\mathrm{nom}}$. Part of it is a phase offset of
half a step, and part is a phase drift that accumulates over the transient.
A comparison uniform in $\omega^{*2}$ must therefore be insensitive to the
phase of the fast oscillation, or be made with a solution at matched
frequency, as for least squares.
\end{remark}

\paragraph{Curvature of order one: the curvature window and the switch.}

When $\lambda=\eta\varepsilon p\varrho$ stays of order one, LS is stable up to the
curvature ceiling, the threshold on the per-step curvature below which the
moment recursion is stable, $\lambda<2(1+\beta)\to4$ (Corollary~\ref{cor:stab} and
Lemma~\ref{lem:stab-all}; Corollary~\ref{cor:curv} for the window). For LR
the ceiling is lower, and it is set by a resonance between the signal and
the bulk that the dynamic $\varrho$ turns into an instability. Write
$h:=\eta\varepsilon p$, so that $\lambda=h\varrho$, and consider the
drift recursion of Lemma~\ref{lem:LRdrift} with the tame coefficients of
Lemma~\ref{lem:B}, in the variables
\[
(\theta,\tilde m,R,\tilde v,\tilde c):=(\theta_\parallel,\ \eta m_\parallel,\ R_\perp,\ \eta^2V_\perp,\ \eta C_\perp).
\]
These scalings, and the order of the coordinates, differ from those of the
fast variables above. As $\varepsilon\to0$ at fixed $h$, the recursion tends
to the undamped map
\begin{equation}\label{eq:window-map}
\begin{aligned}
\tilde m'&=\tilde m+h(\varrho\theta-r),&\theta'&=\theta-\tilde m',\\
\tilde v'&=\tilde v+2h\varrho\tilde c+h^2\varrho^2R,&\tilde c'&=\tilde c+h\varrho R-\tilde v',\\
R'&=R-2(\tilde c+h\varrho R)+\tilde v',
\end{aligned}
\end{equation}
with $\varrho=\varrho(\theta,R)$; the damping and the noise are the
$O(\varepsilon)$ terms dropped here. Let $\theta^*>0$ and $R^*\ge0$ satisfy
the signal condition $\varrho(\theta^*,R^*)\theta^*=r$, put
$a:=\varrho(\theta^*,R^*)$, and let $\lambda:=ha$. Then
$\hat x^*=(\theta^*,0,R^*,\lambda R^*,-\tfrac12\lambda R^*)$ is a fixed point of
\eqref{eq:window-map}, by direct substitution. In
Proposition~\ref{prop:W2} we regard the Jacobian of \eqref{eq:window-map}
at $\hat x^*$ as a function of the independent parameters
$(\theta^*,R^*,\lambda)$.

\begin{proposition}[Cells 7--8: the LR curvature ceiling]\label{prop:W2}
Let $J_0$ be the Jacobian of \eqref{eq:window-map} at $\hat x^*$, and let
\[
\mathsf P(\lambda):=\mathsf a\,\lambda^2-4\mathsf b\,\lambda+16,\qquad
\begin{aligned}
\mathsf a&:=2+2\theta^{*2}+R^*,\\ \mathsf b&:=3+\theta^{*2}+R^*,
\end{aligned}
\]
and let $\lambda_c:=2\bigl(\mathsf b-\sqrt{\mathsf b^2-4\mathsf a}\bigr)/\mathsf a$
be the smaller root of $\mathsf P$. Let $R^*>0$. Then $0<\lambda_c<2$ and:
\begin{enumerate}
\item[(i)] $\det(J_0+I)=-(\lambda-2)\,\mathsf P(\lambda)$ and $1$ is an eigenvalue of $J_0$;
\item[(ii)] for $0<\lambda<\lambda_c$ the eigenvalues of $J_0$ are $1$ and two conjugate
pairs on the unit circle, with distinct real parts in $(-1,1)$, and all five
are simple;
\item[(iii)] at $\lambda=\lambda_c$ a complex pair reaches $-1$, and for $\lambda_c<\lambda<2$ it is
a real pair with one eigenvalue below $-1$. Hence, if $(\theta^*,R^*)$ is the
limit as $\varepsilon\to0$, at fixed $h$, of the signal and bulk components
of the fixed point of the drift recursion, and $\lambda_c<\lambda<2$, then that fixed
point is linearly unstable for every $\varepsilon$ small enough.
\end{enumerate}
\end{proposition}

\noindent The instability is a period doubling: an eigenvalue of the
linearized map passes through $-1$, so the per-step rotation angle of an
oscillatory mode reaches $\pi$ and the deviation from the fixed point flips
sign at every step. In LS this happens to the bulk at $\lambda=2$, where both
oscillatory eigenvalues of the one-step map equal $-\beta$
(Remark~\ref{rem:nyquist}). There the two eigenvalues only touch, and
separate again along their circle without instability, because nothing
couples the two modes at order one. For LR the signal couples to the bulk
through $\varrho$ at order one; the pair of eigenvalues in which the two
oscillations mix reaches $-1$ earlier, at $\lambda_c<2$, and leaves the unit
circle. Since $\mathsf P(2)=-4R^*<0$, the ceiling lies below $\lambda=2$, the LS
point of Remark~\ref{rem:nyquist}, for every $R^*>0$.

As $R^*\to0$ the signal and bulk oscillations decouple, and
$\lambda_c\to\min\bigl(2,\,4/(1+\theta^{*2})\bigr)$: the signal, whose per-step
curvature is $\lambda(1+\theta^{*2})$, period-doubles at $\lambda(1+\theta^{*2})=4$, and
the bulk at $\lambda=2$. This is the case of cell~7, where $\Phi^*=0$ and the
bulk floor vanishes in the limit. At $R^*=0$ itself the discriminant
$\mathsf d(\lambda)=4(\lambda-3+\theta^{*2})^2$ of the proof vanishes at
$\lambda=3-\theta^{*2}$, which can lie in $(0,\lambda_c)$, so the eigenvalues need not
be simple there; this is why the proposition assumes $R^*>0$.

\begin{proof}
The proof reduces the spectrum of $J_0$ to the roots of a quadratic. The
characteristic polynomial is $z-1$ times a palindromic quartic, and the
substitution $m=z+1/z$ turns the quartic into a quadratic $\mathsf p(m)$;
an eigenvalue lies on the unit circle exactly when its $m$ is real and in
$[-2,2]$. Parts (ii) and (iii) then follow from the signs of $\mathsf p(\pm2)$
and of the discriminant of $\mathsf p$, and the last paragraph carries (iii)
over to the drift recursion.

\emph{The quadratic.} Differentiate \eqref{eq:window-map} at $\hat x^*$, using
$\partial_\theta\varrho=\varrho\theta$ and $\partial_R\varrho=\tfrac12\varrho$
and $r=a\theta^*$. The resulting $J_0$ has $\det J_0=1$, and a direct
computation gives $\det(J_0+I)=-(\lambda-2)\mathsf P(\lambda)$ and $\det(J_0-I)=0$, which is (i).
Since $\mathsf P(0)=16>0$ and $\mathsf P(2)=-4R^*<0$, the smaller root
satisfies $0<\lambda_c<2$; at $R^*=0$ the roots of $\mathsf P$ are $2$ and
$4/(1+\theta^{*2})$, which gives the limit of $\lambda_c$ stated before the proof.
The characteristic polynomial is therefore $(z-1)\mathsf q(z)$ with
\[
\begin{aligned}
\mathsf q(z)&=z^4+\mathsf c_1z^3+\mathsf c_2z^2+\mathsf c_1z+1,\\
\mathsf c_1&=-\tfrac12\bigl[(2+R^*)\lambda^2-2(5+R^*+\theta^{*2})\lambda+8\bigr],\\
\mathsf c_2&=6-(10+2\theta^{*2}+2R^*)\lambda+(6+4\theta^{*2}+2R^*)\lambda^2-\tfrac12\mathsf a\,\lambda^3,
\end{aligned}
\]
a palindromic quartic, so its roots come in reciprocal pairs and
$\mathsf q(z)=z^2\,\mathsf p(m)$ with $m:=z+1/z$ and
$\mathsf p(m)=m^2+\mathsf c_1m+\mathsf c_2-2$. A root $z$ lies on
the unit circle if and only if $m$ is real with $|m|\le2$, and $z=\pm1$
corresponds to $m=\pm2$. One finds
\[
\begin{aligned}
\mathsf p(-2)&=\tfrac12(2-\lambda)\,\mathsf P(\lambda),\\
\mathsf p(2)&=\tfrac12\lambda^2\bigl[2(4+4\theta^{*2}+R^*)-\mathsf a \lambda\bigr],
\end{aligned}
\]
and the discriminant of $\mathsf p$ is $\frac14\lambda^2\mathsf d(\lambda)$ with
\[
\mathsf d(\lambda)=R^{*2}(\lambda-2)^2+4R^*(\lambda-2)(\lambda-3-\theta^{*2})+4(\lambda-3+\theta^{*2})^2 .
\]

\emph{Below the ceiling.} As a quadratic in $R^*$, $\mathsf d$ has leading coefficient $(\lambda-2)^2$ and roots
$2(\sqrt{3-\lambda}\mp\theta^*)^2/(\lambda-2)\le0$ for $\lambda<2$, so $\mathsf d>0$ for $0<\lambda<2$ and
$R^*>0$. Also $\mathsf p(2)>0$ for $\lambda<2$, because
$2(4+4\theta^{*2}+R^*)\ge2\mathsf a$, and $\mathsf p(-2)>0$ for $0<\lambda<\lambda_c$,
because $\mathsf P(0)=16$ and $\lambda_c<2$ is the smaller root of $\mathsf P$. Hence on
$(0,\lambda_c)$ the two roots $m_\pm$ of $\mathsf p$ are real and distinct, and
each of $\pm2$ lies on the same side of both.

It remains to place both roots between $-2$ and $2$. Their sum is
$-\mathsf c_1=4-(5+R^*+\theta^{*2})\lambda+\tfrac12(2+R^*)\lambda^2<4$, because
$(2+R^*)\lambda<10+2R^*+2\theta^{*2}$ for $\lambda<2$; since $2$ lies on the same side
of both roots, both are below $2$. For the lower bound, let $\lambda\to0$: then
$\mathsf p(m)\to(m-2)^2$ and both roots tend to $2$, so they lie above $-2$
for small $\lambda$, and by continuity in $\lambda$ they remain there, because they
never equal $-2$ on $(0,\lambda_c)$. Thus $m_\pm\in(-2,2)$ on $(0,\lambda_c)$,
which gives (ii): since $\mathsf q(z)=(z^2-m_-z+1)(z^2-m_+z+1)$, each $m_\pm$
gives a conjugate pair on the unit circle with real part $m_\pm/2$, and the
simplicity follows from $m_-\ne m_+$ and $m_\pm\ne2$.

\emph{At and above the ceiling.} At $\lambda=\lambda_c$, $\mathsf p(-2)=0$, so $m_-=-2$ and a conjugate pair meets at $z=-1$.
For $\lambda_c<\lambda<2$, $\mathsf P(\lambda)<0$, because $\mathsf P(2)=-4R^*<0$ and $\lambda_c$ is
the smaller root of $\mathsf P$, so $\mathsf p(-2)<0$; then $-2$ lies between
the two real roots of $\mathsf p$, $m_-<-2$, and the pair is real with
$z<-1<1/z<0$.

\emph{The drift recursion.} Under the hypothesis of (iii), the fixed point
of the drift recursion converges to $\hat x^*$ in the scaled variables, for
three reasons. Its signal component satisfies the signal condition $f=0$ of
the proof of Proposition~\ref{prop:D}, which does not use $\lambda\to0$. Its bulk
components are those of Proposition~\ref{prop:D}(iii) with $\lambda$ fixed. Its
dropped momentum components are $O(\varepsilon)$. The damping and noise
terms dropped in \eqref{eq:window-map} are $O(\varepsilon)$ in $C^1$ near
$\hat x^*$, by the derivative bounds of Lemma~\ref{lem:B}. The Jacobian of
the drift recursion at its fixed point therefore converges to $J_0$ as
$\varepsilon\to0$, and for $\varepsilon$ small the real eigenvalue stays
below $-1$, which is (iii).
\end{proof}

\begin{remark}[Numerical observations on the ceiling]\label{rem:W2}
With the tame coefficients of Lemma~\ref{lem:B}, the fixed point has
$R^*=\Phi/(a(1-\lambda/4))$ by Proposition~\ref{prop:D}(iii). With the exact
logistic coefficients at small $p$, continuation of the fixed point of the
drift recursion loses stability somewhat above the $\lambda_c$ computed with the
tame coefficients, by an amount of the
size of $\epsilon_B$ at the fixed point. Below $\lambda_c$ the per-step rates of
the drift recursion, in units of $\ebar$, converge as $\varepsilon\to0$ to
limits that depend on $\lambda$; they tend to the rates of
Proposition~\ref{prop:W1} as $\lambda\to0$ and differ from them near the
ceiling. For $\theta^*<1$ the multipliers can return to
the unit circle on parts of $(2,\lambda_+)$, with $\lambda_+$ the larger root of
$\mathsf P$; we do not use this.
\end{remark}

\begin{remark}[Cold starts leave the window]\label{rem:W2-basin}
The ceiling of Proposition~\ref{prop:W2} is a statement about the fixed
point. From the cold start $\theta_\parallel=0$, zero momentum, the signal
overshoots $r$, and during the swing the curvature factor, and with it $\lambda$,
exceeds its equilibrium value by a factor that grows quickly with $r$.
Numerically, for some $\lambda^*<\lambda_c$, depending on $r$, the exact drift
recursion from the cold start leaves the tame regime and does not return
to the equilibrium, while a warm start, one close to the fixed point,
converges. The curvature window of LR is therefore a statement about starts
in the basin of attraction of the fixed point, that is, starts from which
the recursion converges to it.
\end{remark}

\begin{assumption}[Curvature window and switch]\label{ass:W2}
In cells~7 and~8, let $\lambda^*:=\lim\eta\varepsilon p\varrho^*$, and let $\lambda_c^*$
be the limit of the ceiling $\lambda_c$ of Proposition~\ref{prop:W2} at the fixed
point along the path. Suppose $\lambda^*<\lambda_c^*$. We assume that there are $r_0,c,M>0$,
independent of $d$ along the path, such that from every start within
distance $r_0$ of that fixed point $x^*$ in the variables
$(\theta,\tilde m,R,\tilde v,\tilde c)$ the drift recursion converges to it,
with $|x_k-x^*|\le Me^{-c\ebar k}$ for all $k\ge0$.
\end{assumption}

\noindent The bulk floor of that fixed point is $R_\perp^*=\varphi/(1-H)$ at
$\varrho^*$ (Proposition~\ref{prop:D}(iii)), which is
$\varphi/(1-H_c)\,(1+o(1))$ because the noise feedback vanishes in the tame
regime.

\noindent On the switch, cell~8, $\Phi^*>0$, so $R^*$ has a positive limit
and $\lambda_c^*<2$. Proposition~\ref{prop:W1} does not give the per-step rates
here; Remark~\ref{rem:W2} describes how they depend on $\lambda^*$.

In cell~7, where $\Phi^*=0$, $R^*\to0$ along the path and
$\lambda_c^*=\min(2,4/(1+\theta^{*2}))$. The Jacobian of the drift recursion at
the limiting fixed point is block triangular, as in the proof of
Proposition~\ref{prop:S}(ii) for $\Phi^*=0$, and it inherits the rates of
cell~5 described after Proposition~\ref{prop:W1}. The envelope rate is then
$\ebar(1+o(1))$.

For $\theta^*<1$ the limiting linearization of cell~7 stays on the unit
circle beyond $\lambda=2$, up to $4/(1+\theta^{*2})$, apart from the tangency at
$\lambda=2$. Along the path, however, $R^*>0$ is small, so $\lambda_c<2$, and a pair leaves
the circle on a short interval below $2$, and possibly on parts of
$(2,\lambda_+)$ (Remark~\ref{rem:W2}), with a growth that vanishes as $R^*\to0$
and competes with the damping of order $\varepsilon$. Whether the window of
cell~7 extends past $\lambda=2$ is therefore not settled, and the assumption
excludes it.

What the assumption adds to Proposition~\ref{prop:W2} is the convergence
itself and a rate of order $\ebar$. Compared with LS, the LR curvature
window is narrower, $\lambda^*<\lambda_c^*\le2$ instead of $\lambda<4$, and it requires a warm
start (Remark~\ref{rem:W2-basin}).

\begin{remark}[Why the curvature window needs both restrictions]
Both restrictions in Assumption~\ref{ass:W2} come from the dynamics of the
recursion, not from the method of proof.

\emph{The ceiling.} The mechanism behind Proposition~\ref{prop:W2} is
already visible in the bulk alone. Freeze $\theta$ at $\theta^*$, so that
$\varrho$ varies only through $R$, and linearize the bulk part of
\eqref{eq:window-map} at $\hat x^*$. The Jacobian $J_b$ in $(R,\tilde v,\tilde c)$
has the eigenvalue $1$ and a pair of eigenvalues with product $1$, and
\[
\det(J_b+I)=(\lambda-2)\bigl((2+R^*)\lambda-4\bigr).
\]
For a constant $\varrho$, as in LS, the two roots coincide at $\lambda=2$: the pair
touches $-1$ and returns to the unit circle, which is the tangency of
Remark~\ref{rem:nyquist}. The dependence of $\varrho$ on $R$ separates the
roots, and for $4/(2+R^*)<\lambda<2$ the pair is real, with one eigenvalue below
$-1$. The coupling to the signal through $\partial_\theta\varrho$ moves the
lower root further down, to $\lambda_c$.

Above $\lambda_c$ the unstable eigenvalue is a property of the undamped map, so
its distance from the unit circle does not shrink with $\varepsilon$, while
the damping is of order $\varepsilon$. Numerically, for $r$ up to about
$1$, the loss of stability is abrupt: just above the threshold, a recursion
started next to the fixed point moves away with alternating signs and
leaves the tame regime quickly, on the per-step clock rather than the
retention clock. At larger $r$, where the fixed point is itself outside the
tame regime, we observed instead a bounded oscillation of period close to
two.

\emph{The warm start.} Remark~\ref{rem:W2-basin} describes the first swing
from the cold start. That swing passes $r$ by a fraction of order one, and
because $\varrho$ grows like $e^{\theta_\parallel^2/2}$, the curvature during
it exceeds its equilibrium value by a factor that grows quickly with $r$.
When the excursion is large enough, the next overshoot is amplified rather than
damped, which raises $\varrho$ further, and the recursion leaves the tame
regime and does not return.

In our runs with the exact coefficients this happens, for $r$ of order one
and larger, at values of $\lambda^*$ below $\lambda_c$ from which a warm start
converges. A statement for every $\lambda^*<\lambda_c$ therefore needs a restriction
on the start.

\emph{What is not proved.} For $\lambda<\lambda_c$ the eigenvalues of $J_0$ are simple
and lie on the unit circle (Proposition~\ref{prop:W2}(ii)), so the damping
moves each of them by an amount of order $\varepsilon$ given by first-order
perturbation theory, as in the proof of Proposition~\ref{prop:W1}. Two
steps are missing. The first is that every first-order shift points into
the unit disc, which is the discrete analogue of the identity
$D^\top \mathbf H+\mathbf HD=-S$ used there. The second is a neighbourhood of size
independent of $\varepsilon$: damping of order $\varepsilon$ dominates the
quadratic terms of the map only within distance of order $\varepsilon$ of
the fixed point, and beyond it the nonlinear terms of the undamped map must
average out, which can fail at resonances among its eigenvalues.

Numerically the first-order shifts point into the disc for every $\lambda<\lambda_c$
we tried, and the per-step rates behave as described in
Remark~\ref{rem:W2}. This is a real difference from LS. In the LS curvature
window with $H_n^*=0$, the case closest to LR, the envelope rate is $1$ on
the retention clock for every $\lambda^*$ (Corollary~\ref{cor:curv}(a)); for LR
the rates of Proposition~\ref{prop:W1} are only the small-$\lambda^*$ limit of
the curvature-window rates.
\end{remark}

Every limit above in which the curvature feedback vanishes (all cells except~7
and~8) has the same equilibrium, a fixed point of the drift recursion, which
we now identify; in cells~7 and~8 the bulk floor carries the extra factor
$1/(1-H_c)$ of part~(iii).
Convergence to it is proved in cells~2 to~4 (Propositions~\ref{prop:LR2}
and~\ref{prop:S}) and, at each fixed $\omega^{*2}$, in the windows of cells~5
and~6 (Corollary~\ref{cor:recursion}). In cells~7 and~8 it is
Assumption~\ref{ass:W2}.

\begin{proposition}[Equilibrium and floor]\label{prop:D}
Work in the tame regime with $\lambda,H_n\to0$. Every fixed point of the drift recursion
satisfies $\varrho\theta_\parallel=r\,(1+O(\epsilon_B))$,
$\varrho R_\perp=\Phi\,(1+O(\epsilon_B+\lambda+H_n))$ and
$\eta C_\perp/\varepsilon=-\eta\cA R_\perp/(1+\beta)=O(\lambda/\varepsilon)$; in the
limit, $\varrho\theta_\parallel=r$, $\varrho R_\perp=\Phi$ and $C_\perp=0$.
Equivalently, $\theta^*$ is the unique positive root of
\[
\theta\exp\Bigl(\frac{\theta^2-r^2+\Phi\theta/r}{2}\Bigr)=r,
\qquad\text{and}\qquad R^*=\frac{\Phi\theta^*}r,\quad\varrho^*=\frac r{\theta^*}.
\]
Moreover, the following hold.
\begin{enumerate}
\item[(i)] $0<\theta^*<r$ for $\Phi>0$.
\item[(ii)] $\log\varrho^*=\dfrac{R^*}{2(1+r^2)}+O(\Phi^2)$.
\item[(iii)] Without assuming $\lambda\to0$: for the exact
recursion, $R^*=\varphi/(1-H)$ evaluated at $\varrho^*$; at fixed $\eta$ it
depends on $\varepsilon$ only through $H_c=O(\eta\varepsilon p\varrho)$.
\item[(iv)] Let $L(\theta):=\E\bigl[\log(1+e^{\langle\theta,x\rangle+b^*})-y(\langle\theta,x\rangle+b^*)\bigr]$
be the population logistic loss and $\mathrm{KL}(\theta):=L(\theta)-L(\mu)$.
At equilibrium,
\[
\mathrm{KL}^*=\tfrac12p\varrho^*R^*(1+O(\Phi+\epsilon_B))=\tfrac12p\Phi(1+O(\Phi+\epsilon_B)) .
\]
\item[(v)] Under the co-scaling with $\alpha_\eta\ge1-\sigma$,
$R^*\asymp d^{1-\sigma-\alpha_\eta}$.
\item[(vi)] The Jacobian of the slow system of Proposition~\ref{prop:LR2} at
the equilibrium is
\[
\varrho^*\begin{pmatrix}-(1+\theta^{*2})&-\theta^*/2\\-2\theta^*R^*&-(2+R^*)\end{pmatrix}.
\]
\end{enumerate}
\end{proposition}

\noindent Part (i) says that the equilibrium is biased against alignment.
Part (iv) says that the equilibrium excess risk is
$\mathrm{KL}^*\approx p\eta d/(4B)$ in every cell. Part (v) gives
$R^*\asymp d^{1-\sigma+\kappa-\gamma}$ at $\alpha_\eta=\gamma-\kappa$. In
(vi), the signal curvature is $1+\theta^{*2}$, not $1+r^2$.

The relative error in (iv) cannot be reduced to $O(\Phi)$. At fixed $p$, as
$\Phi\to0$, the ratio $\mathrm{KL}^*/(\frac12p\Phi)$ tends to
$1-\E[\sigma(\langle\mu,x\rangle+b^*)^2]/p$, the curvature of $L$ at $\mu$
in the directions orthogonal to $\mu$, divided by $p$. This limit is below $1$ by an amount of
order $\epsilon_B$.

The condition $\alpha_\eta\ge1-\sigma$ in (v) is part of admissibility, and it
says that $\Phi$ stays bounded. It is needed: as $\Phi\to\infty$ the root
$\theta^*$ tends to $0$, and $R^*=2\log\Phi\,(1+o(1))$ grows more slowly than $\Phi$.

\begin{proof}
At a fixed point the signal and the bulk decouple, except through $\varrho$,
and the problem reduces to one scalar equation.

\emph{Fixed point.} At a fixed point the mean signal gradient vanishes,
$f=0$, which gives $\varrho\theta_\parallel=r\,(1+O(\epsilon_B))$ by
Lemma~\ref{lem:B}. The bulk
fixed point of \eqref{eq:L1} is $R=\varphi/(1-H)$
(Corollary~\ref{cor:stab}(ii)), which tends to $\Phi/\varrho$ as
$\lambda,H_n,\epsilon_B\to0$, and the fixed-point equations for $R_\perp$ and
$C_\perp$ give $C_\perp=-\varepsilon\cA R_\perp/(1+\beta)$.
Since $R_\perp\le q<e^q$ and $\cA\le Cp\varrho$ by Lemma~\ref{lem:B},
\[
|C_\perp|\le\cA R_\perp\le Cp\,e^{(q-r^2)/2}e^{q}=Cp\,e^{(3q-r^2)/2}\le C\epsilon_B ,
\]
so $C_\perp\to0$ in the tame regime.
In the limit, substituting $\varrho\theta_\parallel=r$ and
$\varrho R_\perp=\Phi$ into $\varrho=e^{(\theta^2+R-r^2)/2}$ gives the scalar
equation.

\emph{Uniqueness and (i).} The left side $g(\theta)$ of the scalar equation
increases from $0$ to $\infty$ on $(0,\infty)$, and $g(r)=re^{\Phi/2}>r$.
Hence the root is unique and lies in $(0,r)$.

\emph{(ii).} Put $\theta^*=r(1-\chi)$ and expand the logarithm of the scalar
equation. This gives $\chi=\Phi/(2(1+r^2))+O(\Phi^2)$, and then
$\log\varrho^*=\chi+O(\chi^2)$ and $R^*=\Phi(1-\chi)$.

\emph{(iv).} Since $\sigma(\langle\mu,x\rangle+b^*)$ is the Bayes posterior,
\[
\mathrm{KL}(\theta)=\E_x\,\mathrm{KL}\bigl(\mathrm{Ber}(\sigma(\langle\mu,x\rangle+b^*))\,\big\|\,\mathrm{Ber}(\sigma(\langle\theta,x\rangle+b^*))\bigr)\ge0,
\]
with minimum at $\theta=\mu$, and $\nabla L=\cA\theta+\mathcal B\mu$. The
Hessian at $\mu$ is $\E[\sigma'(\langle\mu,x\rangle+b^*)xx^\top]$; by the Stein
identities of Lemma~\ref{lem:LRdrift} and Lemma~\ref{lem:B} (where $\varrho=1$
at $\theta=\mu$), its block on $\hat\mu^\perp$ is $p(1+O(\epsilon_B))$ times
the identity, its $\hat\mu\hat\mu$ entry is $p(1+r^2)(1+O(\epsilon_B))$, and
its mixed block vanishes. Because $\mathrm{KL}$ is invariant under rotations
of $\theta_\perp$, it is a smooth function of $(s,R_\perp)$ with
$s:=\theta_\parallel-r$, so
\[
\mathrm{KL}=\tfrac12p\bigl[(1+r^2)s^2+R_\perp\bigr](1+O(\epsilon_B))+O\bigl(p(|s|^3+|s|R_\perp+R_\perp^2)\bigr).
\]
At the equilibrium $s=-r\chi=O(\Phi)$ and $R^*=O(\Phi)$, so
$\mathrm{KL}^*=\frac12pR^*(1+O(\Phi+\epsilon_B))$; the factor
$1+O(\epsilon_B)$ of the quadratic form is kept. Finally $\varrho^*=1+O(\Phi)$
and $\varrho^*R^*=\Phi$ exactly.

\emph{(iii), (v), (vi).} For (iii), the bulk fixed point of the exact
recursion with the parameters frozen at the fixed point is $\varphi/(1-H)$ by
Corollary~\ref{cor:stab}(ii). Here $\varphi=\eta(d-1)\mathcal D_0/(2B\cA)$ and
$H_n$ do not involve $\varepsilon$, while $H_c=\eta\varepsilon\cA/(2(1+\beta))$;
and the signal condition $f=0$ does not involve $\varepsilon$ either. So at
fixed $\eta$ the fixed point depends on $\varepsilon$ only through $H_c$.
For (v), since $\theta^*<r$, the exponent in the scalar equation at
$\theta^*$ is less than $\Phi/2$, so $r<\theta^*e^{\Phi/2}$. With
$R^*=\Phi\theta^*/r$ this gives
\[
\Phi e^{-\Phi/2}<R^*<\Phi .
\]
Under the co-scaling, $\Phi=\eta(d-1)/(2B)\asymp d^{1-\sigma-\alpha_\eta}$, which
is bounded when $\alpha_\eta\ge1-\sigma$, and then $R^*\asymp\Phi$ gives (v).
For (vi) one uses $\partial_\theta\varrho=\theta\varrho$ and
$\partial_R\varrho=\varrho/2$.
\end{proof}

Proposition~\ref{prop:D} says, first, that the floor is set by the learning
rate. The equilibrium excess risk $\mathrm{KL}^*\approx p\eta d/(4B)$ depends
on $\eta/B$ and not on the momentum: by (iii), at a fixed learning rate the
floor is flat in $\varepsilon$, up to the curvature feedback. A longer momentum
memory lowers the floor only if it comes with a smaller learning rate.

Second, the equilibrium is biased against alignment. The bulk noise $R^*>0$
raises the curvature factor $\varrho$, and the signal then settles at
$r/\varrho^*<r$. By (ii), the size of this effect involves the signal as well
as the bulk: the factor $1+r^2$ in $\log\varrho^*$ comes from the shift of
the signal, and the approximation $\varrho^*\approx e^{R^*/2}$, which ignores
that shift, is off by this factor in the exponent.

Third, the linearization at the equilibrium, (vi), has signal curvature
$\varrho^*(1+\theta^{*2})$; this, and not $\varrho^*(1+r^2)$, governs the local
rate of the signal in cell~2.

\section{Benefits of momentum}\label{sec:helps-app}\label{sec:benefits_app}

This section proves the statements of Section~\ref{sec:helps} about when momentum speeds up training. It makes three main claims.
\begin{enumerate}
\item \emph{Momentum does not lower the sample cost of a single feature.} At the fastest learning rate, the number of nonzero gradients per decay time of the risk has the closed form~\eqref{eq:sample_cost} to leading order. Along every co-scaling ray with $\gamma>0$ this is at least the cost of SGD, to leading order (Corollary~\ref{cor:sample_cost} and Lemma~\ref{lem:sgd-cost}).
\item \emph{Near a floor, momentum is at most twice as fast.} A target floor fixes the learning rate. At that learning rate, as $\beta\to1$, momentum decays at most twice as fast as SGD, with an explicit ratio at critical damping (Proposition~\ref{prop:fixed_floor}). A decaying learning rate gives SGD the optimal risk without momentum (Remark~\ref{rem:schedule}).
\item \emph{With one learning rate for a vocabulary, momentum raises the critical batch size by $\sqrt{\kappa_{\mathcal V}}$.} The fewest steps per decay time are about $\kappa_{\mathcal V}/4$ for SGD and about $\sqrt{\kappa_{\mathcal V}}/4$ for tuned momentum, where $\kappa_{\mathcal V}$ is the ratio of the largest to the smallest token frequency. For $\kappa_{\mathcal V}>16$ the fewest samples agree up to an absolute factor, so the critical batch sizes differ by $\sqrt{\kappa_{\mathcal V}}$ up to absolute constants (Propositions~\ref{prop:helps-critical} and~\ref{prop:vocab_full}).
\end{enumerate}

Section~\ref{app:rates} collects the tools: the decay rate per step, its exact value for SGD, the retention cap, and the comparison with the continuum cubic. Section~\ref{app:single_feature} proves the first two claims, and Section~\ref{app:vocab} the third. Section~\ref{app:bigram} describes the experiments behind Figure~\ref{fig:vocab}.

\subsection{Decay rates of the per-step chain}\label{app:rates}
We measure convergence by the decay rate of the linear part of~\eqref{eq:L1}. This linear part acts on symmetric $2\times2$ matrices by
\[
\mathcal T(\Sigma):=F\Sigma F^\top+2\lambda\eps\,H_n\,\Sigma_{11}\,bb^\top .
\]
It is the map $\mathcal L$ of Lemmas~\ref{lem:stab-all} and~\ref{lem:retention-cap}. Its decay rate per step and its decay time, in steps, are
\[
\Lambda:=-\ln\rho(\mathcal T)\in(-\infty,\infty],\qquad \frac1\Lambda .
\]

On the retention clock, the corresponding rate is read off the continuum system~\eqref{eq:L3}. Its characteristic polynomial, from Lemma~\ref{lem:L3}(c), is
\[
\chi(z)=(z+1)(z^2+2z+4\omega^2)-4\omega^2 H ,
\]
and we write
\[
r_c(\omega^2,H):=-\max\{\operatorname{Re}z:\chi(z)=0\}
\]
for its asymptotic decay rate. The two rates agree to leading order. By Lemma~\ref{lem:step_spectrum} below, $\Lambda=\eps\,r_c(\omega^2,H)\bigl(1+o(1)\bigr)$ as $\eps\to0$ at fixed $(\omega^2,H)$.

The next lemma records what we use about $r_c$.

\begin{lemma}[Roots of the characteristic polynomial]\label{lem:chi_roots}
Let $\omega^2>0$ and $H\in[0,1)$.
\begin{enumerate}
\item[(a)] $r_c\le1-H$, with equality only if $H=0$ and $\omega^2\ge\frac14$.
\item[(b)] If $H=0$, then $r_c=1-\sqrt{1-4\omega^2}$ for $\omega^2\le\frac14$ and $r_c=1$ for $\omega^2\ge\frac14$.
\item[(c)] If $\omega^2=\frac14$, then $r_c=1-H^{1/3}$.
\item[(d)] Uniformly in $H\in[0,1)$, $r_c=2\omega^2(1-H)\bigl(1+O(\omega^2)\bigr)$ as $\omega^2\to0$.
\item[(e)] $r_c<4\omega^2(1-H)$ for $H\in(0,1)$, and $r_c\le4\omega^2(1-H)$ for $H=0$, with equality only if $\omega^2=\frac14$.
\end{enumerate}
\end{lemma}

\noindent Part~(d) identifies $2\omega^2(1-H)$ as the rate of SGD on the retention clock, and part~(e) says that momentum never exceeds twice that rate. Parts~(b) and~(c) are the two cases in which $\chi$ factors explicitly. Part~(c) is critical damping, the edge of oscillation, where the quadratic factor of $\chi$ is a perfect square.

\begin{proof}
Each part either factors $\chi$ or locates a real root of $\chi$ by a sign change. For the second, recall that $\chi$ is a monic cubic, so $\chi(z)\to+\infty$ as $z\to+\infty$. A point $a$ with $\chi(a)\le0$ therefore has a real root of $\chi$ in $[a,\infty)$, and so $r_c\le-a$.

For~(a) with $H>0$, evaluate $\chi$ at $-(1-H)$ and at $0$:
\[
\chi\bigl(-(1-H)\bigr)=-(1-H)\,H\,(1+H)<0<4\omega^2(1-H)=\chi(0).
\]
So $\chi$ has a real root in $(-(1-H),0)$, and $r_c<1-H$. This is the bound of Lemma~\ref{lem:L3}(c). The case $H=0$ follows from~(b).

For~(b), the polynomial factors when $H=0$, with roots $-1$ and $-1\pm\sqrt{1-4\omega^2}$. The largest real part is $-1+\sqrt{1-4\omega^2}$ for $\omega^2\le\frac14$ and $-1$ for $\omega^2\ge\frac14$.

For~(c), the quadratic factor is a square when $\omega^2=\frac14$, and
\[
\chi(z)=(z+1)^3-H .
\]
Its roots are $-1+H^{1/3}e^{2\pi ik/3}$ for $k=0,1,2$, and the largest real part is $-1+H^{1/3}$.

For~(d), the roots at $\omega^2=0$ are $0$, $-1$ and $-2$, and only the root at $0$ matters. The rescaling $z=\omega^2(1-H)y$ resolves it:
\[
\frac{\chi\bigl(\omega^2(1-H)y\bigr)}{\omega^2(1-H)}=2y+4+\omega^2\bigl(3(1-H)y^2+4y\bigr)+\omega^4(1-H)^2y^3 .
\]
The right side is smooth in $(y,\omega^2,H)$, and at $\omega^2=0$ it has the simple root $y=-2$. The implicit function theorem gives a root $y=-2+O(\omega^2)$, uniformly in $H\in[0,1]$. The other two roots of $\chi$ stay near $-1$ and $-2$.

For~(e), put $c:=4\omega^2(1-H)$, twice the rate in~(d). Substituting $4\omega^2=c/(1-H)$ gives
\[
\chi(-c)=-c\bigl[(1-c)^2+4\omega^2 H\bigr]\le0,
\]
with equality only if $H=0$ and $c=1$. So $\chi$ has a real root in $[-c,\infty)$, and in $(-c,\infty)$ unless $\chi(-c)=0$. Hence $r_c\le c$, strictly unless $H=0$ and $\omega^2=\frac14$.
\end{proof}

Plain SGD is the reference point for every comparison in this section. Without momentum the risk satisfies a scalar recursion, so its decay rate is exact.

\begin{lemma}[Plain SGD]\label{lem:helps-sgd}
Let Assumption~\ref{ass:G} hold with $\beta=0$, so that $\eps=1$, $\lambda=\eta\cA$ and $H=H_n+\lambda/2$. Then the $(1,1)$ entry of~\eqref{eq:L1} is
\[
R_{k+1}=\bigl(1-2\lambda(1-H)\bigr)R_k+2\lambda\varphi,
\]
and $\rho(\mathcal T)=|1-2\lambda(1-H)|$. The linear part is exponentially stable if and only if $H<1$. In that case $0\le1-2\lambda(1-H)<1$ and $\Lambda=-\ln\bigl(1-2\lambda(1-H)\bigr)$. For LS, $H=\eta/\eta_+^{\mathrm{SGD}}$, where $\eta_+^{\mathrm{SGD}}$ is~\eqref{eq:eta-plus} at $\beta=0$, and
\[
\sup_{0<\eta<\eta_+^{\mathrm{SGD}}}2\lambda(1-H)=\frac{p\,\eta_+^{\mathrm{SGD}}}{2}<1,
\]
attained at $\eta=\eta_+^{\mathrm{SGD}}/2$.
\end{lemma}

\noindent With $\beta=0$ the iteration is SGD, $\theta_{k+1}=\theta_k-\eta g_k$. The momentum variable is forgotten after one step, so the risk obeys a recursion of its own. This exact contraction is used in Lemma~\ref{lem:sgd-cost}, in Section~\ref{app:fixed_floor} and in Lemma~\ref{lem:helps-vocab}(v).

\begin{proof}
Because the momentum is forgotten after one step, the linear part has rank one and is carried by its $(1,1)$ entry. We read off the scalar recursion and the spectrum from this, then the stability condition, and finally the LS formulas.

At $\beta=0$ the matrix
\[
F=\begin{pmatrix}1-\lambda&0\\ \lambda&0\end{pmatrix}
\]
has a zero second column, so $F\Sigma F^\top=\Sigma_{11}\,Fe_1e_1^\top F^\top$. Hence
\[
\mathcal T(\Sigma)=\Sigma_{11}\Pi,\qquad \Pi:=Fe_1e_1^\top F^\top+2\lambda H_n\,bb^\top,\qquad \Pi_{11}=(1-\lambda)^2+2\lambda H_n=1-2\lambda(1-H).
\]
The $(1,1)$ entry of~\eqref{eq:L1} is therefore the scalar recursion. The map $\mathcal T$ has rank one and satisfies $\mathcal T(\Pi)=\Pi_{11}\Pi$, so its spectrum is $\{0,\Pi_{11}\}$ and $\rho(\mathcal T)=|\Pi_{11}|$.

A linear map is exponentially stable exactly when its spectral radius is below one. Here that means $|1-2\lambda(1-H)|<1$, or $0<\lambda(1-H)<1$. The upper constraint never binds. Indeed $\lambda(1-H)=\lambda(1-H_n)-\lambda^2/2$, which is negative if $H_n\ge1$. Otherwise it is a concave quadratic in $\lambda$, with maximum $(1-H_n)^2/2\le\frac12$. The linear part is therefore stable exactly when $\lambda(1-H)>0$, that is, when $H<1$, and then $1-2\lambda(1-H)\in[0,1)$.

For LS, \eqref{eq:LSoracle} gives $H_n=\eta(d+2-p)/(2B)$ and $\lambda/2=\eta p/2$, so $H=\eta/\eta_+^{\mathrm{SGD}}$. The quantity to maximize is
\[
2\lambda(1-H)=2\eta p\Bigl(1-\frac{\eta}{\eta_+^{\mathrm{SGD}}}\Bigr),
\]
a concave quadratic in $\eta$ with maximum $p\eta_+^{\mathrm{SGD}}/2$ at $\eta=\eta_+^{\mathrm{SGD}}/2$. This maximum is below one because $1/\eta_+^{\mathrm{SGD}}>p/2$.
\end{proof}

With momentum no such closed form is available. The next lemma gives two upper bounds on the decay rate instead, valid for every $\beta$.

\begin{lemma}[Per-step speed limit]\label{lem:step_speed}\label{lem:helps-onecopy}
Let $\beta\in[0,1)$, $\lambda>0$ and $H_n\ge0$, and put $\omega^2:=\lambda/\eps^2$, so that $\eps\omega^2=\eta\cA$.
\begin{enumerate}
\item[(i)] $\rho(\mathcal T)\ge\rho(F)^2\ge\beta$, with $\rho(\mathcal T)=\rho(F)^2$ if $H_n=0$. Hence $\Lambda\le\ebar=\ln(1/\beta)$.
\item[(ii)] If $4\eps\omega^2<1$, then $\Lambda\le-\ln(1-4\eps\omega^2)$.
\end{enumerate}
\end{lemma}

\noindent Part~(i) is the retention cap of Lemma~\ref{lem:retention-cap}, with equality without gradient noise. No learning rate, batch size or noise feedback makes the chain decay faster than the buffer forgets. Part~(ii) caps the rate, to first order, by $4\eta\cA$. This is twice the rate of SGD at small feedback (Lemma~\ref{lem:helps-sgd}), and it holds for every $\beta$ and every $H_n$.

\begin{proof}
Both parts compare $\mathcal T$ with its noiseless part $\mathcal T_0(\Sigma):=F\Sigma F^\top$. Noise can only slow the chain, and this gives~(i). For~(ii) we use that $F$ has a real eigenvalue close to $1$ when $\eps\omega^2$ is small, and we locate it by a sign change of its characteristic polynomial.

For~(i), observe that $\mathcal T_0$ and the noise term $\mathcal T-\mathcal T_0$, which is $\Sigma\mapsto2\lambda\varepsilon H_n\Sigma_{11}bb^\top$, both map positive semidefinite matrices to positive semidefinite matrices. The telescoping identity
\[
\mathcal T^k-\mathcal T_0^k=\sum_{j=0}^{k-1}\mathcal T^{j}(\mathcal T-\mathcal T_0)\mathcal T_0^{k-1-j}
\]
shows that $\mathcal T^k-\mathcal T_0^k$ has the same property. In particular $\mathcal T_0^k(I)\preceq\mathcal T^k(I)$.

To turn this into a comparison of spectral radii, we use a fact about a linear map $\mathcal S$ that preserves positive semidefiniteness. Its operator norm, with respect to the spectral norm, is $\|\mathcal S(I)\|$. Indeed, if $\|\Sigma\|\le1$, then $-I\preceq\Sigma\preceq I$, so $-\mathcal S(I)\preceq\mathcal S(\Sigma)\preceq\mathcal S(I)$ and $\|\mathcal S(\Sigma)\|\le\|\mathcal S(I)\|$. Applied to $\mathcal T_0^k$ and $\mathcal T^k$, this gives $\|\mathcal T_0^k\|\le\|\mathcal T^k\|$ for every $k$. Gelfand's formula \citep[Corollary~5.6.14]{HornJohnson2013} then gives $\rho(\mathcal T_0)\le\rho(\mathcal T)$.

It remains to compute $\rho(\mathcal T_0)$. If $\nu_1,\nu_2$ are the eigenvalues of $F$, those of $\mathcal T_0$ are $\nu_1^2$, $\nu_1\nu_2$ and $\nu_2^2$. Hence
\[
\rho(\mathcal T_0)=\rho(F)^2\ge|\nu_1\nu_2|=|\det F|=\beta .
\]
When $H_n=0$ the noise term vanishes and $\mathcal T=\mathcal T_0$.

For~(ii), part~(i) gives $\Lambda\le-2\ln\rho(F)$, so it suffices to bound $\rho(F)$ from below. The characteristic polynomial of $F$ is $f(\xi)=\xi^2-(1+\beta-\lambda)\xi+\beta$. It satisfies $f(1)=\lambda>0$ and, with $\beta=1-\eps$ and $\lambda=\eps^2\omega^2$,
\[
f(1-2\eps\omega^2)=-\eps^2\omega^2\bigl(1-(4-2\eps)\omega^2\bigr).
\]
Let $4\eps\omega^2<1$. Two cases arise, according to the sign of $1-(4-2\eps)\omega^2$.

If $(4-2\eps)\omega^2<1$, then $f(1-2\eps\omega^2)<0<f(1)$, so $F$ has a real eigenvalue in $(1-2\eps\omega^2,1)$. Since $1-2\eps\omega^2>\frac12$, this gives $\rho(F)>1-2\eps\omega^2>0$, and
\[
\Lambda<-2\ln(1-2\eps\omega^2)\le-\ln(1-4\eps\omega^2),
\]
because $(1-2x)^2\ge1-4x$. If instead $(4-2\eps)\omega^2\ge1$, then $\eps\le(4-2\eps)\eps\omega^2\le4\eps\omega^2$, and part~(i) gives $\Lambda\le\ebar=-\ln(1-\eps)\le-\ln(1-4\eps\omega^2)$.
\end{proof}

The comparison $\Lambda\approx\eps\,r_c$ between the chain and the continuum system rests on an exact formula for the spectrum of $\mathcal T$. Written as $1+\eps z$, on the retention clock, the eigenvalues of $\mathcal T$ correspond to the roots $z$ of a cubic that differs from $\chi$ by terms of order $\eps$.

\begin{lemma}[Spectrum of the per-step map]\label{lem:step_spectrum}
Let $\beta=1-\eps$, write $\omega^2=\eta p/\eps$ (so that $\lambda=\eps^2\omega^2$), and let $H=H_n+H_c\in[0,1)$. The eigenvalues of $\mathcal T$ are $1+\eps z$, where $z$ ranges over the roots of
\[
q(z)=z^3+(3+\eps b_2)\,z^2+(2+4\omega^2+\eps b_1)\,z+2\omega^2(2-\eps)(1-H),
\]
with $|b_1|,|b_2|\le1+\frac{16}3\omega^2+\eps\omega^4$ for $\eps\le\frac12$. In particular
\[
q(z)-\chi(z)=\eps\bigl(b_2z^2+b_1z-2\omega^2(1-H)\bigr).
\]
\end{lemma}

\begin{proof}
The formula is a direct expansion. Expanding $\det\bigl(zI-(\mathcal T-I)/\eps\bigr)$ gives a monic cubic with constant term $2\omega^2(2-\eps)(1-H)$, and with $b_1,b_2$ determined by
\begin{align*}
(2-\eps)\,b_1&=-(2-\eps)-\omega^2\bigl(12H-4+\eps(6-10H)+\eps^2(2H-2)\bigr)-2\eps(1-\eps)\omega^4,\\
(2-\eps)\,b_2&=-(2-\eps)+\omega^2\bigl(8-\eps(8+4H)+\eps^2(2+2H)\bigr)-2\eps^2(1-\eps)\omega^4.
\end{align*}

For the bounds, let $\eps\le\frac12$ and $H\in[0,1]$. The first bracket increases in $H$ and the second decreases, so they lie in $[-4,8]$ and $[3,8]$ respectively. Dividing by $2-\eps\ge\frac32$, and using $2(1-\eps)\le2-\eps$ for the terms in $\omega^4$, gives $|b_1|,|b_2|\le1+\frac{16}3\omega^2+\eps\omega^4$.
\end{proof}

When $\omega^2$ is small as well as $\eps$, the slow root of $q$ can be computed with explicit constants. The result is that momentum then decays at the rate of SGD with the same learning rate and feedback, up to a relative error of order $\eps+\omega^2$.

\begin{lemma}[Small $\omega^2$: the per-step rate of SGD]\label{lem:small_delta}
If $\eps\le\frac1{50}$ and $\omega^2\le\frac1{50}$, then for every $H\in[0,1)$ the decay rate per step of the chain is
\[
\Lambda=2\eta p\,(1-H)(1+\vartheta),\qquad|\vartheta|\le25(\eps+\omega^2).
\]
\end{lemma}

\begin{proof}
We locate the slow root of $q$ by a sign change, check that the other two roots are far to its left, and then compute the slow root to first order. Write $q(z)=z^3+a_2z^2+a_1z+a_0$. In this range Lemma~\ref{lem:step_spectrum} gives $|b_i|\le1.17$, so
\[
|a_2-3|\le0.024,\qquad a_1\in[1.97,2.11],\qquad 0<a_0\le4\omega^2\le0.08 .
\]

\emph{The slow root.} Put $t_\pm:=(a_0/a_1)(1\pm4a_0)$. At these two points,
\[
q(-t_+)\le a_0^2\Bigl(-4+\frac{a_2(1+4a_0)^2}{a_1^2}\Bigr)<0,\qquad q(-t_-)=4a_0^2+t_-^2(a_2-t_-)>0,
\]
The first bound drops the term $-t_+^3$, and it is negative because $a_2(1+4a_0)^2/a_1^2<1.5$. The second value is positive because $t_-<a_2$. So $q$ has a root between $-t_+$ and $-t_-$, which we write as
\[
z_s=-\frac{a_0}{a_1}(1+\vartheta_1),\qquad|\vartheta_1|\le4a_0\le16\omega^2 .
\]
It satisfies $|z_s|\le t_+\le0.06$.

\emph{The other roots.} Dividing out the slow root, $q(z)=(z-z_s)(z^2+bz+c)$ with $b=a_2+z_s\in[2.9,3.03]$ and $c=a_1+z_sb\in[1.78,2.11]$. The roots of $z^2+bz+c$ have real part at most $-b/2+\sqrt{\max(b^2/4-c,0)}\le-0.7$ and modulus at most $3.1$. The corresponding eigenvalues of $\mathcal T$ satisfy
\[
|1+\eps z|^2\le1-1.4\eps+9.7\eps^2\le1-\eps,
\]
so they decay at rate at least $\eps/2$ per step. The eigenvalue $1+\eps z_s\in(0,1)$ decays at rate $-\ln(1+\eps z_s)\le0.06\,\eps<\eps/2$. It is therefore the slowest, and it sets the decay rate:
\[
\Lambda=\eps|z_s|(1+\vartheta_2),\qquad 0\le\vartheta_2\le\eps .
\]

\emph{The constant.} It remains to compare $a_0/a_1$ with $2\omega^2(1-H)$. Dividing the constant term by the linear coefficient,
\[
\frac{a_0}{a_1}=\frac{2\omega^2(1-H)(1-\eps/2)}{1+x},\qquad x:=2\omega^2+\frac{\eps b_1}2,\qquad|x|\le2\omega^2+0.59\,\eps .
\]
So $a_0/a_1=2\omega^2(1-H)(1+\vartheta_3)$ with $|\vartheta_3|\le3(\eps+\omega^2)$. Since $\eps\omega^2=\eta p$, the three factors combine to
\[
\Lambda=2\eta p(1-H)(1+\vartheta_1)(1+\vartheta_2)(1+\vartheta_3),
\]
and $(1+16\omega^2)(1+\eps)(1+3\eps+3\omega^2)-1\le25(\eps+\omega^2)$ in this range.
\end{proof}

\subsection{A single feature}\label{app:single_feature}
This subsection compares momentum with SGD on one feature in two ways. Far above the floor, we compare the fewest samples per decay time over all learning rates (Corollary~\ref{cor:sample_cost}). Near a floor, which fixes the learning rate, we compare the decay rates at the same learning rate (Proposition~\ref{prop:fixed_floor}).

\paragraph{Sample cost.}\label{app:sample_cost}
Hold every parameter fixed except the learning rate, and let $\Lambda(\eta)$ be the decay rate per step at learning rate $\eta$. The fastest rate and the corresponding cost are
\[
\Lambda^*:=\sup_{0<\eta<\eta_+}\Lambda(\eta),\qquad N_1:=\frac{Bp}{\Lambda^*}.
\]
A feature receives $Bp$ nonzero gradients per step on average. So $N_1$ is the number of nonzero gradients it receives in one decay time at the fastest learning rate.

Two reference costs appear in the answer:
\[
N_{\mathrm{stab}}:=\frac{2B}{\eta_+}=d+2-p+Bp\,\frac{1-\beta}{1+\beta},\qquad N_{\mathrm{mem}}:=\frac{Bp}{\eps}.
\]
The first is set by the stability threshold, and the second counts the nonzero gradients in one retention time $1/\eps$.

As $\eta$ varies, the point $(\omega^2,H)=(\eta p/\eps,\eta/\eta_+)$ moves along a ray from the origin. With $s:=2\eps/(p\eta_+)$, this ray is
\[
2H=s\omega^2,\qquad 0<\omega^2<2/s,
\]
and $s\,N_{\mathrm{mem}}=N_{\mathrm{stab}}$ exactly. The slope $s$ thus compares the two reference costs. The next lemma finds the fastest decay rate of $\chi$ along each ray.

\begin{lemma}[Fastest decay along a ray]\label{lem:ray}
Fix $s>0$, and for $0<\omega^2<2/s$ let $r_c(\omega^2)$ be the decay rate of $\chi$ at $(\omega^2,H)=(\omega^2,s\omega^2/2)$. Then
\[
\sup_{0<\omega^2<2/s}r_c(\omega^2)=\mu^*(s):=1-e^*,\qquad e^*:=\frac{s}{1+\sqrt{1+s^2}},
\]
attained only at $\omega^{*2}=(1-e^{*2})/2=1/(1+\sqrt{1+s^2})$.
\end{lemma}

\noindent The best rate interpolates between the two limits of the slope. For large $s$, where stability binds, $\mu^*(s)$ is close to $1/s$. As $s\to0$, where the memory binds, it tends to the retention rate $1$.

\begin{proof}
We bound $r_c$ on the ray by evaluating $\chi$ at a fixed real point $-\mu$ and maximizing over $\omega^2$. On the ray $4\omega^2 H=2s\omega^4$, so for $\mu\in(0,1)$,
\[
g(\omega^2):=\chi(-\mu)=-\mu(1-\mu)(2-\mu)+4(1-\mu)\omega^2-2s\omega^4 .
\]
This is a concave quadratic in $\omega^2$. Its maximum, at $\omega^2=(1-\mu)/s$, is
\[
(1-\mu)\Bigl[\frac{2(1-\mu)}s-\mu(2-\mu)\Bigr]=\frac{1-\mu}s\bigl(s\mu^2-2(1+s)\mu+2\bigr).
\]
The quadratic $s\mu^2-2(1+s)\mu+2$ vanishes at its smaller root $\mu=\mu^*$ and is negative on $(\mu^*,1)$. Hence $\chi(-\mu^*)\le0$ for every $\omega^2$ on the ray, with equality only at $\omega^{*2}=(1-\mu^*)/s$. On the stable range $\chi(0)=4\omega^2(1-s\omega^2/2)>0$, so $\chi$ has a real root in $[-\mu^*,0)$. Therefore $r_c(\omega^2)\le\mu^*$, strictly for $\omega^2\ne\omega^{*2}$.

It remains to show that the bound is attained at $\omega^{*2}$. There $-\mu^*$ is a simple root, with $\chi'(-\mu^*)=2-2\omega^{*2}>0$. The quotient $\chi(z)/(z+\mu^*)$ is a monic quadratic that is positive at $-\mu^*$, with vertex at $(\mu^*-3)/2<-\mu^*$. Its roots therefore have real parts below $-\mu^*$, and $r_c(\omega^{*2})=\mu^*$.

Finally, the identity
\[
\chi(s\omega^2-1)=s\omega^2\,(s^2\omega^4-1+2\omega^2)
\]
shows that $e^*=s\omega^{*2}=1-\mu^*$ satisfies $e^{*2}=1-2\omega^{*2}$, which is the stated formula for $\omega^{*2}$.
\end{proof}

\begin{corollary}[Sample cost]\label{cor:sample_cost}\label{cor:samplecost}
\begin{enumerate}
\item[(i)] For every $d$, $N_1\ge Bp/\ln(1/\beta)$.
\item[(ii)] Along a co-scaling ray with $\gamma>0$,
\begin{equation}
N_1=\tfrac12\Bigl(N_{\mathrm{stab}}+N_{\mathrm{mem}}+\sqrt{N_{\mathrm{stab}}^2+N_{\mathrm{mem}}^2}\Bigr)\bigl(1+o(1)\bigr).
\label{eq:sample_cost}
\end{equation}
\end{enumerate}
\end{corollary}

\noindent Part~(i) is the retention cap counted in samples. Part~(ii) is a hyperbola in steps and samples. The cost is close to $N_{\mathrm{stab}}$ when $N_{\mathrm{stab}}\gg N_{\mathrm{mem}}$ and close to $N_{\mathrm{mem}}$ in the opposite case, and it is never below $\max(N_{\mathrm{stab}},N_{\mathrm{mem}})$, to leading order.

\begin{proof}
Part~(i) is Lemma~\ref{lem:step_speed}(i), which gives $\Lambda^*\le\ln(1/\beta)$.

For~(ii), we first reduce to a statement on the retention clock. The identity
\[
\bigl(1+s+\sqrt{1+s^2}\bigr)\bigl(1+\sqrt{1+s^2}-s\bigr)=2\bigl(1+\sqrt{1+s^2}\bigr)
\]
gives
\[
\frac{N_{\mathrm{mem}}}{\mu^*(s)}=\tfrac12\Bigl(N_{\mathrm{stab}}+N_{\mathrm{mem}}+\sqrt{N_{\mathrm{stab}}^2+N_{\mathrm{mem}}^2}\Bigr).
\]
Since $N_1=N_{\mathrm{mem}}\,\eps/\Lambda^*$, it suffices to show that $\Lambda^*/\eps=\mu^*(s)\bigl(1+o(1)\bigr)$.

Along a co-scaling ray with $\gamma>0$, $\eps\to0$ and $s$ converges in $[0,\infty]$. Two consequences of Lemma~\ref{lem:step_spectrum} carry the argument. First, for $\omega^2$ in a bounded set the coefficients of $q$ converge to those of $\chi$ uniformly, so the roots converge uniformly, and
\[
\frac\Lambda\eps=-\frac1\eps\ln\max_z|1+\eps z|=r_c(\omega^2,H)+o(1)
\]
uniformly in $H\in[0,1]$ and in $\omega^2$ in compact subsets of $(0,\infty)$. Second, Lemma~\ref{lem:small_delta} controls small $\omega^2$. The three possible limits of $s$ are the three positions relative to the resonance line.
\begin{itemize}
\item \emph{$s\to\infty$ (below the resonance line).} Every stable learning rate has $\omega^2<2/s\to0$. By Lemma~\ref{lem:small_delta}, $\Lambda=2\eps\omega^2(1-H)(1+o(1))=(4\eps/s)\,H(1-H)\,(1+o(1))$, uniformly in $H$. Maximizing over $H$ gives $\Lambda^*=(\eps/s)(1+o(1))$, while $\mu^*(s)=(1/s)(1+O(1/s))$.
\item \emph{$s\to s_0\in(0,\infty)$ (on the resonance line).} Small $\omega^2$ cannot be optimal. By Lemma~\ref{lem:small_delta}, $\Lambda/\eps\le4\omega^2$ for $\omega^2\le\delta_0:=\min(\frac1{50},\mu^*(s_0)/8)$, and this is at most $\mu^*(s_0)/2$. On the compact set $[\delta_0,2/s]$, $\Lambda/\eps=r_c(\omega^2,s\omega^2/2)+o(1)$ uniformly, and $r_c$ is continuous in $s$. The maximizer lies in this set: by Lemma~\ref{lem:chi_roots}(e), $\mu^*(s_0)=r_c(\omega^{*2})\le4\omega^{*2}$, so $\omega^{*2}(s_0)\ge\mu^*(s_0)/4>\delta_0$. Hence $\Lambda^*/\eps\to\sup_{\omega^2} r_c(\omega^2,s_0\omega^2/2)=\mu^*(s_0)$ by Lemma~\ref{lem:ray}.
\item \emph{$s\to0$ (above the resonance line).} The retention cap is attained. By Lemma~\ref{lem:step_speed}(i), $\Lambda\le\ln(1/\beta)=\eps(1+O(\eps))$ for every learning rate. At $\omega^2=\frac12$, so that $H=s/4\to0$, the cubic tends to $(z+1)(z^2+2z+2)$, whose decay rate is $1$, so $\Lambda/\eps\to1$. Hence $\Lambda^*/\eps\to1=\lim\mu^*(s)$. \qedhere
\end{itemize}
\end{proof}

The comparison with SGD is the content of the next lemma.

\begin{lemma}[Sample cost of SGD]\label{lem:sgd-cost}\label{rem:sgd_cost}
At $\beta=0$ the risk contracts by exactly $1-2\eta p(1-H)$ per step, and the best
contraction, at $H=\frac12$, is
\[
1-\frac{Bp}{N_{\mathrm{stab},0}},\qquad N_{\mathrm{stab},0}:=\frac{2B}{\eta_+^{\mathrm{SGD}}}=d+2-p+Bp .
\]
SGD's sample cost is therefore
\[
\frac{Bp}{\ln\bigl(1+Bp/(d+2-p)\bigr)}\le N_{\mathrm{stab},0},
\]
it increases with $B$, and it is at most
\[
N_{\mathrm{stab}}+\eps N_{\mathrm{mem}}\le(1+\eps)\max(N_{\mathrm{stab}},N_{\mathrm{mem}}).
\]
\end{lemma}

\begin{proof}
The contraction is Lemma~\ref{lem:helps-sgd}; it is a concave quadratic in $H$,
maximized at $H=\frac12$. The cost follows from the decay time $1/\Lambda$ with
$\Lambda=-\ln(1-Bp/N_{\mathrm{stab},0})$, and $\ln(1+x)\ge x/(1+x)$ gives the
bound by $N_{\mathrm{stab},0}$. Monotonicity in $B$ is Step~3 of the proof of
Proposition~\ref{prop:helps-critical}. Finally,
$N_{\mathrm{stab},0}-N_{\mathrm{stab}}=2Bp(1-\eps)/(2-\eps)\le\eps N_{\mathrm{mem}}$.
\end{proof}

By~\eqref{eq:sample_cost}, momentum has the cost
$N_1\ge\max(N_{\mathrm{stab}},N_{\mathrm{mem}})(1+o(1))$. Along every co-scaling ray
with $\gamma>0$, momentum therefore does not lower the sample cost to leading order.

\begin{remark}[Critical batch size of a single feature]\label{rem:crit_single}
Let $S:=N_1/(Bp)$ be the decay time in steps at the fastest learning rate, and $E:=BS=N_1/p$ the number of samples per decay time. For one row at a fixed $\beta$, these are $S_\beta(B)$ and $E_\beta(B)$ in the notation of Proposition~\ref{prop:helps-critical}.

In the noise-limited region, $N_{\mathrm{stab}}=(d+2)(1+o(1))$ does not depend on $B$, while $N_{\mathrm{mem}}$ grows linearly in $B$. So
\[
S\approx\max\Bigl(\frac{d+2}{Bp},\frac1\eps\Bigr)
\]
halves with each doubling of the batch until $B\approx(d+2)\eps/p$, and then stays at $S^{(1)}_{\min}=1/\eps$. With $E^{(1)}_{\min}=N_{\mathrm{stab}}/p$, the critical batch size of \citet{mccandlish2018empirical} is
\[
B_{\mathrm{crit}}=\frac{E^{(1)}_{\min}}{S^{(1)}_{\min}}=\frac{N_{\mathrm{stab}}\,\eps}p .
\]

The identity $(N_1-N_{\mathrm{stab}})(N_1-N_{\mathrm{mem}})=N_{\mathrm{stab}}N_{\mathrm{mem}}/2$ turns~\eqref{eq:sample_cost} into their hyperbola
\[
\Bigl(\frac S{S^{(1)}_{\min}}-1\Bigr)\Bigl(\frac E{E^{(1)}_{\min}}-1\Bigr)=\frac12,
\]
with the constant $\frac12$ in place of their $1$. In the $(\kappa,\gamma)$ plane, $N_{\mathrm{stab}}\asymp N_{\mathrm{mem}}$ is the resonance line.
\end{remark}

\paragraph{Speed at a fixed floor.}\label{app:fixed_floor}
With label noise the learning rate sets the floor. A target floor therefore fixes the learning rate, to leading order, for SGD and momentum alike, and we compare the two at the same learning rate.

At $\beta=0$ the risk contracts by exactly $1-2\eta p(1-H)$ per step, with $H=H_n+\eta p/2$ (Lemma~\ref{lem:helps-sgd}). With momentum, Lemma~\ref{lem:step_spectrum} gives the decay rate $\eps\,r_c(\omega^2,H)(1+o(1))$ per step as $\eps\to0$ at fixed $(\omega^2,H)$. SGD at the same learning rate decays at
\[
2\eta p(1-H)\bigl(1+O(\eta p)\bigr)=2\eps\omega^2(1-H)\bigl(1+o(1)\bigr),
\]
where $H$ now denotes momentum's total feedback. This is legitimate because SGD's total feedback $H_n+\eta p/2$ differs from $H=H_n+H_c$ by $O(\eta p)$, which tends to $0$. The ratio of the two rates therefore tends to $r_c(\omega^2,H)/(2\omega^2(1-H))$.

\begin{proposition}[Speed at a fixed floor]\label{prop:fixed_floor}\label{lem:speedup}
Fix $\omega^2>0$ and $H\in[0,1)$, and let $\eps\to0$ with $\omega^2$ and the noise feedback $H_n=H$ fixed, so that $\eta=\eps\omega^2/p\to0$. Then the decay rate per step with momentum, divided by that of SGD at the same learning rate, tends to
\[
\frac{r_c(\omega^2,H)}{2\omega^2(1-H)} .
\]
\begin{enumerate}
\item[(i)] This limit is at most $2$, with equality only if $H=0$ and $\omega^2=\frac14$.
\item[(ii)] At $\omega^2=\frac14$ it equals $2/(1+a+a^2)$, where $a:=H^{1/3}$.
\end{enumerate}
\end{proposition}

\noindent Holding $H_n$ fixed is the same as letting the total feedback tend to $H$, since $H_c=O(\eta p)\to0$. The proposition says that near a floor momentum at most doubles the speed. The doubling is reached only without gradient noise and at critical damping. At critical damping the ratio is explicit for every value of the feedback.

\begin{proof}
The convergence of the ratio is the paragraph before the proposition, so only the value of the limit needs proof. For~(i), put $c:=4\omega^2(1-H)$, twice SGD's rate on the retention clock, so that the limit is $2r_c/c$. Lemma~\ref{lem:chi_roots}(e) gives $r_c\le c$, strictly unless $H=0$ and $\omega^2=\frac14$. In that case $r_c=1=c$ by Lemma~\ref{lem:chi_roots}(b).

For~(ii), Lemma~\ref{lem:chi_roots}(c) gives $r_c=1-a$ at $\omega^2=\frac14$. SGD's rate on the same clock is $2\omega^2(1-H)=(1-a^3)/2$, so the ratio is
\[
\frac{2(1-a)}{1-a^3}=\frac2{1+a+a^2} .
\]
\end{proof}

Lemma~\ref{lem:step_speed}(ii) gives a non-asymptotic counterpart of~(i). Let $\beta\in[0,1)$, let $H=H_n+\eta p/2<1$, and let $4\eta p<1$. Momentum then decays at rate at most $-\ln(1-4\eta p)\le4\eta p/(1-4\eta p)$. SGD at the same learning rate decays at rate $-\ln(1-2\eta p(1-H))\ge2\eta p(1-H)$, by Lemma~\ref{lem:helps-sgd}. The ratio of the two rates is therefore at most
\[
\frac2{(1-H)(1-4\eta p)},
\]
with no limit taken.

Two further comparisons follow from the same formulas. The first is with SGD at twice the learning rate. For $H<\frac12$ and to leading order, critically damped momentum at learning rate $\eta$ decays $(1-a)/(1-2a^3)\ge1-a$ times as fast as SGD at $2\eta$. It also keeps the floor $\varphi/(1-H)$ of SGD at $\eta$, which is about half that of SGD at $2\eta$.

The second concerns small $\omega^2$, where the sign of the gain depends on the feedback. Lemma~\ref{lem:chi_roots}(d), carried to the next order of the small root, gives the ratio
\[
1+\omega^2(1-3H)+O(\omega^4).
\]
So for fixed $H\ne\frac13$ and small $\omega^2$, momentum helps exactly when $H<\frac13$.

\begin{remark}[A decaying learning rate]\label{rem:schedule}
Take $\eta_k=1/(pk)$ and $\beta=0$. The parameters~\eqref{eq:params} are then
\[
\lambda_k=\frac1k,\qquad H_{n,k}=\frac{d+2-p}{2Bpk},\qquad \varphi_k=\frac{\varsigma^2d}{2Bpk},
\]
and, once $k$ is large enough that $\eta_k<\eta_+$, \eqref{eq:L1} reads
\[
R_{k+1}=\Bigl(1-\frac2k+O(k^{-2})\Bigr)R_k+\frac{\varsigma^2d}{Bpk^2}.
\]
A sequence with $a_{k+1}=(1-c/k+O(k^{-2}))a_k+D/k^2$ and $c>1$ satisfies $ka_k\to D/(c-1)$. With $c=2$ this gives
\[
R_k\sim\frac{\varsigma^2d}{Bpk}=\frac{\varsigma^2d}{N_k},
\]
where $N_k=Bpk$ is the expected number of nonzero gradients after $k$ steps. Iterating the exact recursion reproduces this limit.

The risk $\varsigma^2d/N$ is, to leading order, that of least squares on the $N$ samples in which the feature is active. For Gaussian label noise it is the Cramér--Rao bound for unbiased estimators of $\theta_\star$ from those samples, and asymptotically the minimax risk.
\end{remark}

\subsection{A vocabulary}\label{app:vocab}
A vocabulary trained with one $(\eta,\beta)$ is a family of single features that share their hyperparameters. Lemma~\ref{lem:vocab_rows} makes this exact. Lemma~\ref{lem:helps-vocab} bounds the decay rate of the slowest row, and Lemma~\ref{lem:helps-twocurv} shows that the classical parameters of its part~(iv) are optimal without gradient noise. Proposition~\ref{prop:helps-critical} then computes the critical batch sizes, and Proposition~\ref{prop:vocab_full} restates the result for Section~\ref{sec:helps}.

\paragraph{The model.} Let $p_1\ge\dots\ge p_{|\mathcal V|}>0$ with $\sum_jp_j=1$, and let $\theta_\star\in\mathbb{R}^{|\mathcal V|\times d}$ have rows $\theta_{\star,j}$. A sample is a triple $(t,\tilde x,\zeta)$ of independent variables: a token $t$ with $\Pr(t=j)=p_j$, an input $\tilde x\sim\N(0,I_d)$, and label noise $\zeta$ as in Section~\ref{sec:ls-model}. The parameters form a matrix $\Theta\in\mathbb{R}^{|\mathcal V|\times d}$ with rows $\theta_j$. The prediction is $\langle\tilde x,\theta_t\rangle$, and the label and the per-sample loss are
\[
y=\langle\tilde x,\theta_{\star,t}\rangle+\zeta,\qquad \tfrac12\bigl(\langle\tilde x,\theta_t\rangle-y\bigr)^2 .
\]

Equivalently, with the input $x:=e_t\tilde x^\top\in\mathbb{R}^{|\mathcal V|\times d}$, a one-hot vector times a Gaussian vector, the prediction is $\langle x,\Theta\rangle$. This is least squares in $|\mathcal V|d$ dimensions, with mean Hessian $\operatorname{diag}(p)\otimes I_d$.

A minibatch is $B$ independent samples. One pair $(\eta,\beta)$ updates every row:
\[
G=\frac1B\sum_{i=1}^Be_{t_i}\bigl(\langle\tilde x_i,\theta_{t_i}\rangle-y_i\bigr)\tilde x_i^\top,\qquad M\leftarrow\beta M+\eps G,\qquad \Theta\leftarrow\Theta-\eta M .
\]

\begin{lemma}[The rows decouple]\label{lem:vocab_rows}
For each $j$, the pair $(\theta_{j,k},m_{j,k})$ has the law of the single-feature process of Section~\ref{sec:ls-model} with $p=p_j$ and target $\theta_{\star,j}$. Hence $R_{j,k}:=\E|\theta_{j,k}-\theta_{\star,j}|^2$ satisfies~\eqref{eq:L1} with parameters $(\beta,\,\lambda_j,\,H_{n,j},\,\varphi)$, where
\[
\lambda_j=\eta\eps p_j,\qquad H_{n,j}=\frac{\eta(d+2-p_j)}{2B},\qquad \varphi=\frac{\eta\varsigma^2d}{2B},
\]
and the expected excess loss on a fresh sample is $\frac12\sum_jp_jR_{j,k}$.
\end{lemma}

\noindent The rows are not independent, since each sample activates exactly one of them. Only the law of each row enters the loss, however, so the vocabulary behaves like $|\mathcal V|$ single features that share $(\eta,\beta,B)$.

\begin{proof}
Fix a row $j$. The idea is that row $j$ sees exactly the gated samples of Section~\ref{sec:ls-model}, with the gate given by the event that the token is $j$. Write $s_{j,i}:=\mathbf 1[t_i=j]$. Row $j$ of $G$ is
\[
\frac1B\sum_{i=1}^Bs_{j,i}\,\tilde x_i\bigl(\langle\tilde x_i,\theta_j\rangle-y_i\bigr),
\]
and on $\{s_{j,i}=1\}$ the label is $y_i=\langle\tilde x_i,\theta_{\star,j}\rangle+\zeta_i$.

The triples $(s_{j,i},\tilde x_i,\zeta_i)$ are independent across $i$, and $s_{j,i}\sim\Ber(p_j)$ is independent of $(\tilde x_i,\zeta_i)$. Put $x_i:=s_{j,i}\tilde x_i$. The per-sample gradients $s_{j,i}\tilde x_i(\langle\tilde x_i,\theta_j-\theta_{\star,j}\rangle-\zeta_i)$ are then exactly those of the gated samples $(x_i,\langle x_i,\theta_{\star,j}\rangle+\zeta_i)$ of Section~\ref{sec:ls-model}.

The momentum and the update act on each row with the same scalars. The row-$j$ process is therefore a function of these triples alone, and it has the law of the single-feature process. Lemma~\ref{lem:L1} gives the recursion. The gates of different rows are dependent, since $\sum_js_{j,i}=1$. This changes the joint law of the rows, but not the law of any one row.

For the loss, condition a fresh sample on its token. The excess loss is $\frac12\sum_jp_j|\theta_j-\theta_{\star,j}|^2$, and its expectation is $\frac12\sum_jp_jR_{j,k}$.
\end{proof}

\paragraph{Notation for a vocabulary.} Let $|\mathcal V|\ge2$, write $p_{\max}=p_1\ge\dots\ge p_{|\mathcal V|}=p_{\min}>0$, and let
\[
\kappa_{\mathcal V}:=\frac{p_{\max}}{p_{\min}}
\]
be the frequency imbalance. By Lemma~\ref{lem:vocab_rows}, row $j$ is the single-feature model with $p=p_j$. For row $j$ we write $\lambda_j:=\eta\eps p_j$ and $\omega_j^2:=\eta p_j/\eps$, and $H_{n,j}$, $H_{c,j}$, $H_j$, $\mathcal T_j$ and $\Lambda_j$ for its noise, curvature and total feedback~\eqref{eq:params}, linear part and decay rate. We write $\eta_+(p_j)$ for~\eqref{eq:eta-plus} with $p=p_j$. The loss decays at the rate of its slowest row,
\[
\Lambda_{\min}:=\min_j\Lambda_j .
\]

\begin{lemma}[One learning rate for a vocabulary]\label{lem:helps-vocab}
Train the vocabulary of Section~\ref{app:vocab} with one $(\eta,\beta)$.
\begin{enumerate}
\item[(i)] Every $\mathcal T_j$ is exponentially stable if and only if $\eta<\min\bigl(\eta_+(p_{\max}),\eta_+(p_{\min})\bigr)$. If $B>(1+\beta)/(1-\beta)$, this minimum is $\eta_+(p_{\max})$.
\item[(ii)] Let $\beta=0$ and let every $\mathcal T_j$ be stable. Then
\[
\frac1{\Lambda_{\min}}\ge\frac{d+2-p_{\min}}{Bp_{\min}}-1,
\]
and if $\kappa_{\mathcal V}>4$, also $1/\Lambda_{\min}\ge\kappa_{\mathcal V}/4-1$.
\item[(iii)] Let $\beta\in[0,1)$ and let every $\mathcal T_j$ be stable. Then
\[
\frac1{\Lambda_{\min}}\ge\frac{d+2-p_{\min}}{8Bp_{\min}}-1,
\]
and if $\kappa_{\mathcal V}>16$, also $1/\Lambda_{\min}\ge\sqrt{\kappa_{\mathcal V}}/4-1$. If moreover the parameters of row $|\mathcal V|$ vary along a co-scaling ray with $\gamma>0$, then
\[
\frac1{\Lambda_{\min}}\ge\frac{\max(N_{\mathrm{stab}},N_{\mathrm{mem}})}{Bp_{\min}}\bigl(1+o(1)\bigr)\ge\frac{d+2-p_{\min}}{Bp_{\min}}\bigl(1+o(1)\bigr),
\]
with $N_{\mathrm{stab}}$ and $N_{\mathrm{mem}}$ of Section~\ref{app:sample_cost} at $p=p_{\min}$.
\item[(iv)] Let $\kappa_{\mathcal V}>1$, let $\Lambda_j^0$ be the decay rate of row $j$ with $H_{n,j}$ replaced by $0$, and put $\Lambda_{\min}^0:=\min_j\Lambda_j^0$.
\begin{enumerate}
\item[(a)] If $\beta=0$ and $\eta=2/(p_{\max}+p_{\min})$, then $\Lambda_{\min}^0=2\ln\bigl((\kappa_{\mathcal V}+1)/(\kappa_{\mathcal V}-1)\bigr)$ and $1/\Lambda_{\min}^0\le\kappa_{\mathcal V}/4$.
\item[(b)] If $\beta\in(0,1)$, then $\Lambda_j^0\le\ebar$, with equality if and only if $(1-\sqrt\beta)^2\le \lambda_j\le(1+\sqrt\beta)^2$. If
\[
\sqrt\beta=\frac{\sqrt{\kappa_{\mathcal V}}-1}{\sqrt{\kappa_{\mathcal V}}+1}\qquad\text{and}\qquad\eta\eps p_{\max}=(1+\sqrt\beta)^2,
\]
then $\Lambda_j^0=\ebar$ for every $j$, and $1/\Lambda_{\min}^0=1/\ebar\le\sqrt{\kappa_{\mathcal V}}/4$.
\item[(c)] At the parameters of~(a) and~(b), every row with $H_{n,j}$ replaced by $0$ is stable. For $B$ large every row is stable, and $\Lambda_{\min}\to\Lambda_{\min}^0$ as $B\to\infty$ with $(\eta,\beta)$, $(p_j)$ and $d$ fixed.
\end{enumerate}
\item[(v)] Let $\beta=0$, and put
\[
D:=d+2-p_{\min},\qquad \eta^*_j:=\frac B{d+2+(B-1)p_j},\qquad D_B:=d+2+(B-1)(p_{\max}+p_{\min}).
\]
Maximize $\Lambda_{\min}$ over the $\eta$ at which every $\mathcal T_j$ is stable. If $(B-1)(p_{\max}-p_{\min})\le d+2$, then
\[
\max_\eta\Lambda_{\min}=\ln\Bigl(1+\frac{Bp_{\min}}D\Bigr),
\]
attained at $\eta=\eta^*_{|\mathcal V|}$, and $\eta^*_{|\mathcal V|}/\eta^*_j\in[1,2)$ for every $j$. Otherwise
\[
\max_\eta\Lambda_{\min}=-\ln\Bigl(1-\frac{4B(B-1)p_{\max}p_{\min}}{D_B^2}\Bigr),
\]
attained at $\eta=2B/D_B$.
\end{enumerate}
\end{lemma}

\noindent The decay time of the loss is that of its slowest row, $1/\Lambda_{\min}$. Parts~(ii) and~(iii) bound it from below for every learning rate and every batch size. The first bound in each part comes from the noise term $(d+2-p)/(2B)$ of~\eqref{eq:eta-plus}, which does not depend on $\beta$. This bound falls like $1/(Bp_{\min})$ as the batch grows. The second bound comes from the curvature term of~\eqref{eq:eta-plus} for the most frequent row, $j=1$, which caps the learning rate. It does not depend on $B$: it is $\kappa_{\mathcal V}/4$ for SGD and $\sqrt{\kappa_{\mathcal V}}/4$ with momentum. Part~(v) gives the best decay rate of SGD at every batch size exactly.

Part~(iv) shows that both large-batch bounds are attained, up to the additive constant, when the gradient noise vanishes. The choices in~(iv) are the classical optimal parameters of gradient descent and of Polyak's heavy-ball method on quadratics with curvatures in $[p_{\min},p_{\max}]$ \citep{polyak1964some}, with step $\eta$ in~(a) and $\eta\eps$ in~(b). In~(b) they place every row in the range $|\cos\theta_j|\le1$. There $F_j$ has complex or double eigenvalues of modulus $\sqrt\beta$, so the noiseless map $\Sigma\mapsto F_j\Sigma F_j^\top$ decays at exactly $\beta^k$ (Lemma~\ref{lem:retention-cap}). In terms of $\omega_j^2$, the lower end of this range is $\omega_j^2\ge1/(1+\sqrt\beta)^2$. As $\beta\to1$ it tends to the critically damped value $\omega^2=\frac14$ of Proposition~\ref{prop:fixed_floor}(ii).

The proofs use only the per-row parameters, so (i)--(v) hold for any $p_j\in(0,1]$, whether or not they sum to one.

\begin{proof}
Since $\Lambda_{\min}\le\Lambda_{|\mathcal V|}$, each lower bound on $1/\Lambda_{\min}$ in~(ii) and~(iii) follows from an upper bound on the rate $\Lambda_{|\mathcal V|}$ of the rarest row. Stability supplies two constraints, used throughout. Row $|\mathcal V|$ must have $\eta<\eta_+(p_{\min})$, which bounds $\eta$ through the noise term of~\eqref{eq:eta-plus}. Row $1$ must be stable as well, which bounds $\eta$ through the curvature term. Part~(iv) computes the rates at the classical parameters and lets $B\to\infty$. Part~(v) maximizes the exact rates of SGD over $\eta$.

Several steps convert a bound $x\le y<1$ on the contraction $x$ of a row into a bound on its decay time. They all use
\begin{equation}
-\ln(1-y)\le\frac y{1-y},\qquad\text{so that}\qquad \frac1{-\ln(1-x)}\ge\frac1y-1\quad\text{for }0<x\le y<1 .
\label{eq:helps-log}
\end{equation}

\emph{(i).} By Lemma~\ref{lem:stab-all} and the LS identity $H=\eta/\eta_+$, row $j$ is stable exactly when $\eta<\eta_+(p_j)$. By~\eqref{eq:eta-plus},
\[
\frac1{\eta_+(p)}=\frac{d+2}{2B}+p\Bigl(\frac\eps{2(2-\eps)}-\frac1{2B}\Bigr)
\]
is affine in $p$. Its largest value over $p_{\min}\le p\le p_{\max}$ is therefore attained at $p_{\max}$ or at $p_{\min}$. It is attained at $p_{\max}$ when the slope is positive, that is, when $B>(2-\eps)/\eps=(1+\beta)/(1-\beta)$.

\emph{(ii).} By Lemma~\ref{lem:helps-sgd}, the rarest row decays at rate $\Lambda_{|\mathcal V|}=-\ln(1-x_{|\mathcal V|})$, where
\[
x_{|\mathcal V|}:=2\eta p_{\min}(1-H_{|\mathcal V|}),\qquad H_{|\mathcal V|}=\frac\eta{\eta_+(p_{\min})}<1 .
\]
Here $\eta_+(p_{\min})$ is the threshold $\eta_+^{\mathrm{SGD}}$ of that lemma at $p=p_{\min}$, since $\beta=0$. The same lemma gives $0<x_{|\mathcal V|}\le p_{\min}\eta_+(p_{\min})/2<1$, so $\Lambda_{|\mathcal V|}$ is finite. The two bounds come from two upper bounds on $x_{|\mathcal V|}$.

For the first, the noise term gives $1/\eta_+(p_{\min})\ge(d+2-p_{\min})/(2B)$, and so
\[
x_{|\mathcal V|}\le\frac{p_{\min}\eta_+(p_{\min})}2\le y:=\frac{Bp_{\min}}{d+2-p_{\min}} .
\]
If $y\ge1$ the first bound is trivial. Otherwise~\eqref{eq:helps-log} gives $1/\Lambda_{\min}\ge1/\Lambda_{|\mathcal V|}\ge1/y-1$.

For the second, the stability of row $1$ gives $\eta<\eta_+(p_{\max})<2/p_{\max}$. Hence $x_{|\mathcal V|}<2\eta p_{\min}<4/\kappa_{\mathcal V}$, and~\eqref{eq:helps-log} gives $1/\Lambda_{\min}\ge\kappa_{\mathcal V}/4-1$.

\emph{(iii).} Lemma~\ref{lem:step_speed} bounds the rate of the rarest row in two ways:
\[
\Lambda_{|\mathcal V|}\le\ebar=-\ln(1-\eps),\qquad \Lambda_{|\mathcal V|}\le-\ln(1-4\eta p_{\min})\quad\text{if }4\eta p_{\min}<1 .
\]
Each lower bound of~(iii) comes from one of these, combined with a stability constraint.

For the first bound, the stability of row $|\mathcal V|$ gives $\eta<\eta_+(p_{\min})\le2B/D$, with $D:=d+2-p_{\min}$. So $4\eta p_{\min}<y:=8Bp_{\min}/D$. If $y<1$, then $\Lambda_{|\mathcal V|}\le-\ln(1-y)$, and~\eqref{eq:helps-log} gives $1/\Lambda_{\min}\ge1/y-1$.

For the second, the stability of row $1$ gives $H_{c,1}<1$, that is, $\lambda_1<2(1+\beta)$, or $\eta p_{\max}<2(2-\eps)/\eps$. Hence $4\eta p_{\min}<16/(\eps\kappa_{\mathcal V})$. The two bounds on $\Lambda_{|\mathcal V|}$ now pull in opposite directions. The first is small when $\eps$ is small, that is, when $\beta$ is close to $1$, and the second is small when $\eps$ is large. Their minimum is largest where they balance, at $\eps=4/\sqrt{\kappa_{\mathcal V}}$:
\[
\min\Bigl(\eps,\frac{16}{\eps\kappa_{\mathcal V}}\Bigr)\le\frac4{\sqrt{\kappa_{\mathcal V}}}<1 .
\]
One of the two bounds therefore gives $\Lambda_{|\mathcal V|}\le-\ln(1-4/\sqrt{\kappa_{\mathcal V}})$, and~\eqref{eq:helps-log} gives $1/\Lambda_{\min}\ge\sqrt{\kappa_{\mathcal V}}/4-1$.

For the last claim, let $\Lambda^*$, $N_1$, $N_{\mathrm{stab}}$ and $N_{\mathrm{mem}}$ be the quantities of Section~\ref{app:sample_cost} for $p=p_{\min}$. The stability of row $|\mathcal V|$ restricts $\eta$ to $(0,\eta_+(p_{\min}))$, so $\Lambda_{\min}\le\Lambda_{|\mathcal V|}\le\Lambda^*=Bp_{\min}/N_1$. Corollary~\ref{cor:sample_cost}(ii) and the inequality $\frac12(a+b+\sqrt{a^2+b^2})\ge\max(a,b)$ for $a,b\ge0$ then give
\[
\frac1{\Lambda_{\min}}\ge\frac{N_1}{Bp_{\min}}\ge\frac{\max(N_{\mathrm{stab}},N_{\mathrm{mem}})}{Bp_{\min}}\bigl(1+o(1)\bigr),
\]
and $N_{\mathrm{stab}}=2B/\eta_+(p_{\min})\ge d+2-p_{\min}$.

\emph{(iv)(a).} With $\beta=0$ and $H_{n,j}=0$, Lemma~\ref{lem:helps-sgd} gives $H_j=\lambda_j/2$ and
\[
\Lambda_j^0=-\ln\bigl(1-2\lambda_j+\lambda_j^2\bigr)=-2\ln|1-\lambda_j| .
\]
For $\eta=2/(p_{\max}+p_{\min})$, the curvature $\lambda_j=2p_j/(p_{\max}+p_{\min})$ lies in $[2/(\kappa_{\mathcal V}+1),2\kappa_{\mathcal V}/(\kappa_{\mathcal V}+1)]$. So $\lambda_j<2$, every row is stable, and $|1-\lambda_j|\le(\kappa_{\mathcal V}-1)/(\kappa_{\mathcal V}+1)$, with equality for $j=1$ and $j=|\mathcal V|$. Hence $\Lambda_{\min}^0=2\ln((\kappa_{\mathcal V}+1)/(\kappa_{\mathcal V}-1))$. Since $\ln\frac{1+x}{1-x}\ge2x$ for $x\in(0,1)$, $\Lambda_{\min}^0\ge4/\kappa_{\mathcal V}$.

\emph{(iv)(b).} With $H_{n,j}=0$ the linear part is $\mathcal T_j^0(\Sigma)=F_j\Sigma F_j^\top$, and Lemma~\ref{lem:step_speed}(i) gives $\rho(\mathcal T_j^0)=\rho(F_j)^2$. So $\Lambda_j^0=-2\ln\rho(F_j)\le-\ln\beta=\ebar$. Equality holds exactly when $|\cos\theta_j|\le1$ (Lemma~\ref{lem:retention-cap}), that is, when
\[
-2\sqrt\beta\le1+\beta-\lambda_j\le2\sqrt\beta,\quad\text{or equivalently}\quad(1-\sqrt\beta)^2\le \lambda_j\le(1+\sqrt\beta)^2.
\]
For the stated choice, $(1+\sqrt\beta)/(1-\sqrt\beta)=\sqrt{\kappa_{\mathcal V}}$. The curvature $\lambda_j=(1+\sqrt\beta)^2p_j/p_{\max}$ therefore lies in
\[
\Bigl[\frac{(1+\sqrt\beta)^2}{\kappa_{\mathcal V}},(1+\sqrt\beta)^2\Bigr]=\bigl[(1-\sqrt\beta)^2,(1+\sqrt\beta)^2\bigr],
\]
and $\Lambda_j^0=\ebar$ for every $j$. Every row is stable, because $H_{c,j}\le H_{c,1}<1$ is the inequality $(1+\sqrt\beta)^2<2(1+\beta)$, which is $(1-\sqrt\beta)^2>0$. Finally,
\[
\ebar=-\ln\beta=2\ln\frac{1+1/\sqrt{\kappa_{\mathcal V}}}{1-1/\sqrt{\kappa_{\mathcal V}}}\ge\frac4{\sqrt{\kappa_{\mathcal V}}} .
\]

\emph{(iv)(c).} In~(a) and~(b) the curvature feedbacks $H_{c,j}$ are below one. For $B$ large, every row is therefore stable with $H_{n,j}=\eta(d+2-p_j)/(2B)$. The coefficients of the characteristic polynomial of the matrix of $\mathcal T_j$ are polynomials in $H_{n,j}$, and the roots of a monic polynomial depend continuously on its coefficients. Hence $\rho(\mathcal T_j)\to\rho(\mathcal T_j^0)$ as $B\to\infty$, for each of the finitely many $j$. Since $-\ln$ is continuous from $[0,1)$ into $(0,\infty]$, $\Lambda_j\to\Lambda_j^0$ in $(0,\infty]$, and so $\Lambda_{\min}\to\Lambda_{\min}^0$. This limit is finite. In~(a) it is attained at $j=1$ or $j=|\mathcal V|$, where $|1-\lambda_j|=(\kappa_{\mathcal V}-1)/(\kappa_{\mathcal V}+1)>0$, and in~(b) it equals $\ebar$.

\emph{(v).} At $\beta=0$ every rate is exact, and the problem reduces to the two extreme rows. By Lemma~\ref{lem:helps-sgd}, row $j$ contracts by exactly $1-g_j(\eta)$ per step, where
\[
g_j(\eta)=2\eta p_j(1-H_j)=p_j\eta\Bigl(2-\frac\eta{\eta^*_j}\Bigr),
\]
since $H_j=\eta(d+2+(B-1)p_j)/(2B)$ at $\beta=0$. Row $j$ is stable exactly when $g_j(\eta)>0$, that is, when $\eta<2\eta^*_j=\eta_+(p_j)$. The loss then decays at rate
\[
\Lambda_{\min}=-\ln\Bigl(1-\min_jg_j(\eta)\Bigr).
\]

Two observations reduce the minimum over $j$ to the extreme rows. At fixed $\eta$, $g$ is a concave quadratic in $p$, since its coefficient of $p^2$ is $-\eta^2(B-1)/B$. So $\min_jg_j$ is attained at $p_{\max}$ or at $p_{\min}$. In $\eta$, each $g_j$ is concave with its maximum at $\eta^*_j$, and $\eta^*_1\le\eta^*_{|\mathcal V|}$.

The question is whether the rarest token can run at its own optimum. At $\eta^*_{|\mathcal V|}$ the most frequent token decays at least as fast as the rarest exactly when $g_1(\eta^*_{|\mathcal V|})\ge g_{|\mathcal V|}(\eta^*_{|\mathcal V|})$, and this rearranges to $(B-1)(p_{\max}-p_{\min})\le d+2$. In that case the best $\eta$ runs the rarest token at its own optimum, where
\[
g_{|\mathcal V|}=\frac{Bp_{\min}}{d+2+(B-1)p_{\min}},\qquad 1-g_{|\mathcal V|}=\frac D{d+2+(B-1)p_{\min}} .
\]
No per-token learning rates can do better. Every other token's optimum satisfies
\[
\frac{\eta^*_{|\mathcal V|}}{\eta^*_j}=\frac{d+2+(B-1)p_j}{d+2+(B-1)p_{\min}}\in[1,2),
\]
so every row is stable there.

Otherwise $g_1<g_{|\mathcal V|}$ at $\eta^*_{|\mathcal V|}$. At $\eta^*_1$, on the other hand, the difference $g_1-g_{|\mathcal V|}$ is a positive multiple of $(p_{\max}-p_{\min})\bigl(d+2+(B-1)(p_{\max}-p_{\min})\bigr)>0$. Between these two points $g_1$ decreases and $g_{|\mathcal V|}$ increases, so the maximin over $\eta$ is where $g_1=g_{|\mathcal V|}$. This happens at $\eta=2B/D_B$, which is below $\eta_+(p_j)=2B/(d+2+(B-1)p_j)$ for every $j$, and the common value is $4B(B-1)p_{\max}p_{\min}/D_B^2$.
\end{proof}

The next lemma shows that the choices of Lemma~\ref{lem:helps-vocab}(iv) are optimal without gradient noise. It is the classical lower bound for one step size on two curvatures, and we include a short proof.

\begin{lemma}[One step size for two curvatures]\label{lem:helps-twocurv}
Let $p_{\max}>p_{\min}>0$, $\kappa_{\mathcal V}:=p_{\max}/p_{\min}$, $\beta\in[0,1)$ and $c>0$. Write $p_1:=p_{\max}$ and $p_{|\mathcal V|}:=p_{\min}$, and for $j\in\{1,|\mathcal V|\}$ let $F_j$ be the matrix $F$ of~\eqref{eq:L1} with $\lambda=cp_j$. Then
\[
\max\bigl(\rho(F_1),\rho(F_{|\mathcal V|})\bigr)\ge\frac{\sqrt{\kappa_{\mathcal V}}-1}{\sqrt{\kappa_{\mathcal V}}+1},
\]
and if $\beta=0$, also $\max\bigl(\rho(F_1),\rho(F_{|\mathcal V|})\bigr)\ge(\kappa_{\mathcal V}-1)/(\kappa_{\mathcal V}+1)$.
\end{lemma}

\noindent In Lemma~\ref{lem:helps-vocab}, $c=\eta\eps$. With $H_{n,j}$ replaced by $0$, the decay rate of row $j$ is $-2\ln\rho(F_j)$ (Lemma~\ref{lem:step_speed}(i)). Without gradient noise, then, no $(\eta,\beta)$ gives the slowest row a rate above $2\ln\frac{\sqrt{\kappa_{\mathcal V}}+1}{\sqrt{\kappa_{\mathcal V}}-1}$, and no learning rate of SGD a rate above $2\ln\frac{\kappa_{\mathcal V}+1}{\kappa_{\mathcal V}-1}$. The choices of Lemma~\ref{lem:helps-vocab}(iv) attain both.

Polyak's optimal heavy-ball rate on the interval $[cp_{\min},cp_{\max}]$ \citep{polyak1964some} gives the same bound. Indeed, $\rho(F)$ is nonincreasing in $\lambda$ up to $1+\beta$ and nondecreasing after it (Step~6 of the proof of Proposition~\ref{prop:helps-critical}), so the two extreme curvatures carry the whole interval.

\begin{proof}
Both eigenvalues of each $F_j$ lie in the disc of radius $r:=\max(\rho(F_1),\rho(F_{|\mathcal V|}))$, so each trace is at most $2r$ in absolute value. The traces of $F_1$ and $F_{|\mathcal V|}$ differ by $c(p_{\max}-p_{\min})=(\kappa_{\mathcal V}-1)cp_{\min}$, while $cp_{\min}$ is at least $(1-r)^2$. Together these force $r$ up to the stated bound.

The matrix $F_j$ has trace $t_j:=1+\beta-cp_j$ and determinant $\beta$, so its characteristic polynomial is $f_j(z)=z^2-t_jz+\beta$. If $\beta>0$, the eigenvalues of $F_j$ have product $\beta$, so $r\ge\sqrt\beta>0$. If $\beta=0$, they are $0$ and $1-cp_j$, and $r>0$ because $cp_{\max}\ne cp_{\min}$. Either way $\beta\le r^2$. The two lower bounds of the lemma are below one, so we may assume $r<1$.

We first bound the traces. Let $j\in\{1,|\mathcal V|\}$, and let $z,z'$ be the eigenvalues of $F_j$, so that $f_j(x)=(x-z)(x-z')$ for real $x$. If they are not real, then $z'=\bar z$ and $f_j(x)=|x-z|^2\ge0$. If they are real, they lie in $[-r,r]$, so $f_j(r)=(r-z)(r-z')\ge0$ and $f_j(-r)=(r+z)(r+z')\ge0$. In both cases $r^2\mp t_jr+\beta\ge0$, that is,
\[
|t_j|\le r+\frac\beta r\le2r,
\]
the last inequality because $\beta\le r^2$.

Next we combine the two curvatures. For $j=|\mathcal V|$ the trace bound gives
\[
cp_{\min}=1+\beta-t_{|\mathcal V|}\ge1+\beta-r-\frac\beta r=(1-r)\Bigl(1-\frac\beta r\Bigr)\ge(1-r)^2,
\]
again by $\beta/r\le r$. For $j=1$ and $j=|\mathcal V|$ together it gives $c(p_{\max}-p_{\min})=t_{|\mathcal V|}-t_1\le4r$. Since $c(p_{\max}-p_{\min})=(\kappa_{\mathcal V}-1)cp_{\min}$,
\[
g(r):=(\kappa_{\mathcal V}-1)(1-r)^2-4r\le0 .
\]
The function $g$ is strictly decreasing on $[0,1]$, and it vanishes at $q:=(\sqrt{\kappa_{\mathcal V}}-1)/(\sqrt{\kappa_{\mathcal V}}+1)$, because $1-q=2/(\sqrt{\kappa_{\mathcal V}}+1)$ and $(\kappa_{\mathcal V}-1)\cdot4/(\sqrt{\kappa_{\mathcal V}}+1)^2=4q$. Hence $r\ge q$.

If $\beta=0$, the sharper bound $|t_j|\le r$ applies. It gives $cp_{\min}=1-t_{|\mathcal V|}\ge1-r$ and $c(p_{\max}-p_{\min})\le2r$, so $(\kappa_{\mathcal V}-1)(1-r)\le2r$, that is, $r\ge(\kappa_{\mathcal V}-1)/(\kappa_{\mathcal V}+1)$.
\end{proof}

\paragraph{Critical batch sizes with explicit constants.}\label{app:critical_batch}
Following \citet{mccandlish2018empirical}, the critical batch size is the fewest samples divided by the fewest steps, each minimized over all batch sizes and hyperparameters. Proposition~\ref{prop:helps-critical} computes the fewest steps in closed form, for SGD, for every fixed $\beta$ and with $\beta$ tuned. It bounds the fewest samples of momentum within an absolute factor of those of SGD. Proposition~\ref{prop:vocab_full} then reads off the statement of Section~\ref{sec:helps}.

\begin{proposition}[Critical batch size of a vocabulary]\label{prop:helps-critical}
Consider the rows of Lemma~\ref{lem:helps-vocab}, with $d$ and $p_1\ge\dots\ge p_{|\mathcal V|}>0$ fixed, and let $\kappa_{\mathcal V}>1$. Let the optimizer class $O$ be SGD, $\beta=0$, or momentum, $\beta\in[0,1)$, or, in part~(iv), a single fixed $\beta$. For an integer $B\ge1$, let $\mathcal A_O(B)$ be the set of $(\eta,\beta)$, with $\eta>0$ and $\beta$ in the class, at which every $\mathcal T_j$ is exponentially stable, and put
\[
S_O(B):=\inf_{(\eta,\beta)\in\mathcal A_O(B)}\frac1{\Lambda_{\min}},\qquad
E_O(B):=B\,S_O(B),
\]
\[
S_O:=\inf_{B\ge1}S_O(B),\qquad E_O:=\inf_{B\ge1}E_O(B),\qquad
B_O:=\frac{E_O}{S_O}.
\]
Then the following hold.
\begin{enumerate}
\item[(i)] As $B\to\infty$, $S_O(B)\to S_O$, where
\[
S_{\mathrm{SGD}}=\frac1{2\ln\frac{\kappa_{\mathcal V}+1}{\kappa_{\mathcal V}-1}},\qquad
S_{\mathrm{mom}}=\frac1{2\ln\frac{\sqrt{\kappa_{\mathcal V}}+1}{\sqrt{\kappa_{\mathcal V}}-1}},
\]
so that
\[
\frac{\kappa_{\mathcal V}}4\Bigl(1-\frac1{\kappa_{\mathcal V}^2}\Bigr)\le S_{\mathrm{SGD}}\le\frac{\kappa_{\mathcal V}}4,
\qquad
\frac{\sqrt{\kappa_{\mathcal V}}}4\Bigl(1-\frac1{\kappa_{\mathcal V}}\Bigr)\le S_{\mathrm{mom}}\le\frac{\sqrt{\kappa_{\mathcal V}}}4.
\]
\item[(ii)] $E_{\mathrm{SGD}}=E_{\mathrm{SGD}}(1)=1\big/\ln\frac{d+2}{d+2-p_{\min}}$, so that $(d+2-p_{\min})/p_{\min}\le E_{\mathrm{SGD}}\le(d+2)/p_{\min}$, and $E_{\mathrm{mom}}\le E_{\mathrm{SGD}}$. If $\kappa_{\mathcal V}>16$, also
\[
c_\kappa\,\frac{d+2-p_{\min}}{8p_{\min}}\le E_{\mathrm{mom}},
\qquad c_\kappa:=\max\Bigl(1-\frac4{\sqrt{\kappa_{\mathcal V}}},\frac1{\sqrt5}\Bigr).
\]
\item[(iii)] $4(d+2-p_{\min})/p_{\max}\le B_{\mathrm{SGD}}\le\frac{\kappa_{\mathcal V}^2}{\kappa_{\mathcal V}^2-1}\cdot\frac{4(d+2)}{p_{\max}}$, and
\[
\frac{B_{\mathrm{mom}}}{B_{\mathrm{SGD}}}\le\frac{\kappa_{\mathcal V}\sqrt{\kappa_{\mathcal V}}}{\kappa_{\mathcal V}-1},
\qquad
B_{\mathrm{mom}}\le\frac{\kappa_{\mathcal V}}{\kappa_{\mathcal V}-1}\cdot\frac{4(d+2)}{\sqrt{p_{\max}p_{\min}}}.
\]
If $\kappa_{\mathcal V}>16$, also
\[
\frac{c_\kappa\sqrt{\kappa_{\mathcal V}}}{8}\Bigl(1-\frac1{\kappa_{\mathcal V}^2}\Bigr)\Bigl(1-\frac{p_{\min}}{d+2}\Bigr)
\le\frac{B_{\mathrm{mom}}}{B_{\mathrm{SGD}}},
\qquad
c_\kappa\,\frac{d+2-p_{\min}}{2\sqrt{p_{\max}p_{\min}}}\le B_{\mathrm{mom}},
\]
and the first left side is at least $\sqrt{\kappa_{\mathcal V}}/19$.
\item[(iv)] Fix $\beta\in[0,1)$, and write $S_\beta(B)$ and $S_\beta$ for $S_O(B)$ and $S_O$ with the class $\{\beta\}$. Let $\rho_\beta$ be the largest modulus of a root of
\[
z^2-(1+\beta)\frac{\kappa_{\mathcal V}-1}{\kappa_{\mathcal V}+1}\,z+\beta .
\]
Then $S_\beta(B)\to S_\beta=-1/(2\ln\rho_\beta)$ as $B\to\infty$. Here $\rho_0=(\kappa_{\mathcal V}-1)/(\kappa_{\mathcal V}+1)$, and $\rho_\beta=\sqrt\beta$ exactly when $\kappa_{\mathcal V}\le\kappa_\beta:=\bigl((1+\sqrt\beta)/(1-\sqrt\beta)\bigr)^2$. The minimum over $\beta$ is $\min_\beta\rho_\beta=\rho_{\beta_P}=(\sqrt{\kappa_{\mathcal V}}-1)/(\sqrt{\kappa_{\mathcal V}}+1)$, at Polyak's $\beta_P:=\bigl(\frac{\sqrt{\kappa_{\mathcal V}}-1}{\sqrt{\kappa_{\mathcal V}}+1}\bigr)^2$, which recovers~(i). Finally, put
\[
A_\beta:=\frac{(1-\beta)(\kappa_{\mathcal V}+1)}{4(1+\beta)},\qquad
\omega_\beta^2:=\frac{2(1+\beta)}{(1-\beta)^2(\kappa_{\mathcal V}+1)}.
\]
If $\kappa_{\mathcal V}\ge\kappa_\beta$, then $S_\beta\le A_\beta$, and if $\omega_\beta^2\le\frac14$, then
\[
\tfrac12\bigl(1+\sqrt{1-4\omega_\beta^2}\bigr)A_\beta\le S_\beta\le A_\beta .
\]
\end{enumerate}
\end{proposition}

\noindent Here $S_O(B)$ is the shortest decay time, in steps, of the slowest row at batch size $B$, over the hyperparameters of the class, and $E_O(B)$ is the number of samples processed in that time. Minimizing over $B$ as well, $S_O$ and $E_O$ are the fewest steps and the fewest samples. Their ratio $B_O$ is the critical batch size of \citet{mccandlish2018empirical}, with one decay time of the excess risk in place of the time to their target loss.

The two minima sit at opposite ends of the batch axis. The fewest samples of SGD are attained at $B=1$. There the noise term of~\eqref{eq:eta-plus} dominates the stability threshold of every row, and the rarest row, $j=|\mathcal V|$, is the slowest. The fewest steps are approached as $B\to\infty$. There the noise feedback vanishes, and $S_O$ is the classical iteration count of gradient descent and of the heavy ball on the curvatures $p_1,\dots,p_{|\mathcal V|}$ (Lemmas~\ref{lem:helps-vocab}(iv) and~\ref{lem:helps-twocurv}).

Part~(iii) says three things about the critical batch sizes, all with absolute constants. First, $B_{\mathrm{SGD}}$ is $4d/p_{\max}$ up to a factor $1+O(d^{-1}+\kappa_{\mathcal V}^{-2})$. Second, $B_{\mathrm{mom}}$ is at most about $4d/\sqrt{p_{\max}p_{\min}}$. Third, for $\kappa_{\mathcal V}>16$, momentum multiplies $B_{\mathrm{SGD}}$ by a factor between $\sqrt{\kappa_{\mathcal V}}/19$ and $\sqrt{\kappa_{\mathcal V}}\,\bigl(1+O(\kappa_{\mathcal V}^{-1})\bigr)$. For large $\kappa_{\mathcal V}$ the lower end of this factor is $\frac18\sqrt{\kappa_{\mathcal V}}\,\bigl(1-O(\kappa_{\mathcal V}^{-1/2})\bigr)$.

The lower bound in~(iii) exceeds $1$ when $\kappa_{\mathcal V}\ge150$. For such $\kappa_{\mathcal V}$, momentum provably raises the critical batch size. For smaller $\kappa_{\mathcal V}$ the proposition gives its order of magnitude only.

Part~(iv) gives the fewest steps at each fixed $\beta$. The quantity $\omega_\beta^2$ is the $\omega^2$ of the rarest row at the balancing learning rate, defined by $\eta\eps(p_{\max}+p_{\min})=2(1+\beta)$. The two-sided bound therefore holds while that row is at most critically damped, and there a fixed $\beta$ divides the fewest steps of SGD by about $(1+\beta)/(1-\beta)$. Once $\kappa_{\mathcal V}\le\kappa_\beta$, every row can sit in the range of Lemma~\ref{lem:helps-vocab}(iv)(b), and $S_\beta=1/\ln(1/\beta)$ is the retention cap.

For $\kappa_{\mathcal V}>16$, part~(ii) says that momentum divides the fewest samples by at most $8(d+2)/(c_\kappa(d+2-p_{\min}))$. This factor is below $19$, and it is $8\bigl(1+O(\kappa_{\mathcal V}^{-1/2})\bigr)$ for large $\kappa_{\mathcal V}$. The $8$ comes from the bound of Lemma~\ref{lem:step_speed}(ii) on the rate of the rarest row.

To leading order the factor $8$ disappears in the regime of the last claim of Lemma~\ref{lem:helps-vocab}(iii). There momentum needs at least $(d+2-p_{\min})/p_{\min}$ samples per decay time, the lower end of $E_{\mathrm{SGD}}$ in~(ii). If $E_{\mathrm{mom}}=E_{\mathrm{SGD}}$, then $B_{\mathrm{mom}}/B_{\mathrm{SGD}}=S_{\mathrm{SGD}}/S_{\mathrm{mom}}$, which lies between $\sqrt{\kappa_{\mathcal V}}(1-\kappa_{\mathcal V}^{-2})$ and $\kappa_{\mathcal V}\sqrt{\kappa_{\mathcal V}}/(\kappa_{\mathcal V}-1)$.

The proposition concerns only the two minima. For SGD, Lemma~\ref{lem:helps-vocab}(v) also gives $S_{\mathrm{SGD}}(B)$ at every batch size. For momentum the proposition does not describe the curve of steps against batch size between its two asymptotes, the batch size at which the curve bends, or how close $S_{\mathrm{mom}}(B)$ comes to $S_{\mathrm{mom}}$ at a finite batch with gradient noise. These are observed numerically (Remark~\ref{rem:helps-numerics}).

The decay rates do not involve $\varsigma$, so neither do $S_O$, $E_O$ and $B_O$. With $\varsigma>0$, however, the floors at the two ends of the batch axis differ. At the parameters that attain $E_{\mathrm{SGD}}$ (Step~3 of the proof), every row has $H_j=\frac12$, and by Lemma~\ref{lem:helps-sgd} its floor is $R_\infty=\varphi/(1-H_j)=\varsigma^2d/(d+2)$. At fixed $(\eta,\beta)$, the floor vanishes as $B\to\infty$, since $\varphi=\eta\varsigma^2d/(2B)$. Our $B_O$ therefore matches a definition with a fixed target loss only for the noiseless model, $\varsigma=0$, and a small target.

\begin{proof}
The proof has six steps. Steps~1 and~2 compute the fewest steps: the noiseless two-curvature bound of Lemma~\ref{lem:helps-twocurv} applies at every batch size, and the classical parameters of Lemma~\ref{lem:helps-vocab}(iv) attain it as $B\to\infty$. Steps~3 and~4 bound the fewest samples through the rarest row, whose rate the noise term of~\eqref{eq:eta-plus} caps at every batch size. SGD attains this bound at $B=1$, and for momentum Lemma~\ref{lem:step_speed}(ii) costs the factor~$8$. Step~5 divides the two minima. Step~6 treats a fixed $\beta$ by the arguments of Steps~1 and~2, with the monotonicity of $\rho(F)$ in $\lambda$ in place of Lemma~\ref{lem:helps-twocurv}.

We write
\[
D:=d+2-p_{\min},\qquad
r_{\mathrm{SGD}}:=2\ln\frac{\kappa_{\mathcal V}+1}{\kappa_{\mathcal V}-1},\qquad
r_{\mathrm{mom}}:=2\ln\frac{\sqrt{\kappa_{\mathcal V}}+1}{\sqrt{\kappa_{\mathcal V}}-1},
\]
so that $r_{\mathrm{SGD}}\le r_{\mathrm{mom}}$, and $r_O$ for the value of the class $O$. The infima are finite. For every $B$, the set $\mathcal A_O(B)$ contains $(\eta,\beta)$ for every small $\eta>0$ and every $\beta$ in the class, by Lemma~\ref{lem:helps-vocab}(i). On $\mathcal A_O(B)$, $\Lambda_{\min}>0$, since every $\mathcal T_j$ then has spectral radius below one. So $S_O(B)<\infty$, and Step~1 shows $S_O(B)>0$.

\emph{Step 1: fewest steps, lower bound.} Let $B\ge1$ and $(\eta,\beta)\in\mathcal A_O(B)$, and let $F_j$ be the matrix $F$ of row $j$, with $\lambda=\lambda_j=\eta\eps p_j$. Gradient noise can only slow a row: by Lemma~\ref{lem:step_speed}(i), $\rho(\mathcal T_j)\ge\rho(F_j)^2$, so $\Lambda_j\le-2\ln\rho(F_j)$ for every $j$. Lemma~\ref{lem:helps-twocurv} with $c=\eta\eps$ gives $\max(\rho(F_1),\rho(F_{|\mathcal V|}))\ge e^{-r_{\mathrm{mom}}/2}$, and $\ge e^{-r_{\mathrm{SGD}}/2}$ when $\beta=0$. Therefore
\[
\Lambda_{\min}\le\min(\Lambda_1,\Lambda_{|\mathcal V|})\le-2\ln\max\bigl(\rho(F_1),\rho(F_{|\mathcal V|})\bigr)\le r_O,
\]
where, for the momentum class at $\beta=0$, we used $r_{\mathrm{SGD}}\le r_{\mathrm{mom}}$. So $S_O(B)\ge1/r_O$ for every $B$, and $S_O\ge1/r_O$.

\emph{Step 2: fewest steps, upper bound.} For SGD let $(\eta,\beta):=(2/(p_{\max}+p_{\min}),0)$, and for momentum let $(\eta,\beta)$ be the parameters of Lemma~\ref{lem:helps-vocab}(iv)(b). By Lemma~\ref{lem:helps-vocab}(iv), $\Lambda^0_{\min}=r_O$. In~(a) this is the stated value. In~(b), $\Lambda^0_{\min}=\ebar=-\ln\beta=r_{\mathrm{mom}}$, because $\sqrt\beta=(\sqrt{\kappa_{\mathcal V}}-1)/(\sqrt{\kappa_{\mathcal V}}+1)$. Lemma~\ref{lem:helps-vocab}(iv)(c) provides $B_0$ such that $(\eta,\beta)\in\mathcal A_O(B)$ for every integer $B\ge B_0$, and such that the value of $\Lambda_{\min}$ at $(\eta,\beta,B)$ tends to $r_O$ as $B\to\infty$. For $B\ge B_0$, $S_O(B)\le1/\Lambda_{\min}$, and letting $B\to\infty$ gives $\limsup_BS_O(B)\le1/r_O$. With Step~1, $S_O(B)\to S_O=1/r_O$.

It remains to bound $r_O$ in terms of $\kappa_{\mathcal V}$. For $x\in(0,1)$, the series $\ln\frac{1+x}{1-x}=2\sum_{k\text{ odd}}x^k/k$ lies between $2x$ and $2\sum_{k\text{ odd}}x^k=2x/(1-x^2)$. Taking $x=1/\kappa_{\mathcal V}$ and $x=1/\sqrt{\kappa_{\mathcal V}}$,
\[
\frac4{\kappa_{\mathcal V}}\le r_{\mathrm{SGD}}\le\frac{4\kappa_{\mathcal V}}{\kappa_{\mathcal V}^2-1},
\qquad
\frac4{\sqrt{\kappa_{\mathcal V}}}\le r_{\mathrm{mom}}\le\frac{4\sqrt{\kappa_{\mathcal V}}}{\kappa_{\mathcal V}-1},
\]
which are the bounds of~(i).

\emph{Step 3: fewest samples of SGD.} The rarest row bounds the rate at every batch size, and the bound is attained at $B=1$. Let $B\ge1$ and $(\eta,0)\in\mathcal A_{\mathrm{SGD}}(B)$. Since $\beta=0$, the threshold $\eta_+^{\mathrm{SGD}}$ of Lemma~\ref{lem:helps-sgd} at $p=p_{\min}$ is $\eta_+(p_{\min})$, and~\eqref{eq:eta-plus} gives
\[
\frac1{\eta_+(p_{\min})}=\frac D{2B}+\frac{p_{\min}}2,\qquad\text{so that}\qquad \frac{p_{\min}\eta_+(p_{\min})}2=\frac{Bp_{\min}}{D+Bp_{\min}} .
\]
For row $|\mathcal V|$, Lemma~\ref{lem:helps-sgd} gives $\Lambda_{|\mathcal V|}=-\ln(1-x_{|\mathcal V|})$ with $0<x_{|\mathcal V|}\le p_{\min}\eta_+(p_{\min})/2$, so $\Lambda_{|\mathcal V|}\le\ln(1+Bp_{\min}/D)$. The function $x\mapsto x/\ln(1+x)$ is increasing on $(0,\infty)$, since the numerator of its derivative is $\ln(1+x)-x/(1+x)>0$. Hence, for $B\ge1$,
\[
\frac B{\Lambda_{\min}}\ge\frac B{\Lambda_{|\mathcal V|}}\ge\frac{B}{\ln(1+Bp_{\min}/D)}
\ge\frac1{\ln(1+p_{\min}/D)}=\frac1{\ln\frac{d+2}{D}},
\]
and $E_{\mathrm{SGD}}(B)\ge1/\ln\frac{d+2}D$ for every $B$.

Conversely, at $B=1$, \eqref{eq:eta-plus} at $\beta=0$ gives $1/\eta_+(p_j)=(d+2)/2$ for every $j$. So $(\eta,0)\in\mathcal A_{\mathrm{SGD}}(1)$ for $\eta:=1/(d+2)=\eta_+(p_j)/2$, and the last claim of Lemma~\ref{lem:helps-sgd} gives $2\lambda_j(1-H_j)=p_j\eta_+(p_j)/2=p_j/(d+2)$. Thus $\Lambda_j=-\ln(1-p_j/(d+2))$, the smallest is $\Lambda_{|\mathcal V|}=\ln\frac{d+2}D$, and $E_{\mathrm{SGD}}(1)\le1/\ln\frac{d+2}D$. The two bounds together give $E_{\mathrm{SGD}}=E_{\mathrm{SGD}}(1)=1/\ln\frac{d+2}D$. The bounds in~(ii) follow from $x/(1+x)\le\ln(1+x)\le x$ with $x=p_{\min}/D$.

\emph{Step 4: fewest samples with momentum.} The upper bound is immediate. Since $\beta=0$ belongs to the momentum class, $\mathcal A_{\mathrm{SGD}}(B)\subset\mathcal A_{\mathrm{mom}}(B)$, so $E_{\mathrm{mom}}(B)\le E_{\mathrm{SGD}}(B)$ for every $B$, and $E_{\mathrm{mom}}\le E_{\mathrm{SGD}}$.

The lower bound, for $\kappa_{\mathcal V}>16$, is the main point of~(ii). The rarest row again sets the cost, but now Lemma~\ref{lem:step_speed} caps its rate by $\eps$ or by $4\eta p_{\min}$. The noise term bounds $\eta$ by the batch size, and the curvature term of row~$1$ ties $\eps$ to $\eta$, so the two caps cannot both be large. Combining the noise bound on $\eta$ with the cap $4\eta p_{\min}$ produces the factor $8$. Balancing the two caps shows that the logarithm in the rate costs at most the further factor $c_\kappa$.

Let $B\ge1$ and $(\eta,\beta)\in\mathcal A_{\mathrm{mom}}(B)$. If $\beta=0$, Step~3 gives $B/\Lambda_{\min}\ge D/p_{\min}$. So let $\beta\in(0,1)$, and put
\[
m:=\min(\eps,4\eta p_{\min})\in(0,1).
\]
Row $|\mathcal V|$ has $\eps\omega_{|\mathcal V|}^2=\eta p_{\min}$, so Lemma~\ref{lem:step_speed} gives $\Lambda_{|\mathcal V|}\le-\ln(1-\eps)$, and also $\Lambda_{|\mathcal V|}\le-\ln(1-4\eta p_{\min})$ when $4\eta p_{\min}<1$. In both cases $\Lambda_{|\mathcal V|}\le-\ln(1-m)$.

Two stability constraints now bound $m$. By~\eqref{eq:eta-plus}, $1/\eta_+(p_{\min})>D/(2B)$, so the stability of row $|\mathcal V|$ gives $\eta<2B/D$, and
\[
B>\frac{\eta D}2\ge\frac{mD}{8p_{\min}} .
\]
The stability of row $1$ gives $H_{c,1}\le H_1<1$ (Lemma~\ref{lem:stab-all}), that is, $\eta\eps p_{\max}<2(1+\beta)=2(2-\eps)$. Hence $4\eta p_{\min}<8(2-\eps)/(\eps\kappa_{\mathcal V})$, and
\[
m\le\min\Bigl(\eps,\frac{8(2-\eps)}{\eps\kappa_{\mathcal V}}\Bigr)\le m_\kappa:=\frac4{1+\sqrt{1+\kappa_{\mathcal V}}} .
\]
The last inequality holds because the first term increases and the second decreases in $\eps>0$, and they are equal at $\eps=m_\kappa$, the positive root of $\kappa_{\mathcal V}\eps^2+8\eps-16=0$. Since $\sqrt{1+\kappa_{\mathcal V}}>4$,
\[
m_\kappa<\frac4{\sqrt{\kappa_{\mathcal V}}},\qquad
1-m_\kappa=\frac{\sqrt{1+\kappa_{\mathcal V}}-3}{\sqrt{1+\kappa_{\mathcal V}}+1}>\frac15 .
\]

It remains to turn the bound on $m$ into a bound on the decay time. The function $\phi(m):=m/(-\ln(1-m))$ is decreasing on $(0,1)$, since $-\ln(1-m)/m=\sum_{k\ge1}m^{k-1}/k$ is increasing. Moreover $\phi(m)\ge\sqrt{1-m}$. With $s:=(1-m)^{-1/2}>1$ this reads $2\ln s\le s-1/s$, which holds because both sides vanish at $s=1$ and the derivative of $s-1/s-2\ln s$ is $(1-1/s)^2\ge0$. Combining,
\[
\frac B{\Lambda_{\min}}\ge\frac B{\Lambda_{|\mathcal V|}}>\frac{D}{8p_{\min}}\,\phi(m)
\ge\frac{D}{8p_{\min}}\,\phi(m_\kappa)
\ge\sqrt{1-m_\kappa}\,\frac{D}{8p_{\min}}\ge c_\kappa\,\frac{D}{8p_{\min}},
\]
the last inequality because $\sqrt{1-m_\kappa}$ is at least $1-m_\kappa>1-4/\sqrt{\kappa_{\mathcal V}}$ and greater than $1/\sqrt5$. The same bound holds when $\beta=0$, since $D/p_{\min}\ge c_\kappa D/(8p_{\min})$. So $E_{\mathrm{mom}}(B)\ge c_\kappa D/(8p_{\min})$ for every $B$, which completes~(ii).

\emph{Step 5: critical batch sizes.} By Steps~2 and~3, $B_{\mathrm{SGD}}=r_{\mathrm{SGD}}E_{\mathrm{SGD}}$, and
\[
\frac{4D}{\kappa_{\mathcal V}p_{\min}}\le B_{\mathrm{SGD}}\le\frac{4\kappa_{\mathcal V}(d+2)}{(\kappa_{\mathcal V}^2-1)p_{\min}} .
\]
This is the first claim of~(iii), since $\kappa_{\mathcal V}p_{\min}=p_{\max}$. For the ratio, write
\[
\frac{B_{\mathrm{mom}}}{B_{\mathrm{SGD}}}
=\frac{E_{\mathrm{mom}}}{E_{\mathrm{SGD}}}\cdot\frac{r_{\mathrm{mom}}}{r_{\mathrm{SGD}}}.
\]
By~(ii), the first factor is at most $1$, and at least $c_\kappa D/(8(d+2))$ if $\kappa_{\mathcal V}>16$. By Step~2, the second factor lies between
\[
\frac4{\sqrt{\kappa_{\mathcal V}}}\cdot\frac{\kappa_{\mathcal V}^2-1}{4\kappa_{\mathcal V}}=\sqrt{\kappa_{\mathcal V}}\bigl(1-\kappa_{\mathcal V}^{-2}\bigr)
\qquad\text{and}\qquad
\frac{4\sqrt{\kappa_{\mathcal V}}}{\kappa_{\mathcal V}-1}\cdot\frac{\kappa_{\mathcal V}}4=\frac{\kappa_{\mathcal V}\sqrt{\kappa_{\mathcal V}}}{\kappa_{\mathcal V}-1}.
\]
Multiplying, with $D/(d+2)=1-p_{\min}/(d+2)$, gives the bounds on the ratio. Similarly $B_{\mathrm{mom}}=r_{\mathrm{mom}}E_{\mathrm{mom}}$, and the bounds of Step~2 on $r_{\mathrm{mom}}$ and of~(ii) on $E_{\mathrm{mom}}$, with $\sqrt{\kappa_{\mathcal V}}\,p_{\min}=\sqrt{p_{\max}p_{\min}}$, give the bounds on $B_{\mathrm{mom}}$.

For the last claim, $p_{\min}=p_{\max}/\kappa_{\mathcal V}<\frac1{16}$ because $p_{\max}\le1$, and $d\ge1$, so $p_{\min}/(d+2)<\frac1{48}$. With $c_\kappa\ge1/\sqrt5$ and $\kappa_{\mathcal V}^{-2}<\frac1{256}$, the first left side in~(iii) is at least
\[
\frac{\sqrt{\kappa_{\mathcal V}}}8\cdot\frac1{\sqrt5}\cdot\frac{255}{256}\cdot\frac{47}{48}>\frac{\sqrt{\kappa_{\mathcal V}}}{19}.
\]

\emph{Step 6: a fixed $\beta$.} Fix $\beta\in[0,1)$. We first establish the monotonicity of $\rho(F)$ in $\lambda$ that replaces Lemma~\ref{lem:helps-twocurv}. For $0<\lambda<2(1+\beta)$, write $\rho(\lambda)$ for the spectral radius of $F$, whose eigenvalues are the roots of $z^2-tz+\beta$ with $t=1+\beta-\lambda$. Replacing $\lambda$ by $2(1+\beta)-\lambda$ replaces $t$ by $-t$ and the eigenvalues by their negatives, so $\rho$ is symmetric about $\lambda=1+\beta$. For $\lambda\le(1-\sqrt\beta)^2$, $t\ge2\sqrt\beta$, the roots are real and nonnegative, and the larger one, $\frac12\bigl(t+\sqrt{t^2-4\beta}\bigr)$, decreases in $\lambda$. For $(1-\sqrt\beta)^2\le \lambda\le1+\beta$, $0\le t\le2\sqrt\beta$ and $\rho=\sqrt\beta$. So $\rho$ is nonincreasing on $(0,1+\beta]$ and nondecreasing on $[1+\beta,2(1+\beta))$.

The lower bound on $S_\beta(B)$ follows. Let
\[
\lambda_*:=\frac{2(1+\beta)}{\kappa_{\mathcal V}+1}<1+\beta .
\]
Then $1+\beta-\lambda_*=(1+\beta)\frac{\kappa_{\mathcal V}-1}{\kappa_{\mathcal V}+1}$, so $\rho(\lambda_*)=\rho_\beta$, and $\rho(2(1+\beta)-\lambda_*)=\rho_\beta$ by the symmetry. Let $B\ge1$ and $(\eta,\beta)\in\mathcal A_{\{\beta\}}(B)$. Then $\lambda_j=\eta\eps p_j<2(1+\beta)$ for every $j$ (Lemma~\ref{lem:stab-all}), and $\lambda_1=\kappa_{\mathcal V}\lambda_{|\mathcal V|}$. If $\lambda_1+\lambda_{|\mathcal V|}\le2(1+\beta)$, then $\lambda_{|\mathcal V|}\le \lambda_*$, so $\rho(F_{|\mathcal V|})\ge\rho_\beta$. Otherwise $\lambda_1>2(1+\beta)-\lambda_*>1+\beta$, so $\rho(F_1)\ge\rho_\beta$. With Lemma~\ref{lem:step_speed}(i), $\Lambda_{\min}\le\min(\Lambda_1,\Lambda_{|\mathcal V|})\le-2\ln\rho_\beta$, so $S_\beta(B)\ge-1/(2\ln\rho_\beta)$ for every $B$.

For the matching upper bound, take $\eta$ with $\eta\eps(p_{\max}+p_{\min})=2(1+\beta)$, so that $\lambda_{|\mathcal V|}=\lambda_*$ and $\lambda_1=2(1+\beta)-\lambda_*$. Every $\lambda_j$ lies in $[\lambda_*,2(1+\beta)-\lambda_*]$, where $\rho\le\rho_\beta$, and $H_{c,j}\le H_{c,1}<1$. With $H_{n,j}$ replaced by $0$, Lemma~\ref{lem:step_speed}(i) gives row $j$ the rate $-2\ln\rho(F_j)\ge-2\ln\rho_\beta$, with equality at $j=|\mathcal V|$. So $\Lambda^0_{\min}=-2\ln\rho_\beta$, which is finite since $\rho_\beta>0$. The argument for Lemma~\ref{lem:helps-vocab}(iv)(c) uses only $H_{c,1}<1$. It puts $(\eta,\beta)$ in $\mathcal A_{\{\beta\}}(B)$ for large $B$ and gives $\Lambda_{\min}\to\Lambda^0_{\min}$ as $B\to\infty$. Hence $\limsup_BS_\beta(B)\le-1/(2\ln\rho_\beta)$, and with the lower bound, $S_\beta(B)\to S_\beta=-1/(2\ln\rho_\beta)$.

Next come the values of $\rho_\beta$. The value of $\rho_0$ is read off from the roots $0$ and $(\kappa_{\mathcal V}-1)/(\kappa_{\mathcal V}+1)$. The roots have modulus $\sqrt\beta$ exactly when they are non-real or double, that is, when $(1+\beta)\frac{\kappa_{\mathcal V}-1}{\kappa_{\mathcal V}+1}\le2\sqrt\beta$, which rearranges to $\kappa_{\mathcal V}\le\kappa_\beta$. For the minimum over $\beta$, apply Lemma~\ref{lem:helps-twocurv} with $c=\lambda_*/p_{\min}$, so that $cp_{\max}=2(1+\beta)-\lambda_*$. It gives $\rho_\beta\ge(\sqrt{\kappa_{\mathcal V}}-1)/(\sqrt{\kappa_{\mathcal V}}+1)$. At $\beta_P$, $\kappa_{\mathcal V}=\kappa_{\beta_P}$, so $\rho_{\beta_P}=\sqrt{\beta_P}$ attains this bound.

Finally, the bounds in terms of $A_\beta$. Let $\kappa_{\mathcal V}\ge\kappa_\beta$, that is, $\lambda_*\le(1-\sqrt\beta)^2$, so that $\rho_\beta\ge\sqrt\beta$ is the larger real root. Write $\rho_\beta=e^{-L}$ with $0<L\le L_m:=\frac12\ln(1/\beta)$, where $L_m=\infty$ if $\beta=0$. Since $z^2-(1+\beta-\lambda)z+\beta=(z-1)(z-\beta)+\lambda z$, the root satisfies
\[
\lambda_*=g(L):=(e^L-1)(e^{-L}-\beta)=1+\beta-e^{-L}-\beta e^L .
\]
The function $g$ is concave, with $g(0)=0$ and $g'(L)=e^{-L}-\beta e^L$, so $g'(0)=1-\beta$ and $g$ is increasing on $[0,L_m]$. Concavity gives $\lambda_*\le(1-\beta)L$, that is,
\[
S_\beta=\frac1{2L}\le\frac{1-\beta}{2\lambda_*}=A_\beta .
\]
For the lower bound let $\omega_\beta^2\le\frac14$. This implies $\kappa_{\mathcal V}\ge\kappa_\beta$, since $\lambda_*=\omega_\beta^2(1-\beta)^2\le(1-\beta)^2/4<(1-\sqrt\beta)^2$. The smaller root
\[
L_-:=\tfrac12(1-\beta)\bigl(1-\sqrt{1-4\omega_\beta^2}\bigr)
\]
of $L(1-\beta-L)=\lambda_*$ is real, and $L_-\le\frac12(1-\beta)\le L_m$, since $1-\beta\le\ln(1/\beta)$. The function $h(L):=g(L)-L(1-\beta-L)$ has $h(0)=h'(0)=0$ and $h''(L)=2-e^{-L}-\beta e^L\ge1-\sqrt\beta\ge0$ on $[0,L_m]$. So $g(L_-)\ge L_-(1-\beta-L_-)=\lambda_*=g(L)$. As $g$ is increasing on $[0,L_m]$, $L\le L_-$, and
\[
S_\beta\ge\frac1{2L_-}=\tfrac12\bigl(1+\sqrt{1-4\omega_\beta^2}\bigr)A_\beta .
\]
\end{proof}

Proposition~\ref{prop:vocab} of the main text is a short form of the next statement, which collects the consequences of Lemma~\ref{lem:helps-vocab} and Proposition~\ref{prop:helps-critical} for a vocabulary.

\begin{proposition}[Full statement of Proposition~\ref{prop:vocab}]\label{prop:vocab_full}\label{rem:vocab}
Let $\kappa_{\mathcal V}>1$ and $B_\times:=(d+2)/(p_{\max}-p_{\min})$. For each optimizer class of Proposition~\ref{prop:helps-critical}, let $S=S_O(B)$ be the decay time in steps of the loss at the best learning rate, and at the best $\beta$ for momentum. Let $E_{\min}:=E_{\mathrm{SGD}}=1/\ln\frac{d+2}{d+2-p_{\min}}$.
\begin{enumerate}
\item[(i)] SGD runs the rarest token at its own fastest learning rate if and only if $B-1\le B_\times$. In that case
\[
S_{\mathrm{SGD}}(B)=\frac1{\ln\bigl(1+Bp_{\min}/(d+2-p_{\min})\bigr)},
\]
and per-token learning rates do not lower $S$. For $B-1\le B_\times$, momentum lowers $S$ by at most an absolute factor. If $\kappa_{\mathcal V}>16$, then
\[
S_{\mathrm{mom}}(B)\ge\frac{c_\kappa(\kappa_{\mathcal V}-1)(d+2-p_{\min})}{8\kappa_{\mathcal V}(d+2)}\,S_{\mathrm{SGD}}(B),
\]
with $c_\kappa$ of Proposition~\ref{prop:helps-critical}. Along a co-scaling ray with $\gamma>0$ for the rarest token, and still for $B-1\le B_\times$, $S_\beta(B)\ge S_{\mathrm{SGD}}(B)(1+o(1))$.
\item[(ii)] At every batch size $S\ge S_\infty$, and $S\to S_\infty$ as $B\to\infty$, where $S_\infty$ is as follows. For SGD, $S_\infty=S_{\mathrm{SGD}}$, and
\[
\frac{\kappa_{\mathcal V}}4\bigl(1-\kappa_{\mathcal V}^{-2}\bigr)\le S_{\mathrm{SGD}}\le\frac{\kappa_{\mathcal V}}4 .
\]
At fixed $\beta$, $S_\infty=S_\beta$. It equals $1/\ln(1/\beta)$ if $\kappa_{\mathcal V}\le\kappa_\beta$, and if $\omega_\beta^2\le\frac14$ it satisfies
\[
\tfrac12\bigl(1+\sqrt{1-4\omega_\beta^2}\bigr)A_\beta\le S_\beta\le A_\beta,
\]
with $\kappa_\beta$, $A_\beta$ and $\omega_\beta^2$ of Proposition~\ref{prop:helps-critical}(iv). With $\beta$ tuned, $S_\infty=S_{\mathrm{mom}}=\min_\beta S_\beta$, attained at Polyak's $\beta_P$, and
\[
\frac{\sqrt{\kappa_{\mathcal V}}}4\bigl(1-\kappa_{\mathcal V}^{-1}\bigr)\le S_{\mathrm{mom}}\le\frac{\sqrt{\kappa_{\mathcal V}}}4 .
\]
\item[(iii)] The critical batch size $B_{\mathrm{SGD}}=E_{\min}/S_{\mathrm{SGD}}$ of SGD satisfies
\[
\frac{4(d+2-p_{\min})}{p_{\max}}\le B_{\mathrm{SGD}}\le\frac{\kappa_{\mathcal V}^2}{\kappa_{\mathcal V}^2-1}\cdot\frac{4(d+2)}{p_{\max}} .
\]
With $\beta$ tuned, the critical batch size $B_{\mathrm{mom}}=E_{\mathrm{mom}}/S_{\mathrm{mom}}$ is at most $\frac{\kappa_{\mathcal V}}{\kappa_{\mathcal V}-1}\cdot\frac{4(d+2)}{\sqrt{p_{\max}p_{\min}}}$, and if $\kappa_{\mathcal V}>16$, at least $c_\kappa\frac{d+2-p_{\min}}{2\sqrt{p_{\max}p_{\min}}}$.
\end{enumerate}
\end{proposition}

\noindent Here $E_{\min}$ is the fewest samples per decay time of SGD, which lies between $(d+2-p_{\min})/p_{\min}$ and $(d+2)/p_{\min}$ (Proposition~\ref{prop:helps-critical}(ii)). Part~(i) says that below $B_\times$ one learning rate costs SGD nothing, and momentum gains at most an absolute factor over it. Parts~(ii) and~(iii) are Proposition~\ref{prop:helps-critical} read for the vocabulary. The large-batch decay time is set by the two extreme tokens, and the critical batch size of SGD by the most frequent one.

\begin{proof}
\emph{(i), SGD.} By Lemma~\ref{lem:helps-vocab}(v), the best shared learning rate is the rarest token's own optimum $\eta^*_{|\mathcal V|}$ exactly when $B-1\le B_\times$, and then $S$ has the stated value. Since the loss decays no faster than its rarest row, per-token learning rates cannot lower $S$. For $B-1>B_\times$, the best shared learning rate is $2B/D_B$, where $g_{|\mathcal V|}<g_{|\mathcal V|}(\eta^*_{|\mathcal V|})$, so per-token learning rates would lower $S$.

\emph{(i), momentum.} Let $B-1\le B_\times$, and first let $\kappa_{\mathcal V}>16$. Proposition~\ref{prop:helps-critical}(ii) bounds the decay time of momentum from below:
\[
BS_{\mathrm{mom}}(B)=E_{\mathrm{mom}}(B)\ge E_{\mathrm{mom}}\ge c_\kappa\,\frac{d+2-p_{\min}}{8p_{\min}} .
\]
For SGD we need the matching upper bound. Put $y:=Bp_{\min}/(d+2+(B-1)p_{\min})$, so that $S_{\mathrm{SGD}}(B)=-1/\ln(1-y)\le1/y$. Since $(B-1)p_{\min}\le B_\times p_{\min}=(d+2)/(\kappa_{\mathcal V}-1)$,
\[
BS_{\mathrm{SGD}}(B)\le\frac By\le\frac{(d+2)\kappa_{\mathcal V}}{(\kappa_{\mathcal V}-1)p_{\min}},
\]
and dividing the two bounds gives the first claim.

For the claim along a co-scaling ray with $\gamma>0$ for the rarest token, compare both optimizers with the sample cost of that token. A shared learning rate cannot make the loss decay faster than its rarest row, so $S_\beta(B)\ge N_1/(Bp_{\min})$, with $N_1$ of Section~\ref{app:sample_cost} at $p=p_{\min}$. Corollary~\ref{cor:sample_cost}(ii) gives $N_1\ge\max(N_{\mathrm{stab}},N_{\mathrm{mem}})(1+o(1))$. By Lemma~\ref{lem:sgd-cost}, SGD's sample cost at $p=p_{\min}$ is at most $(1+\eps)\max(N_{\mathrm{stab}},N_{\mathrm{mem}})$, and for $B-1\le B_\times$ it equals $Bp_{\min}S_{\mathrm{SGD}}(B)$. Since $\eps\to0$ along the ray, $S_\beta(B)\ge S_{\mathrm{SGD}}(B)(1+o(1))$.

\emph{(ii) and (iii)} are Proposition~\ref{prop:helps-critical}(i), (iii) and~(iv), with $E_{\min}=E_{\mathrm{SGD}}$ by Proposition~\ref{prop:helps-critical}(ii).
\end{proof}

\noindent The critical batch size at a fixed $\beta$ is not determined. At fixed $\beta$ with $\kappa_{\mathcal V}\ge\kappa_\beta$, Proposition~\ref{prop:helps-critical}(ii) and~(iv) give
\[
\frac{E_{\min}}{S_\beta}\ge\frac{4(d+2-p_{\min})(1+\beta)}{(1-\beta)(p_{\max}+p_{\min})},
\]
with near equality when $\omega_\beta^2$ is small. This ratio would be the critical batch size of that class if momentum at a fixed $\beta$ also needed $E_{\min}$ samples per decay time, which we do not prove.

\begin{remark}[Numerical checks of the constants]\label{rem:helps-numerics}
All checks use the exact recursion~\eqref{eq:L1}, with $\Lambda$ computed from the eigenvalues of the $3\times3$ matrix of $\mathcal T$.

Two facts proved above only for SGD, or only as $B\to\infty$, hold numerically for momentum at a finite batch. Over every token of a Zipf vocabulary and a range of $(B,\beta,\eta)$, the slowest row is always one of the two extremes. The tuned $\beta$ is zero just below $B-1=B_\times$ and positive just above it, both in a Zipf vocabulary and with the bigram frequencies.

For a Zipf vocabulary, $p_j\propto1/j$, with $\eta$ tuned for SGD and $(\eta,\beta)$ for momentum, the decay time of the slowest row follows the bounds of Lemma~\ref{lem:helps-vocab}. Beyond $B_\times$, SGD levels off near $S_{\mathrm{SGD}}\approx\kappa_{\mathcal V}/4$. Momentum keeps falling and approaches $S_{\mathrm{mom}}\approx\sqrt{\kappa_{\mathcal V}}/4$ at the largest batches, with the tuned $\beta$ close to Polyak's $\beta_P$. So the fewest steps of Proposition~\ref{prop:helps-critical}(i) are nearly reached at a finite batch. Below $B_\times$ the two optimizers agree.

For two rows and a range of $\kappa_{\mathcal V}>16$, a search over the stable $(\eta,\beta,B)$ finds no point with fewer samples per decay time than $E_{\mathrm{SGD}}$. Numerically, then, $E_{\mathrm{mom}}=E_{\mathrm{SGD}}$ and $B_{\mathrm{mom}}/B_{\mathrm{SGD}}=S_{\mathrm{SGD}}/S_{\mathrm{mom}}\approx\sqrt{\kappa_{\mathcal V}}$, and the factor $8$ in Proposition~\ref{prop:helps-critical}(ii) is a loss in the proof.
\end{remark}

\subsection{Experiments}\label{app:bigram}
This subsection describes the two experiments of Figure~\ref{fig:vocab}: the exact loss curves of the least-squares vocabulary in panel~(c), and the bigram softmax model in panel~(d).

\paragraph{Loss curves of the least-squares vocabulary.}\label{app:vocab_curves}
Figure~\ref{fig:vocab}(c) uses the input-token frequencies of the bigram data ($|\mathcal V|=4096$), grouped into 51 logarithmic frequency bins, each at its mean frequency, with $d=64$ and label noise $\varsigma^2=10^{-3}$. Every row starts at $R_{j,0}=1$ with zero momentum. The curves are the exact expected excess loss $\frac12\sum_jp_jR_{j,k}$, computed from the affine recursion~\eqref{eq:L1} by powers of its matrix. At each batch size, both optimizers run at the learning rate whose exact floor is $10^{-3}$ of the initial loss, the convention of Figure~\ref{fig:phase-diagram}(d), found by bisection on the floor of the affine recursion. For SGD this fixes $\eta$. For momentum the floor-matched learning rate depends on $\beta$, which is chosen for the fewest steps to reach $1\%$ of the initial loss from $\beta=0$ and the 26 values $1-10^{-x}$ with $x$ evenly spaced from $0.05$ to $4$. Since this grid includes $\beta=0$, momentum is never slower than SGD. Table~\ref{tab:vocab_floor} gives the resulting learning rates, $\beta$ and speedups.

Up to $B=1024$ the floor-matched learning rate is below both stability limits, and momentum gains at most $1\%$, although for a single feature at the same floor the ratio of decay rates at critical damping, $2/(1+a+a^2)$, would be about $1.7$ (Proposition~\ref{prop:fixed_floor}). One $\beta$ cannot critically damp the rare tokens without putting the frequent tokens, which carry most of the loss, into fast oscillation, where they decay no faster than the retention rate. At larger batch sizes, SGD's floor-matched learning rate is just below its stability limit, so its steps stop falling, while momentum's learning rate keeps growing with $\beta$. Far above $B=65536$, a floor of $10^{-3}$ is reached only near the stability limit of either optimizer, so we stop there.

\begin{table}[!htbp]
\centering
\caption{The least-squares vocabulary of Figure~\ref{fig:vocab}(c) at a floor of $10^{-3}$ of the initial loss: the floor-matched learning rates, momentum's $\beta$, and momentum's speedup over SGD in steps to reach $1\%$ of the initial loss.}
\label{tab:vocab_floor}
\begin{tabular}{rcccc}
\toprule
& SGD & \multicolumn{2}{c}{momentum} & \\
\cmidrule(lr){2-2}\cmidrule(lr){3-4}
Batch (tokens) & $\eta$ & $\eta$ & $\beta$ & speedup \\
\midrule
64 & 0.979 & 0.979 & 0 & 1.00 \\
256 & 3.84 & 3.86 & 0.11 & 1.00 \\
1024 & 13.8 & 15.3 & 0.57 & 1.01 \\
4096 & 28.3 & 57.9 & 0.70 & 1.41 \\
16384 & 35.4 & 209 & 0.86 & 3.52 \\
65536 & 37.6 & 279 & 0.79 & 4.49 \\
\bottomrule
\end{tabular}
\end{table}

\paragraph{Bigram experiment.}
The code is in the directory \texttt{bigram\_model} of the companion repository \citep{sparsesgdrepo}.
\emph{Data.} We tokenize the training split of WikiText-103 \citep{merity2017pointer} (raw) with the Pythia tokenizer \citep{biderman2023pythia}, keep the $|\mathcal V|=4096$ most frequent tokens, and count the consecutive pairs $(x,y)$ in which both tokens are kept. WikiText-103 is available under the Creative Commons Attribution-ShareAlike license, and the Pythia models and tokenizer, including the Pythia-70M model of Figure~\ref{fig:tokens}, under the Apache 2.0 license. The empirical joint distribution $P(x,y)$ of these pairs is the data distribution. Each minibatch is $B$ independent draws from it, and the population loss
\[
L=-\sum_{x,y}P(x,y)\log\operatorname{softmax}_y\bigl(\langle H_x,\lambda_\cdot\rangle\bigr)
\]
is computed exactly. The input frequencies range from $p_{\max}=0.059$ to $p_{\min}=3.6\times10^{-5}$.

\emph{Model and optimizer.} Each token has an embedding $H_x\in\mathbb{R}^d$ and a readout vector $\lambda_y\in\mathbb{R}^d$, with $d=64$, untied, no bias, initialized with independent $\N(0,1/d)$ entries. All parameters are trained with one constant $\eta$ and one $\beta$:
\[
m\leftarrow\beta m+(1-\beta)g,\qquad \theta\leftarrow\theta-\eta m .
\]

\emph{Sweep.} For each batch size $B\in\{1,2,4,\dots,256,1024,4096,16384\}$ and two seeds, the grid $\beta\in\{0,0.5,0.8,0.9,0.95,0.98,0.99,0.995\}$ times 13 learning rates $\eta=\eta_0(1+\beta)/(1-\beta)$, $\eta_0\in\{2^{-4},\dots,2^8\}$, trains on the same batches from the same initialization. Runs last $\max(2\cdot10^4,\,4\cdot10^7/B)$ steps, capped at $10^7$, with 400 exact evaluations. At $B=1$ no configuration reaches the target.

\emph{Target.} The reference loss $L_{\mathrm{ref}}$ is the best loss of full-batch Adam on the exact loss (cosine schedule, $2\cdot10^5$ steps). A run's time is the number of steps to close $92\%$ of the gap $\ln|\mathcal V|-L_{\mathrm{ref}}$, with linear interpolation between evaluations. SGD's time is the minimum over $\eta$ at $\beta=0$, and momentum's the minimum over $(\eta,\beta)$. Figure~\ref{fig:bigram_bins} applies the same criterion to each of 8 log-spaced input-rank bins, with each bin's own reference loss. The runs used less than 100 chip-hours in total on TPU v4 and v6e chips.
\begin{figure}[!htbp]
\centering
\includegraphics[width=0.42\linewidth]{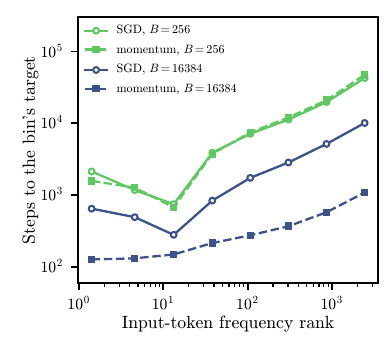}
\caption{\textbf{At large batch, the rarest tokens set SGD's pace in the bigram model.} Steps for each input-frequency bin to close $92\%$ of its own gap, for SGD (solid) and momentum (dashed), each tuned for the global target, at 256 and 16384 tokens per batch; medians over two seeds.}
\label{fig:bigram_bins}
\end{figure}

\clearpage
\section{Numerical illustrations}\label{app:numerics}

This section compares the exact per-step chains with the limits of
Sections~\ref{sec:cor} and~\ref{sec:lr}. The comparisons are evidence for
these results, not part of their proofs, and each caption names the results
that the figure illustrates. Except for the bigram experiment (Appendix~\ref{app:bigram}), all numerical simulations in this paper were run on laptop CPUs and took at most a few hours.

For least squares we plot the risk $R_k$ given by the exact
recursion~\eqref{eq:L1} of Lemma~\ref{lem:L1}, computed at finite $d$ in the
co-scaling~\eqref{eq:coscaling-main} from the cold start
$\Sigma_0=\operatorname{diag}(R_0,0)$, together with the limit of the cell of
Table~\ref{tab:cells}, on the clock natural to that cell: the learning clock
$t=\eta pk$ in cell~2 and the retention clock $\tau=\ebar k$ otherwise. No
sampling is involved, so the gap between the two curves is the
approximation error of the corollaries. The uniform error is
$\sup_k|R_k-R^*(x_k)|$ over the plotted horizon, where $R^*$ is the limit and
$x_k$ is the clock. In the windows, where $R_k$ has no limit, we plot half the
slow energy $S_k$, the eigenfunctional of Lemma~\ref{lem:L3}(d) evaluated on
the matched state $\Sh_k=P^{-1}\Sigma_kP^{-\top}$ of
Lemma~\ref{lem:match}, against half the limit $\bar S$ of
Corollary~\ref{cor:window}.

For logistic regression we plot the signal overlap $\theta_\parallel$ and the
bulk second moment $R_\perp$ along two kinds of paths: the drift recursion,
whose step is the exact conditional mean of one step of the five-variable
chain of Lemma~\ref{lem:LR0}, and sample paths of that chain itself. The
limits are those of Propositions~\ref{prop:LR2} and~\ref{prop:LR34}, and the
equilibrium is that of Proposition~\ref{prop:D}.

Apart from the LR sample paths, the figures of this section use the
recursions themselves and involve no sampling. The check that the exact recursion reproduces the mean
of Monte Carlo runs of the algorithm is Figure~\ref{fig:convergence} in the
main text, which shows individual runs, the recursion and the limit in
cells~2, 3 and~4 at $\sigma=1.2$. The phase figure of the main text,
Figure~\ref{fig:phase-main}, also at $\sigma=1.2$, involves no computation
beyond the exponents: it draws Table~\ref{tab:cells},
Lemma~\ref{lem:nonboundary} and Remark~\ref{rem:cost}.

Figure~\ref{fig:app-ls-cells} takes one slice of the LS dictionary across the
resonance line, below it (cell~2), on it (cell~4) and above it with a
subcritical learning rate (cell~3). Figure~\ref{fig:app-ls-windows} shows the
two windows at the critical learning rate. Figure~\ref{fig:app-lr-cells}
repeats the comparison of Figure~\ref{fig:app-ls-cells} for LR.

\begin{figure}[ht]
\centering
\IfFileExists{figures/fig_app_ls_cells.pdf}%
  {\includegraphics[width=\textwidth]{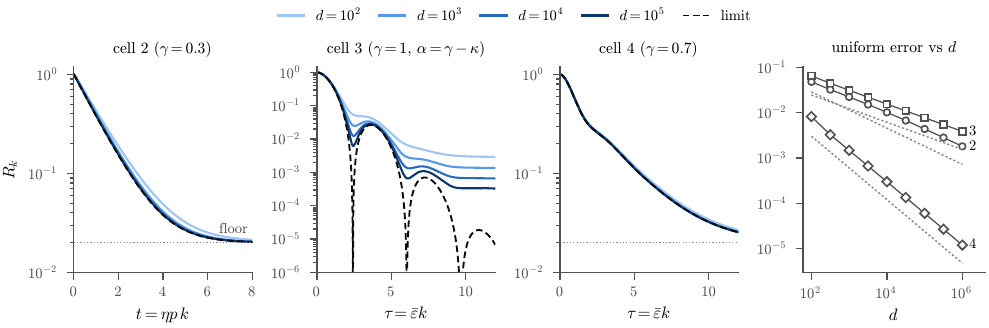}}%
  {\fbox{figure pending: fig\_app\_ls\_cells}}
\caption{\textbf{LS: the exact chain against the limit in cells~2, 3 and~4.}
Parameters: $\sigma=1.5$, $\kappa=1.2$ (so the resonance line is at
$\gamma=0.7$), $p=d^{-1.2}$, $B=\operatorname{round}(d^{1.5})$,
$\varepsilon=d^{-\gamma}$, $\eta=d^{-\alpha}$, $\varsigma^2=0.02$, $R_0=1$.
Blue lines: $R_k$ of the exact recursion at $d=10^2,10^3,10^4,10^5$, darker
with $d$; black dashed: the limit; dotted: the floor $0.02$.
\emph{Cell~2} ($\gamma=0.3$, $\alpha=-0.5$; Corollary~\ref{cor:tikhonov}):
$r'=-2(1-H^*)r+2\varphi^*=-r+0.02$, with $H^*=\frac12$, $\varphi^*=0.01$ and
floor $\varphi^*/(1-H^*)=0.02$.
\emph{Cell~3} ($\gamma=1$, subcritical $\alpha=\gamma-\kappa=-0.2$;
Corollaries~\ref{cor:regular} and~\ref{cor:lift}): $R=X^2$ with
$X''+X'+X=0$, $X(0)=1$, $X'(0)=0$; the downward spikes are the zeros of $X$.
At finite $d$ the chain levels off at its floor
$\varphi/(1-H)\asymp d^{-0.3}$, which vanishes as $d$ grows.
\emph{Cell~4} ($\gamma=0.7$, $\alpha=-0.5$; Corollary~\ref{cor:resonance}):
the system~\eqref{eq:L3} at $(\omega^{*2},H^*,\varphi^*)=(1,\frac12,0.01)$,
with floor $0.02$.
\emph{Right:} the uniform error against $d=10^2,\dots,10^6$. Fitted slopes
(last four points) and predicted slopes: cell~2, $-0.38$ against $-0.4$
(the error $O(\omega^2)$ of Corollary~\ref{cor:tikhonov}, with
$\omega^2\asymp d^{\gamma-\kappa-\alpha}$); cell~3, $-0.30$ against $-0.3$ (the
defect of Corollary~\ref{cor:lift}, linear in the feedback and temperature $\ut,\pt\asymp d^{1-\sigma-\alpha}$); cell~4, $-0.70$ against $-0.7$ (the
rate $d^{-\min(\gamma,1)}$ of Corollary~\ref{cor:resonance}). The dotted
reference lines have the predicted slopes and are offset downward.}
\label{fig:app-ls-cells}
\end{figure}

\begin{figure}[ht]
\centering
\IfFileExists{figures/fig_app_ls_windows.pdf}%
  {\includegraphics[width=\textwidth]{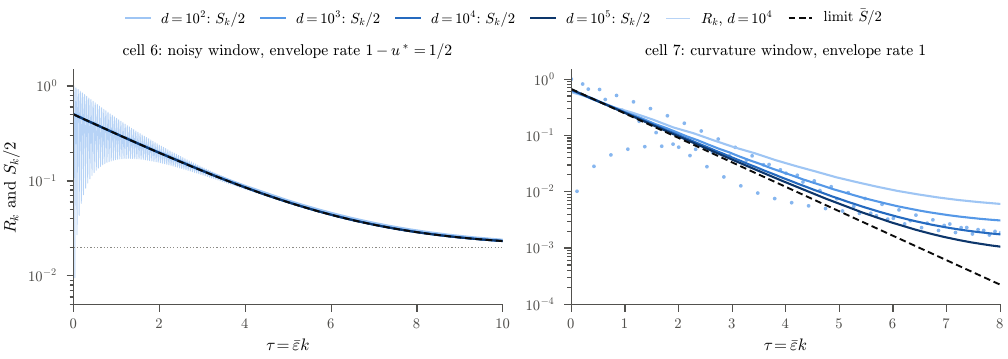}}%
  {\fbox{figure pending: fig\_app\_ls\_windows}}
\caption{\textbf{LS: the slow energy in the two windows, at the critical
learning rate.} Parameters: $\sigma=1.5$, $B=\operatorname{round}(d^{1.5})$,
$\varsigma^2=0.02$, all prefactors equal to one, cold start. Blue lines:
$S_k/2$ of the exact chain at $d=10^2,10^3,10^4,10^5$, darker with $d$; black
dashed: $\bar S/2$; light blue: $R_k$ itself at $d=10^4$.
\emph{Left, cell~6} (noisy window; $\kappa=0.6$, $\gamma=1$,
$\alpha=\alpha_c=1-\sigma=-0.5$, so $\bar\omega^2\approx63$, $501$, $3981$,
$31623$, $H^*=\frac12$, $\varphi^*=0.01$; Corollary~\ref{cor:window}(a)):
$\bar S'=-(1-H^*)\bar S+2\varphi^*$ with $\bar S(0)=1$, envelope rate
$\frac12$ per unit of $\tau$ and floor $2\varphi^*/(1-H^*)=0.04$, so that
$\bar S/2$ tends to $0.02$ (dotted). The uniform error
$\sup_k|S_k-\bar S(\tau_k)|$ is $0.030$, $0.011$, $0.0039$, $0.0014$. The
risk $R_k$ oscillates between deep troughs and an envelope around $S_k/2$, as
in Corollary~\ref{cor:window}(c); the per-step angle is small, so it fills a
band. By Remark~\ref{rem:window}, the troughs are not convergence: the
envelope decays at the rate of $\bar S$.
\emph{Right, cell~7} (curvature window; $\kappa=0$, $\gamma=0.25$,
$\alpha=\alpha_c=-\kappa-\gamma=-0.25$, so $\lambda=1$, $H_c\to\frac14$ and
$\theta^*\to\pi/3$; Corollary~\ref{cor:curv}(a)): $\bar S=\frac43e^{-\tau}$,
with envelope rate $1$ although the stability margin is only
$1-H_c^*=\frac34$. Here $\bar S(0)=\frac43$ is the limit of the matched
initial slow energy $S_0$ from the cold start, computed numerically. The
uniform error is $0.16$, $0.085$, $0.046$, $0.026$, of order
$\ebar=d^{-0.25}$. At finite $d$ the feedback $\ut\asymp d^{-0.25}$ and the temperature $\pt$
are positive, and the curves level off at floors of order $d^{-0.25}$ that
vanish as $d$ grows. The per-step angle $2\theta^*$ is close to $2\pi/3$, so
$R_k$ (dots) cycles with a period of about three steps.}
\label{fig:app-ls-windows}
\end{figure}

\begin{figure}[ht]
\centering
\IfFileExists{figures/fig_app_lr_cells.pdf}%
  {\includegraphics[width=\textwidth]{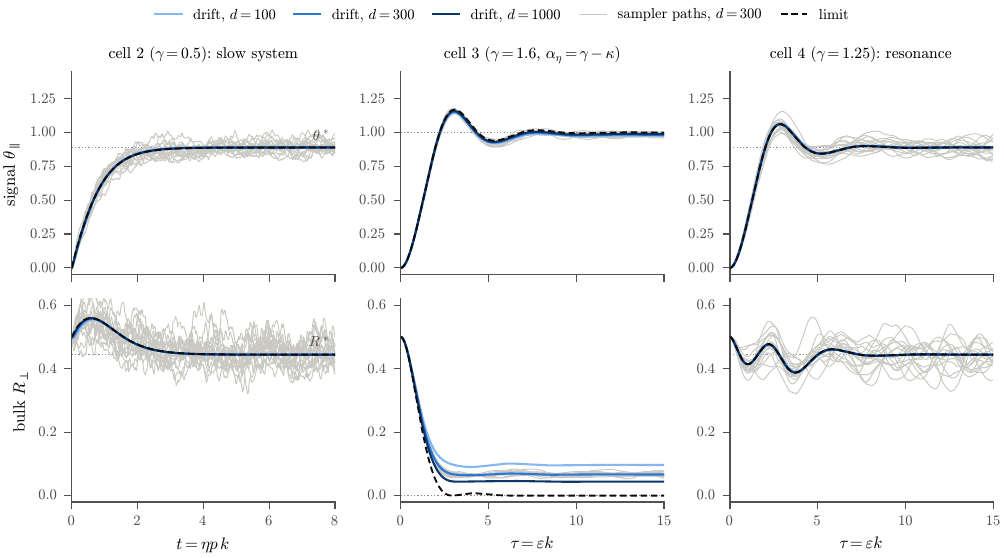}}%
  {\fbox{figure pending: fig\_app\_lr\_cells}}
\caption{\textbf{LR: drift recursion, sample paths and limits in cells~2, 3
and~4.} Top row: signal overlap $\theta_\parallel$; bottom row: bulk second
moment $R_\perp$. Parameters: $r=1$, $\sigma=1$ ($B=d$), $\kappa=1.25$,
$\eta=d^{-\alpha_\eta}$, $\varepsilon=d^{-\gamma}$; start $\theta_\parallel=0$,
zero momentum, $R_\perp=0.5$, $V_\perp=C_\perp=0$. Blue lines: the drift
recursion at $d=100,300,1000$, darker with $d$; grey: $16$ sample paths of
the exact chain of Lemma~\ref{lem:LR0} at $d=300$; black dashed: the limit;
dotted: the equilibrium of Proposition~\ref{prop:D}.
\emph{Cell~2} ($\gamma=0.5$, $\alpha_\eta=0$, so $\omega^2=d^{-0.75}$ and
$\Phi\to\frac12$; clock $t=\eta pk$; Proposition~\ref{prop:LR2}): the slow
system $\theta'=r-\varrho\theta$, $R'=2\Phi-2\varrho R$ at $\Phi=\frac12$.
\emph{Cell~3} ($\gamma=1.6$, subcritical $\alpha_\eta=\gamma-\kappa=0.35$, so
$\omega_{\mathrm{nom}}^2=1$ and $\Phi=0.099$, $0.068$, $0.045\to0$; clock
$\tau=\varepsilon k$; Proposition~\ref{prop:LR34}): the
system~\eqref{eq:LR5} with $\omega^{*2}=1$, $\Phi^*=0$.
\emph{Cell~4} ($\gamma=1.25=1-\sigma+\kappa$, $\alpha_\eta=0$, so
$\omega_{\mathrm{nom}}^2=1$ and $\Phi\to\frac12$; clock $\tau=\varepsilon k$;
Proposition~\ref{prop:LR34}):
the system~\eqref{eq:LR5} with $\omega^{*2}=1$, $\Phi^*=\frac12$.
Equilibrium: at $\Phi=\frac12$, $\theta^*=0.889<r$, $R^*=0.445$ and
$\varrho^*=1.125$, biased against alignment as in
Proposition~\ref{prop:D}(i); at $\Phi=0$, $(\theta^*,R^*)=(1,0)$.
Uniform error of the drift recursion against the limit, as
$(\theta_\parallel,R_\perp)$, for $d=100/300/1000$: cell~2
$(0.025,0.018)/(0.012,0.0087)/(0.0052,0.0038)$; cell~3
$(0.026,0.10)/(0.019,0.070)/(0.0135,0.046)$, where the error in $R_\perp$ is
the finite-$d$ bulk floor, of order $\Phi\asymp d^{-0.35}$; cell~4
$(0.011,0.0047)/(0.0030,0.0011)/(0.0007,0.0003)$. The largest tameness
error $\epsilon_B$ of Lemma~\ref{lem:B} along the drift paths, over the three
cells, is $0.069$, $0.017$ and $0.0039$ for $d=100$, $300$ and $1000$. The largest deviation of a sample path from the
drift recursion over the horizon is $0.165$ (cell~2), $0.041$ (cell~3) and
$0.137$ (cell~4), against $\sqrt{\log d/d}=0.138$ at $d=300$, the scale of
Corollary~\ref{cor:fluid}. All statements are over finite horizons.
Convergence to the equilibrium is proved in cell~2
(Proposition~\ref{prop:LR2}) and in cells~3 and~4
(Proposition~\ref{prop:S}).
}
\label{fig:app-lr-cells}
\end{figure}

\clearpage

\section{Extended Related Work}
\label{app:relwork}

This appendix expands on the main-text related work (Section~\ref{sec:relwork}). Sections~\ref{app:rw-stochastic}--\ref{app:rw-noise} cover momentum theory, Section~\ref{app:rw-architecture} covers the sources of sparse updates, and Section~\ref{app:rw-basis} the alignment of sparsity with the parameter basis. We highlight the closest prior comparisons to our two-timescale phase diagram and to the spectral conflict of Section~\ref{sec:vocab_main}.

\subsection{Acceleration of stochastic momentum}
\label{app:rw-stochastic}

\paragraph{Classical foundation.}
Polyak introduced the heavy-ball method with the goal of incorporating a physical notion of momentum into iterative optimization \citep{polyak1964some}. A velocity variable carries past gradients forward, and on deterministic strongly convex quadratics this changes the contraction rate from $(1-1/\kappa)$ to $(1-1/\sqrt{\kappa})$. Nesterov gave a different update structure that achieves the same rate and extends to general smooth convex problems \citep{nesterov1983method}.

The continuous-time perspective came later, relating acceleration to damped second-order dynamics \citep{qian1999momentum}, and it has a couple of advantages: it more easily connects acceleration to physical mechanics and bypasses some of the algebraic complexity of the discrete analyses. \citet{su2016differential} furnish a second-order ODE limit of Nesterov's method and use it to explain restarting. \citet{wibisono2016variational} reframe acceleration in a variational Bregman--Lagrangian language. \citet{shi2022highresolution} sharpen the picture with high-resolution differential equations that distinguish Polyak heavy ball from Nesterov via a gradient-correction term. We inherit this damped-oscillator picture. For least squares, in every regime with a stability margin and a Nyquist margin, the risk is the second moment of a stochastic heavy ball sampled on a grid, up to an error of the order of the retention rate (Theorem~\ref{thm:M}). In the heavy-ball cell, from a cold start, it is exactly the square of the damped oscillator $X''+X'+\omega^{*2}X=0$ on the retention clock, where $\omega^{*2}$ is the limiting ratio of the learning rate to the retention rate (Corollary~\ref{cor:lift}); on the resonance line the noise keeps the limit irreducibly three-dimensional (Corollary~\ref{cor:resonance}).

\paragraph{Stochastic refinements.}
With noisy gradients the heavy-ball story becomes conditional. Stochastic heavy ball converges almost surely on broad nonconvex coercive objectives \citep{gadat2018stochastic}, with quantitative rates in convex and strongly convex settings \citep{sebbouh2021almost}. But \citet{kidambi2018insufficiency} construct a linear-regression instance on which standard stochastic heavy-ball and Nesterov schemes cannot outperform SGD for any choice of step size and momentum, and they propose an accelerated alternative that requires explicit oracle access.

When acceleration over SGD is recovered, it is recovered by careful design. \citet{jain2018accelerating} and \citet{varre2022accelerated} achieve it on least-squares stochastic approximation by decoupling bias decay from noise. The closest comparison cases for our paper are \citet{paquette2021sgdlarge, paquette2021dynamics, ferbach2025dana}: they prove deterministic high-dimensional limits for SGD and stochastic heavy ball on random least-squares problems, and show that small-batch heavy ball with fixed momentum is dynamically equivalent to SGD. Only carefully scaled $n$-dependent momentum (SDANA, SDAHB, DANA) recovers Nesterov-style behavior. Our model differs in where the stochasticity comes from: not finite-sum row sampling but Bernoulli activation and long inactive intervals. The relevant parameters are not Hessian spectral data but the sparsity, which multiplies the curvature, together with the batch size, the learning rate and the momentum; the arrival statistics of active batches do not enter the per-step dynamics at all (Lemma~\ref{lem:L1}).

\paragraph{Batch size as a gate.}
Several lines of work make momentum acceleration explicitly conditional on the mini-batch scale. \citet{lee2022trajectory} identify an implicit conditioning ratio (ICR) in dense random least squares: above the ICR, SGDM follows the deterministic heavy-ball rate; below, it reduces to single-batch SGD. \citet{bollapragada2025fast} prove a parallel result for minibatch heavy ball on strongly convex quadratics, with a smoothed-condition-number threshold playing the gating role.

The empirical picture matches. \citet{zhang2019which} use a noisy quadratic to show that momentum gives no benefit at small batch but extends scaling at large batch. \citet{fu2023when} attribute the SGD/SGDM gap in deep networks to an effective-learning-rate threshold that controls when momentum can postpone abrupt sharpening.

The practical batch-size literature for language models points the same way. Training speeds up in proportion to the batch only up to a critical batch size, which \citet{mccandlish2018empirical} tie to a gradient noise scale and \citet{shallue2019measuring} measure across workloads. \citet{mccandlish2018empirical} derive a hyperbola relating the steps and samples needed to reach a target loss for SGD, and fit the same form for other optimizers; for one feature trained with heavy ball at a fixed $\beta$, our model gives this hyperbola exactly, with the constant $\frac12$ in place of their $1$ (Remark~\ref{rem:crit_single}). \citet{marek2026small} argue that, with hyperparameters scaled to the small-batch token clock, very small batches train stably and make momentum much less necessary. \citet{zhang2024criticalbatch} and \citet{merrill2026critical} formalize a critical batch size that depends on problem and training stage. \citet{wang2026fastcatchup} studies optimal scheduling between small and large batches via functional scaling laws. In our setting the batch size enters through the noise feedback $H_n=\eta(d+2-p)/(2B)$, and hence through the critical learning rate. At the critical learning rate, momentum enters the limit at order one together with the surviving gradient noise only on the resonance line $\gamma=1-\sigma+\kappa$, where the batch exponent compensates the sparsity and retention exponents (Corollary~\ref{cor:resonance}); below the line the limit is one-dimensional SGD, and above it the envelope of the risk obeys a one-dimensional law (Section~\ref{sec:dict}). Within one feature our models are isotropic, so this is a statement about the shape of the dynamics. Across the features of a vocabulary trained with one learning rate, momentum does accelerate: it multiplies the critical batch size, the ratio of the fewest samples to the fewest steps, both counted per decay time of the excess risk, by the square root of the frequency imbalance, up to an absolute constant, and for a large spread it raises it (Proposition~\ref{prop:helps-critical}).

\subsection{Momentum, noise, smoothing, and effective learning rate}
\label{app:rw-noise}

A common heuristic holds that momentum reduces stochastic gradient noise. The picture is more delicate. \citet{gitman2019understanding} unify heavy ball, Nesterov, and quasi-hyperbolic momentum under one QHM parameterization and analyze stability regions and stationary distributions. When momentum averages gradients, as in our parameterization, they show that the stationary covariance does not depend on $\beta$ to first order in the learning rate; $\beta$ enters only at second order. \citet{wang2024marginal} prove that in small-learning-rate noise-dominated regimes, SGD with and without momentum track each other to leading order in $\sqrt{\eta/(1-\beta)}$ over short and long horizons.

Both results point the same way as ours: neither finds momentum reducing the noise at leading order. Our model adds a further observable: whether the momentum buffer survives between active updates long enough to remain a co-evolving state variable. Below the resonance line the answer is no: the limit is one-dimensional SGD with a noise floor on the learning clock, and momentum is adiabatically eliminated (Corollary~\ref{cor:tikhonov}). On the line it survives at order one; above the line it survives as an oscillation whose envelope obeys a one-dimensional law (Corollary~\ref{cor:window}). In no regime does momentum reduce the noise: the noise floor $\varphi/(1-H)$ is set by the learning rate, and a longer momentum memory lowers it only to the extent that it comes with a smaller learning rate (Corollary~\ref{cor:stab} and Remark~\ref{rem:cost}).

\subsection{Sources of effective sparsity in transformers}
\label{app:rw-architecture}

On the data side, token frequencies in language follow Zipfian power laws, so the parameters indexed by tokens receive active gradients at rates spanning orders of magnitude. Recent work analyzes stochastic optimization under this kind of class imbalance \citep{kunstner2024heavy, kunstner2026scaling, wang2025muon, pezzicoli2025classimbalance}. Whether momentum primarily accelerates, primarily smooths, or has any robust effect at all in such stochastic schemes remains debated.

Sparse, intermittent updates arise not only from heavy-tailed data but also from architecture itself. Modern transformers route gradient signal through input-dependent subsets of weights at every layer, so any given parameter, neuron, head, or expert receives nonzero gradients only at irregular steps --- exactly the regime our sparse-momentum model studies. We organize the relevant literature into three groups: sparsity that emerges naturally in trained networks (whether feed-forward or attention), gating mechanisms that artificially induce it, and mixture-of-experts routing.

\paragraph{Emergent sparsity in trained networks.}
Trained transformers exhibit substantial sparsity even when the data is dense: a small subset of weights or activations carries most of the signal on each input. Strictly this is closer to a statement about highly correlated or low-rank gradients than about literal gradient zeros. \citet{li2023lazyneuron} report on the order of $3\%$ nonzero post-activation entries for T5-Base and $6.3\%$ for ViT-B/16, with sparsity increasing in model size and persisting across training and evaluation data. \citet{voita2019heads} make the analogous observation for attention: most heads can be pruned with little performance loss, while a few specialized heads carry most of the function. \citet{zhang2022moefication} show that dense FFNs can be partitioned into expert-like subnetworks that activate only $10$--$30\%$ of parameters per input.

The inference-systems literature pushes this picture to LLM scale. \citet{liu2023dejavu} introduce \emph{contextual sparsity}, small input-dependent subsets of attention heads and MLP weights that approximately reproduce the dense output, and they train lightweight predictors to exploit it at inference time. Perhaps the most striking observation is the power law of \citet{song2024powerinfer}: a few ``hot'' neurons fire on most inputs, while a heavy tail of ``cold'' neurons is highly input-dependent. The takeaway is that contemporary transformers already have sparse, input-conditioned computational paths even before any explicit gating or experts.

\paragraph{Gating mechanisms.}
A second source is explicit gating built into the architecture. Gated feed-forward variants such as GLU \citep{dauphin2017gatedcnn}, GEGLU, and SwiGLU \citep{shazeer2020glu} multiply one linear projection by a learned gate; the gate is smooth rather than exactly zero, but it modulates which channels carry signal forward, and therefore which downstream weights receive substantial gradient mass. \citet{qiu2025gatedattention} add a head-specific sigmoid gate after scaled dot-product attention and attribute the gains partly to query-dependent sparse gating scores. Other long-context designs build sparsity directly into attention via fixed patterns \citep{child2019sparse}, hardware-aligned trainable schemes \citep{yuan2025nsa}, or content-routed expert selection \citep{piekos2025mosa}. Closely related concentration phenomena, attention sinks \citep{xiao2024streamingllm} and massive activations \citep{sun2024massiveactivations}, show that even within ostensibly dense attention, the realized signal flow is highly nonuniform.

\paragraph{Mixture-of-experts routing.}
The mixture-of-experts (MoE) layer is perhaps the closest architectural analogue of our gated-regression setup. When a token is routed to a few experts, each expert's parameters and optimizer state evolve on a sparse, route-dependent schedule, and the load-balancing, capacity-factor, and dropped-token issues that the MoE literature manages are architectural versions of the sparse-update problem we analyze in a simplified model.

\citet{shazeer2017outrageously} introduced the sparsely gated MoE layer: a router selects a small subset of experts per example, adding capacity without proportional compute. \citet{lepikhin2020gshard} bring this to large translation Transformers with automatic sharding, and \citet{fedus2022switch} simplify routing to top-1 expert selection at trillion-parameter scale.

Subsequent work studies how to scale, stabilize, and specialize expert routing. GLaM \citep{du2022glam} demonstrates parameter-efficient scaling; ST-MoE \citep{zoph2022stmoe} addresses stability and transferability; expert-choice routing \citep{zhou2022expertchoice} reverses the assignment direction to mitigate dropped tokens; DeepSeekMoE \citep{dai2024deepseekmoe} pursues fine-grained expert specialization.

\subsection{Basis alignment of sparsity}
\label{app:rw-basis}

If sparsity is aligned with the parameter basis, the optimizer could update parameters from their momentum buffers only for parameters that received a gradient on the current step, as is done in SparseAdam\footnote{\url{https://github.com/pytorch/pytorch/blob/main/torch/optim/sparse_adam.py}} \citep{paszke2019pytorch}. This applies to the embedding layer, where one-hot inputs zero out entire rows of the gradient. But sparsity in transformer features need not be basis-aligned: mechanistic interpretability work finds interpretable sparse features in non-axial directions of the activation space, visible only after dictionary learning with a sparse autoencoder \citep{elhage2022toy, bricken2023monosemanticity, cunningham2024sparse}. For adaptive methods where the second-moment buffer accumulates coordinate-wise, it is less clear how optimizers might handle sparsity that is not basis-aligned. In contrast, SGD and heavy-ball momentum with a scalar coefficient commute with orthogonal changes of basis; the sparsity in our models gates whole samples rather than coordinates, so the dynamics we characterize do not depend on a choice of basis.

\end{document}